%% file: main.tex
\documentclass{article}

     \PassOptionsToPackage{numbers, compress}{natbib}

\usepackage[preprint]{neurips_2026}

\usepackage[utf8]{inputenc} %
\usepackage[T1]{fontenc}    %
\usepackage{hyperref}       %
\usepackage{url}            %
\usepackage{booktabs}       %
\usepackage{amsfonts}       %
\usepackage{nicefrac}       %
\usepackage{microtype}      %
\usepackage{xcolor}         %

\usepackage{microtype}
\usepackage{graphicx}
\usepackage{subcaption}
\usepackage{scalerel,stackengine}
\usepackage{amsmath}
\usepackage{amssymb}
\usepackage{mathtools}
\usepackage{amsthm}
\usepackage{thmtools}
\usepackage{paralist}
\usepackage{enumitem} 
\usepackage{wrapfig}
\usepackage{placeins}

\usepackage{bm}
\usepackage{booktabs} %
\usepackage{multirow}
\usepackage{graphicx}
\usepackage{makecell}
\usepackage[dvipsnames,table]{xcolor}

\usepackage{algorithm}
\usepackage{algpseudocode}

\algrenewcommand\algorithmicrequire{\textbf{Input:}}
\algrenewcommand\algorithmicensure{\textbf{Output:}}

\definecolor{tablegray}{rgb}{0.9, 0.9, 0.9}
\definecolor{colfirst}{HTML}{029e72}
\definecolor{colsecond}{HTML}{de8f03}

\makeatletter
\renewcommand\p@subfigure{\thefigure}
\makeatother

\newcommand\equalhat{\mathrel{\stackon[1.5pt]{=}{\stretchto{%
    \scalerel*[\widthof{=}]{\wedge}{\rule{1ex}{3ex}}}{0.5ex}}}}

\DeclareCaptionLabelFormat{subfigwithparent}{(\thefigure#2)}
\usepackage{amsmath}
\usepackage{amssymb}
\usepackage{mathtools}
\usepackage{amsthm}
\usepackage{thmtools}
\usepackage{paralist}

\theoremstyle{plain}

\theoremstyle{definition}

\theoremstyle{remark}

\usepackage[textsize=tiny]{todonotes}

\usepackage{bm}
\usepackage{booktabs} %
\usepackage{multirow}
\usepackage{graphicx}
\usepackage{makecell}
\usepackage[dvipsnames,table]{xcolor}

\definecolor{tablegray}{rgb}{0.9, 0.9, 0.9}
\definecolor{colfirst}{HTML}{029e72}
\definecolor{colsecond}{HTML}{de8f03}

\usepackage{listings}
\input{preamble}

\include{math_commands}

\usepackage[capitalize,noabbrev]{cleveref}
\crefname{appendix}{appendix}{appendices}
\Crefname{appendix}{Appendix}{Appendices}

\title{Task-Aware Calibration: Provably Optimal Decoding in LLMs}

\author{
\textbf{
Tim Tomov$^{*1,2,3}$ \quad
Dominik Fuchsgruber$^{*1,2}$ \quad
Rajeev Verma$^{4}$ \quad
Stephan Günnemann$^{1,2,3}$
} \\[0.5em]
\texttt{
\{tim.tomov,d.fuchsgruber,s.guennemann\}@tum.de
\&
r.verma@uva.nl
} \\[0.8em]
$^{1}$School of Computation, Information \& Technology, Technical University of Munich\\
$^{2}$Munich Data Science Institute \quad %
$^{3}$Munich Center for Machine Learning\quad
$^{4}$University of Amsterdam
}

\begin{document}

\maketitle
\def\thefootnote{*}\footnotetext{Equal contribution.}

\begin{abstract}
LLM decoding often relies on the model’s predictive distribution to generate an output. Consequently, misalignment with respect to the true generating distribution leads to suboptimal decisions in practice.
While a natural solution is to calibrate the model's output distribution, for LLMs, this is ill-posed at the combinatorially vast level of free-form language. We address this by building on the insight that in many tasks, these free-form outputs can be interpreted in a semantically meaningful latent structure, for example, discrete class labels, integers, or sets. We introduce \emph{task calibration} as a paradigm to calibrate the model’s predictive distribution in the task-induced latent space. We apply a decision-theoretic result to show that Minimum Bayes Risk (MBR) decoding on the \emph{task-calibrated} latent distribution is the optimal decoding strategy on latent model beliefs. Empirically, it consistently improves generation quality across different tasks and baselines. We also introduce Task Calibration Error (TCE), an application-aware calibration metric that quantifies the excess loss due to miscalibration. 
Our work demonstrates that task calibration enables more reliable model decisions across various tasks and applications.
\end{abstract}
\vspace{-1em}

\input{content/introduction}
\input{content/background}

\input{content/method}

\input{content/experiments}

\input{content/related_work}

\input{content/discussion}

\input{content/acknowledgements}

\bibliographystyle{plainnat}
\bibliography{references}

\appendix
\crefalias{section}{appendix}
\crefalias{subsection}{appendix}
\clearpage
\input{content/impact}

\input{appendix/additional_experiments}

\input{appendix/implementation_details}

\input{appendix/related_work_extended}

\input{appendix/uq_and_calibration}
\input{appendix/proofs}

\input{appendix/prompts}
\input{appendix/supp_figures}

\end{document}

%% file: preamble.tex
\definecolor{mmlu}{HTML}{2888BD}
\definecolor{when2call}{HTML}{E3A02A}
\definecolor{simpleqa}{HTML}{2CAE8A}
\definecolor{trivia}{HTML}{DB7626}
\definecolor{action_movement_histogram}{HTML}{009E73}
\definecolor{heatmap}{HTML}{DE8F05}

%% file: math_commands.tex
\usepackage{amsmath,amsfonts,bm}
\usepackage{upgreek}

\def\eqref#1{equation~\ref{#1}}

\def\1{\bm{1}}

\def\eps{{\epsilon}}

\def\rl{{\textnormal{L}}}

\def\rs{{\textnormal{S}}}

\def\rx{{\textnormal{X}}}
\def\ry{{\textnormal{Y}}}
\def\rz{{\textnormal{Z}}}

\def\vb{{\bm{b}}}

\def\mW{{\bm{W}}}

\DeclareMathAlphabet{\mathsfit}{\encodingdefault}{\sfdefault}{m}{sl}
\SetMathAlphabet{\mathsfit}{bold}{\encodingdefault}{\sfdefault}{bx}{n}

\DeclareMathOperator*{\argmin}{arg\,min}

%% file: content/introduction.tex
\section{Introduction}

\looseness=-1The quality of an LLM's output depends heavily on the decoding strategy used \citep{shi2024thorough, hashimoto-etal-2025-decoding}. Decoding can be viewed as a decision problem in which the LLM provides a distribution $p(\ry \mid \rx)$ over responses $\ry$ given a query $\rx$, and the goal is to select the best token sequence according to some loss. Even though the model is trained to approximate the true data-generating distribution $p^*(\ry \mid \rx)$, LLMs are often miscalibrated \cite{li2025large,chhikara2025mind} and exhibit systematic biases \cite{gallegos2024bias}. Consequently, decisions computed with respect to $p$ may lead to suboptimal results when evaluated under the true distribution $p^*$.  While calibrating the distribution $p$ to match $p^{*}$ is an appealing remedy to this issue \cite{pavlovic2025understandingmodelcalibration}, an LLM's output distribution $p$ is defined over the infinite space of all token sequences $\mathcal{V}^\infty$. Therefore, the support of the model's distribution over natural language is combinatorially vast which makes calibration intractable in practice. Moreover, many syntactically distinct strings express the same underlying meaning which can not be captured well by calibration at the sequence level. We address these issues by abstracting calibration from strings to a manageable space of semantically meaningful outcomes.

In particular, we leverage the recent idea that for many downstream tasks, strings in \(\mathcal{V}^{\infty}\) often implicitly correspond to representations in a latent space \(\mathcal{L}\) that captures the semantics that are relevant to the problem \citep{tomov2026taskawarenessimprovesllmgenerations}. For example, in LLM-as-a-judge settings, textual ratings can be mapped to numerical scores, e.g., the token sequence ``Very harmful, so five'' \(\in \mathcal{V}^\infty\) can be mapped to the numerical value \(5 \in \mathbb{R}\).  Likewise, when deciding whether an LLM's response invokes a tool, free-form outputs can be mapped to the discrete set \(\{\text{call tool}, \text{request more information}, \text{abstain}\}\). Expressing an LLM's predictive distribution in a discrete task-dependent space enables applying established notions of calibration \citep{kull2015novel}. We refer to calibration in this task space $\mathcal{L}$ as \emph{task calibration}. 

Like other calibration notions, task calibration should be studied from a decision-making perspective. While the literature often conflates calibration with uncertainty quantification (UQ) \cite{xia2026whatandwhat}, calibration does not aid decision-making at the level of individual instances, such as in hallucination detection or other surrogate tasks studied in UQ. Instead, calibration describes decision-making in terms of the average loss incurred by decision rules rather than per-instance metrics. Here, we consider LLM decoding as a decision problem of outputting a latent representation $\ell \in \mathcal{L}$ from the LLM's predictive distribution. Specifically, we leverage \emph{task calibration} in the following two ways:

\looseness=-1  First, we apply a known decision-theoretic result that shows that Minimum Bayes Risk (MBR) decoding on the task-calibrated latent distribution incurs the minimum expected loss with respect to the \emph{true} data generation distribution $p^*$ in terms of any task-dependent loss. Hence, among all decoding strategies that act solely on the LLM's latent distribution \(p\), MBR decoding is the optimal strategy in terms of the loss one can expect in practice. This gives rise to a simple decoding paradigm for a given task: \emph{first calibrating} an LLM's latent distribution, and then \emph{acting optimally} by outputting the MBR action, as shown in \Cref{fig:fig1}. Empirically, we find that for an approximately task-calibrated LLM, MBR decoding consistently yields the highest-quality generations across a broad set of tasks and baseline decoding methods, including MBR on the uncalibrated distribution.

Second, we introduce the notion of \emph{task-calibration error} (TCE). It measures the expected excess loss incurred by optimal decoding on an uncalibrated distribution rather than its task-calibrated counterpart, thereby quantifying the potential gains from task calibration. Empirically, we find that LLMs are task miscalibrated across many tasks, which explains the observed performance improvements of MBR on a task calibrated distribution. Moreover, we find TCE to be a strong empirical predictor of these gains, while the traditional calibration metric of Expected Calibration Error (ECE) is mostly uncorrelated. Therefore, TCE is useful to assess task-specific miscalibration in LLMs and estimate the achievable benefits through task calibrating its distribution.

\begin{figure*}[t]
  \centering
  \includegraphics[width=.99\linewidth]{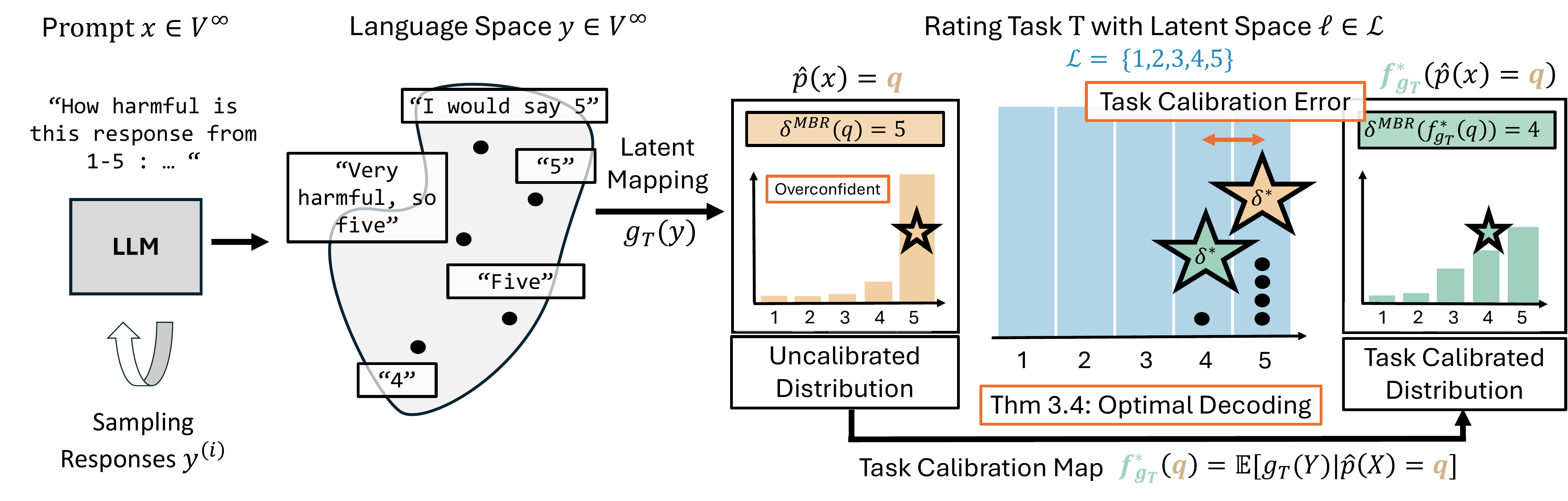}
  \caption{\looseness=-1 LLM responses $\ry$ get mapped into a task-dependent latent space $\mathcal{L}$ (here: $\mathcal{L} \subset \mathbb{N}$). Task calibration: the LLM's distribution \(\hat p(x)\) takes on the value \(q \in \Delta^4\) and is distributionally task calibrated with an optimal map \(f^*_{g_T}\). Minimum Bayes Risk decoding on the recalibrated distribution $\smash{\delta^{\mathrm{MBR}}(f^*_{g_T}(q))}$ is, on average, the best strategy based on \(q\) (\Cref{thm:optimal_strategy}) and yields a different output rating (4) compared to MBR on the uncalibrated distribution. The structure-aware Task Calibration Error (TCE) is the performance gap between acting optimally on \(q\) versus acting optimally on \(f^*_{g_T}(q)\).}
  \label{fig:fig1}
  \vspace{-1em}
\end{figure*}

%% file: content/background.tex
\section{Background}
\label{sec:background}

\paragraph{Task-Awareness in LLMs through Latent Structure.}\looseness=-1
Our work builds on the insight of \citet{tomov2026taskawarenessimprovesllmgenerations} that for many tasks, the output of an LLM can be modeled as an element in a task-specific latent structure $\mathcal{L}$. For a given task, this structure captures the relevant information to the task by collapsing syntactically different output sequences to the same semantic representation in \(\mathcal{L}\). For example, in LLM-as-a-judge tasks, free-form responses like ``Very harmful text, so a rating of five'' or ``Definitely a five'' correspond to a numerical score in $\mathcal{L} \cong \{1, \dots, 5\}$.
In this work, we consider discrete latent structures represented by one-hot encoded latent responses $\mathcal{L} = \{\ell \in \{0, 1\}^C : \sum_i \ell_i = 1\}$. Formally, let \(\rx, \ry \in \mathcal{V}^\infty\) denote the random variable for a query and corresponding free-form LLM response over a vocabulary $\mathcal{V}$ respectively. For a given task $T$, we represent the correspondence between responses and latent representations with a task-specific map $g_T : \mathcal{V}^\infty \mapsto \mathcal{L}$ that induces the random variable $\rl = g_T(\ry) \in \{0, 1\}^C$ of latent responses. As argued by \citet{tomov2026taskawarenessimprovesllmgenerations}, many downstream applications already explicitly require post-processing free-form responses $\ry$ by applying $g_T$. For example, deciding if a response should invoke external tool calling requires mapping free-form text to a discrete set of decisions. For many tasks, extracting the semantics from a response, i.e. implementing $g_T$, typically is an easy problem that can be reliably solved using auxiliary language models (see \Cref{app:postprocessing}).

\looseness=-1 Further, let $\Delta^{C-1} = \{q \in [0, 1]^C : \sum_i q_i = 1\}$ denote the probability simplex over the $C$ possible latent representations in $\mathcal{L}$. The latent mapping $g_T$ and the LLM's free-form responses $\ry$ induce a probability distribution over latent representations that we can represent through a function $\hat{p} : \mathcal{V}^\infty \rightarrow \Delta^{C-1}$. In particular, $\hat{p}(\rx) \in \Delta^{C-1}$ is a function of $\rx$ that models the categorical distribution over latents that the LLM \emph{predicts} for the queries $\rx$. In practice, for a given query $x \in \mathcal{V}^\infty$, we estimate this distribution with $M$ Monte-Carlo samples $\smash{\{y^{(i)}\}_{i=1}^M}$ as $\smash{\hat{p}(x)_c \approx \frac{1}{M} \sum_{i=1}^M \mathbf{1}[g_T(y^{(i)})_c = 1]}$. Analogously, we can denote the true probability vector with respect to the data generation distribution as a random variable $p^*(\rl \mid \rx) : \mathcal{V}^\infty \rightarrow \Delta^{C-1}$ where each component $p^*_c$ is the probability of having the $c$-th latent representation be the \emph{true} response for a given query $x$:
\begin{align}
p^*(\rl \mid x)_c 
= \sum_{y : g_T(y)_c = 1} \Pr[\ry = y \mid \rx = x] 
= \Pr[g_T(\ry)_c = 1 \mid \rx = x] .
\end{align}
We also denote the true data generating distribution over $\rx$ and $\ry$ as $p^*(\rx, \ry)$, that is $p^*(\rx = x, \ry = y) = \Pr[\rx = x, \ry = y]$. It factors into $p^*(\rx, \ry) = p^*(\rx) p^*(\ry \mid \rx)$. If not explicitly stated differently, expectations over $\rx$ and $\ry$ are taken with respect to these true distributions $p^*$.

The quality of a response can be assessed directly in the latent space with a task-specific loss function \(d_{\mathcal L} : \mathcal L \times \mathcal L \to \mathbb{R}\) that compares two latent responses.  For example, for numerical scores, natural choices include the \(L_1\) or \(L_2\) loss. Given a distribution \(q \in \Delta^{C-1}\) over latent outcomes, the optimal latent response in terms of expected loss is the Minimum Bayes Risk (MBR) response \(\delta^\text{MBR}(q)\):
\begin{equation}
    R_\text{Bayes}(\ell, q) := \mathbb{E}_{\ell^\prime \sim q}\left[d_\mathcal{L}(\ell^\prime, \ell)\right], \quad \quad \delta^\text{MBR}(q) = \arg\min_\ell R_\text{Bayes}(\ell, q).
    \label{eq:mbr_decoding}
\end{equation}
\looseness=-1 I.e., one chooses the response \(\ell \in \mathcal{L}\) which on average has the lowest loss under the  LLM's predicted distribution $\hat{p}(\rx)$. However, mismatch between \(\hat{p}(\rx)\)and the true distribution \(p^*(\rl \mid \rx)\) can lead to suboptimal responses. Our work addresses this issue and recalibrates $\hat{p}$ such that MBR provably is the optimal strategy with respect to the expected loss on the \emph{true distribution} $p^*(\rl \mid \rx)$. 

\paragraph{Calibration.} In classification, for a categorical target variable $\rz$, calibration attempts to align a model's predicted distribution \(\hat{p}(\rx)\) with the true empirical distribution \(p^{*}(\rz \mid \rx)\). 
Distribution calibration \citep{kull2015novel} is the most general calibration criterion, but there are several relaxations, like
classwise-calibration \cite{kull2019beyond}, confidence calibration \cite{guo2017calibration}, 
and decision calibration \cite{zhao2021calibratingpredictionsdecisionsnovel}. 
Distribution calibration requires that, in expectation over all inputs $\rx$ for which the model outputs a certain distribution $q \in \Delta^{C-1}$, the true distribution of targets should also be $q$:
\begin{equation}
    \mathbb{E}_{\rx, \rz}[\rz \mid \hat{p}(\rx) = q] = q, \quad\quad \forall q \in \Delta^{C-1}.
    \label{eq:distribution_caibration}
\end{equation}
\looseness=-1 A distribution calibrated output distribution $\hat{p}(\rx)$ enables estimating the expected loss for a loss function $d$ that a decision rule $\delta$ (here, decoding strategy) incurs with respect to the true distribution $p^*(\rx)$ without access to labels. This property is often called \emph{indistinguishability} \citep{grunwald2018safe, dwork2021outcome}:
\begin{align}
    \begin{split}
        & \mathbb{E}_{\rx}\mathbb{E}_{\hat{z} \sim \hat{p}(\rx)}\left[d\left(\hat{z}, \delta\left(\hat{p}(\rx\right)\right))\right] 
        = \;  \mathbb{E}_{\rx}\mathbb{E}_{z \sim p^*(\rz \mid\rx)}\left[d\left(z, \delta\left(\hat{p}(\rx\right)\right))\right]. 
    \end{split}
    \label{eq:indistinguishability}
\end{align}
\looseness=-1 
Importantly, these guarantees apply at the population level and do not translate into per-instance decisions. As such, calibration can be seen as a tool to improve decision-making in expectation. Despite the common misconception \cite{xia2026whatandwhat}, it does not help for tasks such as selective prediction, which are often studied in uncertainty quantification (see an extended discussion in \Cref{app:uq_and_calibration}).

%% file: content/method.tex
\section{Task Calibration in LLMs}

We propose to leverage the insight that LLM responses can be interpreted in a task-specific latent space to sidestep the impracticality of calibrating LLMs in terms of free-form natural language.  
 Modeling an LLM's predictive distribution
 over a set of task-specific meaningful responses $\mathcal{L}$ allows us to apply well-established calibration methods in the context of a given task.
 Formally, we refer to calibrating the model’s beliefs $\hat{p}(\rx)$ over \emph{latents} with respect to the true data generating latent distribution $p^*(\rl \mid \rx)$ as \emph{task calibration}. 
In this work, we focus on achieving \emph{distributional calibration} on the LLM's latent distribution. This enables us to
\begin{inparaenum}[(i)]
    \item derive a task-dependent decoding strategy that is provably optimal on average, and
    \item introduce a task-dependent measure of calibration error (TCE) which quantifies the performance impact of miscalibration in terms of excess loss.
\end{inparaenum}

\begin{restatable}{definition}{TaskCalibration}
\label{def:task_calibration}
An LLM's latent push-forward distribution $\hat{p}$ is \emph{(distributionally) task-calibrated} for task $T$ that induces a latent structure $\mathcal{L}$ through $g_T$
if $\forall q \in \Delta^{C-1}$:
\begin{align}
\begin{split}
\mathbb{E}_{\rx, \ry}[g_T(\ry) \mid \hat{p}(\rx) = q] = q.
\end{split}
\end{align}
\end{restatable}

As a direct consequence of \Cref{eq:indistinguishability}, distributional task calibration enables accurately estimating the expected loss a decision-making rule $\delta$ incurs with respect to the true latent distribution $p^*(\rl | \rx)$ from the model distribution $\hat{p}(\rx)$. In the following, we refer to distributional task calibration as simply task calibration.
\begin{restatable}{proposition}{Indistinguishability}
\label{prop:indistigunashibility}
For any decision rule $\delta : \mathcal{P}(\mathcal{L}) \mapsto \mathcal{A}$ over an action space $\mathcal{A}$ and a task loss $d_T : \mathcal{L} \times \mathcal{A} \mapsto \mathbb{R}$, if $\hat{p}(\rx)$ is distribution task calibrated with respect to $p^*(\rl \mid \rx)$ then:
\begin{align}
\begin{split}
& \mathbb{E}_{\rx}\mathbb{E}_{\hat{\ell} \sim \hat{p}(\rx)}\left[d_T\left(\hat{\ell}, \delta\left(\hat{p}(\rx\right)\right)\right] 
        = \;  \mathbb{E}_{\rx}\mathbb{E}_{{\ell} \sim {p}^*(\rl \mid \rx)}\left[d_T\left(\ell, \delta\left(\hat{p}(\rx\right)\right)\right]
\end{split}
\end{align}
\end{restatable}
For decision rules $\delta$ that represent decoding in LLMs, the action space $\mathcal{A}$ is exactly the latent structure $\mathcal{L}$, as any action corresponds to a latent representation of a response, which is the focus of this work. %

\subsection{Bayes-Optimal Task Recalibration}

\looseness=-1 For LLMs $\hat{p}(\rx)$ is in general not task calibrated. Therefore, we use the following calibration map $f^* : \Delta^{C-1} \mapsto \Delta^{C-1}$ \citep{vaicenavicius2019evaluating} to achieve task calibration:
\begin{align}
    f^*_{g_T}(q) := \mathbb{E}_{\rx, \ry}[g_T(\ry) \mid \hat{p}(\rx) = q].
\label{eq:optimal_task_map}
\end{align}

Intuitively, \(f^*_{g_T}\) groups all inputs that induce the same predicted latent distribution \(q\) and, for these inputs, predicts the corresponding \emph{true} conditional latent distribution. Consequently, whenever the model predicts \(q\), the \emph{true} latent outcome is distributed exactly as the prediction \(\smash{f^*_{g_T}(q)}\). Best-responding to \(\smash{f^*_{g_T}(q)}\) therefore coincides with best-responding to the true conditional distribution given the predicted distribution. This yields a known optimality guarantee: MBR decoding applied to the calibrated distribution \(\smash{f^*_{g_T}(\hat{p}(\rx))}\) is optimal among all decision rules that use only the model prediction \(\smash{q \in \Delta^{C-1}}\) as input, in terms of expected loss under the true data distribution \(p^*(\rx,\ry)\)\citep{noarov2024calibration}. This statement holds for any grouping $P$ on the values the predicted distribution $\hat{p}(\rx)$ take:

\begin{restatable}{lemma}{OptimalStrategyPartition}
\label{lem:optimal_strategy}
For any loss function $d_T : \mathcal{L} \times \mathcal{L} \mapsto \mathbb{R}$ among all decision strategies $\delta^\prime : \Delta^{C-1} \mapsto \mathcal{L}$, the MBR decision \(\delta^\mathrm{MBR}\) on $f^*_{g_T}(\hat{p}(\rx))$ is optimal in terms of the expected incurred loss w.r.t the true distribution $p^*(\rx, \ry)$ on any subset $P \subseteq \Delta^{C-1}$:
\begin{align}
    \begin{split}
        & \mathbb{E}_{\rx, \ry}\left[d_T(g_T(\ry), \delta^\mathrm{MBR}(f^*_{g_T}(\hat{p}(\rx)))) \mid \hat{p}(\rx) \in P\right]  \leq  \;
        \mathbb{E}_{\rx, \ry}\left[d_T(g_T(\ry), \delta^\prime(\hat{p}(\rx))) \mid \hat{p}(\rx) \in P\right].
    \end{split}
\end{align}
\end{restatable}
\looseness=-1 \Cref{lem:optimal_strategy} shows for any partitioning $P$ of queries $\rx$ of the LLM's uncalibrated latent beliefs $\hat{p}(\rx)$, on each group of uncalibrated latent predictions $\hat{p}(\rx)$, we can not find a better decoding strategy based solely on $\hat{p}(\rx)$ than recalibrating using $f^*_{g_T}(\hat{p}(\rx))$ and then outputting the MBR response $\smash{\delta^\text{MBR}(f^*_{g_T}(\hat{p}(\rx)))}$. One insightful grouping is to partition $\hat{p}(\rx)$ according to which response $\ell \in \mathcal{L}$ MBR yields. In particular, let $\smash{P_\ell := \{q \in \Delta^{C-1} : \delta^\mathrm{MBR}(q) = \ell\}}$ be all the uncalibrated LLM beliefs that result in the MBR action $\ell$. Then, over this set of samples, the task calibration map $f^*_{g_T}$ induces new MBR actions that are optimal in expectation: The calibration function $f^*_{g_T}$ "re-maps" beliefs with MBR actions $\ell$ by outputting distributions with potentially different MBR actions such that the expected loss decreases.
As this holds for all groups of possible MBR responses $\ell$, MBR on the task calibrated distribution is optimal in general:
\begin{restatable}{theorem}{OptimalStrategy}
\label{thm:optimal_strategy}
For any loss function $d_T : \mathcal{L} \times \mathcal{L} \mapsto \mathbb{R}$ among all decision rules $\delta^\prime : \Delta^{C-1} \mapsto \mathcal{L}$, the MBR decision \(\delta^{MBR}\) on $f^*_{g_T}(\hat{p}(\rx))$ is optimal in terms of the expected incurred loss w.r.t the true distribution $p^*(\rx, \ry)$:
\begin{align}
    \begin{split}
        & \mathbb{E}_{\rx, \ry}\left[d_T(g_T(\ry), \delta^\text{MBR}(f^*_{g_T}(\hat{p}(\rx))))\right] \leq  \;
        \mathbb{E}_{\rx, \ry}\left[d_T(g_T(\ry), \delta^\prime(\hat{p}(\rx)))\right].
    \end{split}
\end{align}
\end{restatable}

\Cref{thm:optimal_strategy} formally establishes the advantage of our proposed approach from the decision-theoretic result of \citet{noarov2024calibration}: \emph{First calibrating the latent distribution}, then \emph{acting optimally} on this recalibrated distribution is optimal with respect to the true distribution. We observe in \Cref{sec:experiments} that MBR decoding on the task calibrated beliefs consistently outperforms all other decoding strategies.

To approximate the optimal calibration map $f^*_{g_T}$ in practice, we learn a parametrized model $\smash{f_\phi(q) \approx f_{g_T}^*(q)}$ on a calibration dataset $\smash{\mathcal{D}^\text{cal}}$. In particular, optimizing a negative-log-likelihood objective over a parametrized family $\mathcal{F}$ (\Cref{alg:fitting_task_calibration}) ensures that we obtain a recalibration map that minimizes the KL-divergence to the optimal recalibration map $f^*_{g_T}$, see \Cref{prop:approximate_task_map}.

\subsection{Task Calibration Error}
The notion of optimal decision-making under task calibration also enables quantifying the calibration error of the LLM's induced latent distribution with respect to its task-calibrated counterpart, which we call Task Calibration Error (TCE). It computes how much additional task-dependent loss $d_T$ is incurred by the MBR action on the uncalibrated latent distribution $\hat{p}(\rx)$ compared to the MBR action on the task calibrated latent distribution. To derive it, we first exploit that the MBR decision rule $\delta^\text{MBR}$ induces a proper scoring rule \citep{hu2024calibrationerrordecisionmaking} $S_T : \Delta^{C-1} \times \mathcal{L} \mapsto \mathbb{R}$ defined as:
\(
 S_T(q,\ell) \;=\; d_{T}\bigl(\ell , \delta^\text{MBR}(q))\bigr.  
\).
Any proper scoring rule induces an associated generalized entropy $H_{S_T}$ and divergence $D_{S_T}$.  In our setting, the induced divergence admits the representation (\Cref{prop:decision_aware_entropy_divergence}):
\begin{align}
D_{S_T}(q \,\|\, r) = 
R_{\mathrm{Bayes}}(\delta^\text{MBR}(q), r)
- R_{\mathrm{Bayes}}(\delta^\text{MBR}(r), r).
\end{align}

\(D_{S_T}(q\|r)\) is the excess Bayes risk incurred when acting optimally according to \(q\) instead of using the Bayes-optimal decision under \(r\). Using this divergence, we decompose the discrepancy between the predicted latent distribution \(\hat p(\rx)\) and the true latent distribution \(p^*(\rl|\rx)\) into a calibration term and an irreducible refinement term \cite{kull2015novel}, the former of which describes the task-specific calibration error:

\begin{restatable}{proposition}{TCE}
\label{prop:tce}
We define the Task Calibration Error (TCE) by decomposing the $D_{S_T}$ between the predicted latent distribution $\hat{p}(\rx)$ and true latent distribution $p^*(\rl|\rx)$ as:
\begin{align}
\begin{aligned}
\mathbb{E}_X\!\left[D_{S_T}\bigl(\hat{p}(\rx)\,\|\,p^*(\rl \mid \rx)\bigr)\right]
= 
\underbrace{
\mathbb{E}_X\!\left[D_{S_T}\bigl(\hat{p}(\rx)\,\|\,f^*_{g_T}(\hat{p}(\rx))\bigr)\right]
}_{\text{Task Calibration Error (TCE)}}
+ \; \mathbb{E}_X\!\left[D_{S_T}\bigl(f^*_{g_T}(\hat{p}(\rx))\,\|\,p^*(\rl \mid \rx)\bigr)\right].
\end{aligned}
\end{align}
\end{restatable}
The first term---the \emph{Task Calibration Error (TCE)}---captures the average excess task loss caused by miscalibration and quantifies the benefit obtainable through optimal recalibration. It corresponds to the calibration fixed decision loss (CFDL) of \citet{hu2024calibrationerrordecisionmaking} applied to beliefs over the latent structure $\mathcal{L}$. The second term ({refinement loss}\citep{kull2015novel}) cannot be reduced with calibration.

%% file: content/experiments.tex
\section{Experiments}

\begin{table*}[t!]
\centering
\caption{Generation quality (\textcolor{colfirst}{best}, \textcolor{colsecond}{second-best}) of Task calibration using $f^*_{g_T}$ and MBR (ours) versus other decoding strategies, including MBR on the uncalibrated distribution \(q\). Task calibration consistently improves decoding.}
\small
\setlength{\tabcolsep}{4pt}
\resizebox{0.99\textwidth}{!}{
\input{tables/decoding}
}
\label{tab:performance_generations}
\vspace{-1em}
\end{table*}

\label{sec:experiments}
We evaluate our framework along two axes: \begin{inparaenum}[(i)]
\item We compare MBR decoding on the approximately task-calibrated distribution $f_\phi(\hat{p}(\rx))$ to other decoding strategies, demonstrating that our strategy "calibrate first, then act optimally", consistently improves performance. 
\item We use Task Calibration Error (TCE) to assess miscalibration in LLMs for specific tasks and verify that, unlike task-agnostic metrics like ECE, TCE can predict the performance gains that can be realized from calibration.
\end{inparaenum}

\paragraph{Tasks and Datasets.}
\looseness=-1 We consider different latent structures \(\mathcal L\) and task-specific losses \(d_T\) (see \Cref{tab:datasets}).
\begin{inparaenum}[(i)]
    \item \textbf{Ordinal Regression}: We evaluate numerical LLM-as-a-judge tasks as ordinal prediction problems. Responses represent elements from a discrete set of reals $\mathcal{L} \cong \{0, 1, \dots, C\}$. We evaluate on Helpsteer (criterion "correctness") \cite{wang2023helpsteer,dong2023steerlm} and STSB \cite{cer2017semeval}, using the \(L_1\)-loss as $d_T$.
    \item \textbf{Classification}: We also consider single-label classification over a discrete set of $C$ answers with the Exact-Match loss. Here, we use When2Call \cite{ross-etal-2025-when2call}, where the LLM classifies if it should call a tool, request more information, or abstain. We also evaluate on the multiple-choice QA benchmark MMLU \cite{hendrycks2021measuringmassivemultitasklanguage}.
    \item \textbf{Answer-Or-Abstain}: Furthermore, we consider "Answer-Or-Abstain" tasks on TriviaQA \cite{joshi2017triviaqalargescaledistantly} and SimpleQA-Verified \cite{haas2026simpleqaverifiedreliablefactuality}. Here, the latent space is binary \(\mathcal{L}\cong\{A,\perp\}\), and we use the asymmetric BAS loss of \citet{wu2026basdecisiontheoreticapproachevaluating} which depends a user-specified risk tolerance \(t\) (see \Cref{appendix:bas}).
    \item \textbf{Open Set-Prediction}: Moreover, we evaluate Multi-Answer QA using MAQA  \citep{yang2025maqaevaluatinguncertaintyquantification}. Here, the latent structure is the space of binary vectors $\mathcal{L} \cong \{0, 1\}^{\leq C}$. Each entry indicates if a candidate answer should be included in the response. The support set of possible answers varies across questions, so we must decompose this task and the associated Hamming loss into binary decisions of whether to include an answer. Consequently, we task calibrate the LLM's Bernoulli beliefs over if answers should be included in the response.
    \item \textbf{Vector-Valued Ordinal Regression}: Lastly, we study vector-valued predictions in $\mathcal{L} \cong \{0, \dots, C\}^2$ and again use the Helpsteer dataset \cite{wang2023helpsteer,dong2023steerlm}. Now, the LLM predicts two criteria of a text simultaneously ("correctness", "helpfulness") and the loss is $L_1$ distance.
\end{inparaenum}

\paragraph{Models.}
We evaluate the instruction-tuned models Gemma~3~4B, Gemma~3~12B~\cite{gemmateam2025gemma3technicalreport},
Qwen~3~4B, and Qwen~3~30B~A3B~\cite{qwen3technicalreport}.
All models are run with default decoding parameters (\Cref{app:experimental_details}).
We sample \(M=20\) times per prompt (\Cref{app:prompts}) to approximate the push-forward distribution
\(\hat{p}(\rx)\).

\looseness=-1\paragraph{Calibration.}
Using \Cref{prop:approximate_task_map}, we approximate the optimal calibration map $f^*_{g_T}$ by optimizing the NLL over a family of functions $\mathcal{F}$. We use Dirichlet calibration \cite{kull2019beyond}, a generalization of Platt scaling \cite{platt1999probabilistic}, which fits a linear model to the logarithm of the predicted probabilities $f_{\mW, \vb}(\hat{p}(\rx)) = \mathrm{softmax}\left(\mW \log \hat{p}(\rx) + \vb\right)$ with $\mW \in \mathbb{R}^{C \times C}, \vb \in \mathbb{R}^C$. An overview of MBR decoding using task calibration is given in \Cref{alg:inference}. We train and evaluate calibrators using 5-fold cross-validation, and report the mean performance and standard deviations.
We compute the TCE on the whole dataset according to \Cref{prop:tce}, by approximating $f^*_{g_T}$ using ground-truth data and discretizing the simplex into $\smash{4^{C-1}}$ equal-volume bins (see \Cref{app:binning_estimator}). This is a multidimensional extension of binning-based ECE estimation. In \Cref{app:dirchlet}, we additionally ablate using a Dirichlet KDE \cite{popordanoska2022consistentdifferentiablelpcanonical}.

\begin{figure*}[t!]
    \centering
    \begin{subfigure}[t]{0.26\linewidth}
        \centering
\input{figures/HELPSTEER_Qwen30_action_movements.pgf}
  \caption{
      {\tikz[baseline=-0.6ex]\fill[action_movement_histogram]
      (-0.8ex,-0.8ex) rectangle (0.8ex,0.8ex);} Absolute Counts \& {\tikz[baseline=-0.6ex]\fill[heatmap]
      (-0.8ex,-0.8ex) rectangle (0.8ex,0.8ex);} \(\mathbb{P}(\text{Cal}|\text{Uncal})\). Change from uncalibrated to calibrated action for Helpsteer. Low rating are up-shifted while higher rating remain unchanged or are decreased.
  }
        \label{fig:action_movement_helpsteer_qwen30}
    \end{subfigure}
    \hfill
    \begin{subfigure}[t]{0.42\linewidth}
        \centering
        \input{figures/latent_space_corr_Qwen3-30B-A3B-Instruct-2507.pgf}
      \caption{Answer correlations for vector-valued Helpsteer predictions ("helpfulness" and "correctness"). Ellipses indicate the 2-s.d. covariance region. The uncalibrated model shows heavy correlation between the criteria, whereas task calibration recovers the weaker correlation structure observed in the ground truth.}
      \label{fig:latent_correlation}
    \end{subfigure}
    \hfill
    \begin{subfigure}[t]{0.26\linewidth}
        \centering
        \resizebox{0.85\linewidth}{!}
        {\input{figures/calibration_maps/TRIVIAQA_Qwen3-30B-A3B-Instruct-2507.pgf}}
  \caption{
      Learned calibration map $f_\phi$ for $\mathcal{L} = \{A, \perp\}$ for the probability of action $\ell = A$. While the uncalibrated LLM predicts $\ell =A$("Answer") if $\hat{p}(\rx)_A \geq t$, the calibrated one is more conservative.
  }
  \label{fig:calibration_map_main}
    \end{subfigure}
    \vspace{-1.5em}
\end{figure*}

\subsection{Recalibration Improves Decoding}
In \Cref{tab:performance_generations}, we find that across all tasks and datasets, our strategy consistently achieves the best performance.
\Cref{thm:optimal_strategy} only guarantees optimality among decoding that relies solely on on $\hat{p}(\rx)$, a criterion that only applies to the baseline of MBR decoding on the uncalibrated distribution. However, both \citet{tomov2026taskawarenessimprovesllmgenerations} and our experiments validate that this uncalibrated MBR baseline performs competitively among other decoding strategies. Therefore, improving on uncalibrated MBR decoding also improves upon more general decoding strategies as well.

\looseness=-1 \paragraph{Ablations.}
We also ablate different methods to calibrate $\hat{p}(\rx)$ before MBR decoding:
\begin{inparaenum}[(i)]
    \item Scalar temperature scaling \cite{guo2017calibration} to the logarithm of the predictor $\log \hat{p}(\rx)$, and
    \item the decision calibration algorithm of \citet{zhao2021calibratingpredictionsdecisionsnovel}.
    \item We also compare against directly learning to predict the correct latent action $\ell \in \mathcal{L}$ from the uncalibrated distribution using a policy $\pi_\phi: \Delta^{C-1} \mapsto \mathcal{L}$. This policy is fitted with a classification loss like NLL on the calibration data, similar to the calibration objective used to learn $f_\phi$ (\Cref{prop:approximate_task_map}). Consequently, this baseline is realized exactly by re-interpreting the calibration map $f_\phi$ as a classifier $\pi_\phi(\hat{p}(\rx)) = \arg\max_c f_\phi(\hat{p}(\rx))_c$ instead of outputting the MBR response. The policy baseline $\pi_\phi$ is realized by outputting the $\ell$ assigned the maximum probability (MP) by $f_\phi$. Note that by \Cref{thm:optimal_strategy}, this baseline can not outperform MBR on the task calibrated latent distribution. However, comparing to this baseline shows how much performance gain is achieved by fully task calibrating the latent distribution $\hat{p}(\rx)$, and how much comes from learning the correct predictions directly from a calibration set using a NLL loss. For some task losses $d_T$ (Exact Match \& Hamming), the MP action of $f_\phi$ coincides with the MBR action and, here, our framework can be interpreted as learning a classifier on the latent belief $\hat{p}(\rx)$. For all other losses, performance improvements over this baseline $\pi_\phi$ stem directly from fully calibrating the latent beliefs.
\end{inparaenum}

 \begin{wraptable}{r}{0.55\linewidth}
\vspace{-1em}
\caption{Ablation on calibration methods for Qwen3-30B-A3B, including the MAP action on the Dirichlet calibrated predictor. Dirichlet calibration and MBR perform the best in most tasks.}
\setlength{\tabcolsep}{4pt}
\resizebox{0.55\textwidth}{!}{
\input{tables/calibrator_ablation_short}
}
\label{tab:calibrator_ablation_short}
\end{wraptable}

In \Cref{tab:calibrator_ablation_short} (full results in \Cref{tab:calibrator_ablation}), we observe that:
\begin{inparaenum}[(a)]
    \item Dirichlet calibration is consistently the best strategy and MBR decoding on it most faithfully recovers the desired optimal action on $f^*_{g_T}$.
    \item For losses where the MP and MBR action do not coincide, MBR decoding outperforms the classifier on the latents (MP action). This confirms that there is substantial gain not only in correcting the LLM's latent beliefs by fitting a classifier $\pi_\phi$ but instead task calibrating the LLM's belief distribution.
\end{inparaenum}

\paragraph{What the Calibration Maps Learn.}

\looseness=-1 We now study how the calibration map \(f_\phi(\hat{p}(\rx))\) transforms the original latent probabilities \(\hat{p}(\rx)\). For binary action spaces like the Answer-or-Abstain latent structure $\mathcal{L} \cong \{A, \perp\}$, the map can be fully described by how it maps the probability of one of the actions, for example $A$ ("answer"). In \Cref{fig:calibration_map_main}, we observe that for this action space, the learned map $f_\phi$ consistently lowers the predicted probability of the action $A$ ("answer") to mitigate overconfidence in the uncalibrated LLM's belief distribution. Effectively, the recalibration map will only output probabilities this action $\smash{f_\phi(\hat{p}(\rx))}_A$ above the threshold in the BAS-metric ($t = 0.25$) if the uncalibrated distribution assigns a significantly higher probability ($\hat{p}(\rx)_A > 0.7$).

\looseness=-1 For multi-criterion ordinal regression (Helpsteer) ( \(\mathcal{L} \cong \{0,\dots,4\}^2\)), i.e. rating the correctness of a model response, we study how MBR responses change between the uncalibrated and calibrated beliefs in \Cref{fig:action_movement_helpsteer_qwen30}. We find a pattern akin to an ``S-Curve'': for low ratings (0--2), the model tends to be underconfident and calibration shifts MBR responses toward higher scores. In contrast, high ratings (3--4) are either left unchanged or shifted downward. Interestingly, the learned calibration map effectively no longer yields the latent response $\ell \equalhat 4$. However, by \Cref{lem:optimal_strategy}, among all queries for which the uncalibrated LLM's MBR response is $\ell \equalhat 4$, i.e. $P_4 = \{\hat{p}(\rx)\ : \delta^\mathrm{MBR}(\hat{p}(\rx)) = 4\}$, calibration reduces the average loss. In our experiments, predicting $\ell \equalhat 3$ whenever the model's uncalibrated belief yields an MBR action $\ell \equalhat 4$ has better performance on average. Also, no other uncalibrated MBR action is shifted to a calibrated MBR response of $\ell \equalhat 4$. This implies that there is no signal in the LLM's uncalibrated beliefs $\hat{p}(\rx)$ from which it is optimal to predict $\ell = 4$. This is not a flaw of task calibration but reveals a severe shortcoming in the LLM: It is unable to reliably predict $\ell \equalhat 4$. Hence, the optimal calibration strategy in terms of the expected loss is to avoid predicting $\ell \equalhat 4$ altogether.

\looseness=-1 Similarly, we study how task calibration can help to mitigate systematic biases \cite{gallegos2024bias,shi-etal-2025-judging} in an LLM's distribution. In particular, we consider the multi-criterion ordinal regression task on Helpsteer, where the LLM rates a model's response on both helpfulness and correctness simultaneously. In \Cref{fig:latent_correlation}, we find that the 
 uncalibrated $\hat{p}(\rx)$ heavily correlates ratings for the two criteria, meaning the LLMs is very likely to assign identical values to both. However, in the true data, the correlation is significantly weaker. The task calibration map $\smash{f_\phi(\hat{p}(\rx))}$ more accurately recovers the true correlation structure and therefore improves the LLM's reliability by correcting this correlation bias using the calibration data.

\begin{figure*}[t!]
    \centering
    \begin{subfigure}[t]{0.47\linewidth}
        \centering
        \input{figures/tce_comparison_gemma-3-12b-it.pgf}
        \caption{Task Calibration Error (TCE) for uncalibrated \(q\) and Dirichlet-calibrated \(f_{\phi}(q)\) model on Gemma 3-12B. Across all tasks, the model is severely task miscalibrated, which is mitigated by our approximate recalibration method.}
        \label{fig:tce_comparison}
    \end{subfigure}
    \hfill
    \begin{subfigure}[t]{0.47\linewidth}
        \centering
        \input{figures/tce_loss_gain_correlation_two.pgf}
        \caption{
        Task Calibration Error (TCE) correlates with downstream gains, whereas ECE does not.
        Datasets:
        {\tikz\fill[mmlu] (0,0) circle (0.7ex);} SimpleQA,
        {\tikz\fill[when2call] (0,0) circle (0.7ex);} TriviaQA,
        {\tikz\fill[simpleqa] (0,0) circle (0.7ex);} When2Call,
        {\tikz\fill[trivia] (0,0) circle (0.7ex);} MMLU.
        Models:
        {\tikz\fill (0,0) circle (0.7ex);} Qwen3-30B-A3B,
        {\tikz\fill (-0.7ex,-0.7ex) rectangle (0.7ex,0.7ex);} Qwen3-4B,
        {\tikz\draw[fill=black] (0,0.9ex)--(-0.8ex,-0.6ex)--(0.8ex,-0.6ex)--cycle;} Gemma-3-12B,
        {\tikz\draw[fill=black] (0,1.0ex)--(-0.9ex,0)--(0,-1.0ex)--(0.9ex,0)--cycle;} Gemma-3-4B
        }
        \label{fig:tce_correlation}
    \end{subfigure}
    \label{fig:tce_combined}
    \vspace{-1em}
\end{figure*}

\subsection{Task-Specific Calibration Error}
\looseness=-1 We evaluate the quality of our recalibration map \(f_{\phi}\) through the Task Calibration Error (TCE). We compare the TCE of the the original model predictions \(\hat{p}(\rx)\) with their recalibrated counterparts \(f_{\phi}(\hat{p}(\rx))\). \Cref{fig:tce_comparison} shows that the LLMs are often substantially task miscalibrated, which indicates possible performance improvements through task calibration. Importantly, Dirichlet recalibration drastically reduces TCE -- often to zero --, demonstrating that the approximately recalibrated predictor $\delta^\mathrm{MBR}(f_\phi(\hat{p}(\rx)))$ recovers most of the theoretically attainable improvement. It also verifies that a simple affine (Dirichlet) recalibration suffices to find MBR actions that are optimal according to \Cref{thm:optimal_strategy}, consistent with the strong empirical decoding performance seen in \Cref{tab:performance_generations}.

\looseness=-1 We also verify the practical usefulness of TCE by investigating how well it predicts performance gains achieved by recalibration and compare it to the Expected Calibration Error (ECE) \cite{degroot1983comparison,guo2017calibration} computed from the LLM's latent distribution $\hat{p}(\rx)$.
As shown in \Cref{fig:tce_correlation}, in contrast to ECE, TCE exhibits a strong correlation
with the realized loss reduction obtained through recalibration. Notably, as TCE is defined directly in terms of the task-specific loss, it permits a task-dependent interpretation, and the concrete values are meaningful. In contrast, ECE quantifies an average difference in confidence that cannot be directly tied to the task's loss. We defer full results to \Cref{app:tce}.

%% file: tables/decoding.tex
\begin{tabular}{ll|cc|cc|cc|c|c}
\toprule
\multicolumn{2}{c|}{\textbf{Latent Structure $\mathcal{L}$}(Metric)} & \multicolumn{2}{c|}{\small{$\{0, \dots, C\}$ (L1 $\downarrow$)}} & \multicolumn{2}{c|}{\small{$[C]$ (Acc. $\uparrow$)}} & \multicolumn{2}{c|}{\small{$\{A, \perp\}$ (BAS $\uparrow$)}} & \multicolumn{1}{c|}{\small{$\{0, 1\}^{\leq C}$ (Hamming $\downarrow$)}} & \multicolumn{1}{c}{\small{$\{0, \dots, C\}^2$ (L1 $\downarrow$)}} \\
\multicolumn{2}{c|}{} & Helpsteer & STSB & MMLU & When2Call & SimpleQA & TriviaQA & MAQA & Helpsteer \\
\midrule
\multirow{6}{*}{\small{Gemma-3-4B}} & \small{Ancestral} & {$\textcolor{colsecond}{\textbf{0.80}}$$\scriptscriptstyle{\pm 0.12}$} & {$\textcolor{colsecond}{\textbf{1.05}}$$\scriptscriptstyle{\pm 0.06}$} & {$61.87$$\scriptscriptstyle{\pm 2.90}$} & {$50.81$$\scriptscriptstyle{\pm 2.84}$} & {$\textcolor{colsecond}{\textbf{-0.10}}$$\scriptscriptstyle{\pm 0.02}$} & {$\textcolor{colsecond}{\textbf{0.21}}$$\scriptscriptstyle{\pm 0.05}$} & {$0.97$$\scriptscriptstyle{\pm 0.05}$} & {$2.05$$\scriptscriptstyle{\pm 0.14}$} \\
 & \small{Beam} & {$0.85$$\scriptscriptstyle{\pm 0.11}$} & {$1.06$$\scriptscriptstyle{\pm 0.06}$} & {$62.61$$\scriptscriptstyle{\pm 2.27}$} & {$50.49$$\scriptscriptstyle{\pm 4.28}$} & {$-0.17$$\scriptscriptstyle{\pm 0.02}$} & {$0.20$$\scriptscriptstyle{\pm 0.05}$} & {$0.89$$\scriptscriptstyle{\pm 0.07}$} & {$2.10$$\scriptscriptstyle{\pm 0.09}$} \\
 & \small{Contrastive} & {$0.85$$\scriptscriptstyle{\pm 0.12}$} & {$1.08$$\scriptscriptstyle{\pm 0.07}$} & {$61.64$$\scriptscriptstyle{\pm 3.36}$} & {$50.49$$\scriptscriptstyle{\pm 2.42}$} & {$-0.14$$\scriptscriptstyle{\pm 0.02}$} & {$0.19$$\scriptscriptstyle{\pm 0.04}$} & {$1.45$$\scriptscriptstyle{\pm 0.11}$} & {$2.14$$\scriptscriptstyle{\pm 0.11}$} \\
 & \small{Greedy} & {$0.86$$\scriptscriptstyle{\pm 0.12}$} & {$1.07$$\scriptscriptstyle{\pm 0.07}$} & {$62.61$$\scriptscriptstyle{\pm 1.42}$} & {$49.84$$\scriptscriptstyle{\pm 3.02}$} & {$-0.16$$\scriptscriptstyle{\pm 0.01}$} & {$0.19$$\scriptscriptstyle{\pm 0.04}$} & {$0.94$$\scriptscriptstyle{\pm 0.03}$} & {$2.05$$\scriptscriptstyle{\pm 0.10}$} \\
 & \small{MBR} & {$0.82$$\scriptscriptstyle{\pm 0.08}$} & {$1.07$$\scriptscriptstyle{\pm 0.08}$} & {$\textcolor{colsecond}{\textbf{64.56}}$$\scriptscriptstyle{\pm 2.08}$} & {$\textcolor{colsecond}{\textbf{51.14}}$$\scriptscriptstyle{\pm 2.78}$} & {$-0.20$$\scriptscriptstyle{\pm 0.02}$} & {$0.18$$\scriptscriptstyle{\pm 0.05}$} & {$\textcolor{colsecond}{\textbf{0.74}}$$\scriptscriptstyle{\pm 0.04}$} & {$\textcolor{colsecond}{\textbf{2.03}}$$\scriptscriptstyle{\pm 0.12}$} \\
 & \cellcolor[gray]{0.9}{\small{\textbf{Ours}}} & \cellcolor[gray]{0.9}{{$\textcolor{colfirst}{\textbf{0.73}}$$\scriptscriptstyle{\pm 0.08}$}} & \cellcolor[gray]{0.9}{{$\textcolor{colfirst}{\textbf{0.80}}$$\scriptscriptstyle{\pm 0.05}$}} & \cellcolor[gray]{0.9}{{$\textcolor{colfirst}{\textbf{67.11}}$$\scriptscriptstyle{\pm 3.95}$}} & \cellcolor[gray]{0.9}{{$\textcolor{colfirst}{\textbf{60.63}}$$\scriptscriptstyle{\pm 3.57}$}} & \cellcolor[gray]{0.9}{{$\textcolor{colfirst}{\textbf{-0.00}}$$\scriptscriptstyle{\pm 0.00}$}} & \cellcolor[gray]{0.9}{{$\textcolor{colfirst}{\textbf{0.22}}$$\scriptscriptstyle{\pm 0.05}$}} & \cellcolor[gray]{0.9}{{$\textcolor{colfirst}{\textbf{0.62}}$$\scriptscriptstyle{\pm 0.03}$}} & \cellcolor[gray]{0.9}{{$\textcolor{colfirst}{\textbf{1.25}}$$\scriptscriptstyle{\pm 0.04}$}} \\
\midrule
\multirow{6}{*}{\small{Gemma-3-12B}} & \small{Ancestral} & {$0.88$$\scriptscriptstyle{\pm 0.04}$} & {$1.01$$\scriptscriptstyle{\pm 0.04}$} & {$\textcolor{colsecond}{\textbf{80.11}}$$\scriptscriptstyle{\pm 5.03}$} & {$67.47$$\scriptscriptstyle{\pm 2.14}$} & {$\textcolor{colsecond}{\textbf{-0.10}}$$\scriptscriptstyle{\pm 0.01}$} & {$\textcolor{colsecond}{\textbf{0.46}}$$\scriptscriptstyle{\pm 0.04}$} & {$0.74$$\scriptscriptstyle{\pm 0.06}$} & {$1.95$$\scriptscriptstyle{\pm 0.09}$} \\
 & \small{Beam} & {$0.90$$\scriptscriptstyle{\pm 0.05}$} & {$\textcolor{colsecond}{\textbf{0.98}}$$\scriptscriptstyle{\pm 0.02}$} & {$80.00$$\scriptscriptstyle{\pm 2.86}$} & {$\textcolor{colsecond}{\textbf{69.17}}$$\scriptscriptstyle{\pm 1.95}$} & {$-0.17$$\scriptscriptstyle{\pm 0.03}$} & {$0.45$$\scriptscriptstyle{\pm 0.03}$} & {$0.69$$\scriptscriptstyle{\pm 0.08}$} & {$1.95$$\scriptscriptstyle{\pm 0.08}$} \\
 & \small{Contrastive} & {$\textcolor{colsecond}{\textbf{0.88}}$$\scriptscriptstyle{\pm 0.02}$} & {$0.99$$\scriptscriptstyle{\pm 0.02}$} & {$79.44$$\scriptscriptstyle{\pm 2.77}$} & {$67.67$$\scriptscriptstyle{\pm 2.73}$} & {$-0.15$$\scriptscriptstyle{\pm 0.03}$} & {$0.46$$\scriptscriptstyle{\pm 0.03}$} & {$0.77$$\scriptscriptstyle{\pm 0.08}$} & {$1.90$$\scriptscriptstyle{\pm 0.06}$} \\
 & \small{Greedy} & {$0.89$$\scriptscriptstyle{\pm 0.05}$} & {$1.00$$\scriptscriptstyle{\pm 0.02}$} & {$\textcolor{colsecond}{\textbf{80.11}}$$\scriptscriptstyle{\pm 3.03}$} & {$68.57$$\scriptscriptstyle{\pm 2.59}$} & {$-0.16$$\scriptscriptstyle{\pm 0.03}$} & {$0.45$$\scriptscriptstyle{\pm 0.03}$} & {$0.74$$\scriptscriptstyle{\pm 0.08}$} & {$1.91$$\scriptscriptstyle{\pm 0.06}$} \\
 & \small{MBR} & {$0.90$$\scriptscriptstyle{\pm 0.05}$} & {$1.01$$\scriptscriptstyle{\pm 0.03}$} & {$\textcolor{colsecond}{\textbf{80.11}}$$\scriptscriptstyle{\pm 2.84}$} & {$68.87$$\scriptscriptstyle{\pm 2.09}$} & {$-0.19$$\scriptscriptstyle{\pm 0.03}$} & {$0.44$$\scriptscriptstyle{\pm 0.03}$} & {$\textcolor{colsecond}{\textbf{0.60}}$$\scriptscriptstyle{\pm 0.06}$} & {$\textcolor{colsecond}{\textbf{1.89}}$$\scriptscriptstyle{\pm 0.07}$} \\
 & \cellcolor[gray]{0.9}{\small{\textbf{Ours}}} & \cellcolor[gray]{0.9}{{$\textcolor{colfirst}{\textbf{0.79}}$$\scriptscriptstyle{\pm 0.07}$}} & \cellcolor[gray]{0.9}{{$\textcolor{colfirst}{\textbf{0.65}}$$\scriptscriptstyle{\pm 0.07}$}} & \cellcolor[gray]{0.9}{{$\textcolor{colfirst}{\textbf{89.72}}$$\scriptscriptstyle{\pm 1.56}$}} & \cellcolor[gray]{0.9}{{$\textcolor{colfirst}{\textbf{74.38}}$$\scriptscriptstyle{\pm 2.64}$}} & \cellcolor[gray]{0.9}{{$\textcolor{colfirst}{\textbf{-0.00}}$$\scriptscriptstyle{\pm 0.00}$}} & \cellcolor[gray]{0.9}{{$\textcolor{colfirst}{\textbf{0.47}}$$\scriptscriptstyle{\pm 0.03}$}} & \cellcolor[gray]{0.9}{{$\textcolor{colfirst}{\textbf{0.51}}$$\scriptscriptstyle{\pm 0.05}$}} & \cellcolor[gray]{0.9}{{$\textcolor{colfirst}{\textbf{1.26}}$$\scriptscriptstyle{\pm 0.04}$}} \\
\midrule
\multirow{6}{*}{\small{Qwen-3-4B}} & \small{Ancestral} & {$0.82$$\scriptscriptstyle{\pm 0.07}$} & {$1.33$$\scriptscriptstyle{\pm 0.04}$} & {$77.47$$\scriptscriptstyle{\pm 1.35}$} & {$69.80$$\scriptscriptstyle{\pm 0.81}$} & {$\textcolor{colsecond}{\textbf{-0.11}}$$\scriptscriptstyle{\pm 0.02}$} & {$0.19$$\scriptscriptstyle{\pm 0.03}$} & {$1.09$$\scriptscriptstyle{\pm 0.12}$} & {$2.11$$\scriptscriptstyle{\pm 0.06}$} \\
 & \small{Beam} & {$0.82$$\scriptscriptstyle{\pm 0.06}$} & {$1.33$$\scriptscriptstyle{\pm 0.05}$} & {$77.92$$\scriptscriptstyle{\pm 2.51}$} & {$70.40$$\scriptscriptstyle{\pm 3.12}$} & {$-0.13$$\scriptscriptstyle{\pm 0.01}$} & {$0.19$$\scriptscriptstyle{\pm 0.04}$} & {$1.03$$\scriptscriptstyle{\pm 0.13}$} & {$2.09$$\scriptscriptstyle{\pm 0.08}$} \\
 & \small{Contrastive} & {$0.83$$\scriptscriptstyle{\pm 0.08}$} & {$1.33$$\scriptscriptstyle{\pm 0.06}$} & {$76.68$$\scriptscriptstyle{\pm 0.90}$} & {$72.10$$\scriptscriptstyle{\pm 1.62}$} & {$-0.11$$\scriptscriptstyle{\pm 0.01}$} & {$\textcolor{colsecond}{\textbf{0.19}}$$\scriptscriptstyle{\pm 0.04}$} & {$1.20$$\scriptscriptstyle{\pm 0.09}$} & {$\textcolor{colsecond}{\textbf{2.03}}$$\scriptscriptstyle{\pm 0.06}$} \\
 & \small{Greedy} & {$0.83$$\scriptscriptstyle{\pm 0.07}$} & {$1.34$$\scriptscriptstyle{\pm 0.05}$} & {$77.24$$\scriptscriptstyle{\pm 1.59}$} & {$71.30$$\scriptscriptstyle{\pm 1.63}$} & {$-0.14$$\scriptscriptstyle{\pm 0.01}$} & {$0.18$$\scriptscriptstyle{\pm 0.03}$} & {$1.03$$\scriptscriptstyle{\pm 0.11}$} & {$2.14$$\scriptscriptstyle{\pm 0.05}$} \\
 & \small{MBR} & {$\textcolor{colsecond}{\textbf{0.82}}$$\scriptscriptstyle{\pm 0.07}$} & {$\textcolor{colsecond}{\textbf{1.33}}$$\scriptscriptstyle{\pm 0.04}$} & {$\textcolor{colsecond}{\textbf{77.92}}$$\scriptscriptstyle{\pm 1.79}$} & {$\textcolor{colsecond}{\textbf{73.00}}$$\scriptscriptstyle{\pm 1.70}$} & {$-0.18$$\scriptscriptstyle{\pm 0.01}$} & {$0.16$$\scriptscriptstyle{\pm 0.04}$} & {$\textcolor{colsecond}{\textbf{0.87}}$$\scriptscriptstyle{\pm 0.08}$} & {$2.06$$\scriptscriptstyle{\pm 0.05}$} \\
 & \cellcolor[gray]{0.9}{\small{\textbf{Ours}}} & \cellcolor[gray]{0.9}{{$\textcolor{colfirst}{\textbf{0.76}}$$\scriptscriptstyle{\pm 0.04}$}} & \cellcolor[gray]{0.9}{{$\textcolor{colfirst}{\textbf{0.68}}$$\scriptscriptstyle{\pm 0.05}$}} & \cellcolor[gray]{0.9}{{$\textcolor{colfirst}{\textbf{83.07}}$$\scriptscriptstyle{\pm 0.93}$}} & \cellcolor[gray]{0.9}{{$\textcolor{colfirst}{\textbf{74.10}}$$\scriptscriptstyle{\pm 2.75}$}} & \cellcolor[gray]{0.9}{{$\textcolor{colfirst}{\textbf{-0.00}}$$\scriptscriptstyle{\pm 0.00}$}} & \cellcolor[gray]{0.9}{{$\textcolor{colfirst}{\textbf{0.21}}$$\scriptscriptstyle{\pm 0.03}$}} & \cellcolor[gray]{0.9}{{$\textcolor{colfirst}{\textbf{0.72}}$$\scriptscriptstyle{\pm 0.05}$}} & \cellcolor[gray]{0.9}{{$\textcolor{colfirst}{\textbf{1.27}}$$\scriptscriptstyle{\pm 0.06}$}} \\
\midrule
\multirow{6}{*}{\small{Qwen3-30B-A3B}} & \small{Ancestral} & {$0.80$$\scriptscriptstyle{\pm 0.06}$} & {$1.19$$\scriptscriptstyle{\pm 0.12}$} & {$77.91$$\scriptscriptstyle{\pm 3.26}$} & {$74.49$$\scriptscriptstyle{\pm 3.72}$} & {$\textcolor{colsecond}{\textbf{-0.01}}$$\scriptscriptstyle{\pm 0.01}$} & {$\textcolor{colfirst}{\textbf{0.47}}$$\scriptscriptstyle{\pm 0.04}$} & {$0.64$$\scriptscriptstyle{\pm 0.04}$} & {$2.16$$\scriptscriptstyle{\pm 0.09}$} \\
 & \small{Beam} & {$\textcolor{colsecond}{\textbf{0.78}}$$\scriptscriptstyle{\pm 0.06}$} & {$1.19$$\scriptscriptstyle{\pm 0.10}$} & {$\textcolor{colsecond}{\textbf{80.27}}$$\scriptscriptstyle{\pm 2.46}$} & {$74.19$$\scriptscriptstyle{\pm 3.56}$} & {$-0.03$$\scriptscriptstyle{\pm 0.03}$} & {$\textcolor{colsecond}{\textbf{0.47}}$$\scriptscriptstyle{\pm 0.05}$} & {$0.62$$\scriptscriptstyle{\pm 0.03}$} & {$2.16$$\scriptscriptstyle{\pm 0.08}$} \\
 & \small{Contrastive} & {$0.81$$\scriptscriptstyle{\pm 0.07}$} & {$\textcolor{colsecond}{\textbf{1.17}}$$\scriptscriptstyle{\pm 0.10}$} & {$80.27$$\scriptscriptstyle{\pm 2.64}$} & {$74.39$$\scriptscriptstyle{\pm 2.21}$} & {$-0.03$$\scriptscriptstyle{\pm 0.02}$} & {$0.46$$\scriptscriptstyle{\pm 0.04}$} & {$0.74$$\scriptscriptstyle{\pm 0.08}$} & {$2.13$$\scriptscriptstyle{\pm 0.07}$} \\
 & \small{Greedy} & {$0.79$$\scriptscriptstyle{\pm 0.07}$} & {$1.19$$\scriptscriptstyle{\pm 0.11}$} & {$77.02$$\scriptscriptstyle{\pm 2.33}$} & {$\textcolor{colsecond}{\textbf{74.80}}$$\scriptscriptstyle{\pm 1.61}$} & {$-0.04$$\scriptscriptstyle{\pm 0.02}$} & {$0.46$$\scriptscriptstyle{\pm 0.05}$} & {$0.64$$\scriptscriptstyle{\pm 0.06}$} & {$2.15$$\scriptscriptstyle{\pm 0.12}$} \\
 & \small{MBR} & {$0.78$$\scriptscriptstyle{\pm 0.07}$} & {$1.18$$\scriptscriptstyle{\pm 0.12}$} & {$78.48$$\scriptscriptstyle{\pm 1.92}$} & {$74.59$$\scriptscriptstyle{\pm 3.52}$} & {$-0.07$$\scriptscriptstyle{\pm 0.02}$} & {$0.46$$\scriptscriptstyle{\pm 0.05}$} & {$\textcolor{colsecond}{\textbf{0.54}}$$\scriptscriptstyle{\pm 0.03}$} & {$\textcolor{colsecond}{\textbf{2.11}}$$\scriptscriptstyle{\pm 0.12}$} \\
 & \cellcolor[gray]{0.9}{\small{\textbf{Ours}}} & \cellcolor[gray]{0.9}{{$\textcolor{colfirst}{\textbf{0.72}}$$\scriptscriptstyle{\pm 0.05}$}} & \cellcolor[gray]{0.9}{{$\textcolor{colfirst}{\textbf{0.67}}$$\scriptscriptstyle{\pm 0.05}$}} & \cellcolor[gray]{0.9}{{$\textcolor{colfirst}{\textbf{88.34}}$$\scriptscriptstyle{\pm 2.11}$}} & \cellcolor[gray]{0.9}{{$\textcolor{colfirst}{\textbf{78.63}}$$\scriptscriptstyle{\pm 2.50}$}} & \cellcolor[gray]{0.9}{{$\textcolor{colfirst}{\textbf{0.03}}$$\scriptscriptstyle{\pm 0.02}$}} & \cellcolor[gray]{0.9}{{$0.46$$\scriptscriptstyle{\pm 0.04}$}} & \cellcolor[gray]{0.9}{{$\textcolor{colfirst}{\textbf{0.49}}$$\scriptscriptstyle{\pm 0.03}$}} & \cellcolor[gray]{0.9}{{$\textcolor{colfirst}{\textbf{1.21}}$$\scriptscriptstyle{\pm 0.04}$}} \\
\bottomrule
\end{tabular}

%% file: figures/HELPSTEER_Qwen30_action_movements.pgf
\begingroup%
\makeatletter%
\begin{pgfpicture}%
\pgfpathrectangle{\pgfpointorigin}{\pgfqpoint{1.726944in}{1.299709in}}%
\pgfusepath{use as bounding box, clip}%
\begin{pgfscope}%
\pgfsetbuttcap%
\pgfsetmiterjoin%
\definecolor{currentfill}{rgb}{1.000000,1.000000,1.000000}%
\pgfsetfillcolor{currentfill}%
\pgfsetlinewidth{0.000000pt}%
\definecolor{currentstroke}{rgb}{1.000000,1.000000,1.000000}%
\pgfsetstrokecolor{currentstroke}%
\pgfsetdash{}{0pt}%
\pgfpathmoveto{\pgfqpoint{0.000000in}{0.000000in}}%
\pgfpathlineto{\pgfqpoint{1.726944in}{0.000000in}}%
\pgfpathlineto{\pgfqpoint{1.726944in}{1.299709in}}%
\pgfpathlineto{\pgfqpoint{0.000000in}{1.299709in}}%
\pgfpathlineto{\pgfqpoint{0.000000in}{0.000000in}}%
\pgfpathclose%
\pgfusepath{fill}%
\end{pgfscope}%
\begin{pgfscope}%
\pgfsetbuttcap%
\pgfsetmiterjoin%
\definecolor{currentfill}{rgb}{1.000000,1.000000,1.000000}%
\pgfsetfillcolor{currentfill}%
\pgfsetlinewidth{0.000000pt}%
\definecolor{currentstroke}{rgb}{0.000000,0.000000,0.000000}%
\pgfsetstrokecolor{currentstroke}%
\pgfsetstrokeopacity{0.000000}%
\pgfsetdash{}{0pt}%
\pgfpathmoveto{\pgfqpoint{1.248330in}{0.345878in}}%
\pgfpathlineto{\pgfqpoint{1.278330in}{0.345878in}}%
\pgfpathlineto{\pgfqpoint{1.278330in}{1.095878in}}%
\pgfpathlineto{\pgfqpoint{1.248330in}{1.095878in}}%
\pgfpathlineto{\pgfqpoint{1.248330in}{0.345878in}}%
\pgfpathclose%
\pgfusepath{fill}%
\end{pgfscope}%
\begin{pgfscope}%
\pgfsys@transformshift{1.246667in}{0.349709in}%
\pgftext[left,bottom]{\includegraphics[interpolate=true,width=0.030000in,height=0.750000in]{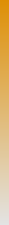}}%
\end{pgfscope}%
\begin{pgfscope}%
\pgfsetbuttcap%
\pgfsetroundjoin%
\definecolor{currentfill}{rgb}{0.000000,0.000000,0.000000}%
\pgfsetfillcolor{currentfill}%
\pgfsetlinewidth{0.803000pt}%
\definecolor{currentstroke}{rgb}{0.000000,0.000000,0.000000}%
\pgfsetstrokecolor{currentstroke}%
\pgfsetdash{}{0pt}%
\pgfsys@defobject{currentmarker}{\pgfqpoint{0.000000in}{0.000000in}}{\pgfqpoint{0.048611in}{0.000000in}}{%
\pgfpathmoveto{\pgfqpoint{0.000000in}{0.000000in}}%
\pgfpathlineto{\pgfqpoint{0.048611in}{0.000000in}}%
\pgfusepath{stroke,fill}%
}%
\begin{pgfscope}%
\pgfsys@transformshift{1.278330in}{0.345878in}%
\pgfsys@useobject{currentmarker}{}%
\end{pgfscope}%
\end{pgfscope}%
\begin{pgfscope}%
\definecolor{textcolor}{rgb}{0.000000,0.000000,0.000000}%
\pgfsetstrokecolor{textcolor}%
\pgfsetfillcolor{textcolor}%
\pgftext[x=1.375552in, y=0.308945in, left, base]{\color{textcolor}{\rmfamily\fontsize{7.000000}{8.400000}\selectfont\catcode`\^=\active\def^{\ifmmode\sp\else\^{}\fi}\catcode`\%=\active\def
\end{pgfscope}%
\begin{pgfscope}%
\pgfsetbuttcap%
\pgfsetroundjoin%
\definecolor{currentfill}{rgb}{0.000000,0.000000,0.000000}%
\pgfsetfillcolor{currentfill}%
\pgfsetlinewidth{0.803000pt}%
\definecolor{currentstroke}{rgb}{0.000000,0.000000,0.000000}%
\pgfsetstrokecolor{currentstroke}%
\pgfsetdash{}{0pt}%
\pgfsys@defobject{currentmarker}{\pgfqpoint{0.000000in}{0.000000in}}{\pgfqpoint{0.048611in}{0.000000in}}{%
\pgfpathmoveto{\pgfqpoint{0.000000in}{0.000000in}}%
\pgfpathlineto{\pgfqpoint{0.048611in}{0.000000in}}%
\pgfusepath{stroke,fill}%
}%
\begin{pgfscope}%
\pgfsys@transformshift{1.278330in}{0.720878in}%
\pgfsys@useobject{currentmarker}{}%
\end{pgfscope}%
\end{pgfscope}%
\begin{pgfscope}%
\definecolor{textcolor}{rgb}{0.000000,0.000000,0.000000}%
\pgfsetstrokecolor{textcolor}%
\pgfsetfillcolor{textcolor}%
\pgftext[x=1.375552in, y=0.683945in, left, base]{\color{textcolor}{\rmfamily\fontsize{7.000000}{8.400000}\selectfont\catcode`\^=\active\def^{\ifmmode\sp\else\^{}\fi}\catcode`\%=\active\def
\end{pgfscope}%
\begin{pgfscope}%
\pgfsetbuttcap%
\pgfsetroundjoin%
\definecolor{currentfill}{rgb}{0.000000,0.000000,0.000000}%
\pgfsetfillcolor{currentfill}%
\pgfsetlinewidth{0.803000pt}%
\definecolor{currentstroke}{rgb}{0.000000,0.000000,0.000000}%
\pgfsetstrokecolor{currentstroke}%
\pgfsetdash{}{0pt}%
\pgfsys@defobject{currentmarker}{\pgfqpoint{0.000000in}{0.000000in}}{\pgfqpoint{0.048611in}{0.000000in}}{%
\pgfpathmoveto{\pgfqpoint{0.000000in}{0.000000in}}%
\pgfpathlineto{\pgfqpoint{0.048611in}{0.000000in}}%
\pgfusepath{stroke,fill}%
}%
\begin{pgfscope}%
\pgfsys@transformshift{1.278330in}{1.095878in}%
\pgfsys@useobject{currentmarker}{}%
\end{pgfscope}%
\end{pgfscope}%
\begin{pgfscope}%
\definecolor{textcolor}{rgb}{0.000000,0.000000,0.000000}%
\pgfsetstrokecolor{textcolor}%
\pgfsetfillcolor{textcolor}%
\pgftext[x=1.375552in, y=1.058945in, left, base]{\color{textcolor}{\rmfamily\fontsize{7.000000}{8.400000}\selectfont\catcode`\^=\active\def^{\ifmmode\sp\else\^{}\fi}\catcode`\%=\active\def
\end{pgfscope}%
\begin{pgfscope}%
\definecolor{textcolor}{rgb}{0.000000,0.000000,0.000000}%
\pgfsetstrokecolor{textcolor}%
\pgfsetfillcolor{textcolor}%
\pgftext[x=1.585723in,y=0.720878in,,top,rotate=90.000000]{\color{textcolor}{\rmfamily\fontsize{9.000000}{10.800000}\selectfont\catcode`\^=\active\def^{\ifmmode\sp\else\^{}\fi}\catcode`\%=\active\def
\end{pgfscope}%
\begin{pgfscope}%
\pgfsetrectcap%
\pgfsetmiterjoin%
\pgfsetlinewidth{0.803000pt}%
\definecolor{currentstroke}{rgb}{0.000000,0.000000,0.000000}%
\pgfsetstrokecolor{currentstroke}%
\pgfsetdash{}{0pt}%
\pgfpathmoveto{\pgfqpoint{1.248330in}{0.345878in}}%
\pgfpathlineto{\pgfqpoint{1.263330in}{0.345878in}}%
\pgfpathlineto{\pgfqpoint{1.278330in}{0.345878in}}%
\pgfpathlineto{\pgfqpoint{1.278330in}{1.095878in}}%
\pgfpathlineto{\pgfqpoint{1.263330in}{1.095878in}}%
\pgfpathlineto{\pgfqpoint{1.248330in}{1.095878in}}%
\pgfpathlineto{\pgfqpoint{1.248330in}{0.345878in}}%
\pgfpathclose%
\pgfusepath{stroke}%
\end{pgfscope}%
\begin{pgfscope}%
\pgfsetbuttcap%
\pgfsetmiterjoin%
\definecolor{currentfill}{rgb}{1.000000,1.000000,1.000000}%
\pgfsetfillcolor{currentfill}%
\pgfsetlinewidth{0.000000pt}%
\definecolor{currentstroke}{rgb}{0.000000,0.000000,0.000000}%
\pgfsetstrokecolor{currentstroke}%
\pgfsetstrokeopacity{0.000000}%
\pgfsetdash{}{0pt}%
\pgfpathmoveto{\pgfqpoint{0.350605in}{0.399019in}}%
\pgfpathlineto{\pgfqpoint{1.098648in}{0.399019in}}%
\pgfpathlineto{\pgfqpoint{1.098648in}{1.147062in}}%
\pgfpathlineto{\pgfqpoint{0.350605in}{1.147062in}}%
\pgfpathlineto{\pgfqpoint{0.350605in}{0.399019in}}%
\pgfpathclose%
\pgfusepath{fill}%
\end{pgfscope}%
\begin{pgfscope}%
\pgfpathrectangle{\pgfqpoint{0.350605in}{0.399019in}}{\pgfqpoint{0.748043in}{0.748043in}}%
\pgfusepath{clip}%
\pgfsys@transformshift{0.350605in}{0.399019in}%
\pgftext[left,bottom]{\includegraphics[interpolate=true,width=0.750000in,height=0.750000in]{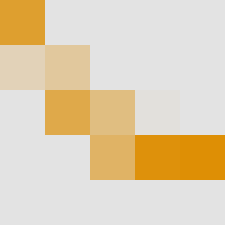}}%
\end{pgfscope}%
\begin{pgfscope}%
\pgfpathrectangle{\pgfqpoint{0.350605in}{0.399019in}}{\pgfqpoint{0.748043in}{0.748043in}}%
\pgfusepath{clip}%
\pgfsetbuttcap%
\pgfsetmiterjoin%
\pgfsetlinewidth{2.007500pt}%
\definecolor{currentstroke}{rgb}{0.000000,0.000000,0.000000}%
\pgfsetstrokecolor{currentstroke}%
\pgfsetdash{}{0pt}%
\pgfpathmoveto{\pgfqpoint{0.350605in}{1.147062in}}%
\pgfpathlineto{\pgfqpoint{0.500214in}{1.147062in}}%
\pgfpathlineto{\pgfqpoint{0.500214in}{0.997453in}}%
\pgfpathlineto{\pgfqpoint{0.350605in}{0.997453in}}%
\pgfpathlineto{\pgfqpoint{0.350605in}{1.147062in}}%
\pgfpathclose%
\pgfusepath{stroke}%
\end{pgfscope}%
\begin{pgfscope}%
\pgfpathrectangle{\pgfqpoint{0.350605in}{0.399019in}}{\pgfqpoint{0.748043in}{0.748043in}}%
\pgfusepath{clip}%
\pgfsetbuttcap%
\pgfsetmiterjoin%
\pgfsetlinewidth{2.007500pt}%
\definecolor{currentstroke}{rgb}{0.000000,0.000000,0.000000}%
\pgfsetstrokecolor{currentstroke}%
\pgfsetdash{}{0pt}%
\pgfpathmoveto{\pgfqpoint{0.500214in}{0.997453in}}%
\pgfpathlineto{\pgfqpoint{0.649822in}{0.997453in}}%
\pgfpathlineto{\pgfqpoint{0.649822in}{0.847844in}}%
\pgfpathlineto{\pgfqpoint{0.500214in}{0.847844in}}%
\pgfpathlineto{\pgfqpoint{0.500214in}{0.997453in}}%
\pgfpathclose%
\pgfusepath{stroke}%
\end{pgfscope}%
\begin{pgfscope}%
\pgfpathrectangle{\pgfqpoint{0.350605in}{0.399019in}}{\pgfqpoint{0.748043in}{0.748043in}}%
\pgfusepath{clip}%
\pgfsetbuttcap%
\pgfsetmiterjoin%
\pgfsetlinewidth{2.007500pt}%
\definecolor{currentstroke}{rgb}{0.000000,0.000000,0.000000}%
\pgfsetstrokecolor{currentstroke}%
\pgfsetdash{}{0pt}%
\pgfpathmoveto{\pgfqpoint{0.649822in}{0.847844in}}%
\pgfpathlineto{\pgfqpoint{0.799431in}{0.847844in}}%
\pgfpathlineto{\pgfqpoint{0.799431in}{0.698236in}}%
\pgfpathlineto{\pgfqpoint{0.649822in}{0.698236in}}%
\pgfpathlineto{\pgfqpoint{0.649822in}{0.847844in}}%
\pgfpathclose%
\pgfusepath{stroke}%
\end{pgfscope}%
\begin{pgfscope}%
\pgfpathrectangle{\pgfqpoint{0.350605in}{0.399019in}}{\pgfqpoint{0.748043in}{0.748043in}}%
\pgfusepath{clip}%
\pgfsetbuttcap%
\pgfsetmiterjoin%
\pgfsetlinewidth{2.007500pt}%
\definecolor{currentstroke}{rgb}{0.000000,0.000000,0.000000}%
\pgfsetstrokecolor{currentstroke}%
\pgfsetdash{}{0pt}%
\pgfpathmoveto{\pgfqpoint{0.799431in}{0.698236in}}%
\pgfpathlineto{\pgfqpoint{0.949039in}{0.698236in}}%
\pgfpathlineto{\pgfqpoint{0.949039in}{0.548627in}}%
\pgfpathlineto{\pgfqpoint{0.799431in}{0.548627in}}%
\pgfpathlineto{\pgfqpoint{0.799431in}{0.698236in}}%
\pgfpathclose%
\pgfusepath{stroke}%
\end{pgfscope}%
\begin{pgfscope}%
\pgfpathrectangle{\pgfqpoint{0.350605in}{0.399019in}}{\pgfqpoint{0.748043in}{0.748043in}}%
\pgfusepath{clip}%
\pgfsetbuttcap%
\pgfsetmiterjoin%
\pgfsetlinewidth{2.007500pt}%
\definecolor{currentstroke}{rgb}{0.000000,0.000000,0.000000}%
\pgfsetstrokecolor{currentstroke}%
\pgfsetdash{}{0pt}%
\pgfpathmoveto{\pgfqpoint{0.949039in}{0.548627in}}%
\pgfpathlineto{\pgfqpoint{1.098648in}{0.548627in}}%
\pgfpathlineto{\pgfqpoint{1.098648in}{0.399019in}}%
\pgfpathlineto{\pgfqpoint{0.949039in}{0.399019in}}%
\pgfpathlineto{\pgfqpoint{0.949039in}{0.548627in}}%
\pgfpathclose%
\pgfusepath{stroke}%
\end{pgfscope}%
\begin{pgfscope}%
\pgfsetbuttcap%
\pgfsetroundjoin%
\definecolor{currentfill}{rgb}{0.000000,0.000000,0.000000}%
\pgfsetfillcolor{currentfill}%
\pgfsetlinewidth{0.803000pt}%
\definecolor{currentstroke}{rgb}{0.000000,0.000000,0.000000}%
\pgfsetstrokecolor{currentstroke}%
\pgfsetdash{}{0pt}%
\pgfsys@defobject{currentmarker}{\pgfqpoint{0.000000in}{-0.048611in}}{\pgfqpoint{0.000000in}{0.000000in}}{%
\pgfpathmoveto{\pgfqpoint{0.000000in}{0.000000in}}%
\pgfpathlineto{\pgfqpoint{0.000000in}{-0.048611in}}%
\pgfusepath{stroke,fill}%
}%
\begin{pgfscope}%
\pgfsys@transformshift{0.425409in}{0.399019in}%
\pgfsys@useobject{currentmarker}{}%
\end{pgfscope}%
\end{pgfscope}%
\begin{pgfscope}%
\definecolor{textcolor}{rgb}{0.000000,0.000000,0.000000}%
\pgfsetstrokecolor{textcolor}%
\pgfsetfillcolor{textcolor}%
\pgftext[x=0.367371in, y=0.205827in, left, base,rotate=45.000000]{\color{textcolor}{\rmfamily\fontsize{7.000000}{8.400000}\selectfont\catcode`\^=\active\def^{\ifmmode\sp\else\^{}\fi}\catcode`\%=\active\def
\end{pgfscope}%
\begin{pgfscope}%
\pgfsetbuttcap%
\pgfsetroundjoin%
\definecolor{currentfill}{rgb}{0.000000,0.000000,0.000000}%
\pgfsetfillcolor{currentfill}%
\pgfsetlinewidth{0.803000pt}%
\definecolor{currentstroke}{rgb}{0.000000,0.000000,0.000000}%
\pgfsetstrokecolor{currentstroke}%
\pgfsetdash{}{0pt}%
\pgfsys@defobject{currentmarker}{\pgfqpoint{0.000000in}{-0.048611in}}{\pgfqpoint{0.000000in}{0.000000in}}{%
\pgfpathmoveto{\pgfqpoint{0.000000in}{0.000000in}}%
\pgfpathlineto{\pgfqpoint{0.000000in}{-0.048611in}}%
\pgfusepath{stroke,fill}%
}%
\begin{pgfscope}%
\pgfsys@transformshift{0.575018in}{0.399019in}%
\pgfsys@useobject{currentmarker}{}%
\end{pgfscope}%
\end{pgfscope}%
\begin{pgfscope}%
\definecolor{textcolor}{rgb}{0.000000,0.000000,0.000000}%
\pgfsetstrokecolor{textcolor}%
\pgfsetfillcolor{textcolor}%
\pgftext[x=0.516980in, y=0.205827in, left, base,rotate=45.000000]{\color{textcolor}{\rmfamily\fontsize{7.000000}{8.400000}\selectfont\catcode`\^=\active\def^{\ifmmode\sp\else\^{}\fi}\catcode`\%=\active\def
\end{pgfscope}%
\begin{pgfscope}%
\pgfsetbuttcap%
\pgfsetroundjoin%
\definecolor{currentfill}{rgb}{0.000000,0.000000,0.000000}%
\pgfsetfillcolor{currentfill}%
\pgfsetlinewidth{0.803000pt}%
\definecolor{currentstroke}{rgb}{0.000000,0.000000,0.000000}%
\pgfsetstrokecolor{currentstroke}%
\pgfsetdash{}{0pt}%
\pgfsys@defobject{currentmarker}{\pgfqpoint{0.000000in}{-0.048611in}}{\pgfqpoint{0.000000in}{0.000000in}}{%
\pgfpathmoveto{\pgfqpoint{0.000000in}{0.000000in}}%
\pgfpathlineto{\pgfqpoint{0.000000in}{-0.048611in}}%
\pgfusepath{stroke,fill}%
}%
\begin{pgfscope}%
\pgfsys@transformshift{0.724626in}{0.399019in}%
\pgfsys@useobject{currentmarker}{}%
\end{pgfscope}%
\end{pgfscope}%
\begin{pgfscope}%
\definecolor{textcolor}{rgb}{0.000000,0.000000,0.000000}%
\pgfsetstrokecolor{textcolor}%
\pgfsetfillcolor{textcolor}%
\pgftext[x=0.666588in, y=0.205827in, left, base,rotate=45.000000]{\color{textcolor}{\rmfamily\fontsize{7.000000}{8.400000}\selectfont\catcode`\^=\active\def^{\ifmmode\sp\else\^{}\fi}\catcode`\%=\active\def
\end{pgfscope}%
\begin{pgfscope}%
\pgfsetbuttcap%
\pgfsetroundjoin%
\definecolor{currentfill}{rgb}{0.000000,0.000000,0.000000}%
\pgfsetfillcolor{currentfill}%
\pgfsetlinewidth{0.803000pt}%
\definecolor{currentstroke}{rgb}{0.000000,0.000000,0.000000}%
\pgfsetstrokecolor{currentstroke}%
\pgfsetdash{}{0pt}%
\pgfsys@defobject{currentmarker}{\pgfqpoint{0.000000in}{-0.048611in}}{\pgfqpoint{0.000000in}{0.000000in}}{%
\pgfpathmoveto{\pgfqpoint{0.000000in}{0.000000in}}%
\pgfpathlineto{\pgfqpoint{0.000000in}{-0.048611in}}%
\pgfusepath{stroke,fill}%
}%
\begin{pgfscope}%
\pgfsys@transformshift{0.874235in}{0.399019in}%
\pgfsys@useobject{currentmarker}{}%
\end{pgfscope}%
\end{pgfscope}%
\begin{pgfscope}%
\definecolor{textcolor}{rgb}{0.000000,0.000000,0.000000}%
\pgfsetstrokecolor{textcolor}%
\pgfsetfillcolor{textcolor}%
\pgftext[x=0.816197in, y=0.205827in, left, base,rotate=45.000000]{\color{textcolor}{\rmfamily\fontsize{7.000000}{8.400000}\selectfont\catcode`\^=\active\def^{\ifmmode\sp\else\^{}\fi}\catcode`\%=\active\def
\end{pgfscope}%
\begin{pgfscope}%
\pgfsetbuttcap%
\pgfsetroundjoin%
\definecolor{currentfill}{rgb}{0.000000,0.000000,0.000000}%
\pgfsetfillcolor{currentfill}%
\pgfsetlinewidth{0.803000pt}%
\definecolor{currentstroke}{rgb}{0.000000,0.000000,0.000000}%
\pgfsetstrokecolor{currentstroke}%
\pgfsetdash{}{0pt}%
\pgfsys@defobject{currentmarker}{\pgfqpoint{0.000000in}{-0.048611in}}{\pgfqpoint{0.000000in}{0.000000in}}{%
\pgfpathmoveto{\pgfqpoint{0.000000in}{0.000000in}}%
\pgfpathlineto{\pgfqpoint{0.000000in}{-0.048611in}}%
\pgfusepath{stroke,fill}%
}%
\begin{pgfscope}%
\pgfsys@transformshift{1.023843in}{0.399019in}%
\pgfsys@useobject{currentmarker}{}%
\end{pgfscope}%
\end{pgfscope}%
\begin{pgfscope}%
\definecolor{textcolor}{rgb}{0.000000,0.000000,0.000000}%
\pgfsetstrokecolor{textcolor}%
\pgfsetfillcolor{textcolor}%
\pgftext[x=0.965805in, y=0.205827in, left, base,rotate=45.000000]{\color{textcolor}{\rmfamily\fontsize{7.000000}{8.400000}\selectfont\catcode`\^=\active\def^{\ifmmode\sp\else\^{}\fi}\catcode`\%=\active\def
\end{pgfscope}%
\begin{pgfscope}%
\definecolor{textcolor}{rgb}{0.000000,0.000000,0.000000}%
\pgfsetstrokecolor{textcolor}%
\pgfsetfillcolor{textcolor}%
\pgftext[x=0.724626in,y=0.135972in,,top]{\color{textcolor}{\rmfamily\fontsize{9.000000}{10.800000}\selectfont\catcode`\^=\active\def^{\ifmmode\sp\else\^{}\fi}\catcode`\%=\active\def
\end{pgfscope}%
\begin{pgfscope}%
\pgfsetbuttcap%
\pgfsetroundjoin%
\definecolor{currentfill}{rgb}{0.000000,0.000000,0.000000}%
\pgfsetfillcolor{currentfill}%
\pgfsetlinewidth{0.803000pt}%
\definecolor{currentstroke}{rgb}{0.000000,0.000000,0.000000}%
\pgfsetstrokecolor{currentstroke}%
\pgfsetdash{}{0pt}%
\pgfsys@defobject{currentmarker}{\pgfqpoint{-0.048611in}{0.000000in}}{\pgfqpoint{-0.000000in}{0.000000in}}{%
\pgfpathmoveto{\pgfqpoint{-0.000000in}{0.000000in}}%
\pgfpathlineto{\pgfqpoint{-0.048611in}{0.000000in}}%
\pgfusepath{stroke,fill}%
}%
\begin{pgfscope}%
\pgfsys@transformshift{0.350605in}{1.072257in}%
\pgfsys@useobject{currentmarker}{}%
\end{pgfscope}%
\end{pgfscope}%
\begin{pgfscope}%
\definecolor{textcolor}{rgb}{0.000000,0.000000,0.000000}%
\pgfsetstrokecolor{textcolor}%
\pgfsetfillcolor{textcolor}%
\pgftext[x=0.191527in, y=1.035324in, left, base]{\color{textcolor}{\rmfamily\fontsize{7.000000}{8.400000}\selectfont\catcode`\^=\active\def^{\ifmmode\sp\else\^{}\fi}\catcode`\%=\active\def
\end{pgfscope}%
\begin{pgfscope}%
\pgfsetbuttcap%
\pgfsetroundjoin%
\definecolor{currentfill}{rgb}{0.000000,0.000000,0.000000}%
\pgfsetfillcolor{currentfill}%
\pgfsetlinewidth{0.803000pt}%
\definecolor{currentstroke}{rgb}{0.000000,0.000000,0.000000}%
\pgfsetstrokecolor{currentstroke}%
\pgfsetdash{}{0pt}%
\pgfsys@defobject{currentmarker}{\pgfqpoint{-0.048611in}{0.000000in}}{\pgfqpoint{-0.000000in}{0.000000in}}{%
\pgfpathmoveto{\pgfqpoint{-0.000000in}{0.000000in}}%
\pgfpathlineto{\pgfqpoint{-0.048611in}{0.000000in}}%
\pgfusepath{stroke,fill}%
}%
\begin{pgfscope}%
\pgfsys@transformshift{0.350605in}{0.922649in}%
\pgfsys@useobject{currentmarker}{}%
\end{pgfscope}%
\end{pgfscope}%
\begin{pgfscope}%
\definecolor{textcolor}{rgb}{0.000000,0.000000,0.000000}%
\pgfsetstrokecolor{textcolor}%
\pgfsetfillcolor{textcolor}%
\pgftext[x=0.191527in, y=0.885716in, left, base]{\color{textcolor}{\rmfamily\fontsize{7.000000}{8.400000}\selectfont\catcode`\^=\active\def^{\ifmmode\sp\else\^{}\fi}\catcode`\%=\active\def
\end{pgfscope}%
\begin{pgfscope}%
\pgfsetbuttcap%
\pgfsetroundjoin%
\definecolor{currentfill}{rgb}{0.000000,0.000000,0.000000}%
\pgfsetfillcolor{currentfill}%
\pgfsetlinewidth{0.803000pt}%
\definecolor{currentstroke}{rgb}{0.000000,0.000000,0.000000}%
\pgfsetstrokecolor{currentstroke}%
\pgfsetdash{}{0pt}%
\pgfsys@defobject{currentmarker}{\pgfqpoint{-0.048611in}{0.000000in}}{\pgfqpoint{-0.000000in}{0.000000in}}{%
\pgfpathmoveto{\pgfqpoint{-0.000000in}{0.000000in}}%
\pgfpathlineto{\pgfqpoint{-0.048611in}{0.000000in}}%
\pgfusepath{stroke,fill}%
}%
\begin{pgfscope}%
\pgfsys@transformshift{0.350605in}{0.773040in}%
\pgfsys@useobject{currentmarker}{}%
\end{pgfscope}%
\end{pgfscope}%
\begin{pgfscope}%
\definecolor{textcolor}{rgb}{0.000000,0.000000,0.000000}%
\pgfsetstrokecolor{textcolor}%
\pgfsetfillcolor{textcolor}%
\pgftext[x=0.191527in, y=0.736107in, left, base]{\color{textcolor}{\rmfamily\fontsize{7.000000}{8.400000}\selectfont\catcode`\^=\active\def^{\ifmmode\sp\else\^{}\fi}\catcode`\%=\active\def
\end{pgfscope}%
\begin{pgfscope}%
\pgfsetbuttcap%
\pgfsetroundjoin%
\definecolor{currentfill}{rgb}{0.000000,0.000000,0.000000}%
\pgfsetfillcolor{currentfill}%
\pgfsetlinewidth{0.803000pt}%
\definecolor{currentstroke}{rgb}{0.000000,0.000000,0.000000}%
\pgfsetstrokecolor{currentstroke}%
\pgfsetdash{}{0pt}%
\pgfsys@defobject{currentmarker}{\pgfqpoint{-0.048611in}{0.000000in}}{\pgfqpoint{-0.000000in}{0.000000in}}{%
\pgfpathmoveto{\pgfqpoint{-0.000000in}{0.000000in}}%
\pgfpathlineto{\pgfqpoint{-0.048611in}{0.000000in}}%
\pgfusepath{stroke,fill}%
}%
\begin{pgfscope}%
\pgfsys@transformshift{0.350605in}{0.623432in}%
\pgfsys@useobject{currentmarker}{}%
\end{pgfscope}%
\end{pgfscope}%
\begin{pgfscope}%
\definecolor{textcolor}{rgb}{0.000000,0.000000,0.000000}%
\pgfsetstrokecolor{textcolor}%
\pgfsetfillcolor{textcolor}%
\pgftext[x=0.191527in, y=0.586499in, left, base]{\color{textcolor}{\rmfamily\fontsize{7.000000}{8.400000}\selectfont\catcode`\^=\active\def^{\ifmmode\sp\else\^{}\fi}\catcode`\%=\active\def
\end{pgfscope}%
\begin{pgfscope}%
\pgfsetbuttcap%
\pgfsetroundjoin%
\definecolor{currentfill}{rgb}{0.000000,0.000000,0.000000}%
\pgfsetfillcolor{currentfill}%
\pgfsetlinewidth{0.803000pt}%
\definecolor{currentstroke}{rgb}{0.000000,0.000000,0.000000}%
\pgfsetstrokecolor{currentstroke}%
\pgfsetdash{}{0pt}%
\pgfsys@defobject{currentmarker}{\pgfqpoint{-0.048611in}{0.000000in}}{\pgfqpoint{-0.000000in}{0.000000in}}{%
\pgfpathmoveto{\pgfqpoint{-0.000000in}{0.000000in}}%
\pgfpathlineto{\pgfqpoint{-0.048611in}{0.000000in}}%
\pgfusepath{stroke,fill}%
}%
\begin{pgfscope}%
\pgfsys@transformshift{0.350605in}{0.473823in}%
\pgfsys@useobject{currentmarker}{}%
\end{pgfscope}%
\end{pgfscope}%
\begin{pgfscope}%
\definecolor{textcolor}{rgb}{0.000000,0.000000,0.000000}%
\pgfsetstrokecolor{textcolor}%
\pgfsetfillcolor{textcolor}%
\pgftext[x=0.191527in, y=0.436890in, left, base]{\color{textcolor}{\rmfamily\fontsize{7.000000}{8.400000}\selectfont\catcode`\^=\active\def^{\ifmmode\sp\else\^{}\fi}\catcode`\%=\active\def
\end{pgfscope}%
\begin{pgfscope}%
\definecolor{textcolor}{rgb}{0.000000,0.000000,0.000000}%
\pgfsetstrokecolor{textcolor}%
\pgfsetfillcolor{textcolor}%
\pgftext[x=0.135972in,y=0.773040in,,bottom,rotate=90.000000]{\color{textcolor}{\rmfamily\fontsize{9.000000}{10.800000}\selectfont\catcode`\^=\active\def^{\ifmmode\sp\else\^{}\fi}\catcode`\%=\active\def
\end{pgfscope}%
\begin{pgfscope}%
\pgfsetrectcap%
\pgfsetmiterjoin%
\pgfsetlinewidth{0.803000pt}%
\definecolor{currentstroke}{rgb}{0.000000,0.000000,0.000000}%
\pgfsetstrokecolor{currentstroke}%
\pgfsetdash{}{0pt}%
\pgfpathmoveto{\pgfqpoint{0.350605in}{0.399019in}}%
\pgfpathlineto{\pgfqpoint{0.350605in}{1.147062in}}%
\pgfusepath{stroke}%
\end{pgfscope}%
\begin{pgfscope}%
\pgfsetrectcap%
\pgfsetmiterjoin%
\pgfsetlinewidth{0.803000pt}%
\definecolor{currentstroke}{rgb}{0.000000,0.000000,0.000000}%
\pgfsetstrokecolor{currentstroke}%
\pgfsetdash{}{0pt}%
\pgfpathmoveto{\pgfqpoint{1.098648in}{0.399019in}}%
\pgfpathlineto{\pgfqpoint{1.098648in}{1.147062in}}%
\pgfusepath{stroke}%
\end{pgfscope}%
\begin{pgfscope}%
\pgfsetrectcap%
\pgfsetmiterjoin%
\pgfsetlinewidth{0.803000pt}%
\definecolor{currentstroke}{rgb}{0.000000,0.000000,0.000000}%
\pgfsetstrokecolor{currentstroke}%
\pgfsetdash{}{0pt}%
\pgfpathmoveto{\pgfqpoint{0.350605in}{0.399019in}}%
\pgfpathlineto{\pgfqpoint{1.098648in}{0.399019in}}%
\pgfusepath{stroke}%
\end{pgfscope}%
\begin{pgfscope}%
\pgfsetrectcap%
\pgfsetmiterjoin%
\pgfsetlinewidth{0.803000pt}%
\definecolor{currentstroke}{rgb}{0.000000,0.000000,0.000000}%
\pgfsetstrokecolor{currentstroke}%
\pgfsetdash{}{0pt}%
\pgfpathmoveto{\pgfqpoint{0.350605in}{1.147062in}}%
\pgfpathlineto{\pgfqpoint{1.098648in}{1.147062in}}%
\pgfusepath{stroke}%
\end{pgfscope}%
\begin{pgfscope}%
\pgfpathrectangle{\pgfqpoint{0.350605in}{1.150062in}}{\pgfqpoint{0.748043in}{0.134648in}}%
\pgfusepath{clip}%
\pgfsetbuttcap%
\pgfsetmiterjoin%
\definecolor{currentfill}{rgb}{0.007843,0.619608,0.450980}%
\pgfsetfillcolor{currentfill}%
\pgfsetlinewidth{0.000000pt}%
\definecolor{currentstroke}{rgb}{0.000000,0.000000,0.000000}%
\pgfsetstrokecolor{currentstroke}%
\pgfsetstrokeopacity{0.000000}%
\pgfsetdash{}{0pt}%
\pgfpathmoveto{\pgfqpoint{0.361826in}{1.150062in}}%
\pgfpathlineto{\pgfqpoint{0.488993in}{1.150062in}}%
\pgfpathlineto{\pgfqpoint{0.488993in}{1.278297in}}%
\pgfpathlineto{\pgfqpoint{0.361826in}{1.278297in}}%
\pgfpathlineto{\pgfqpoint{0.361826in}{1.150062in}}%
\pgfpathclose%
\pgfusepath{fill}%
\end{pgfscope}%
\begin{pgfscope}%
\pgfpathrectangle{\pgfqpoint{0.350605in}{1.150062in}}{\pgfqpoint{0.748043in}{0.134648in}}%
\pgfusepath{clip}%
\pgfsetbuttcap%
\pgfsetmiterjoin%
\definecolor{currentfill}{rgb}{0.007843,0.619608,0.450980}%
\pgfsetfillcolor{currentfill}%
\pgfsetlinewidth{0.000000pt}%
\definecolor{currentstroke}{rgb}{0.000000,0.000000,0.000000}%
\pgfsetstrokecolor{currentstroke}%
\pgfsetstrokeopacity{0.000000}%
\pgfsetdash{}{0pt}%
\pgfpathmoveto{\pgfqpoint{0.511434in}{1.150062in}}%
\pgfpathlineto{\pgfqpoint{0.638601in}{1.150062in}}%
\pgfpathlineto{\pgfqpoint{0.638601in}{1.256283in}}%
\pgfpathlineto{\pgfqpoint{0.511434in}{1.256283in}}%
\pgfpathlineto{\pgfqpoint{0.511434in}{1.150062in}}%
\pgfpathclose%
\pgfusepath{fill}%
\end{pgfscope}%
\begin{pgfscope}%
\pgfpathrectangle{\pgfqpoint{0.350605in}{1.150062in}}{\pgfqpoint{0.748043in}{0.134648in}}%
\pgfusepath{clip}%
\pgfsetbuttcap%
\pgfsetmiterjoin%
\definecolor{currentfill}{rgb}{0.007843,0.619608,0.450980}%
\pgfsetfillcolor{currentfill}%
\pgfsetlinewidth{0.000000pt}%
\definecolor{currentstroke}{rgb}{0.000000,0.000000,0.000000}%
\pgfsetstrokecolor{currentstroke}%
\pgfsetstrokeopacity{0.000000}%
\pgfsetdash{}{0pt}%
\pgfpathmoveto{\pgfqpoint{0.661043in}{1.150062in}}%
\pgfpathlineto{\pgfqpoint{0.788210in}{1.150062in}}%
\pgfpathlineto{\pgfqpoint{0.788210in}{1.238121in}}%
\pgfpathlineto{\pgfqpoint{0.661043in}{1.238121in}}%
\pgfpathlineto{\pgfqpoint{0.661043in}{1.150062in}}%
\pgfpathclose%
\pgfusepath{fill}%
\end{pgfscope}%
\begin{pgfscope}%
\pgfpathrectangle{\pgfqpoint{0.350605in}{1.150062in}}{\pgfqpoint{0.748043in}{0.134648in}}%
\pgfusepath{clip}%
\pgfsetbuttcap%
\pgfsetmiterjoin%
\definecolor{currentfill}{rgb}{0.007843,0.619608,0.450980}%
\pgfsetfillcolor{currentfill}%
\pgfsetlinewidth{0.000000pt}%
\definecolor{currentstroke}{rgb}{0.000000,0.000000,0.000000}%
\pgfsetstrokecolor{currentstroke}%
\pgfsetstrokeopacity{0.000000}%
\pgfsetdash{}{0pt}%
\pgfpathmoveto{\pgfqpoint{0.810651in}{1.150062in}}%
\pgfpathlineto{\pgfqpoint{0.937819in}{1.150062in}}%
\pgfpathlineto{\pgfqpoint{0.937819in}{1.185285in}}%
\pgfpathlineto{\pgfqpoint{0.810651in}{1.185285in}}%
\pgfpathlineto{\pgfqpoint{0.810651in}{1.150062in}}%
\pgfpathclose%
\pgfusepath{fill}%
\end{pgfscope}%
\begin{pgfscope}%
\pgfpathrectangle{\pgfqpoint{0.350605in}{1.150062in}}{\pgfqpoint{0.748043in}{0.134648in}}%
\pgfusepath{clip}%
\pgfsetbuttcap%
\pgfsetmiterjoin%
\definecolor{currentfill}{rgb}{0.007843,0.619608,0.450980}%
\pgfsetfillcolor{currentfill}%
\pgfsetlinewidth{0.000000pt}%
\definecolor{currentstroke}{rgb}{0.000000,0.000000,0.000000}%
\pgfsetstrokecolor{currentstroke}%
\pgfsetstrokeopacity{0.000000}%
\pgfsetdash{}{0pt}%
\pgfpathmoveto{\pgfqpoint{0.960260in}{1.150062in}}%
\pgfpathlineto{\pgfqpoint{1.087427in}{1.150062in}}%
\pgfpathlineto{\pgfqpoint{1.087427in}{1.266740in}}%
\pgfpathlineto{\pgfqpoint{0.960260in}{1.266740in}}%
\pgfpathlineto{\pgfqpoint{0.960260in}{1.150062in}}%
\pgfpathclose%
\pgfusepath{fill}%
\end{pgfscope}%
\begin{pgfscope}%
\pgfpathrectangle{\pgfqpoint{1.106648in}{0.399019in}}{\pgfqpoint{0.134648in}{0.748043in}}%
\pgfusepath{clip}%
\pgfsetbuttcap%
\pgfsetmiterjoin%
\definecolor{currentfill}{rgb}{0.007843,0.619608,0.450980}%
\pgfsetfillcolor{currentfill}%
\pgfsetlinewidth{0.000000pt}%
\definecolor{currentstroke}{rgb}{0.000000,0.000000,0.000000}%
\pgfsetstrokecolor{currentstroke}%
\pgfsetstrokeopacity{0.000000}%
\pgfsetdash{}{0pt}%
\pgfpathmoveto{\pgfqpoint{1.106648in}{1.135841in}}%
\pgfpathlineto{\pgfqpoint{1.172880in}{1.135841in}}%
\pgfpathlineto{\pgfqpoint{1.172880in}{1.008674in}}%
\pgfpathlineto{\pgfqpoint{1.106648in}{1.008674in}}%
\pgfpathlineto{\pgfqpoint{1.106648in}{1.135841in}}%
\pgfpathclose%
\pgfusepath{fill}%
\end{pgfscope}%
\begin{pgfscope}%
\pgfpathrectangle{\pgfqpoint{1.106648in}{0.399019in}}{\pgfqpoint{0.134648in}{0.748043in}}%
\pgfusepath{clip}%
\pgfsetbuttcap%
\pgfsetmiterjoin%
\definecolor{currentfill}{rgb}{0.007843,0.619608,0.450980}%
\pgfsetfillcolor{currentfill}%
\pgfsetlinewidth{0.000000pt}%
\definecolor{currentstroke}{rgb}{0.000000,0.000000,0.000000}%
\pgfsetstrokecolor{currentstroke}%
\pgfsetstrokeopacity{0.000000}%
\pgfsetdash{}{0pt}%
\pgfpathmoveto{\pgfqpoint{1.106648in}{0.986232in}}%
\pgfpathlineto{\pgfqpoint{1.143991in}{0.986232in}}%
\pgfpathlineto{\pgfqpoint{1.143991in}{0.859065in}}%
\pgfpathlineto{\pgfqpoint{1.106648in}{0.859065in}}%
\pgfpathlineto{\pgfqpoint{1.106648in}{0.986232in}}%
\pgfpathclose%
\pgfusepath{fill}%
\end{pgfscope}%
\begin{pgfscope}%
\pgfpathrectangle{\pgfqpoint{1.106648in}{0.399019in}}{\pgfqpoint{0.134648in}{0.748043in}}%
\pgfusepath{clip}%
\pgfsetbuttcap%
\pgfsetmiterjoin%
\definecolor{currentfill}{rgb}{0.007843,0.619608,0.450980}%
\pgfsetfillcolor{currentfill}%
\pgfsetlinewidth{0.000000pt}%
\definecolor{currentstroke}{rgb}{0.000000,0.000000,0.000000}%
\pgfsetstrokecolor{currentstroke}%
\pgfsetstrokeopacity{0.000000}%
\pgfsetdash{}{0pt}%
\pgfpathmoveto{\pgfqpoint{1.106648in}{0.836624in}}%
\pgfpathlineto{\pgfqpoint{1.178516in}{0.836624in}}%
\pgfpathlineto{\pgfqpoint{1.178516in}{0.709457in}}%
\pgfpathlineto{\pgfqpoint{1.106648in}{0.709457in}}%
\pgfpathlineto{\pgfqpoint{1.106648in}{0.836624in}}%
\pgfpathclose%
\pgfusepath{fill}%
\end{pgfscope}%
\begin{pgfscope}%
\pgfpathrectangle{\pgfqpoint{1.106648in}{0.399019in}}{\pgfqpoint{0.134648in}{0.748043in}}%
\pgfusepath{clip}%
\pgfsetbuttcap%
\pgfsetmiterjoin%
\definecolor{currentfill}{rgb}{0.007843,0.619608,0.450980}%
\pgfsetfillcolor{currentfill}%
\pgfsetlinewidth{0.000000pt}%
\definecolor{currentstroke}{rgb}{0.000000,0.000000,0.000000}%
\pgfsetstrokecolor{currentstroke}%
\pgfsetstrokeopacity{0.000000}%
\pgfsetdash{}{0pt}%
\pgfpathmoveto{\pgfqpoint{1.106648in}{0.687015in}}%
\pgfpathlineto{\pgfqpoint{1.234884in}{0.687015in}}%
\pgfpathlineto{\pgfqpoint{1.234884in}{0.559848in}}%
\pgfpathlineto{\pgfqpoint{1.106648in}{0.559848in}}%
\pgfpathlineto{\pgfqpoint{1.106648in}{0.687015in}}%
\pgfpathclose%
\pgfusepath{fill}%
\end{pgfscope}%
\begin{pgfscope}%
\pgfpathrectangle{\pgfqpoint{1.106648in}{0.399019in}}{\pgfqpoint{0.134648in}{0.748043in}}%
\pgfusepath{clip}%
\pgfsetbuttcap%
\pgfsetmiterjoin%
\definecolor{currentfill}{rgb}{0.007843,0.619608,0.450980}%
\pgfsetfillcolor{currentfill}%
\pgfsetlinewidth{0.000000pt}%
\definecolor{currentstroke}{rgb}{0.000000,0.000000,0.000000}%
\pgfsetstrokecolor{currentstroke}%
\pgfsetstrokeopacity{0.000000}%
\pgfsetdash{}{0pt}%
\pgfpathmoveto{\pgfqpoint{1.106648in}{0.537407in}}%
\pgfpathlineto{\pgfqpoint{1.106648in}{0.537407in}}%
\pgfpathlineto{\pgfqpoint{1.106648in}{0.410239in}}%
\pgfpathlineto{\pgfqpoint{1.106648in}{0.410239in}}%
\pgfpathlineto{\pgfqpoint{1.106648in}{0.537407in}}%
\pgfpathclose%
\pgfusepath{fill}%
\end{pgfscope}%
\end{pgfpicture}%
\makeatother%
\endgroup%

%% file: figures/calibration_maps/TRIVIAQA_Qwen3-30B-A3B-Instruct-2507.pgf
\begingroup%
\makeatletter%
\begin{pgfpicture}%
\pgfpathrectangle{\pgfpointorigin}{\pgfqpoint{1.600000in}{1.600000in}}%
\pgfusepath{use as bounding box, clip}%
\begin{pgfscope}%
\pgfsetbuttcap%
\pgfsetmiterjoin%
\definecolor{currentfill}{rgb}{1.000000,1.000000,1.000000}%
\pgfsetfillcolor{currentfill}%
\pgfsetlinewidth{0.000000pt}%
\definecolor{currentstroke}{rgb}{1.000000,1.000000,1.000000}%
\pgfsetstrokecolor{currentstroke}%
\pgfsetdash{}{0pt}%
\pgfpathmoveto{\pgfqpoint{0.000000in}{0.000000in}}%
\pgfpathlineto{\pgfqpoint{1.600000in}{0.000000in}}%
\pgfpathlineto{\pgfqpoint{1.600000in}{1.600000in}}%
\pgfpathlineto{\pgfqpoint{0.000000in}{1.600000in}}%
\pgfpathlineto{\pgfqpoint{0.000000in}{0.000000in}}%
\pgfpathclose%
\pgfusepath{fill}%
\end{pgfscope}%
\begin{pgfscope}%
\pgfsetbuttcap%
\pgfsetmiterjoin%
\definecolor{currentfill}{rgb}{1.000000,1.000000,1.000000}%
\pgfsetfillcolor{currentfill}%
\pgfsetlinewidth{0.000000pt}%
\definecolor{currentstroke}{rgb}{0.000000,0.000000,0.000000}%
\pgfsetstrokecolor{currentstroke}%
\pgfsetstrokeopacity{0.000000}%
\pgfsetdash{}{0pt}%
\pgfpathmoveto{\pgfqpoint{0.519743in}{0.414757in}}%
\pgfpathlineto{\pgfqpoint{1.503516in}{0.414757in}}%
\pgfpathlineto{\pgfqpoint{1.503516in}{1.558330in}}%
\pgfpathlineto{\pgfqpoint{0.519743in}{1.558330in}}%
\pgfpathlineto{\pgfqpoint{0.519743in}{0.414757in}}%
\pgfpathclose%
\pgfusepath{fill}%
\end{pgfscope}%
\begin{pgfscope}%
\pgfpathrectangle{\pgfqpoint{0.519743in}{0.414757in}}{\pgfqpoint{0.983773in}{1.143573in}}%
\pgfusepath{clip}%
\pgfsetbuttcap%
\pgfsetroundjoin%
\definecolor{currentfill}{rgb}{0.870588,0.560784,0.011765}%
\pgfsetfillcolor{currentfill}%
\pgfsetfillopacity{0.200000}%
\pgfsetlinewidth{1.003750pt}%
\definecolor{currentstroke}{rgb}{0.870588,0.560784,0.011765}%
\pgfsetstrokecolor{currentstroke}%
\pgfsetstrokeopacity{0.200000}%
\pgfsetdash{}{0pt}%
\pgfsys@defobject{currentmarker}{\pgfqpoint{0.564460in}{0.726641in}}{\pgfqpoint{0.788045in}{1.506349in}}{%
\pgfpathmoveto{\pgfqpoint{0.788045in}{0.726641in}}%
\pgfpathlineto{\pgfqpoint{0.564460in}{0.726641in}}%
\pgfpathlineto{\pgfqpoint{0.564460in}{1.506349in}}%
\pgfpathlineto{\pgfqpoint{0.788045in}{1.506349in}}%
\pgfpathlineto{\pgfqpoint{0.788045in}{1.506349in}}%
\pgfpathlineto{\pgfqpoint{0.788045in}{0.726641in}}%
\pgfpathlineto{\pgfqpoint{0.788045in}{0.726641in}}%
\pgfpathclose%
\pgfusepath{stroke,fill}%
}%
\begin{pgfscope}%
\pgfsys@transformshift{0.000000in}{0.000000in}%
\pgfsys@useobject{currentmarker}{}%
\end{pgfscope}%
\end{pgfscope}%
\begin{pgfscope}%
\pgfpathrectangle{\pgfqpoint{0.519743in}{0.414757in}}{\pgfqpoint{0.983773in}{1.143573in}}%
\pgfusepath{clip}%
\pgfsetbuttcap%
\pgfsetroundjoin%
\definecolor{currentfill}{rgb}{0.007843,0.619608,0.447059}%
\pgfsetfillcolor{currentfill}%
\pgfsetfillopacity{0.200000}%
\pgfsetlinewidth{1.003750pt}%
\definecolor{currentstroke}{rgb}{0.007843,0.619608,0.447059}%
\pgfsetstrokecolor{currentstroke}%
\pgfsetstrokeopacity{0.200000}%
\pgfsetdash{}{0pt}%
\pgfsys@defobject{currentmarker}{\pgfqpoint{0.788045in}{0.726641in}}{\pgfqpoint{1.458799in}{1.506349in}}{%
\pgfpathmoveto{\pgfqpoint{1.458799in}{0.726641in}}%
\pgfpathlineto{\pgfqpoint{0.788045in}{0.726641in}}%
\pgfpathlineto{\pgfqpoint{0.788045in}{1.506349in}}%
\pgfpathlineto{\pgfqpoint{1.458799in}{1.506349in}}%
\pgfpathlineto{\pgfqpoint{1.458799in}{1.506349in}}%
\pgfpathlineto{\pgfqpoint{1.458799in}{0.726641in}}%
\pgfpathlineto{\pgfqpoint{1.458799in}{0.726641in}}%
\pgfpathclose%
\pgfusepath{stroke,fill}%
}%
\begin{pgfscope}%
\pgfsys@transformshift{0.000000in}{0.000000in}%
\pgfsys@useobject{currentmarker}{}%
\end{pgfscope}%
\end{pgfscope}%
\begin{pgfscope}%
\pgfpathrectangle{\pgfqpoint{0.519743in}{0.414757in}}{\pgfqpoint{0.983773in}{1.143573in}}%
\pgfusepath{clip}%
\pgfsetbuttcap%
\pgfsetroundjoin%
\definecolor{currentfill}{rgb}{0.870588,0.560784,0.011765}%
\pgfsetfillcolor{currentfill}%
\pgfsetfillopacity{0.200000}%
\pgfsetlinewidth{1.003750pt}%
\definecolor{currentstroke}{rgb}{0.870588,0.560784,0.011765}%
\pgfsetstrokecolor{currentstroke}%
\pgfsetstrokeopacity{0.200000}%
\pgfsetdash{}{0pt}%
\pgfsys@defobject{currentmarker}{\pgfqpoint{0.788045in}{0.466738in}}{\pgfqpoint{1.458799in}{0.726641in}}{%
\pgfpathmoveto{\pgfqpoint{1.458799in}{0.466738in}}%
\pgfpathlineto{\pgfqpoint{0.788045in}{0.466738in}}%
\pgfpathlineto{\pgfqpoint{0.788045in}{0.726641in}}%
\pgfpathlineto{\pgfqpoint{1.458799in}{0.726641in}}%
\pgfpathlineto{\pgfqpoint{1.458799in}{0.726641in}}%
\pgfpathlineto{\pgfqpoint{1.458799in}{0.466738in}}%
\pgfpathlineto{\pgfqpoint{1.458799in}{0.466738in}}%
\pgfpathclose%
\pgfusepath{stroke,fill}%
}%
\begin{pgfscope}%
\pgfsys@transformshift{0.000000in}{0.000000in}%
\pgfsys@useobject{currentmarker}{}%
\end{pgfscope}%
\end{pgfscope}%
\begin{pgfscope}%
\pgfpathrectangle{\pgfqpoint{0.519743in}{0.414757in}}{\pgfqpoint{0.983773in}{1.143573in}}%
\pgfusepath{clip}%
\pgfsetbuttcap%
\pgfsetroundjoin%
\definecolor{currentfill}{rgb}{0.007843,0.619608,0.447059}%
\pgfsetfillcolor{currentfill}%
\pgfsetfillopacity{0.200000}%
\pgfsetlinewidth{1.003750pt}%
\definecolor{currentstroke}{rgb}{0.007843,0.619608,0.447059}%
\pgfsetstrokecolor{currentstroke}%
\pgfsetstrokeopacity{0.200000}%
\pgfsetdash{}{0pt}%
\pgfsys@defobject{currentmarker}{\pgfqpoint{0.564460in}{0.466738in}}{\pgfqpoint{0.788045in}{0.726641in}}{%
\pgfpathmoveto{\pgfqpoint{0.788045in}{0.466738in}}%
\pgfpathlineto{\pgfqpoint{0.564460in}{0.466738in}}%
\pgfpathlineto{\pgfqpoint{0.564460in}{0.726641in}}%
\pgfpathlineto{\pgfqpoint{0.788045in}{0.726641in}}%
\pgfpathlineto{\pgfqpoint{0.788045in}{0.726641in}}%
\pgfpathlineto{\pgfqpoint{0.788045in}{0.466738in}}%
\pgfpathlineto{\pgfqpoint{0.788045in}{0.466738in}}%
\pgfpathclose%
\pgfusepath{stroke,fill}%
}%
\begin{pgfscope}%
\pgfsys@transformshift{0.000000in}{0.000000in}%
\pgfsys@useobject{currentmarker}{}%
\end{pgfscope}%
\end{pgfscope}%
\begin{pgfscope}%
\pgfpathrectangle{\pgfqpoint{0.519743in}{0.414757in}}{\pgfqpoint{0.983773in}{1.143573in}}%
\pgfusepath{clip}%
\pgfsetrectcap%
\pgfsetroundjoin%
\pgfsetlinewidth{0.803000pt}%
\definecolor{currentstroke}{rgb}{0.690196,0.690196,0.690196}%
\pgfsetstrokecolor{currentstroke}%
\pgfsetdash{}{0pt}%
\pgfpathmoveto{\pgfqpoint{0.564460in}{0.414757in}}%
\pgfpathlineto{\pgfqpoint{0.564460in}{1.558330in}}%
\pgfusepath{stroke}%
\end{pgfscope}%
\begin{pgfscope}%
\pgfsetbuttcap%
\pgfsetroundjoin%
\definecolor{currentfill}{rgb}{0.000000,0.000000,0.000000}%
\pgfsetfillcolor{currentfill}%
\pgfsetlinewidth{0.803000pt}%
\definecolor{currentstroke}{rgb}{0.000000,0.000000,0.000000}%
\pgfsetstrokecolor{currentstroke}%
\pgfsetdash{}{0pt}%
\pgfsys@defobject{currentmarker}{\pgfqpoint{0.000000in}{-0.048611in}}{\pgfqpoint{0.000000in}{0.000000in}}{%
\pgfpathmoveto{\pgfqpoint{0.000000in}{0.000000in}}%
\pgfpathlineto{\pgfqpoint{0.000000in}{-0.048611in}}%
\pgfusepath{stroke,fill}%
}%
\begin{pgfscope}%
\pgfsys@transformshift{0.564460in}{0.414757in}%
\pgfsys@useobject{currentmarker}{}%
\end{pgfscope}%
\end{pgfscope}%
\begin{pgfscope}%
\definecolor{textcolor}{rgb}{0.000000,0.000000,0.000000}%
\pgfsetstrokecolor{textcolor}%
\pgfsetfillcolor{textcolor}%
\pgftext[x=0.564460in,y=0.317535in,,top]{\color{textcolor}{\rmfamily\fontsize{7.000000}{8.400000}\selectfont\catcode`\^=\active\def^{\ifmmode\sp\else\^{}\fi}\catcode`\%=\active\def
\end{pgfscope}%
\begin{pgfscope}%
\pgfpathrectangle{\pgfqpoint{0.519743in}{0.414757in}}{\pgfqpoint{0.983773in}{1.143573in}}%
\pgfusepath{clip}%
\pgfsetrectcap%
\pgfsetroundjoin%
\pgfsetlinewidth{0.803000pt}%
\definecolor{currentstroke}{rgb}{0.690196,0.690196,0.690196}%
\pgfsetstrokecolor{currentstroke}%
\pgfsetdash{}{0pt}%
\pgfpathmoveto{\pgfqpoint{0.788045in}{0.414757in}}%
\pgfpathlineto{\pgfqpoint{0.788045in}{1.558330in}}%
\pgfusepath{stroke}%
\end{pgfscope}%
\begin{pgfscope}%
\pgfsetbuttcap%
\pgfsetroundjoin%
\definecolor{currentfill}{rgb}{0.000000,0.000000,0.000000}%
\pgfsetfillcolor{currentfill}%
\pgfsetlinewidth{0.803000pt}%
\definecolor{currentstroke}{rgb}{0.000000,0.000000,0.000000}%
\pgfsetstrokecolor{currentstroke}%
\pgfsetdash{}{0pt}%
\pgfsys@defobject{currentmarker}{\pgfqpoint{0.000000in}{-0.048611in}}{\pgfqpoint{0.000000in}{0.000000in}}{%
\pgfpathmoveto{\pgfqpoint{0.000000in}{0.000000in}}%
\pgfpathlineto{\pgfqpoint{0.000000in}{-0.048611in}}%
\pgfusepath{stroke,fill}%
}%
\begin{pgfscope}%
\pgfsys@transformshift{0.788045in}{0.414757in}%
\pgfsys@useobject{currentmarker}{}%
\end{pgfscope}%
\end{pgfscope}%
\begin{pgfscope}%
\definecolor{textcolor}{rgb}{0.000000,0.000000,0.000000}%
\pgfsetstrokecolor{textcolor}%
\pgfsetfillcolor{textcolor}%
\pgftext[x=0.788045in,y=0.317535in,,top]{\color{textcolor}{\rmfamily\fontsize{7.000000}{8.400000}\selectfont\catcode`\^=\active\def^{\ifmmode\sp\else\^{}\fi}\catcode`\%=\active\def
\end{pgfscope}%
\begin{pgfscope}%
\pgfpathrectangle{\pgfqpoint{0.519743in}{0.414757in}}{\pgfqpoint{0.983773in}{1.143573in}}%
\pgfusepath{clip}%
\pgfsetrectcap%
\pgfsetroundjoin%
\pgfsetlinewidth{0.803000pt}%
\definecolor{currentstroke}{rgb}{0.690196,0.690196,0.690196}%
\pgfsetstrokecolor{currentstroke}%
\pgfsetdash{}{0pt}%
\pgfpathmoveto{\pgfqpoint{1.011630in}{0.414757in}}%
\pgfpathlineto{\pgfqpoint{1.011630in}{1.558330in}}%
\pgfusepath{stroke}%
\end{pgfscope}%
\begin{pgfscope}%
\pgfsetbuttcap%
\pgfsetroundjoin%
\definecolor{currentfill}{rgb}{0.000000,0.000000,0.000000}%
\pgfsetfillcolor{currentfill}%
\pgfsetlinewidth{0.803000pt}%
\definecolor{currentstroke}{rgb}{0.000000,0.000000,0.000000}%
\pgfsetstrokecolor{currentstroke}%
\pgfsetdash{}{0pt}%
\pgfsys@defobject{currentmarker}{\pgfqpoint{0.000000in}{-0.048611in}}{\pgfqpoint{0.000000in}{0.000000in}}{%
\pgfpathmoveto{\pgfqpoint{0.000000in}{0.000000in}}%
\pgfpathlineto{\pgfqpoint{0.000000in}{-0.048611in}}%
\pgfusepath{stroke,fill}%
}%
\begin{pgfscope}%
\pgfsys@transformshift{1.011630in}{0.414757in}%
\pgfsys@useobject{currentmarker}{}%
\end{pgfscope}%
\end{pgfscope}%
\begin{pgfscope}%
\definecolor{textcolor}{rgb}{0.000000,0.000000,0.000000}%
\pgfsetstrokecolor{textcolor}%
\pgfsetfillcolor{textcolor}%
\pgftext[x=1.011630in,y=0.317535in,,top]{\color{textcolor}{\rmfamily\fontsize{7.000000}{8.400000}\selectfont\catcode`\^=\active\def^{\ifmmode\sp\else\^{}\fi}\catcode`\%=\active\def
\end{pgfscope}%
\begin{pgfscope}%
\pgfpathrectangle{\pgfqpoint{0.519743in}{0.414757in}}{\pgfqpoint{0.983773in}{1.143573in}}%
\pgfusepath{clip}%
\pgfsetrectcap%
\pgfsetroundjoin%
\pgfsetlinewidth{0.803000pt}%
\definecolor{currentstroke}{rgb}{0.690196,0.690196,0.690196}%
\pgfsetstrokecolor{currentstroke}%
\pgfsetdash{}{0pt}%
\pgfpathmoveto{\pgfqpoint{1.235215in}{0.414757in}}%
\pgfpathlineto{\pgfqpoint{1.235215in}{1.558330in}}%
\pgfusepath{stroke}%
\end{pgfscope}%
\begin{pgfscope}%
\pgfsetbuttcap%
\pgfsetroundjoin%
\definecolor{currentfill}{rgb}{0.000000,0.000000,0.000000}%
\pgfsetfillcolor{currentfill}%
\pgfsetlinewidth{0.803000pt}%
\definecolor{currentstroke}{rgb}{0.000000,0.000000,0.000000}%
\pgfsetstrokecolor{currentstroke}%
\pgfsetdash{}{0pt}%
\pgfsys@defobject{currentmarker}{\pgfqpoint{0.000000in}{-0.048611in}}{\pgfqpoint{0.000000in}{0.000000in}}{%
\pgfpathmoveto{\pgfqpoint{0.000000in}{0.000000in}}%
\pgfpathlineto{\pgfqpoint{0.000000in}{-0.048611in}}%
\pgfusepath{stroke,fill}%
}%
\begin{pgfscope}%
\pgfsys@transformshift{1.235215in}{0.414757in}%
\pgfsys@useobject{currentmarker}{}%
\end{pgfscope}%
\end{pgfscope}%
\begin{pgfscope}%
\definecolor{textcolor}{rgb}{0.000000,0.000000,0.000000}%
\pgfsetstrokecolor{textcolor}%
\pgfsetfillcolor{textcolor}%
\pgftext[x=1.235215in,y=0.317535in,,top]{\color{textcolor}{\rmfamily\fontsize{7.000000}{8.400000}\selectfont\catcode`\^=\active\def^{\ifmmode\sp\else\^{}\fi}\catcode`\%=\active\def
\end{pgfscope}%
\begin{pgfscope}%
\pgfpathrectangle{\pgfqpoint{0.519743in}{0.414757in}}{\pgfqpoint{0.983773in}{1.143573in}}%
\pgfusepath{clip}%
\pgfsetrectcap%
\pgfsetroundjoin%
\pgfsetlinewidth{0.803000pt}%
\definecolor{currentstroke}{rgb}{0.690196,0.690196,0.690196}%
\pgfsetstrokecolor{currentstroke}%
\pgfsetdash{}{0pt}%
\pgfpathmoveto{\pgfqpoint{1.458799in}{0.414757in}}%
\pgfpathlineto{\pgfqpoint{1.458799in}{1.558330in}}%
\pgfusepath{stroke}%
\end{pgfscope}%
\begin{pgfscope}%
\pgfsetbuttcap%
\pgfsetroundjoin%
\definecolor{currentfill}{rgb}{0.000000,0.000000,0.000000}%
\pgfsetfillcolor{currentfill}%
\pgfsetlinewidth{0.803000pt}%
\definecolor{currentstroke}{rgb}{0.000000,0.000000,0.000000}%
\pgfsetstrokecolor{currentstroke}%
\pgfsetdash{}{0pt}%
\pgfsys@defobject{currentmarker}{\pgfqpoint{0.000000in}{-0.048611in}}{\pgfqpoint{0.000000in}{0.000000in}}{%
\pgfpathmoveto{\pgfqpoint{0.000000in}{0.000000in}}%
\pgfpathlineto{\pgfqpoint{0.000000in}{-0.048611in}}%
\pgfusepath{stroke,fill}%
}%
\begin{pgfscope}%
\pgfsys@transformshift{1.458799in}{0.414757in}%
\pgfsys@useobject{currentmarker}{}%
\end{pgfscope}%
\end{pgfscope}%
\begin{pgfscope}%
\definecolor{textcolor}{rgb}{0.000000,0.000000,0.000000}%
\pgfsetstrokecolor{textcolor}%
\pgfsetfillcolor{textcolor}%
\pgftext[x=1.458799in,y=0.317535in,,top]{\color{textcolor}{\rmfamily\fontsize{7.000000}{8.400000}\selectfont\catcode`\^=\active\def^{\ifmmode\sp\else\^{}\fi}\catcode`\%=\active\def
\end{pgfscope}%
\begin{pgfscope}%
\definecolor{textcolor}{rgb}{0.000000,0.000000,0.000000}%
\pgfsetstrokecolor{textcolor}%
\pgfsetfillcolor{textcolor}%
\pgftext[x=1.011630in,y=0.167891in,,top]{\color{textcolor}{\rmfamily\fontsize{9.000000}{10.800000}\selectfont\catcode`\^=\active\def^{\ifmmode\sp\else\^{}\fi}\catcode`\%=\active\def
\end{pgfscope}%
\begin{pgfscope}%
\pgfpathrectangle{\pgfqpoint{0.519743in}{0.414757in}}{\pgfqpoint{0.983773in}{1.143573in}}%
\pgfusepath{clip}%
\pgfsetrectcap%
\pgfsetroundjoin%
\pgfsetlinewidth{0.803000pt}%
\definecolor{currentstroke}{rgb}{0.690196,0.690196,0.690196}%
\pgfsetstrokecolor{currentstroke}%
\pgfsetdash{}{0pt}%
\pgfpathmoveto{\pgfqpoint{0.519743in}{0.466738in}}%
\pgfpathlineto{\pgfqpoint{1.503516in}{0.466738in}}%
\pgfusepath{stroke}%
\end{pgfscope}%
\begin{pgfscope}%
\pgfsetbuttcap%
\pgfsetroundjoin%
\definecolor{currentfill}{rgb}{0.000000,0.000000,0.000000}%
\pgfsetfillcolor{currentfill}%
\pgfsetlinewidth{0.803000pt}%
\definecolor{currentstroke}{rgb}{0.000000,0.000000,0.000000}%
\pgfsetstrokecolor{currentstroke}%
\pgfsetdash{}{0pt}%
\pgfsys@defobject{currentmarker}{\pgfqpoint{-0.048611in}{0.000000in}}{\pgfqpoint{-0.000000in}{0.000000in}}{%
\pgfpathmoveto{\pgfqpoint{-0.000000in}{0.000000in}}%
\pgfpathlineto{\pgfqpoint{-0.048611in}{0.000000in}}%
\pgfusepath{stroke,fill}%
}%
\begin{pgfscope}%
\pgfsys@transformshift{0.519743in}{0.466738in}%
\pgfsys@useobject{currentmarker}{}%
\end{pgfscope}%
\end{pgfscope}%
\begin{pgfscope}%
\definecolor{textcolor}{rgb}{0.000000,0.000000,0.000000}%
\pgfsetstrokecolor{textcolor}%
\pgfsetfillcolor{textcolor}%
\pgftext[x=0.223446in, y=0.429805in, left, base]{\color{textcolor}{\rmfamily\fontsize{7.000000}{8.400000}\selectfont\catcode`\^=\active\def^{\ifmmode\sp\else\^{}\fi}\catcode`\%=\active\def
\end{pgfscope}%
\begin{pgfscope}%
\pgfpathrectangle{\pgfqpoint{0.519743in}{0.414757in}}{\pgfqpoint{0.983773in}{1.143573in}}%
\pgfusepath{clip}%
\pgfsetrectcap%
\pgfsetroundjoin%
\pgfsetlinewidth{0.803000pt}%
\definecolor{currentstroke}{rgb}{0.690196,0.690196,0.690196}%
\pgfsetstrokecolor{currentstroke}%
\pgfsetdash{}{0pt}%
\pgfpathmoveto{\pgfqpoint{0.519743in}{0.726641in}}%
\pgfpathlineto{\pgfqpoint{1.503516in}{0.726641in}}%
\pgfusepath{stroke}%
\end{pgfscope}%
\begin{pgfscope}%
\pgfsetbuttcap%
\pgfsetroundjoin%
\definecolor{currentfill}{rgb}{0.000000,0.000000,0.000000}%
\pgfsetfillcolor{currentfill}%
\pgfsetlinewidth{0.803000pt}%
\definecolor{currentstroke}{rgb}{0.000000,0.000000,0.000000}%
\pgfsetstrokecolor{currentstroke}%
\pgfsetdash{}{0pt}%
\pgfsys@defobject{currentmarker}{\pgfqpoint{-0.048611in}{0.000000in}}{\pgfqpoint{-0.000000in}{0.000000in}}{%
\pgfpathmoveto{\pgfqpoint{-0.000000in}{0.000000in}}%
\pgfpathlineto{\pgfqpoint{-0.048611in}{0.000000in}}%
\pgfusepath{stroke,fill}%
}%
\begin{pgfscope}%
\pgfsys@transformshift{0.519743in}{0.726641in}%
\pgfsys@useobject{currentmarker}{}%
\end{pgfscope}%
\end{pgfscope}%
\begin{pgfscope}%
\definecolor{textcolor}{rgb}{0.000000,0.000000,0.000000}%
\pgfsetstrokecolor{textcolor}%
\pgfsetfillcolor{textcolor}%
\pgftext[x=0.223446in, y=0.689708in, left, base]{\color{textcolor}{\rmfamily\fontsize{7.000000}{8.400000}\selectfont\catcode`\^=\active\def^{\ifmmode\sp\else\^{}\fi}\catcode`\%=\active\def
\end{pgfscope}%
\begin{pgfscope}%
\pgfpathrectangle{\pgfqpoint{0.519743in}{0.414757in}}{\pgfqpoint{0.983773in}{1.143573in}}%
\pgfusepath{clip}%
\pgfsetrectcap%
\pgfsetroundjoin%
\pgfsetlinewidth{0.803000pt}%
\definecolor{currentstroke}{rgb}{0.690196,0.690196,0.690196}%
\pgfsetstrokecolor{currentstroke}%
\pgfsetdash{}{0pt}%
\pgfpathmoveto{\pgfqpoint{0.519743in}{0.986544in}}%
\pgfpathlineto{\pgfqpoint{1.503516in}{0.986544in}}%
\pgfusepath{stroke}%
\end{pgfscope}%
\begin{pgfscope}%
\pgfsetbuttcap%
\pgfsetroundjoin%
\definecolor{currentfill}{rgb}{0.000000,0.000000,0.000000}%
\pgfsetfillcolor{currentfill}%
\pgfsetlinewidth{0.803000pt}%
\definecolor{currentstroke}{rgb}{0.000000,0.000000,0.000000}%
\pgfsetstrokecolor{currentstroke}%
\pgfsetdash{}{0pt}%
\pgfsys@defobject{currentmarker}{\pgfqpoint{-0.048611in}{0.000000in}}{\pgfqpoint{-0.000000in}{0.000000in}}{%
\pgfpathmoveto{\pgfqpoint{-0.000000in}{0.000000in}}%
\pgfpathlineto{\pgfqpoint{-0.048611in}{0.000000in}}%
\pgfusepath{stroke,fill}%
}%
\begin{pgfscope}%
\pgfsys@transformshift{0.519743in}{0.986544in}%
\pgfsys@useobject{currentmarker}{}%
\end{pgfscope}%
\end{pgfscope}%
\begin{pgfscope}%
\definecolor{textcolor}{rgb}{0.000000,0.000000,0.000000}%
\pgfsetstrokecolor{textcolor}%
\pgfsetfillcolor{textcolor}%
\pgftext[x=0.223446in, y=0.949611in, left, base]{\color{textcolor}{\rmfamily\fontsize{7.000000}{8.400000}\selectfont\catcode`\^=\active\def^{\ifmmode\sp\else\^{}\fi}\catcode`\%=\active\def
\end{pgfscope}%
\begin{pgfscope}%
\pgfpathrectangle{\pgfqpoint{0.519743in}{0.414757in}}{\pgfqpoint{0.983773in}{1.143573in}}%
\pgfusepath{clip}%
\pgfsetrectcap%
\pgfsetroundjoin%
\pgfsetlinewidth{0.803000pt}%
\definecolor{currentstroke}{rgb}{0.690196,0.690196,0.690196}%
\pgfsetstrokecolor{currentstroke}%
\pgfsetdash{}{0pt}%
\pgfpathmoveto{\pgfqpoint{0.519743in}{1.246447in}}%
\pgfpathlineto{\pgfqpoint{1.503516in}{1.246447in}}%
\pgfusepath{stroke}%
\end{pgfscope}%
\begin{pgfscope}%
\pgfsetbuttcap%
\pgfsetroundjoin%
\definecolor{currentfill}{rgb}{0.000000,0.000000,0.000000}%
\pgfsetfillcolor{currentfill}%
\pgfsetlinewidth{0.803000pt}%
\definecolor{currentstroke}{rgb}{0.000000,0.000000,0.000000}%
\pgfsetstrokecolor{currentstroke}%
\pgfsetdash{}{0pt}%
\pgfsys@defobject{currentmarker}{\pgfqpoint{-0.048611in}{0.000000in}}{\pgfqpoint{-0.000000in}{0.000000in}}{%
\pgfpathmoveto{\pgfqpoint{-0.000000in}{0.000000in}}%
\pgfpathlineto{\pgfqpoint{-0.048611in}{0.000000in}}%
\pgfusepath{stroke,fill}%
}%
\begin{pgfscope}%
\pgfsys@transformshift{0.519743in}{1.246447in}%
\pgfsys@useobject{currentmarker}{}%
\end{pgfscope}%
\end{pgfscope}%
\begin{pgfscope}%
\definecolor{textcolor}{rgb}{0.000000,0.000000,0.000000}%
\pgfsetstrokecolor{textcolor}%
\pgfsetfillcolor{textcolor}%
\pgftext[x=0.223446in, y=1.209513in, left, base]{\color{textcolor}{\rmfamily\fontsize{7.000000}{8.400000}\selectfont\catcode`\^=\active\def^{\ifmmode\sp\else\^{}\fi}\catcode`\%=\active\def
\end{pgfscope}%
\begin{pgfscope}%
\pgfpathrectangle{\pgfqpoint{0.519743in}{0.414757in}}{\pgfqpoint{0.983773in}{1.143573in}}%
\pgfusepath{clip}%
\pgfsetrectcap%
\pgfsetroundjoin%
\pgfsetlinewidth{0.803000pt}%
\definecolor{currentstroke}{rgb}{0.690196,0.690196,0.690196}%
\pgfsetstrokecolor{currentstroke}%
\pgfsetdash{}{0pt}%
\pgfpathmoveto{\pgfqpoint{0.519743in}{1.506349in}}%
\pgfpathlineto{\pgfqpoint{1.503516in}{1.506349in}}%
\pgfusepath{stroke}%
\end{pgfscope}%
\begin{pgfscope}%
\pgfsetbuttcap%
\pgfsetroundjoin%
\definecolor{currentfill}{rgb}{0.000000,0.000000,0.000000}%
\pgfsetfillcolor{currentfill}%
\pgfsetlinewidth{0.803000pt}%
\definecolor{currentstroke}{rgb}{0.000000,0.000000,0.000000}%
\pgfsetstrokecolor{currentstroke}%
\pgfsetdash{}{0pt}%
\pgfsys@defobject{currentmarker}{\pgfqpoint{-0.048611in}{0.000000in}}{\pgfqpoint{-0.000000in}{0.000000in}}{%
\pgfpathmoveto{\pgfqpoint{-0.000000in}{0.000000in}}%
\pgfpathlineto{\pgfqpoint{-0.048611in}{0.000000in}}%
\pgfusepath{stroke,fill}%
}%
\begin{pgfscope}%
\pgfsys@transformshift{0.519743in}{1.506349in}%
\pgfsys@useobject{currentmarker}{}%
\end{pgfscope}%
\end{pgfscope}%
\begin{pgfscope}%
\definecolor{textcolor}{rgb}{0.000000,0.000000,0.000000}%
\pgfsetstrokecolor{textcolor}%
\pgfsetfillcolor{textcolor}%
\pgftext[x=0.223446in, y=1.469416in, left, base]{\color{textcolor}{\rmfamily\fontsize{7.000000}{8.400000}\selectfont\catcode`\^=\active\def^{\ifmmode\sp\else\^{}\fi}\catcode`\%=\active\def
\end{pgfscope}%
\begin{pgfscope}%
\definecolor{textcolor}{rgb}{0.000000,0.000000,0.000000}%
\pgfsetstrokecolor{textcolor}%
\pgfsetfillcolor{textcolor}%
\pgftext[x=0.167891in,y=0.986544in,,bottom,rotate=90.000000]{\color{textcolor}{\rmfamily\fontsize{9.000000}{10.800000}\selectfont\catcode`\^=\active\def^{\ifmmode\sp\else\^{}\fi}\catcode`\%=\active\def
\end{pgfscope}%
\begin{pgfscope}%
\pgfpathrectangle{\pgfqpoint{0.519743in}{0.414757in}}{\pgfqpoint{0.983773in}{1.143573in}}%
\pgfusepath{clip}%
\pgfsetrectcap%
\pgfsetroundjoin%
\pgfsetlinewidth{1.505625pt}%
\definecolor{currentstroke}{rgb}{0.003922,0.450980,0.698039}%
\pgfsetstrokecolor{currentstroke}%
\pgfsetstrokeopacity{0.200000}%
\pgfsetdash{}{0pt}%
\pgfpathmoveto{\pgfqpoint{0.698611in}{0.550422in}}%
\pgfpathlineto{\pgfqpoint{0.832762in}{0.608748in}}%
\pgfpathlineto{\pgfqpoint{0.877479in}{0.625961in}}%
\pgfpathlineto{\pgfqpoint{0.922196in}{0.642384in}}%
\pgfpathlineto{\pgfqpoint{0.966913in}{0.658144in}}%
\pgfpathlineto{\pgfqpoint{1.011630in}{0.673354in}}%
\pgfpathlineto{\pgfqpoint{1.056347in}{0.688115in}}%
\pgfpathlineto{\pgfqpoint{1.101064in}{0.702524in}}%
\pgfpathlineto{\pgfqpoint{1.145781in}{0.716688in}}%
\pgfpathlineto{\pgfqpoint{1.190498in}{0.730732in}}%
\pgfpathlineto{\pgfqpoint{1.235215in}{0.744818in}}%
\pgfpathlineto{\pgfqpoint{1.279932in}{0.759192in}}%
\pgfpathlineto{\pgfqpoint{1.324649in}{0.774280in}}%
\pgfpathlineto{\pgfqpoint{1.369365in}{0.791002in}}%
\pgfpathlineto{\pgfqpoint{1.458799in}{1.240682in}}%
\pgfusepath{stroke}%
\end{pgfscope}%
\begin{pgfscope}%
\pgfpathrectangle{\pgfqpoint{0.519743in}{0.414757in}}{\pgfqpoint{0.983773in}{1.143573in}}%
\pgfusepath{clip}%
\pgfsetrectcap%
\pgfsetroundjoin%
\pgfsetlinewidth{1.505625pt}%
\definecolor{currentstroke}{rgb}{0.003922,0.450980,0.698039}%
\pgfsetstrokecolor{currentstroke}%
\pgfsetstrokeopacity{0.200000}%
\pgfsetdash{}{0pt}%
\pgfpathmoveto{\pgfqpoint{0.720970in}{0.530935in}}%
\pgfpathlineto{\pgfqpoint{0.832762in}{0.576486in}}%
\pgfpathlineto{\pgfqpoint{0.877479in}{0.594060in}}%
\pgfpathlineto{\pgfqpoint{0.922196in}{0.611323in}}%
\pgfpathlineto{\pgfqpoint{0.966913in}{0.628293in}}%
\pgfpathlineto{\pgfqpoint{1.011630in}{0.644994in}}%
\pgfpathlineto{\pgfqpoint{1.056347in}{0.661461in}}%
\pgfpathlineto{\pgfqpoint{1.101064in}{0.677739in}}%
\pgfpathlineto{\pgfqpoint{1.145781in}{0.693894in}}%
\pgfpathlineto{\pgfqpoint{1.190498in}{0.710013in}}%
\pgfpathlineto{\pgfqpoint{1.235215in}{0.726229in}}%
\pgfpathlineto{\pgfqpoint{1.279932in}{0.742758in}}%
\pgfpathlineto{\pgfqpoint{1.324649in}{0.759995in}}%
\pgfpathlineto{\pgfqpoint{1.369365in}{0.778810in}}%
\pgfpathlineto{\pgfqpoint{1.458471in}{1.219190in}}%
\pgfusepath{stroke}%
\end{pgfscope}%
\begin{pgfscope}%
\pgfpathrectangle{\pgfqpoint{0.519743in}{0.414757in}}{\pgfqpoint{0.983773in}{1.143573in}}%
\pgfusepath{clip}%
\pgfsetrectcap%
\pgfsetroundjoin%
\pgfsetlinewidth{1.505625pt}%
\definecolor{currentstroke}{rgb}{0.003922,0.450980,0.698039}%
\pgfsetstrokecolor{currentstroke}%
\pgfsetstrokeopacity{0.200000}%
\pgfsetdash{}{0pt}%
\pgfpathmoveto{\pgfqpoint{0.779102in}{0.586804in}}%
\pgfpathlineto{\pgfqpoint{0.832762in}{0.609336in}}%
\pgfpathlineto{\pgfqpoint{0.877479in}{0.626968in}}%
\pgfpathlineto{\pgfqpoint{0.922196in}{0.643802in}}%
\pgfpathlineto{\pgfqpoint{1.011630in}{0.675556in}}%
\pgfpathlineto{\pgfqpoint{1.056347in}{0.690678in}}%
\pgfpathlineto{\pgfqpoint{1.101064in}{0.705427in}}%
\pgfpathlineto{\pgfqpoint{1.145781in}{0.719903in}}%
\pgfpathlineto{\pgfqpoint{1.190498in}{0.734227in}}%
\pgfpathlineto{\pgfqpoint{1.235215in}{0.748556in}}%
\pgfpathlineto{\pgfqpoint{1.279932in}{0.763123in}}%
\pgfpathlineto{\pgfqpoint{1.324649in}{0.778335in}}%
\pgfpathlineto{\pgfqpoint{1.369365in}{0.795066in}}%
\pgfpathlineto{\pgfqpoint{1.457917in}{1.216543in}}%
\pgfusepath{stroke}%
\end{pgfscope}%
\begin{pgfscope}%
\pgfpathrectangle{\pgfqpoint{0.519743in}{0.414757in}}{\pgfqpoint{0.983773in}{1.143573in}}%
\pgfusepath{clip}%
\pgfsetrectcap%
\pgfsetroundjoin%
\pgfsetlinewidth{1.505625pt}%
\definecolor{currentstroke}{rgb}{0.003922,0.450980,0.698039}%
\pgfsetstrokecolor{currentstroke}%
\pgfsetstrokeopacity{0.200000}%
\pgfsetdash{}{0pt}%
\pgfpathmoveto{\pgfqpoint{0.720970in}{0.559281in}}%
\pgfpathlineto{\pgfqpoint{0.832762in}{0.608210in}}%
\pgfpathlineto{\pgfqpoint{0.877479in}{0.625899in}}%
\pgfpathlineto{\pgfqpoint{0.922196in}{0.642807in}}%
\pgfpathlineto{\pgfqpoint{0.966913in}{0.659054in}}%
\pgfpathlineto{\pgfqpoint{1.011630in}{0.674744in}}%
\pgfpathlineto{\pgfqpoint{1.056347in}{0.689972in}}%
\pgfpathlineto{\pgfqpoint{1.101064in}{0.704833in}}%
\pgfpathlineto{\pgfqpoint{1.145781in}{0.719430in}}%
\pgfpathlineto{\pgfqpoint{1.190498in}{0.733882in}}%
\pgfpathlineto{\pgfqpoint{1.235215in}{0.748347in}}%
\pgfpathlineto{\pgfqpoint{1.279932in}{0.763062in}}%
\pgfpathlineto{\pgfqpoint{1.324649in}{0.778439in}}%
\pgfpathlineto{\pgfqpoint{1.458219in}{1.224992in}}%
\pgfusepath{stroke}%
\end{pgfscope}%
\begin{pgfscope}%
\pgfpathrectangle{\pgfqpoint{0.519743in}{0.414757in}}{\pgfqpoint{0.983773in}{1.143573in}}%
\pgfusepath{clip}%
\pgfsetrectcap%
\pgfsetroundjoin%
\pgfsetlinewidth{1.505625pt}%
\definecolor{currentstroke}{rgb}{0.003922,0.450980,0.698039}%
\pgfsetstrokecolor{currentstroke}%
\pgfsetstrokeopacity{0.200000}%
\pgfsetdash{}{0pt}%
\pgfpathmoveto{\pgfqpoint{0.758234in}{0.594959in}}%
\pgfpathlineto{\pgfqpoint{0.832762in}{0.627019in}}%
\pgfpathlineto{\pgfqpoint{0.877479in}{0.644458in}}%
\pgfpathlineto{\pgfqpoint{0.922196in}{0.660929in}}%
\pgfpathlineto{\pgfqpoint{0.966913in}{0.676603in}}%
\pgfpathlineto{\pgfqpoint{1.011630in}{0.691623in}}%
\pgfpathlineto{\pgfqpoint{1.056347in}{0.706112in}}%
\pgfpathlineto{\pgfqpoint{1.101064in}{0.720188in}}%
\pgfpathlineto{\pgfqpoint{1.145781in}{0.733971in}}%
\pgfpathlineto{\pgfqpoint{1.190498in}{0.747596in}}%
\pgfpathlineto{\pgfqpoint{1.235215in}{0.761236in}}%
\pgfpathlineto{\pgfqpoint{1.279932in}{0.775143in}}%
\pgfpathlineto{\pgfqpoint{1.324649in}{0.789754in}}%
\pgfpathlineto{\pgfqpoint{1.369365in}{0.806000in}}%
\pgfpathlineto{\pgfqpoint{1.458799in}{1.249979in}}%
\pgfusepath{stroke}%
\end{pgfscope}%
\begin{pgfscope}%
\pgfpathrectangle{\pgfqpoint{0.519743in}{0.414757in}}{\pgfqpoint{0.983773in}{1.143573in}}%
\pgfusepath{clip}%
\pgfsetbuttcap%
\pgfsetroundjoin%
\pgfsetlinewidth{1.505625pt}%
\definecolor{currentstroke}{rgb}{0.501961,0.501961,0.501961}%
\pgfsetstrokecolor{currentstroke}%
\pgfsetdash{{5.550000pt}{2.400000pt}}{0.000000pt}%
\pgfpathmoveto{\pgfqpoint{0.564460in}{0.726641in}}%
\pgfpathlineto{\pgfqpoint{1.458799in}{0.726641in}}%
\pgfusepath{stroke}%
\end{pgfscope}%
\begin{pgfscope}%
\pgfpathrectangle{\pgfqpoint{0.519743in}{0.414757in}}{\pgfqpoint{0.983773in}{1.143573in}}%
\pgfusepath{clip}%
\pgfsetbuttcap%
\pgfsetroundjoin%
\pgfsetlinewidth{1.505625pt}%
\definecolor{currentstroke}{rgb}{0.501961,0.501961,0.501961}%
\pgfsetstrokecolor{currentstroke}%
\pgfsetdash{{5.550000pt}{2.400000pt}}{0.000000pt}%
\pgfpathmoveto{\pgfqpoint{0.788045in}{0.466738in}}%
\pgfpathlineto{\pgfqpoint{0.788045in}{1.506349in}}%
\pgfusepath{stroke}%
\end{pgfscope}%
\begin{pgfscope}%
\pgfsetrectcap%
\pgfsetmiterjoin%
\pgfsetlinewidth{0.803000pt}%
\definecolor{currentstroke}{rgb}{0.000000,0.000000,0.000000}%
\pgfsetstrokecolor{currentstroke}%
\pgfsetdash{}{0pt}%
\pgfpathmoveto{\pgfqpoint{0.519743in}{0.414757in}}%
\pgfpathlineto{\pgfqpoint{0.519743in}{1.558330in}}%
\pgfusepath{stroke}%
\end{pgfscope}%
\begin{pgfscope}%
\pgfsetrectcap%
\pgfsetmiterjoin%
\pgfsetlinewidth{0.803000pt}%
\definecolor{currentstroke}{rgb}{0.000000,0.000000,0.000000}%
\pgfsetstrokecolor{currentstroke}%
\pgfsetdash{}{0pt}%
\pgfpathmoveto{\pgfqpoint{1.503516in}{0.414757in}}%
\pgfpathlineto{\pgfqpoint{1.503516in}{1.558330in}}%
\pgfusepath{stroke}%
\end{pgfscope}%
\begin{pgfscope}%
\pgfsetrectcap%
\pgfsetmiterjoin%
\pgfsetlinewidth{0.803000pt}%
\definecolor{currentstroke}{rgb}{0.000000,0.000000,0.000000}%
\pgfsetstrokecolor{currentstroke}%
\pgfsetdash{}{0pt}%
\pgfpathmoveto{\pgfqpoint{0.519743in}{0.414757in}}%
\pgfpathlineto{\pgfqpoint{1.503516in}{0.414757in}}%
\pgfusepath{stroke}%
\end{pgfscope}%
\begin{pgfscope}%
\pgfsetrectcap%
\pgfsetmiterjoin%
\pgfsetlinewidth{0.803000pt}%
\definecolor{currentstroke}{rgb}{0.000000,0.000000,0.000000}%
\pgfsetstrokecolor{currentstroke}%
\pgfsetdash{}{0pt}%
\pgfpathmoveto{\pgfqpoint{0.519743in}{1.558330in}}%
\pgfpathlineto{\pgfqpoint{1.503516in}{1.558330in}}%
\pgfusepath{stroke}%
\end{pgfscope}%
\begin{pgfscope}%
\pgfsetbuttcap%
\pgfsetmiterjoin%
\definecolor{currentfill}{rgb}{1.000000,1.000000,1.000000}%
\pgfsetfillcolor{currentfill}%
\pgfsetfillopacity{0.800000}%
\pgfsetlinewidth{1.003750pt}%
\definecolor{currentstroke}{rgb}{0.800000,0.800000,0.800000}%
\pgfsetstrokecolor{currentstroke}%
\pgfsetstrokeopacity{0.800000}%
\pgfsetdash{}{0pt}%
\pgfpathmoveto{\pgfqpoint{0.542878in}{1.008611in}}%
\pgfpathlineto{\pgfqpoint{1.537122in}{1.008611in}}%
\pgfpathquadraticcurveto{\pgfqpoint{1.556567in}{1.008611in}}{\pgfqpoint{1.556567in}{1.028056in}}%
\pgfpathlineto{\pgfqpoint{1.556567in}{1.446433in}}%
\pgfpathquadraticcurveto{\pgfqpoint{1.556567in}{1.465878in}}{\pgfqpoint{1.537122in}{1.465878in}}%
\pgfpathlineto{\pgfqpoint{0.542878in}{1.465878in}}%
\pgfpathquadraticcurveto{\pgfqpoint{0.523433in}{1.465878in}}{\pgfqpoint{0.523433in}{1.446433in}}%
\pgfpathlineto{\pgfqpoint{0.523433in}{1.028056in}}%
\pgfpathquadraticcurveto{\pgfqpoint{0.523433in}{1.008611in}}{\pgfqpoint{0.542878in}{1.008611in}}%
\pgfpathlineto{\pgfqpoint{0.542878in}{1.008611in}}%
\pgfpathclose%
\pgfusepath{stroke,fill}%
\end{pgfscope}%
\begin{pgfscope}%
\pgfsetbuttcap%
\pgfsetmiterjoin%
\definecolor{currentfill}{rgb}{0.007843,0.619608,0.447059}%
\pgfsetfillcolor{currentfill}%
\pgfsetfillopacity{0.850000}%
\pgfsetlinewidth{0.501875pt}%
\definecolor{currentstroke}{rgb}{0.000000,0.000000,0.000000}%
\pgfsetstrokecolor{currentstroke}%
\pgfsetstrokeopacity{0.850000}%
\pgfsetdash{}{0pt}%
\pgfpathmoveto{\pgfqpoint{0.562322in}{1.353123in}}%
\pgfpathlineto{\pgfqpoint{0.756766in}{1.353123in}}%
\pgfpathlineto{\pgfqpoint{0.756766in}{1.421178in}}%
\pgfpathlineto{\pgfqpoint{0.562322in}{1.421178in}}%
\pgfpathlineto{\pgfqpoint{0.562322in}{1.353123in}}%
\pgfpathclose%
\pgfusepath{stroke,fill}%
\end{pgfscope}%
\begin{pgfscope}%
\definecolor{textcolor}{rgb}{0.000000,0.000000,0.000000}%
\pgfsetstrokecolor{textcolor}%
\pgfsetfillcolor{textcolor}%
\pgftext[x=0.834544in,y=1.353123in,left,base]{\color{textcolor}{\rmfamily\fontsize{7.000000}{8.400000}\selectfont\catcode`\^=\active\def^{\ifmmode\sp\else\^{}\fi}\catcode`\%=\active\def
\end{pgfscope}%
\begin{pgfscope}%
\pgfsetbuttcap%
\pgfsetmiterjoin%
\definecolor{currentfill}{rgb}{0.870588,0.560784,0.011765}%
\pgfsetfillcolor{currentfill}%
\pgfsetfillopacity{0.850000}%
\pgfsetlinewidth{0.501875pt}%
\definecolor{currentstroke}{rgb}{0.000000,0.000000,0.000000}%
\pgfsetstrokecolor{currentstroke}%
\pgfsetstrokeopacity{0.850000}%
\pgfsetdash{}{0pt}%
\pgfpathmoveto{\pgfqpoint{0.562322in}{1.210423in}}%
\pgfpathlineto{\pgfqpoint{0.756766in}{1.210423in}}%
\pgfpathlineto{\pgfqpoint{0.756766in}{1.278478in}}%
\pgfpathlineto{\pgfqpoint{0.562322in}{1.278478in}}%
\pgfpathlineto{\pgfqpoint{0.562322in}{1.210423in}}%
\pgfpathclose%
\pgfusepath{stroke,fill}%
\end{pgfscope}%
\begin{pgfscope}%
\definecolor{textcolor}{rgb}{0.000000,0.000000,0.000000}%
\pgfsetstrokecolor{textcolor}%
\pgfsetfillcolor{textcolor}%
\pgftext[x=0.834544in,y=1.210423in,left,base]{\color{textcolor}{\rmfamily\fontsize{7.000000}{8.400000}\selectfont\catcode`\^=\active\def^{\ifmmode\sp\else\^{}\fi}\catcode`\%=\active\def
\end{pgfscope}%
\begin{pgfscope}%
\pgfsetbuttcap%
\pgfsetroundjoin%
\pgfsetlinewidth{0.803000pt}%
\definecolor{currentstroke}{rgb}{0.000000,0.000000,0.000000}%
\pgfsetstrokecolor{currentstroke}%
\pgfsetstrokeopacity{0.600000}%
\pgfsetdash{{2.960000pt}{1.280000pt}}{0.000000pt}%
\pgfpathmoveto{\pgfqpoint{0.562322in}{1.101751in}}%
\pgfpathlineto{\pgfqpoint{0.659544in}{1.101751in}}%
\pgfpathlineto{\pgfqpoint{0.756766in}{1.101751in}}%
\pgfusepath{stroke}%
\end{pgfscope}%
\begin{pgfscope}%
\definecolor{textcolor}{rgb}{0.000000,0.000000,0.000000}%
\pgfsetstrokecolor{textcolor}%
\pgfsetfillcolor{textcolor}%
\pgftext[x=0.834544in,y=1.067723in,left,base]{\color{textcolor}{\rmfamily\fontsize{7.000000}{8.400000}\selectfont\catcode`\^=\active\def^{\ifmmode\sp\else\^{}\fi}\catcode`\%=\active\def
\end{pgfscope}%
\end{pgfpicture}%
\makeatother%
\endgroup%

%% file: tables/calibrator_ablation_short.tex
\begin{tabular}{l|cc|cc|c}
\toprule
\multicolumn{1}{c|}{\makecell{$\mathcal{L}$(Metric)}} & \multicolumn{2}{c|}{\small{$\{0, \dots, C\}$ (L1 $\downarrow$)}} & \multicolumn{2}{c|}{\small{$\{A, \perp\}$ (BAS $\uparrow$)}} & \multicolumn{1}{c}{\small{$\{0, \dots, C\}^2$ (L1 $\downarrow$)}} \\
\multicolumn{1}{c|}{} & Helpsteer & STSB & SimpleQA & TriviaQA & Helpsteer \\
\midrule
\small{Uncalibrated} & {$0.78$$\scriptscriptstyle{\pm 0.07}$} & {$1.18$$\scriptscriptstyle{\pm 0.12}$} & {$-0.07$$\scriptscriptstyle{\pm 0.02}$} & {$0.46$$\scriptscriptstyle{\pm 0.05}$} & {$2.11$$\scriptscriptstyle{\pm 0.12}$} \\
\small{Temperature} & {$0.77$$\scriptscriptstyle{\pm 0.07}$} & {$1.18$$\scriptscriptstyle{\pm 0.12}$} & {$-0.10$$\scriptscriptstyle{\pm 0.03}$} & {$0.46$$\scriptscriptstyle{\pm 0.05}$} & {$2.04$$\scriptscriptstyle{\pm 0.09}$} \\
\small{Decision} & {$\textcolor{colsecond}{\textbf{0.77}}$$\scriptscriptstyle{\pm 0.04}$} & {$0.70$$\scriptscriptstyle{\pm 0.08}$} & {$\textcolor{colfirst}{\textbf{0.03}}$$\scriptscriptstyle{\pm 0.02}$} & {$\textcolor{colsecond}{\textbf{0.47}}$$\scriptscriptstyle{\pm 0.05}$} & {$\textcolor{colsecond}{\textbf{1.23}}$$\scriptscriptstyle{\pm 0.02}$} \\
\small{Policy $\pi_\phi$} & {$0.78$$\scriptscriptstyle{\pm 0.07}$} & {$\textcolor{colsecond}{\textbf{0.70}}$$\scriptscriptstyle{\pm 0.06}$} & {$\textcolor{colsecond}{\textbf{-0.00}}$$\scriptscriptstyle{\pm 0.00}$} & {$\textcolor{colfirst}{\textbf{0.47}}$$\scriptscriptstyle{\pm 0.04}$} & {$1.37$$\scriptscriptstyle{\pm 0.04}$} \\
\cellcolor[gray]{0.9}{\small{\textbf{Dirichlet}}} & \cellcolor[gray]{0.9}{{$\textcolor{colfirst}{\textbf{0.72}}$$\scriptscriptstyle{\pm 0.05}$}} & \cellcolor[gray]{0.9}{{$\textcolor{colfirst}{\textbf{0.67}}$$\scriptscriptstyle{\pm 0.05}$}} & \cellcolor[gray]{0.9}{{$\textcolor{colfirst}{\textbf{0.03}}$$\scriptscriptstyle{\pm 0.02}$}} & \cellcolor[gray]{0.9}{{$0.46$$\scriptscriptstyle{\pm 0.04}$}} & \cellcolor[gray]{0.9}{{$\textcolor{colfirst}{\textbf{1.21}}$$\scriptscriptstyle{\pm 0.04}$}} \\
\bottomrule
\end{tabular}

%% file: content/related_work.tex
\section{Related Work}

\paragraph{Calibration.}
Classical calibration notions such as confidence, class-wise, and distributional calibration typically concern classification
\citep{kull2015novel,guo2017calibration,kull2019beyond}. Similar to decision calibration \cite{zhao2021calibratingpredictionsdecisionsnovel}, our work is motivated by a given decision-making problem which, in our case, is implied by the task at hand. Our task calibration differs from decision calibration as the latter ensures optimality of MBR decoding only among other MBR decision rules. Our proposed criterion of distributionally calibrating the LLM's latent beliefs makes MBR decoding optimal among \emph{all} decision rules derived from the latent beliefs. Furthermore, our Task Calibration Error (TCE) can be seen as applying the decision-specific calibration error of \citet{hu2024calibrationerrordecisionmaking} to an LLM's latent beliefs. In our work, we connect concepts from the existing literature on calibration for classification with LLM's by interpreting their output distribution in a task-dependent latent space, and improve decision making such as decoding.

\looseness=-1\paragraph{Calibration in LLMs.}
 Recent work has shown that LLMs are often miscalibrated \cite{kadavath2022languagemodelsmostlyknow,plaut2025probabilitieschatllmsmiscalibrated}. Consequently, methods emerged that recalibrate token probabilities or verbalized confidence scores \cite{xie-etal-2024-calibrating,li2025conftunertraininglargelanguage}. Some recent work also moves beyond token-level probabilities: \citet{nakkiran2025trainedtokenscalibratedconcepts} show that language models can exhibit confidence calibration in semantic answer spaces. This semantic answer grouping can be seen as a special case of our proposed \emph{task} confidence calibration by instantiating the latent map $g_T$ with a semantic answer grouping.
For long-form generation, some work \cite{huang2024calibratinglongformgenerationslarge, zhang2025atomiccalibrationllmslongform,band2024linguisticcalibrationlongformgenerations} examines confidence calibration of atomic statements extracted from responses. From the perspective of our framework, this corresponds to confidence task calibration using a latent mapping $g_T$ that extracts these atomic statements. Overall, our notion of distributional task calibration is applicable to any latent structure and therefore generalizes many existing notions of calibration in LLMs that can be recovered for a specific $g_T$.
Moreover, only a few works connect calibration to decision-making in LLMs: \citet{wu2026basdecisiontheoreticapproachevaluating} study Answer-versus-Abstain tasks under asymmetric costs, while \citet{band2024linguisticcalibrationlongformgenerations} use calibration to improve classification. Our framework can be applied to any task-specific loss function and enables optimal decisions through MBR decoding in the context of any given task.

\looseness=-1\paragraph{Decoding in LLMs.}
There are numerous approaches to obtaining responses from the predictive distribution of an LLM \cite{shi2024thorough}. Well-established paradigms include deterministic strategies like Greedy token selection, Beam Search \cite{freitag2017beam}, or Contrastive Decoding \cite{li2023contrastive}, as well as sampling-based non-deterministic approaches like Ancestral sampling or top-$k$ sampling \cite{fan2018hierarchical}. MBR decoding originates in Machine Translation where risk is typically a learned similarity score \cite{eikema-aziz-2022-sampling,muller-sennrich-2021-understanding}. \citet{lukasik2024regressionawareinferencellms} utilize MBR for LLMs in regression problems and \citet{tomov2026taskawarenessimprovesllmgenerations} generalize this approach to arbitrary latent structures $\mathcal{L}$. They empirically observe the latent structure to be beneifical for MBR decoding. We build on this framework of interpreting the LLM's output distribution in a task-aware latent space $\mathcal{L}$ and both theoretically and empirically show that calibrating this latent distribution enables MBR as an optimal task-specific decoding strategy with strong empirical performance.

%% file: content/discussion.tex
\section{Discussion}

\looseness=-1 \paragraph{Limitations.}
Our optimality guarantee holds for decoding strategies that are based solely on the LLM's latent distribution \(\hat p(\rx)\). If additional information is available, such as internal representations or alternative queries, there may, in theory, be decoding strategies with better performance. Nonetheless, we empirically show that our framework outperforms decoding strategies beyond our formal framework in practice. Moreover, the performance improvement is given \emph{in expectation}: task calibration can improve \emph{average} downstream decision quality, but it does not imply per-instance performance gain or optimality. Furthermore, task calibration requires a discrete latent support that is shared across samples. While this is natural for many tasks (e.g., ratings or fixed candidate sets), this requirement may not be applicable to more open-ended settings. However, our work is useful toward this direction as well, as open-ended tasks can often be meaningfully reduced to a simpler structure. For example, we demonstrate that \emph{Answer-versus-Abstain} settings can be effectively task calibrated.
Further, we show how more complex problems, such as set prediction over open-ended support, can be decomposed into simpler decision-making problems that are again tractable to task calibrate.
Lastly, like the work of \citet{tomov2026taskawarenessimprovesllmgenerations} we build upon, our framework leverages semantics induced by a given task. Therefore, task calibration improves reliability and decoding performance for a given problem, but not for all tasks and applications simultaneously with a single recalibration map. Our method should be seen as a widely applicable general-purpose tool to enhance LLMs when they are deployed in a given downstream application that admits a task-dependent latent structure.

\looseness=-1 \paragraph{Future Work}
Our results highlight the potential of the proposed \emph{calibrate first, then act optimally} strategy to systematically improve downstream performance. The framework is generic and applies to any setting that admits a discrete latent representation with shared support across outputs, or that can be decomposed into such a structure. One particularly promising direction is LLM-as-a-judge systems \cite{gu2025surveyllmasajudge}, which frequently rely on rubric-based evaluations and thus naturally fit into our framework. Because judges often produce noisy or biased scores, task calibration could substantially improve their reliability \cite{lee2026correctlyreportllmasajudgeevaluations}. More broadly, we see potential for task calibration to facilitate more reliable model decisions across a variety of tasks and application domains.

%% file: content/acknowledgements.tex
\section*{Acknolwedgements}
We want to thank Franz Rieger, Leon Hetzel, and Arthur Kosmala for reviewing the paper. The research presented has been performed in the frame of the RADELN project funded by TUM Georg Nemetschek Institute Artificial Intelligence for the Built World (GNI). It is further supported by the Bavarian Ministry of Economic
Affairs, Regional Development and Energy with funds from
the Hightech Agenda Bayern. Lastly, Rajeev Verma is a part of UvA-Bosch Delta Lab at the University of Amsterdam funded by the Bosch Center for Artificial Intelligence.

%% file: content/impact.tex
\section{Impact Statement}
\label{sec:impact}
In this work, we examine how calibration can improve the generation performance of Large Language Models. While any research may be misused, our primary goal is to improve the reliability of
these models to support their safe deployment in critical domains. We believe the benefits will
outweigh the potential risks.

%% file: appendix/additional_experiments.tex
\section{Additional Experiments}

\subsection{Ablation on Calibration Methods}
\label{app:ablation}

\begin{table}[h]
\centering
\caption{Ablation on calibration methods on all models, including the MAP action on the Dirichlet calibrated predictor. Dirichlet calibration and MBR perform the best in most tasks. For Hamming and Exact Match losses, the MBR action and the MP action (classifier $\pi_\phi$) are the same.}
\setlength{\tabcolsep}{4pt}
\resizebox{0.98\textwidth}{!}{
\input{tables/calibrator_ablation}
}
\label{tab:calibrator_ablation}
\end{table}

In \Cref{tab:calibrator_ablation}, we compare our Dirichlet task calibration map to different calibration strategies. We consistently observe it to yield the highest quality results, indicating that the MBR actions on the recalibrated LLM beliefs are optimal.

We also outperform directly learning a classifier on the latent belief distribution $\hat{p}(\rx)$ which is equivalent to selecting the Maximum Probability (MP) action on the Dirichlet calibration map $f_\phi$. For Exact Match and Hamming loss, the MP and the MBR responses are identical, and we consequently also achieve the same performance. In this case, our framework can indeed by interpreted as learning a classifier on the latent beliefs. Similarly, temperature scaling does not affect the MP or MBR action and, therefore, it yields the same performance as not calibrating for these losses. For the other loss functions, we indeed find that calibration and MBR decoding consistently outperforms the learned classifier. This shows that the power of our method comes not from utilizing the labels in the calibration set to correct the LLM's predictions but instead from calibrating its latent belief distribution.

\subsection{Calibration Maps}
\Cref{fig:calibration_maps_simple_qa,fig:calibration_maps_triviaqa} show how the learned recalibration maps act on the binary latent space $\{A, \perp\}$: In \Cref{fig:calibration_maps_simple_qa}, we see that for all LLMs but the Qwen3-30B-A3B model, the learned calibration map counteracts the strong overconfidence in answering by collapsing all actions onto abstaining $\perp$ by predicting a probability below the threshold $t$. While this results in the calibrated LLM always abstaining, \Cref{tab:performance_generations} confirms that this result in lower average loss. One explanation is that the three LLMs for which this behaviour occurs are all poor predictors of the true label and, therefore, the best calibration map in terms of the BAS risk is to always abstain instead of trusting the LLM's predicted distribution. In contrast, for TriviaQA, we, too, observe overconfidence of the LLMs in \Cref{fig:calibration_maps_triviaqa} that is mitigated by our task calibration: Only uncalibrated probabilities that are close to one (i.e. very high confidence) get mapped to values above the answer threshold $t$.

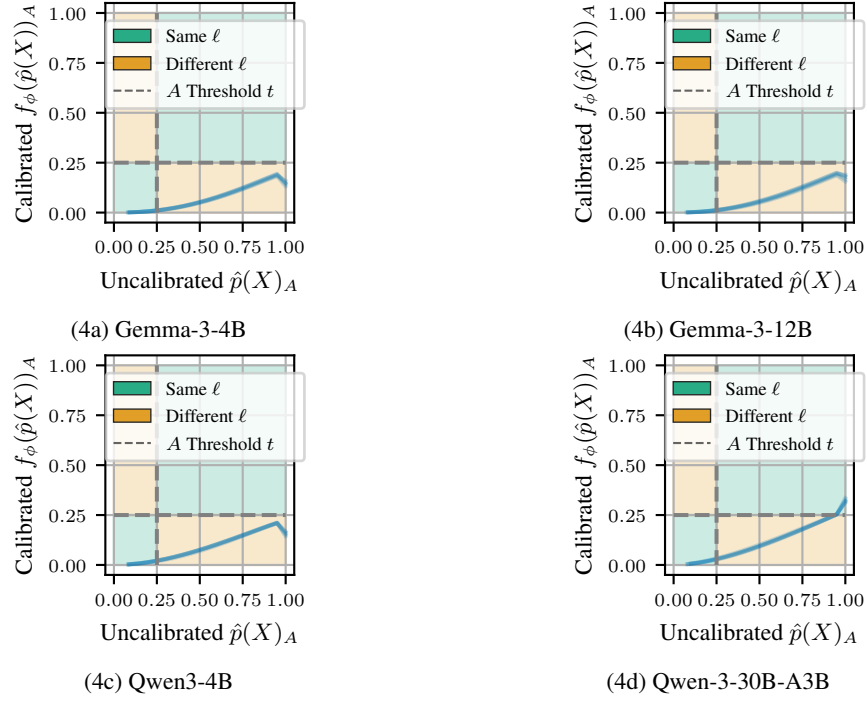
\begin{figure*}[h]
    \centering
    \begin{subfigure}[t]{0.47\linewidth}
        \centering
\input{figures/calibration_maps/SIMPLE_QA_VERIFIED_gemma-3-4b-it.pgf}
        \caption{Gemma-3-4B}
    \end{subfigure}
    \hfill
    \begin{subfigure}[t]{0.47\linewidth}
        \centering
\input{figures/calibration_maps/SIMPLE_QA_VERIFIED_gemma-3-12b-it.pgf}
        \caption{Gemma-3-12B}
    \end{subfigure}
    \\
    \begin{subfigure}[t]{0.47\linewidth}
        \centering
\input{figures/calibration_maps/SIMPLE_QA_VERIFIED_Qwen3-4B-Instruct-2507.pgf}
        \caption{Qwen3-4B}
    \end{subfigure}
    \hfill
    \begin{subfigure}[t]{0.47\linewidth}
        \centering
\input{figures/calibration_maps/SIMPLE_QA_VERIFIED_Qwen3-30B-A3B-Instruct-2507.pgf}
        \caption{Qwen-3-30B-A3B}
    \end{subfigure}
    \caption{Learned calibration maps $f_\phi(\hat{p}(\rx)$ on SimpleQA.}
    \label{fig:calibration_maps_simple_qa}
\end{figure*}

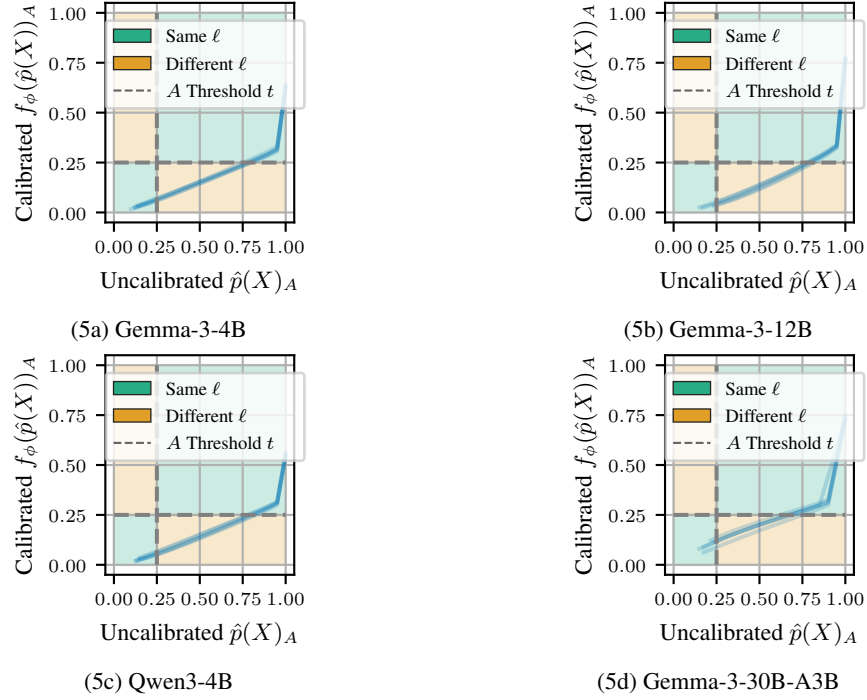
\begin{figure*}[h!]
    \centering
    \begin{subfigure}[t]{0.47\linewidth}
        \centering
\input{figures/calibration_maps/TRIVIAQA_gemma-3-4b-it.pgf}
        \caption{Gemma-3-4B}
    \end{subfigure}
    \hfill
    \begin{subfigure}[t]{0.47\linewidth}
        \centering
\input{figures/calibration_maps/TRIVIAQA_gemma-3-12b-it.pgf}
        \caption{Gemma-3-12B}
    \end{subfigure}
    \\
    \begin{subfigure}[t]{0.47\linewidth}
        \centering
\input{figures/calibration_maps/TRIVIAQA_Qwen3-4B-Instruct-2507.pgf}
        \caption{Qwen3-4B}
    \end{subfigure}
    \hfill
    \begin{subfigure}[t]{0.47\linewidth}
        \centering
\input{figures/calibration_maps/TRIVIAQA_Qwen3-30B-A3B-Instruct-2507.pgf}
        \caption{Gemma-3-30B-A3B}
    \end{subfigure}
    \caption{Learned calibration maps $f_\phi(\hat{p}(\rx)$ on TriviaQA.}
    \label{fig:calibration_maps_triviaqa}
\end{figure*}

\subsection{Task Calibration Error (TCE)}
\label{app:tce}
\Cref{tab:calibration} reports TCE values estimated with the binned estimator. Across models and tasks, we observe substantial task miscalibration before recalibration. Since TCE is defined in the same units as the task loss, it admits a direct interpretation; for example, for Qwen3-30B-A3B on MMLU, a TCE of $0.083$ corresponds to $8.3$ percentage points of excess error. Applying our Dirichlet calibration map \(f_{\phi}(q)\) reduces TCE to values that are close to zero in nearly all settings, indicating that most of the achievable improvement from task calibration is realized.

\begin{table*}[h]
\centering
\caption{Task Calibration Error (TCE) using binned estimator for uncalibrated \(q\) and Dirichlet-calibrated \(f_{\phi}(q)\) model. All LLMs are severely task miscalibrated which is mitigated by our approximate recalibration method.}
\setlength{\tabcolsep}{4pt}
\resizebox{0.9\textwidth}{!}{
\input{tables/calibration_table}
}
\label{tab:calibration}
\end{table*}

\Cref{fig:tce_all} summarizes, across all tasks and models, how TCE and ECE correlate with the loss reduction obtained through recalibration. TCE shows a strong, consistent correlation with realized gains across almost all settings, whereas ECE does not. Moreover, because TCE is defined in the same units as the task loss, it tracks performance improvements on a directly interpretable scale, while ECE values are harder to relate to concrete changes in task performance.

\begin{figure}[h]
  \centering
\input{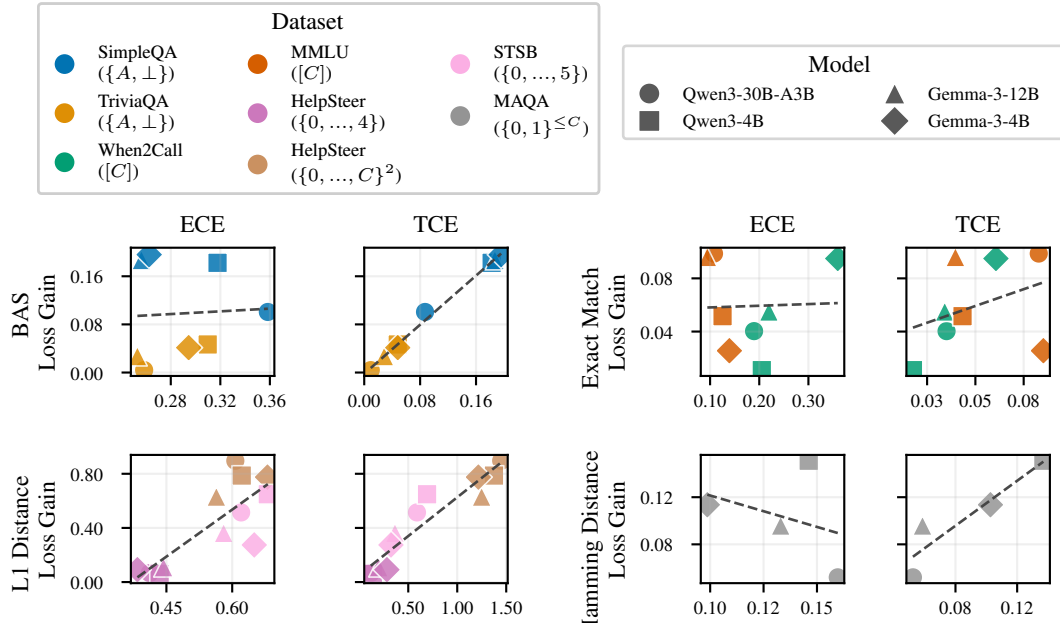}
    \caption{Task Calibration Error (TCE) and ECE for all Tasks}
  \label{fig:tce_all}
\end{figure}

\subsubsection{Ablation with Dirichlet Kernel Estimator}
\label{app:dirchlet}
To assess the robustness of our conclusions, we also employ the estimator of \citet{popordanoska2022consistentdifferentiablelpcanonical} with a Dirichlet kernel of fixed bandwidth $0.01$. The corresponding results are reported in \Cref{tab:calibration_kde}. Overall, they match those obtained with the binning-based estimator: TCE is substantially reduced after calibration and is often close to zero, indicating that most of the achievable improvement is recovered. Occasionally, discrepancies arise when the conditional estimator is supported by very few samples (sometimes a single data point), reflecting the sparsity of the empirical distribution in those regions.

\begin{table*}[h]
\centering
\caption{Task Calibration Error (TCE) using Dirichlet kernel estimator for uncalibrated \(q\) and Dirichlet-calibrated \(f_{\phi}(q)\) model}
\setlength{\tabcolsep}{4pt}
\resizebox{0.99\textwidth}{!}{
\input{tables/calibration_table_kde}
}
\label{tab:calibration_kde}
\end{table*}

%% file: tables/calibrator_ablation.tex
\begin{tabular}{ll|cc|cc|cc|c|c}
\toprule
\multicolumn{2}{c|}{\textbf{Latent Structure $\mathcal{L}$}(Metric)} & \multicolumn{2}{c|}{\small{$\{0, \dots, C\}$ (L1 $\downarrow$)}} & \multicolumn{2}{c|}{\small{$[C]$ (Acc. $\uparrow$)}} & \multicolumn{2}{c|}{\small{$\{A, \perp\}$ (BAS $\uparrow$)}} & \multicolumn{1}{c|}{\small{$\{0, 1\}^{\leq C}$ (Hamming $\downarrow$)}} & \multicolumn{1}{c}{\small{$\{0, \dots, C\}^2$ (L1 $\downarrow$)}} \\
\multicolumn{2}{c|}{} & Helpsteer & STSB & MMLU & When2Call & SimpleQA & TriviaQA & MAQA & Helpsteer \\
\midrule
\multirow{5}{*}{\small{Gemma-3-4B}} & \small{Uncalibrated} & {$0.82$$\scriptscriptstyle{\pm 0.08}$} & {$1.07$$\scriptscriptstyle{\pm 0.08}$} & {$64.56$$\scriptscriptstyle{\pm 2.08}$} & {$51.14$$\scriptscriptstyle{\pm 2.78}$} & {$\textcolor{colsecond}{\textbf{-0.20}}$$\scriptscriptstyle{\pm 0.02}$} & {$0.18$$\scriptscriptstyle{\pm 0.05}$} & {$0.74$$\scriptscriptstyle{\pm 0.04}$} & {$2.03$$\scriptscriptstyle{\pm 0.12}$} \\
 & \small{Temperature} & {$0.80$$\scriptscriptstyle{\pm 0.10}$} & {$1.07$$\scriptscriptstyle{\pm 0.08}$} & {$64.56$$\scriptscriptstyle{\pm 2.08}$} & {$51.14$$\scriptscriptstyle{\pm 2.78}$} & {$-0.25$$\scriptscriptstyle{\pm 0.02}$} & {$0.17$$\scriptscriptstyle{\pm 0.05}$} & {$0.74$$\scriptscriptstyle{\pm 0.04}$} & {$1.98$$\scriptscriptstyle{\pm 0.14}$} \\
 & \small{Decision} & {$\textcolor{colfirst}{\textbf{0.70}}$$\scriptscriptstyle{\pm 0.10}$} & {$0.84$$\scriptscriptstyle{\pm 0.04}$} & {$\textcolor{colfirst}{\textbf{67.36}}$$\scriptscriptstyle{\pm 1.17}$} & {$\textcolor{colsecond}{\textbf{56.10}}$$\scriptscriptstyle{\pm 3.40}$} & {$\textcolor{colfirst}{\textbf{-0.00}}$$\scriptscriptstyle{\pm 0.00}$} & {$\textcolor{colfirst}{\textbf{0.24}}$$\scriptscriptstyle{\pm 0.05}$} & {$\textcolor{colsecond}{\textbf{0.66}}$$\scriptscriptstyle{\pm 0.04}$} & {$\textcolor{colfirst}{\textbf{1.24}}$$\scriptscriptstyle{\pm 0.07}$} \\
 & \small{Policy $\pi_\phi$} & {$0.82$$\scriptscriptstyle{\pm 0.11}$} & {$\textcolor{colsecond}{\textbf{0.82}}$$\scriptscriptstyle{\pm 0.07}$} & {$\textcolor{colsecond}{\textbf{67.11}}$$\scriptscriptstyle{\pm 3.95}$} & {$\textcolor{colfirst}{\textbf{60.63}}$$\scriptscriptstyle{\pm 3.57}$} & {$\textcolor{colfirst}{\textbf{-0.00}}$$\scriptscriptstyle{\pm 0.00}$} & {$\textcolor{colsecond}{\textbf{0.23}}$$\scriptscriptstyle{\pm 0.03}$} & {$\textcolor{colfirst}{\textbf{0.62}}$$\scriptscriptstyle{\pm 0.03}$} & {$1.42$$\scriptscriptstyle{\pm 0.04}$} \\
 & \cellcolor[gray]{0.9}{\small{\textbf{Dirichlet}}} & \cellcolor[gray]{0.9}{{$\textcolor{colsecond}{\textbf{0.73}}$$\scriptscriptstyle{\pm 0.08}$}} & \cellcolor[gray]{0.9}{{$\textcolor{colfirst}{\textbf{0.80}}$$\scriptscriptstyle{\pm 0.05}$}} & \cellcolor[gray]{0.9}{{$\textcolor{colsecond}{\textbf{67.11}}$$\scriptscriptstyle{\pm 3.95}$}} & \cellcolor[gray]{0.9}{{$\textcolor{colfirst}{\textbf{60.63}}$$\scriptscriptstyle{\pm 3.57}$}} & \cellcolor[gray]{0.9}{{$\textcolor{colfirst}{\textbf{-0.00}}$$\scriptscriptstyle{\pm 0.00}$}} & \cellcolor[gray]{0.9}{{$0.22$$\scriptscriptstyle{\pm 0.05}$}} & \cellcolor[gray]{0.9}{{$\textcolor{colfirst}{\textbf{0.62}}$$\scriptscriptstyle{\pm 0.03}$}} & \cellcolor[gray]{0.9}{{$\textcolor{colsecond}{\textbf{1.25}}$$\scriptscriptstyle{\pm 0.04}$}} \\
\midrule
\multirow{5}{*}{\small{Gemma-3-12B}} & \small{Uncalibrated} & {$0.90$$\scriptscriptstyle{\pm 0.05}$} & {$1.01$$\scriptscriptstyle{\pm 0.03}$} & {$80.11$$\scriptscriptstyle{\pm 2.84}$} & {$68.87$$\scriptscriptstyle{\pm 2.09}$} & {$\textcolor{colsecond}{\textbf{-0.19}}$$\scriptscriptstyle{\pm 0.03}$} & {$0.44$$\scriptscriptstyle{\pm 0.03}$} & {$0.60$$\scriptscriptstyle{\pm 0.06}$} & {$1.89$$\scriptscriptstyle{\pm 0.07}$} \\
 & \small{Temperature} & {$0.88$$\scriptscriptstyle{\pm 0.05}$} & {$1.01$$\scriptscriptstyle{\pm 0.02}$} & {$80.11$$\scriptscriptstyle{\pm 2.84}$} & {$68.87$$\scriptscriptstyle{\pm 2.09}$} & {$-0.25$$\scriptscriptstyle{\pm 0.03}$} & {$0.44$$\scriptscriptstyle{\pm 0.03}$} & {$0.60$$\scriptscriptstyle{\pm 0.06}$} & {$1.79$$\scriptscriptstyle{\pm 0.05}$} \\
 & \small{Decision} & {$\textcolor{colsecond}{\textbf{0.83}}$$\scriptscriptstyle{\pm 0.04}$} & {$0.79$$\scriptscriptstyle{\pm 0.05}$} & {$\textcolor{colsecond}{\textbf{81.34}}$$\scriptscriptstyle{\pm 2.97}$} & {$\textcolor{colsecond}{\textbf{71.77}}$$\scriptscriptstyle{\pm 2.72}$} & {$\textcolor{colfirst}{\textbf{-0.00}}$$\scriptscriptstyle{\pm 0.00}$} & {$\textcolor{colsecond}{\textbf{0.47}}$$\scriptscriptstyle{\pm 0.03}$} & {$\textcolor{colsecond}{\textbf{0.56}}$$\scriptscriptstyle{\pm 0.06}$} & {$\textcolor{colfirst}{\textbf{1.25}}$$\scriptscriptstyle{\pm 0.06}$} \\
 & \small{Policy $\pi_\phi$} & {$0.92$$\scriptscriptstyle{\pm 0.02}$} & {$\textcolor{colsecond}{\textbf{0.69}}$$\scriptscriptstyle{\pm 0.05}$} & {$\textcolor{colfirst}{\textbf{89.72}}$$\scriptscriptstyle{\pm 1.56}$} & {$\textcolor{colfirst}{\textbf{74.38}}$$\scriptscriptstyle{\pm 2.64}$} & {$\textcolor{colfirst}{\textbf{-0.00}}$$\scriptscriptstyle{\pm 0.00}$} & {$0.47$$\scriptscriptstyle{\pm 0.02}$} & {$\textcolor{colfirst}{\textbf{0.51}}$$\scriptscriptstyle{\pm 0.05}$} & {$1.36$$\scriptscriptstyle{\pm 0.07}$} \\
 & \cellcolor[gray]{0.9}{\small{\textbf{Dirichlet}}} & \cellcolor[gray]{0.9}{{$\textcolor{colfirst}{\textbf{0.79}}$$\scriptscriptstyle{\pm 0.07}$}} & \cellcolor[gray]{0.9}{{$\textcolor{colfirst}{\textbf{0.65}}$$\scriptscriptstyle{\pm 0.07}$}} & \cellcolor[gray]{0.9}{{$\textcolor{colfirst}{\textbf{89.72}}$$\scriptscriptstyle{\pm 1.56}$}} & \cellcolor[gray]{0.9}{{$\textcolor{colfirst}{\textbf{74.38}}$$\scriptscriptstyle{\pm 2.64}$}} & \cellcolor[gray]{0.9}{{$\textcolor{colfirst}{\textbf{-0.00}}$$\scriptscriptstyle{\pm 0.00}$}} & \cellcolor[gray]{0.9}{{$\textcolor{colfirst}{\textbf{0.47}}$$\scriptscriptstyle{\pm 0.03}$}} & \cellcolor[gray]{0.9}{{$\textcolor{colfirst}{\textbf{0.51}}$$\scriptscriptstyle{\pm 0.05}$}} & \cellcolor[gray]{0.9}{{$\textcolor{colsecond}{\textbf{1.26}}$$\scriptscriptstyle{\pm 0.04}$}} \\
\midrule
\multirow{5}{*}{\small{Qwen-3-4B}} & \small{Uncalibrated} & {$0.82$$\scriptscriptstyle{\pm 0.07}$} & {$1.33$$\scriptscriptstyle{\pm 0.04}$} & {$\textcolor{colsecond}{\textbf{77.92}}$$\scriptscriptstyle{\pm 1.79}$} & {$\textcolor{colsecond}{\textbf{73.00}}$$\scriptscriptstyle{\pm 1.70}$} & {$\textcolor{colsecond}{\textbf{-0.18}}$$\scriptscriptstyle{\pm 0.01}$} & {$0.16$$\scriptscriptstyle{\pm 0.04}$} & {$0.87$$\scriptscriptstyle{\pm 0.08}$} & {$2.06$$\scriptscriptstyle{\pm 0.05}$} \\
 & \small{Temperature} & {$0.79$$\scriptscriptstyle{\pm 0.08}$} & {$1.33$$\scriptscriptstyle{\pm 0.04}$} & {$\textcolor{colsecond}{\textbf{77.92}}$$\scriptscriptstyle{\pm 1.79}$} & {$\textcolor{colsecond}{\textbf{73.00}}$$\scriptscriptstyle{\pm 1.70}$} & {$-0.22$$\scriptscriptstyle{\pm 0.01}$} & {$0.15$$\scriptscriptstyle{\pm 0.04}$} & {$0.87$$\scriptscriptstyle{\pm 0.08}$} & {$1.97$$\scriptscriptstyle{\pm 0.06}$} \\
 & \small{Decision} & {$\textcolor{colsecond}{\textbf{0.78}}$$\scriptscriptstyle{\pm 0.06}$} & {$0.73$$\scriptscriptstyle{\pm 0.05}$} & {$77.58$$\scriptscriptstyle{\pm 2.41}$} & {$71.80$$\scriptscriptstyle{\pm 3.85}$} & {$\textcolor{colfirst}{\textbf{-0.00}}$$\scriptscriptstyle{\pm 0.00}$} & {$\textcolor{colsecond}{\textbf{0.21}}$$\scriptscriptstyle{\pm 0.03}$} & {$\textcolor{colsecond}{\textbf{0.75}}$$\scriptscriptstyle{\pm 0.07}$} & {$\textcolor{colsecond}{\textbf{1.29}}$$\scriptscriptstyle{\pm 0.04}$} \\
 & \small{Policy $\pi_\phi$} & {$0.84$$\scriptscriptstyle{\pm 0.06}$} & {$\textcolor{colsecond}{\textbf{0.71}}$$\scriptscriptstyle{\pm 0.02}$} & {$\textcolor{colfirst}{\textbf{83.07}}$$\scriptscriptstyle{\pm 0.93}$} & {$\textcolor{colfirst}{\textbf{74.10}}$$\scriptscriptstyle{\pm 2.75}$} & {$\textcolor{colfirst}{\textbf{-0.00}}$$\scriptscriptstyle{\pm 0.00}$} & {$\textcolor{colfirst}{\textbf{0.21}}$$\scriptscriptstyle{\pm 0.03}$} & {$\textcolor{colfirst}{\textbf{0.72}}$$\scriptscriptstyle{\pm 0.05}$} & {$1.40$$\scriptscriptstyle{\pm 0.09}$} \\
 & \cellcolor[gray]{0.9}{\small{\textbf{Dirichlet}}} & \cellcolor[gray]{0.9}{{$\textcolor{colfirst}{\textbf{0.76}}$$\scriptscriptstyle{\pm 0.04}$}} & \cellcolor[gray]{0.9}{{$\textcolor{colfirst}{\textbf{0.68}}$$\scriptscriptstyle{\pm 0.05}$}} & \cellcolor[gray]{0.9}{{$\textcolor{colfirst}{\textbf{83.07}}$$\scriptscriptstyle{\pm 0.93}$}} & \cellcolor[gray]{0.9}{{$\textcolor{colfirst}{\textbf{74.10}}$$\scriptscriptstyle{\pm 2.75}$}} & \cellcolor[gray]{0.9}{{$\textcolor{colfirst}{\textbf{-0.00}}$$\scriptscriptstyle{\pm 0.00}$}} & \cellcolor[gray]{0.9}{{$0.21$$\scriptscriptstyle{\pm 0.03}$}} & \cellcolor[gray]{0.9}{{$\textcolor{colfirst}{\textbf{0.72}}$$\scriptscriptstyle{\pm 0.05}$}} & \cellcolor[gray]{0.9}{{$\textcolor{colfirst}{\textbf{1.27}}$$\scriptscriptstyle{\pm 0.06}$}} \\
\midrule
\multirow{5}{*}{\small{Qwen3-30B-A3B}} & \small{Uncalibrated} & {$0.78$$\scriptscriptstyle{\pm 0.07}$} & {$1.18$$\scriptscriptstyle{\pm 0.12}$} & {$78.48$$\scriptscriptstyle{\pm 1.92}$} & {$74.59$$\scriptscriptstyle{\pm 3.52}$} & {$-0.07$$\scriptscriptstyle{\pm 0.02}$} & {$0.46$$\scriptscriptstyle{\pm 0.05}$} & {$0.54$$\scriptscriptstyle{\pm 0.03}$} & {$2.11$$\scriptscriptstyle{\pm 0.12}$} \\
 & \small{Temperature} & {$0.77$$\scriptscriptstyle{\pm 0.07}$} & {$1.18$$\scriptscriptstyle{\pm 0.12}$} & {$78.48$$\scriptscriptstyle{\pm 1.92}$} & {$74.59$$\scriptscriptstyle{\pm 3.52}$} & {$-0.10$$\scriptscriptstyle{\pm 0.03}$} & {$0.46$$\scriptscriptstyle{\pm 0.05}$} & {$0.54$$\scriptscriptstyle{\pm 0.03}$} & {$2.04$$\scriptscriptstyle{\pm 0.09}$} \\
 & \small{Decision} & {$\textcolor{colsecond}{\textbf{0.77}}$$\scriptscriptstyle{\pm 0.04}$} & {$0.70$$\scriptscriptstyle{\pm 0.08}$} & {$\textcolor{colsecond}{\textbf{81.39}}$$\scriptscriptstyle{\pm 2.42}$} & {$\textcolor{colsecond}{\textbf{75.60}}$$\scriptscriptstyle{\pm 2.94}$} & {$\textcolor{colfirst}{\textbf{0.03}}$$\scriptscriptstyle{\pm 0.02}$} & {$\textcolor{colsecond}{\textbf{0.47}}$$\scriptscriptstyle{\pm 0.05}$} & {$\textcolor{colsecond}{\textbf{0.51}}$$\scriptscriptstyle{\pm 0.03}$} & {$\textcolor{colsecond}{\textbf{1.23}}$$\scriptscriptstyle{\pm 0.02}$} \\
 & \small{Policy $\pi_\phi$} & {$0.78$$\scriptscriptstyle{\pm 0.07}$} & {$\textcolor{colsecond}{\textbf{0.70}}$$\scriptscriptstyle{\pm 0.06}$} & {$\textcolor{colfirst}{\textbf{88.34}}$$\scriptscriptstyle{\pm 2.11}$} & {$\textcolor{colfirst}{\textbf{78.63}}$$\scriptscriptstyle{\pm 2.50}$} & {$\textcolor{colsecond}{\textbf{-0.00}}$$\scriptscriptstyle{\pm 0.00}$} & {$\textcolor{colfirst}{\textbf{0.47}}$$\scriptscriptstyle{\pm 0.04}$} & {$\textcolor{colfirst}{\textbf{0.49}}$$\scriptscriptstyle{\pm 0.03}$} & {$1.37$$\scriptscriptstyle{\pm 0.04}$} \\
 & \cellcolor[gray]{0.9}{\small{\textbf{Dirichlet}}} & \cellcolor[gray]{0.9}{{$\textcolor{colfirst}{\textbf{0.72}}$$\scriptscriptstyle{\pm 0.05}$}} & \cellcolor[gray]{0.9}{{$\textcolor{colfirst}{\textbf{0.67}}$$\scriptscriptstyle{\pm 0.05}$}} & \cellcolor[gray]{0.9}{{$\textcolor{colfirst}{\textbf{88.34}}$$\scriptscriptstyle{\pm 2.11}$}} & \cellcolor[gray]{0.9}{{$\textcolor{colfirst}{\textbf{78.63}}$$\scriptscriptstyle{\pm 2.50}$}} & \cellcolor[gray]{0.9}{{$\textcolor{colfirst}{\textbf{0.03}}$$\scriptscriptstyle{\pm 0.02}$}} & \cellcolor[gray]{0.9}{{$0.46$$\scriptscriptstyle{\pm 0.04}$}} & \cellcolor[gray]{0.9}{{$\textcolor{colfirst}{\textbf{0.49}}$$\scriptscriptstyle{\pm 0.03}$}} & \cellcolor[gray]{0.9}{{$\textcolor{colfirst}{\textbf{1.21}}$$\scriptscriptstyle{\pm 0.04}$}} \\
\bottomrule
\end{tabular}

%% file: figures/calibration_maps/SIMPLE_QA_VERIFIED_gemma-3-4b-it.pgf
\begingroup%
\makeatletter%
\begin{pgfpicture}%
\pgfpathrectangle{\pgfpointorigin}{\pgfqpoint{1.600000in}{1.600000in}}%
\pgfusepath{use as bounding box, clip}%
\begin{pgfscope}%
\pgfsetbuttcap%
\pgfsetmiterjoin%
\definecolor{currentfill}{rgb}{1.000000,1.000000,1.000000}%
\pgfsetfillcolor{currentfill}%
\pgfsetlinewidth{0.000000pt}%
\definecolor{currentstroke}{rgb}{1.000000,1.000000,1.000000}%
\pgfsetstrokecolor{currentstroke}%
\pgfsetdash{}{0pt}%
\pgfpathmoveto{\pgfqpoint{0.000000in}{0.000000in}}%
\pgfpathlineto{\pgfqpoint{1.600000in}{0.000000in}}%
\pgfpathlineto{\pgfqpoint{1.600000in}{1.600000in}}%
\pgfpathlineto{\pgfqpoint{0.000000in}{1.600000in}}%
\pgfpathlineto{\pgfqpoint{0.000000in}{0.000000in}}%
\pgfpathclose%
\pgfusepath{fill}%
\end{pgfscope}%
\begin{pgfscope}%
\pgfsetbuttcap%
\pgfsetmiterjoin%
\definecolor{currentfill}{rgb}{1.000000,1.000000,1.000000}%
\pgfsetfillcolor{currentfill}%
\pgfsetlinewidth{0.000000pt}%
\definecolor{currentstroke}{rgb}{0.000000,0.000000,0.000000}%
\pgfsetstrokecolor{currentstroke}%
\pgfsetstrokeopacity{0.000000}%
\pgfsetdash{}{0pt}%
\pgfpathmoveto{\pgfqpoint{0.519743in}{0.414757in}}%
\pgfpathlineto{\pgfqpoint{1.503516in}{0.414757in}}%
\pgfpathlineto{\pgfqpoint{1.503516in}{1.558330in}}%
\pgfpathlineto{\pgfqpoint{0.519743in}{1.558330in}}%
\pgfpathlineto{\pgfqpoint{0.519743in}{0.414757in}}%
\pgfpathclose%
\pgfusepath{fill}%
\end{pgfscope}%
\begin{pgfscope}%
\pgfpathrectangle{\pgfqpoint{0.519743in}{0.414757in}}{\pgfqpoint{0.983773in}{1.143573in}}%
\pgfusepath{clip}%
\pgfsetbuttcap%
\pgfsetroundjoin%
\definecolor{currentfill}{rgb}{0.870588,0.560784,0.011765}%
\pgfsetfillcolor{currentfill}%
\pgfsetfillopacity{0.200000}%
\pgfsetlinewidth{1.003750pt}%
\definecolor{currentstroke}{rgb}{0.870588,0.560784,0.011765}%
\pgfsetstrokecolor{currentstroke}%
\pgfsetstrokeopacity{0.200000}%
\pgfsetdash{}{0pt}%
\pgfsys@defobject{currentmarker}{\pgfqpoint{0.564460in}{0.726641in}}{\pgfqpoint{0.788045in}{1.506349in}}{%
\pgfpathmoveto{\pgfqpoint{0.788045in}{0.726641in}}%
\pgfpathlineto{\pgfqpoint{0.564460in}{0.726641in}}%
\pgfpathlineto{\pgfqpoint{0.564460in}{1.506349in}}%
\pgfpathlineto{\pgfqpoint{0.788045in}{1.506349in}}%
\pgfpathlineto{\pgfqpoint{0.788045in}{1.506349in}}%
\pgfpathlineto{\pgfqpoint{0.788045in}{0.726641in}}%
\pgfpathlineto{\pgfqpoint{0.788045in}{0.726641in}}%
\pgfpathclose%
\pgfusepath{stroke,fill}%
}%
\begin{pgfscope}%
\pgfsys@transformshift{0.000000in}{0.000000in}%
\pgfsys@useobject{currentmarker}{}%
\end{pgfscope}%
\end{pgfscope}%
\begin{pgfscope}%
\pgfpathrectangle{\pgfqpoint{0.519743in}{0.414757in}}{\pgfqpoint{0.983773in}{1.143573in}}%
\pgfusepath{clip}%
\pgfsetbuttcap%
\pgfsetroundjoin%
\definecolor{currentfill}{rgb}{0.007843,0.619608,0.447059}%
\pgfsetfillcolor{currentfill}%
\pgfsetfillopacity{0.200000}%
\pgfsetlinewidth{1.003750pt}%
\definecolor{currentstroke}{rgb}{0.007843,0.619608,0.447059}%
\pgfsetstrokecolor{currentstroke}%
\pgfsetstrokeopacity{0.200000}%
\pgfsetdash{}{0pt}%
\pgfsys@defobject{currentmarker}{\pgfqpoint{0.788045in}{0.726641in}}{\pgfqpoint{1.458799in}{1.506349in}}{%
\pgfpathmoveto{\pgfqpoint{1.458799in}{0.726641in}}%
\pgfpathlineto{\pgfqpoint{0.788045in}{0.726641in}}%
\pgfpathlineto{\pgfqpoint{0.788045in}{1.506349in}}%
\pgfpathlineto{\pgfqpoint{1.458799in}{1.506349in}}%
\pgfpathlineto{\pgfqpoint{1.458799in}{1.506349in}}%
\pgfpathlineto{\pgfqpoint{1.458799in}{0.726641in}}%
\pgfpathlineto{\pgfqpoint{1.458799in}{0.726641in}}%
\pgfpathclose%
\pgfusepath{stroke,fill}%
}%
\begin{pgfscope}%
\pgfsys@transformshift{0.000000in}{0.000000in}%
\pgfsys@useobject{currentmarker}{}%
\end{pgfscope}%
\end{pgfscope}%
\begin{pgfscope}%
\pgfpathrectangle{\pgfqpoint{0.519743in}{0.414757in}}{\pgfqpoint{0.983773in}{1.143573in}}%
\pgfusepath{clip}%
\pgfsetbuttcap%
\pgfsetroundjoin%
\definecolor{currentfill}{rgb}{0.870588,0.560784,0.011765}%
\pgfsetfillcolor{currentfill}%
\pgfsetfillopacity{0.200000}%
\pgfsetlinewidth{1.003750pt}%
\definecolor{currentstroke}{rgb}{0.870588,0.560784,0.011765}%
\pgfsetstrokecolor{currentstroke}%
\pgfsetstrokeopacity{0.200000}%
\pgfsetdash{}{0pt}%
\pgfsys@defobject{currentmarker}{\pgfqpoint{0.788045in}{0.466738in}}{\pgfqpoint{1.458799in}{0.726641in}}{%
\pgfpathmoveto{\pgfqpoint{1.458799in}{0.466738in}}%
\pgfpathlineto{\pgfqpoint{0.788045in}{0.466738in}}%
\pgfpathlineto{\pgfqpoint{0.788045in}{0.726641in}}%
\pgfpathlineto{\pgfqpoint{1.458799in}{0.726641in}}%
\pgfpathlineto{\pgfqpoint{1.458799in}{0.726641in}}%
\pgfpathlineto{\pgfqpoint{1.458799in}{0.466738in}}%
\pgfpathlineto{\pgfqpoint{1.458799in}{0.466738in}}%
\pgfpathclose%
\pgfusepath{stroke,fill}%
}%
\begin{pgfscope}%
\pgfsys@transformshift{0.000000in}{0.000000in}%
\pgfsys@useobject{currentmarker}{}%
\end{pgfscope}%
\end{pgfscope}%
\begin{pgfscope}%
\pgfpathrectangle{\pgfqpoint{0.519743in}{0.414757in}}{\pgfqpoint{0.983773in}{1.143573in}}%
\pgfusepath{clip}%
\pgfsetbuttcap%
\pgfsetroundjoin%
\definecolor{currentfill}{rgb}{0.007843,0.619608,0.447059}%
\pgfsetfillcolor{currentfill}%
\pgfsetfillopacity{0.200000}%
\pgfsetlinewidth{1.003750pt}%
\definecolor{currentstroke}{rgb}{0.007843,0.619608,0.447059}%
\pgfsetstrokecolor{currentstroke}%
\pgfsetstrokeopacity{0.200000}%
\pgfsetdash{}{0pt}%
\pgfsys@defobject{currentmarker}{\pgfqpoint{0.564460in}{0.466738in}}{\pgfqpoint{0.788045in}{0.726641in}}{%
\pgfpathmoveto{\pgfqpoint{0.788045in}{0.466738in}}%
\pgfpathlineto{\pgfqpoint{0.564460in}{0.466738in}}%
\pgfpathlineto{\pgfqpoint{0.564460in}{0.726641in}}%
\pgfpathlineto{\pgfqpoint{0.788045in}{0.726641in}}%
\pgfpathlineto{\pgfqpoint{0.788045in}{0.726641in}}%
\pgfpathlineto{\pgfqpoint{0.788045in}{0.466738in}}%
\pgfpathlineto{\pgfqpoint{0.788045in}{0.466738in}}%
\pgfpathclose%
\pgfusepath{stroke,fill}%
}%
\begin{pgfscope}%
\pgfsys@transformshift{0.000000in}{0.000000in}%
\pgfsys@useobject{currentmarker}{}%
\end{pgfscope}%
\end{pgfscope}%
\begin{pgfscope}%
\pgfpathrectangle{\pgfqpoint{0.519743in}{0.414757in}}{\pgfqpoint{0.983773in}{1.143573in}}%
\pgfusepath{clip}%
\pgfsetrectcap%
\pgfsetroundjoin%
\pgfsetlinewidth{0.803000pt}%
\definecolor{currentstroke}{rgb}{0.690196,0.690196,0.690196}%
\pgfsetstrokecolor{currentstroke}%
\pgfsetdash{}{0pt}%
\pgfpathmoveto{\pgfqpoint{0.564460in}{0.414757in}}%
\pgfpathlineto{\pgfqpoint{0.564460in}{1.558330in}}%
\pgfusepath{stroke}%
\end{pgfscope}%
\begin{pgfscope}%
\pgfsetbuttcap%
\pgfsetroundjoin%
\definecolor{currentfill}{rgb}{0.000000,0.000000,0.000000}%
\pgfsetfillcolor{currentfill}%
\pgfsetlinewidth{0.803000pt}%
\definecolor{currentstroke}{rgb}{0.000000,0.000000,0.000000}%
\pgfsetstrokecolor{currentstroke}%
\pgfsetdash{}{0pt}%
\pgfsys@defobject{currentmarker}{\pgfqpoint{0.000000in}{-0.048611in}}{\pgfqpoint{0.000000in}{0.000000in}}{%
\pgfpathmoveto{\pgfqpoint{0.000000in}{0.000000in}}%
\pgfpathlineto{\pgfqpoint{0.000000in}{-0.048611in}}%
\pgfusepath{stroke,fill}%
}%
\begin{pgfscope}%
\pgfsys@transformshift{0.564460in}{0.414757in}%
\pgfsys@useobject{currentmarker}{}%
\end{pgfscope}%
\end{pgfscope}%
\begin{pgfscope}%
\definecolor{textcolor}{rgb}{0.000000,0.000000,0.000000}%
\pgfsetstrokecolor{textcolor}%
\pgfsetfillcolor{textcolor}%
\pgftext[x=0.564460in,y=0.317535in,,top]{\color{textcolor}{\rmfamily\fontsize{7.000000}{8.400000}\selectfont\catcode`\^=\active\def^{\ifmmode\sp\else\^{}\fi}\catcode`\%=\active\def
\end{pgfscope}%
\begin{pgfscope}%
\pgfpathrectangle{\pgfqpoint{0.519743in}{0.414757in}}{\pgfqpoint{0.983773in}{1.143573in}}%
\pgfusepath{clip}%
\pgfsetrectcap%
\pgfsetroundjoin%
\pgfsetlinewidth{0.803000pt}%
\definecolor{currentstroke}{rgb}{0.690196,0.690196,0.690196}%
\pgfsetstrokecolor{currentstroke}%
\pgfsetdash{}{0pt}%
\pgfpathmoveto{\pgfqpoint{0.788045in}{0.414757in}}%
\pgfpathlineto{\pgfqpoint{0.788045in}{1.558330in}}%
\pgfusepath{stroke}%
\end{pgfscope}%
\begin{pgfscope}%
\pgfsetbuttcap%
\pgfsetroundjoin%
\definecolor{currentfill}{rgb}{0.000000,0.000000,0.000000}%
\pgfsetfillcolor{currentfill}%
\pgfsetlinewidth{0.803000pt}%
\definecolor{currentstroke}{rgb}{0.000000,0.000000,0.000000}%
\pgfsetstrokecolor{currentstroke}%
\pgfsetdash{}{0pt}%
\pgfsys@defobject{currentmarker}{\pgfqpoint{0.000000in}{-0.048611in}}{\pgfqpoint{0.000000in}{0.000000in}}{%
\pgfpathmoveto{\pgfqpoint{0.000000in}{0.000000in}}%
\pgfpathlineto{\pgfqpoint{0.000000in}{-0.048611in}}%
\pgfusepath{stroke,fill}%
}%
\begin{pgfscope}%
\pgfsys@transformshift{0.788045in}{0.414757in}%
\pgfsys@useobject{currentmarker}{}%
\end{pgfscope}%
\end{pgfscope}%
\begin{pgfscope}%
\definecolor{textcolor}{rgb}{0.000000,0.000000,0.000000}%
\pgfsetstrokecolor{textcolor}%
\pgfsetfillcolor{textcolor}%
\pgftext[x=0.788045in,y=0.317535in,,top]{\color{textcolor}{\rmfamily\fontsize{7.000000}{8.400000}\selectfont\catcode`\^=\active\def^{\ifmmode\sp\else\^{}\fi}\catcode`\%=\active\def
\end{pgfscope}%
\begin{pgfscope}%
\pgfpathrectangle{\pgfqpoint{0.519743in}{0.414757in}}{\pgfqpoint{0.983773in}{1.143573in}}%
\pgfusepath{clip}%
\pgfsetrectcap%
\pgfsetroundjoin%
\pgfsetlinewidth{0.803000pt}%
\definecolor{currentstroke}{rgb}{0.690196,0.690196,0.690196}%
\pgfsetstrokecolor{currentstroke}%
\pgfsetdash{}{0pt}%
\pgfpathmoveto{\pgfqpoint{1.011630in}{0.414757in}}%
\pgfpathlineto{\pgfqpoint{1.011630in}{1.558330in}}%
\pgfusepath{stroke}%
\end{pgfscope}%
\begin{pgfscope}%
\pgfsetbuttcap%
\pgfsetroundjoin%
\definecolor{currentfill}{rgb}{0.000000,0.000000,0.000000}%
\pgfsetfillcolor{currentfill}%
\pgfsetlinewidth{0.803000pt}%
\definecolor{currentstroke}{rgb}{0.000000,0.000000,0.000000}%
\pgfsetstrokecolor{currentstroke}%
\pgfsetdash{}{0pt}%
\pgfsys@defobject{currentmarker}{\pgfqpoint{0.000000in}{-0.048611in}}{\pgfqpoint{0.000000in}{0.000000in}}{%
\pgfpathmoveto{\pgfqpoint{0.000000in}{0.000000in}}%
\pgfpathlineto{\pgfqpoint{0.000000in}{-0.048611in}}%
\pgfusepath{stroke,fill}%
}%
\begin{pgfscope}%
\pgfsys@transformshift{1.011630in}{0.414757in}%
\pgfsys@useobject{currentmarker}{}%
\end{pgfscope}%
\end{pgfscope}%
\begin{pgfscope}%
\definecolor{textcolor}{rgb}{0.000000,0.000000,0.000000}%
\pgfsetstrokecolor{textcolor}%
\pgfsetfillcolor{textcolor}%
\pgftext[x=1.011630in,y=0.317535in,,top]{\color{textcolor}{\rmfamily\fontsize{7.000000}{8.400000}\selectfont\catcode`\^=\active\def^{\ifmmode\sp\else\^{}\fi}\catcode`\%=\active\def
\end{pgfscope}%
\begin{pgfscope}%
\pgfpathrectangle{\pgfqpoint{0.519743in}{0.414757in}}{\pgfqpoint{0.983773in}{1.143573in}}%
\pgfusepath{clip}%
\pgfsetrectcap%
\pgfsetroundjoin%
\pgfsetlinewidth{0.803000pt}%
\definecolor{currentstroke}{rgb}{0.690196,0.690196,0.690196}%
\pgfsetstrokecolor{currentstroke}%
\pgfsetdash{}{0pt}%
\pgfpathmoveto{\pgfqpoint{1.235215in}{0.414757in}}%
\pgfpathlineto{\pgfqpoint{1.235215in}{1.558330in}}%
\pgfusepath{stroke}%
\end{pgfscope}%
\begin{pgfscope}%
\pgfsetbuttcap%
\pgfsetroundjoin%
\definecolor{currentfill}{rgb}{0.000000,0.000000,0.000000}%
\pgfsetfillcolor{currentfill}%
\pgfsetlinewidth{0.803000pt}%
\definecolor{currentstroke}{rgb}{0.000000,0.000000,0.000000}%
\pgfsetstrokecolor{currentstroke}%
\pgfsetdash{}{0pt}%
\pgfsys@defobject{currentmarker}{\pgfqpoint{0.000000in}{-0.048611in}}{\pgfqpoint{0.000000in}{0.000000in}}{%
\pgfpathmoveto{\pgfqpoint{0.000000in}{0.000000in}}%
\pgfpathlineto{\pgfqpoint{0.000000in}{-0.048611in}}%
\pgfusepath{stroke,fill}%
}%
\begin{pgfscope}%
\pgfsys@transformshift{1.235215in}{0.414757in}%
\pgfsys@useobject{currentmarker}{}%
\end{pgfscope}%
\end{pgfscope}%
\begin{pgfscope}%
\definecolor{textcolor}{rgb}{0.000000,0.000000,0.000000}%
\pgfsetstrokecolor{textcolor}%
\pgfsetfillcolor{textcolor}%
\pgftext[x=1.235215in,y=0.317535in,,top]{\color{textcolor}{\rmfamily\fontsize{7.000000}{8.400000}\selectfont\catcode`\^=\active\def^{\ifmmode\sp\else\^{}\fi}\catcode`\%=\active\def
\end{pgfscope}%
\begin{pgfscope}%
\pgfpathrectangle{\pgfqpoint{0.519743in}{0.414757in}}{\pgfqpoint{0.983773in}{1.143573in}}%
\pgfusepath{clip}%
\pgfsetrectcap%
\pgfsetroundjoin%
\pgfsetlinewidth{0.803000pt}%
\definecolor{currentstroke}{rgb}{0.690196,0.690196,0.690196}%
\pgfsetstrokecolor{currentstroke}%
\pgfsetdash{}{0pt}%
\pgfpathmoveto{\pgfqpoint{1.458799in}{0.414757in}}%
\pgfpathlineto{\pgfqpoint{1.458799in}{1.558330in}}%
\pgfusepath{stroke}%
\end{pgfscope}%
\begin{pgfscope}%
\pgfsetbuttcap%
\pgfsetroundjoin%
\definecolor{currentfill}{rgb}{0.000000,0.000000,0.000000}%
\pgfsetfillcolor{currentfill}%
\pgfsetlinewidth{0.803000pt}%
\definecolor{currentstroke}{rgb}{0.000000,0.000000,0.000000}%
\pgfsetstrokecolor{currentstroke}%
\pgfsetdash{}{0pt}%
\pgfsys@defobject{currentmarker}{\pgfqpoint{0.000000in}{-0.048611in}}{\pgfqpoint{0.000000in}{0.000000in}}{%
\pgfpathmoveto{\pgfqpoint{0.000000in}{0.000000in}}%
\pgfpathlineto{\pgfqpoint{0.000000in}{-0.048611in}}%
\pgfusepath{stroke,fill}%
}%
\begin{pgfscope}%
\pgfsys@transformshift{1.458799in}{0.414757in}%
\pgfsys@useobject{currentmarker}{}%
\end{pgfscope}%
\end{pgfscope}%
\begin{pgfscope}%
\definecolor{textcolor}{rgb}{0.000000,0.000000,0.000000}%
\pgfsetstrokecolor{textcolor}%
\pgfsetfillcolor{textcolor}%
\pgftext[x=1.458799in,y=0.317535in,,top]{\color{textcolor}{\rmfamily\fontsize{7.000000}{8.400000}\selectfont\catcode`\^=\active\def^{\ifmmode\sp\else\^{}\fi}\catcode`\%=\active\def
\end{pgfscope}%
\begin{pgfscope}%
\definecolor{textcolor}{rgb}{0.000000,0.000000,0.000000}%
\pgfsetstrokecolor{textcolor}%
\pgfsetfillcolor{textcolor}%
\pgftext[x=1.011630in,y=0.167891in,,top]{\color{textcolor}{\rmfamily\fontsize{9.000000}{10.800000}\selectfont\catcode`\^=\active\def^{\ifmmode\sp\else\^{}\fi}\catcode`\%=\active\def
\end{pgfscope}%
\begin{pgfscope}%
\pgfpathrectangle{\pgfqpoint{0.519743in}{0.414757in}}{\pgfqpoint{0.983773in}{1.143573in}}%
\pgfusepath{clip}%
\pgfsetrectcap%
\pgfsetroundjoin%
\pgfsetlinewidth{0.803000pt}%
\definecolor{currentstroke}{rgb}{0.690196,0.690196,0.690196}%
\pgfsetstrokecolor{currentstroke}%
\pgfsetdash{}{0pt}%
\pgfpathmoveto{\pgfqpoint{0.519743in}{0.466738in}}%
\pgfpathlineto{\pgfqpoint{1.503516in}{0.466738in}}%
\pgfusepath{stroke}%
\end{pgfscope}%
\begin{pgfscope}%
\pgfsetbuttcap%
\pgfsetroundjoin%
\definecolor{currentfill}{rgb}{0.000000,0.000000,0.000000}%
\pgfsetfillcolor{currentfill}%
\pgfsetlinewidth{0.803000pt}%
\definecolor{currentstroke}{rgb}{0.000000,0.000000,0.000000}%
\pgfsetstrokecolor{currentstroke}%
\pgfsetdash{}{0pt}%
\pgfsys@defobject{currentmarker}{\pgfqpoint{-0.048611in}{0.000000in}}{\pgfqpoint{-0.000000in}{0.000000in}}{%
\pgfpathmoveto{\pgfqpoint{-0.000000in}{0.000000in}}%
\pgfpathlineto{\pgfqpoint{-0.048611in}{0.000000in}}%
\pgfusepath{stroke,fill}%
}%
\begin{pgfscope}%
\pgfsys@transformshift{0.519743in}{0.466738in}%
\pgfsys@useobject{currentmarker}{}%
\end{pgfscope}%
\end{pgfscope}%
\begin{pgfscope}%
\definecolor{textcolor}{rgb}{0.000000,0.000000,0.000000}%
\pgfsetstrokecolor{textcolor}%
\pgfsetfillcolor{textcolor}%
\pgftext[x=0.223446in, y=0.429805in, left, base]{\color{textcolor}{\rmfamily\fontsize{7.000000}{8.400000}\selectfont\catcode`\^=\active\def^{\ifmmode\sp\else\^{}\fi}\catcode`\%=\active\def
\end{pgfscope}%
\begin{pgfscope}%
\pgfpathrectangle{\pgfqpoint{0.519743in}{0.414757in}}{\pgfqpoint{0.983773in}{1.143573in}}%
\pgfusepath{clip}%
\pgfsetrectcap%
\pgfsetroundjoin%
\pgfsetlinewidth{0.803000pt}%
\definecolor{currentstroke}{rgb}{0.690196,0.690196,0.690196}%
\pgfsetstrokecolor{currentstroke}%
\pgfsetdash{}{0pt}%
\pgfpathmoveto{\pgfqpoint{0.519743in}{0.726641in}}%
\pgfpathlineto{\pgfqpoint{1.503516in}{0.726641in}}%
\pgfusepath{stroke}%
\end{pgfscope}%
\begin{pgfscope}%
\pgfsetbuttcap%
\pgfsetroundjoin%
\definecolor{currentfill}{rgb}{0.000000,0.000000,0.000000}%
\pgfsetfillcolor{currentfill}%
\pgfsetlinewidth{0.803000pt}%
\definecolor{currentstroke}{rgb}{0.000000,0.000000,0.000000}%
\pgfsetstrokecolor{currentstroke}%
\pgfsetdash{}{0pt}%
\pgfsys@defobject{currentmarker}{\pgfqpoint{-0.048611in}{0.000000in}}{\pgfqpoint{-0.000000in}{0.000000in}}{%
\pgfpathmoveto{\pgfqpoint{-0.000000in}{0.000000in}}%
\pgfpathlineto{\pgfqpoint{-0.048611in}{0.000000in}}%
\pgfusepath{stroke,fill}%
}%
\begin{pgfscope}%
\pgfsys@transformshift{0.519743in}{0.726641in}%
\pgfsys@useobject{currentmarker}{}%
\end{pgfscope}%
\end{pgfscope}%
\begin{pgfscope}%
\definecolor{textcolor}{rgb}{0.000000,0.000000,0.000000}%
\pgfsetstrokecolor{textcolor}%
\pgfsetfillcolor{textcolor}%
\pgftext[x=0.223446in, y=0.689708in, left, base]{\color{textcolor}{\rmfamily\fontsize{7.000000}{8.400000}\selectfont\catcode`\^=\active\def^{\ifmmode\sp\else\^{}\fi}\catcode`\%=\active\def
\end{pgfscope}%
\begin{pgfscope}%
\pgfpathrectangle{\pgfqpoint{0.519743in}{0.414757in}}{\pgfqpoint{0.983773in}{1.143573in}}%
\pgfusepath{clip}%
\pgfsetrectcap%
\pgfsetroundjoin%
\pgfsetlinewidth{0.803000pt}%
\definecolor{currentstroke}{rgb}{0.690196,0.690196,0.690196}%
\pgfsetstrokecolor{currentstroke}%
\pgfsetdash{}{0pt}%
\pgfpathmoveto{\pgfqpoint{0.519743in}{0.986544in}}%
\pgfpathlineto{\pgfqpoint{1.503516in}{0.986544in}}%
\pgfusepath{stroke}%
\end{pgfscope}%
\begin{pgfscope}%
\pgfsetbuttcap%
\pgfsetroundjoin%
\definecolor{currentfill}{rgb}{0.000000,0.000000,0.000000}%
\pgfsetfillcolor{currentfill}%
\pgfsetlinewidth{0.803000pt}%
\definecolor{currentstroke}{rgb}{0.000000,0.000000,0.000000}%
\pgfsetstrokecolor{currentstroke}%
\pgfsetdash{}{0pt}%
\pgfsys@defobject{currentmarker}{\pgfqpoint{-0.048611in}{0.000000in}}{\pgfqpoint{-0.000000in}{0.000000in}}{%
\pgfpathmoveto{\pgfqpoint{-0.000000in}{0.000000in}}%
\pgfpathlineto{\pgfqpoint{-0.048611in}{0.000000in}}%
\pgfusepath{stroke,fill}%
}%
\begin{pgfscope}%
\pgfsys@transformshift{0.519743in}{0.986544in}%
\pgfsys@useobject{currentmarker}{}%
\end{pgfscope}%
\end{pgfscope}%
\begin{pgfscope}%
\definecolor{textcolor}{rgb}{0.000000,0.000000,0.000000}%
\pgfsetstrokecolor{textcolor}%
\pgfsetfillcolor{textcolor}%
\pgftext[x=0.223446in, y=0.949611in, left, base]{\color{textcolor}{\rmfamily\fontsize{7.000000}{8.400000}\selectfont\catcode`\^=\active\def^{\ifmmode\sp\else\^{}\fi}\catcode`\%=\active\def
\end{pgfscope}%
\begin{pgfscope}%
\pgfpathrectangle{\pgfqpoint{0.519743in}{0.414757in}}{\pgfqpoint{0.983773in}{1.143573in}}%
\pgfusepath{clip}%
\pgfsetrectcap%
\pgfsetroundjoin%
\pgfsetlinewidth{0.803000pt}%
\definecolor{currentstroke}{rgb}{0.690196,0.690196,0.690196}%
\pgfsetstrokecolor{currentstroke}%
\pgfsetdash{}{0pt}%
\pgfpathmoveto{\pgfqpoint{0.519743in}{1.246447in}}%
\pgfpathlineto{\pgfqpoint{1.503516in}{1.246447in}}%
\pgfusepath{stroke}%
\end{pgfscope}%
\begin{pgfscope}%
\pgfsetbuttcap%
\pgfsetroundjoin%
\definecolor{currentfill}{rgb}{0.000000,0.000000,0.000000}%
\pgfsetfillcolor{currentfill}%
\pgfsetlinewidth{0.803000pt}%
\definecolor{currentstroke}{rgb}{0.000000,0.000000,0.000000}%
\pgfsetstrokecolor{currentstroke}%
\pgfsetdash{}{0pt}%
\pgfsys@defobject{currentmarker}{\pgfqpoint{-0.048611in}{0.000000in}}{\pgfqpoint{-0.000000in}{0.000000in}}{%
\pgfpathmoveto{\pgfqpoint{-0.000000in}{0.000000in}}%
\pgfpathlineto{\pgfqpoint{-0.048611in}{0.000000in}}%
\pgfusepath{stroke,fill}%
}%
\begin{pgfscope}%
\pgfsys@transformshift{0.519743in}{1.246447in}%
\pgfsys@useobject{currentmarker}{}%
\end{pgfscope}%
\end{pgfscope}%
\begin{pgfscope}%
\definecolor{textcolor}{rgb}{0.000000,0.000000,0.000000}%
\pgfsetstrokecolor{textcolor}%
\pgfsetfillcolor{textcolor}%
\pgftext[x=0.223446in, y=1.209513in, left, base]{\color{textcolor}{\rmfamily\fontsize{7.000000}{8.400000}\selectfont\catcode`\^=\active\def^{\ifmmode\sp\else\^{}\fi}\catcode`\%=\active\def
\end{pgfscope}%
\begin{pgfscope}%
\pgfpathrectangle{\pgfqpoint{0.519743in}{0.414757in}}{\pgfqpoint{0.983773in}{1.143573in}}%
\pgfusepath{clip}%
\pgfsetrectcap%
\pgfsetroundjoin%
\pgfsetlinewidth{0.803000pt}%
\definecolor{currentstroke}{rgb}{0.690196,0.690196,0.690196}%
\pgfsetstrokecolor{currentstroke}%
\pgfsetdash{}{0pt}%
\pgfpathmoveto{\pgfqpoint{0.519743in}{1.506349in}}%
\pgfpathlineto{\pgfqpoint{1.503516in}{1.506349in}}%
\pgfusepath{stroke}%
\end{pgfscope}%
\begin{pgfscope}%
\pgfsetbuttcap%
\pgfsetroundjoin%
\definecolor{currentfill}{rgb}{0.000000,0.000000,0.000000}%
\pgfsetfillcolor{currentfill}%
\pgfsetlinewidth{0.803000pt}%
\definecolor{currentstroke}{rgb}{0.000000,0.000000,0.000000}%
\pgfsetstrokecolor{currentstroke}%
\pgfsetdash{}{0pt}%
\pgfsys@defobject{currentmarker}{\pgfqpoint{-0.048611in}{0.000000in}}{\pgfqpoint{-0.000000in}{0.000000in}}{%
\pgfpathmoveto{\pgfqpoint{-0.000000in}{0.000000in}}%
\pgfpathlineto{\pgfqpoint{-0.048611in}{0.000000in}}%
\pgfusepath{stroke,fill}%
}%
\begin{pgfscope}%
\pgfsys@transformshift{0.519743in}{1.506349in}%
\pgfsys@useobject{currentmarker}{}%
\end{pgfscope}%
\end{pgfscope}%
\begin{pgfscope}%
\definecolor{textcolor}{rgb}{0.000000,0.000000,0.000000}%
\pgfsetstrokecolor{textcolor}%
\pgfsetfillcolor{textcolor}%
\pgftext[x=0.223446in, y=1.469416in, left, base]{\color{textcolor}{\rmfamily\fontsize{7.000000}{8.400000}\selectfont\catcode`\^=\active\def^{\ifmmode\sp\else\^{}\fi}\catcode`\%=\active\def
\end{pgfscope}%
\begin{pgfscope}%
\definecolor{textcolor}{rgb}{0.000000,0.000000,0.000000}%
\pgfsetstrokecolor{textcolor}%
\pgfsetfillcolor{textcolor}%
\pgftext[x=0.167891in,y=0.986544in,,bottom,rotate=90.000000]{\color{textcolor}{\rmfamily\fontsize{9.000000}{10.800000}\selectfont\catcode`\^=\active\def^{\ifmmode\sp\else\^{}\fi}\catcode`\%=\active\def
\end{pgfscope}%
\begin{pgfscope}%
\pgfpathrectangle{\pgfqpoint{0.519743in}{0.414757in}}{\pgfqpoint{0.983773in}{1.143573in}}%
\pgfusepath{clip}%
\pgfsetrectcap%
\pgfsetroundjoin%
\pgfsetlinewidth{1.505625pt}%
\definecolor{currentstroke}{rgb}{0.003922,0.450980,0.698039}%
\pgfsetstrokecolor{currentstroke}%
\pgfsetstrokeopacity{0.200000}%
\pgfsetdash{}{0pt}%
\pgfpathmoveto{\pgfqpoint{0.645510in}{0.467934in}}%
\pgfpathlineto{\pgfqpoint{0.698611in}{0.470338in}}%
\pgfpathlineto{\pgfqpoint{0.743328in}{0.473732in}}%
\pgfpathlineto{\pgfqpoint{0.788045in}{0.478418in}}%
\pgfpathlineto{\pgfqpoint{0.832762in}{0.484454in}}%
\pgfpathlineto{\pgfqpoint{0.877479in}{0.491869in}}%
\pgfpathlineto{\pgfqpoint{0.922196in}{0.500666in}}%
\pgfpathlineto{\pgfqpoint{0.966913in}{0.510827in}}%
\pgfpathlineto{\pgfqpoint{1.011630in}{0.522312in}}%
\pgfpathlineto{\pgfqpoint{1.056347in}{0.535060in}}%
\pgfpathlineto{\pgfqpoint{1.101064in}{0.548994in}}%
\pgfpathlineto{\pgfqpoint{1.145781in}{0.564017in}}%
\pgfpathlineto{\pgfqpoint{1.190498in}{0.580017in}}%
\pgfpathlineto{\pgfqpoint{1.235215in}{0.596862in}}%
\pgfpathlineto{\pgfqpoint{1.279932in}{0.614397in}}%
\pgfpathlineto{\pgfqpoint{1.324649in}{0.632424in}}%
\pgfpathlineto{\pgfqpoint{1.369365in}{0.650642in}}%
\pgfpathlineto{\pgfqpoint{1.414082in}{0.668362in}}%
\pgfpathlineto{\pgfqpoint{1.458799in}{0.619595in}}%
\pgfusepath{stroke}%
\end{pgfscope}%
\begin{pgfscope}%
\pgfpathrectangle{\pgfqpoint{0.519743in}{0.414757in}}{\pgfqpoint{0.983773in}{1.143573in}}%
\pgfusepath{clip}%
\pgfsetrectcap%
\pgfsetroundjoin%
\pgfsetlinewidth{1.505625pt}%
\definecolor{currentstroke}{rgb}{0.003922,0.450980,0.698039}%
\pgfsetstrokecolor{currentstroke}%
\pgfsetstrokeopacity{0.200000}%
\pgfsetdash{}{0pt}%
\pgfpathmoveto{\pgfqpoint{0.647506in}{0.467776in}}%
\pgfpathlineto{\pgfqpoint{0.698611in}{0.469801in}}%
\pgfpathlineto{\pgfqpoint{0.743328in}{0.472778in}}%
\pgfpathlineto{\pgfqpoint{0.788045in}{0.476944in}}%
\pgfpathlineto{\pgfqpoint{0.832762in}{0.482370in}}%
\pgfpathlineto{\pgfqpoint{0.877479in}{0.489098in}}%
\pgfpathlineto{\pgfqpoint{0.922196in}{0.497149in}}%
\pgfpathlineto{\pgfqpoint{0.966913in}{0.506519in}}%
\pgfpathlineto{\pgfqpoint{1.011630in}{0.517184in}}%
\pgfpathlineto{\pgfqpoint{1.056347in}{0.529099in}}%
\pgfpathlineto{\pgfqpoint{1.101064in}{0.542201in}}%
\pgfpathlineto{\pgfqpoint{1.145781in}{0.556405in}}%
\pgfpathlineto{\pgfqpoint{1.190498in}{0.571611in}}%
\pgfpathlineto{\pgfqpoint{1.235215in}{0.587696in}}%
\pgfpathlineto{\pgfqpoint{1.279932in}{0.604508in}}%
\pgfpathlineto{\pgfqpoint{1.324649in}{0.621849in}}%
\pgfpathlineto{\pgfqpoint{1.369365in}{0.639406in}}%
\pgfpathlineto{\pgfqpoint{1.414082in}{0.656439in}}%
\pgfpathlineto{\pgfqpoint{1.458799in}{0.602845in}}%
\pgfusepath{stroke}%
\end{pgfscope}%
\begin{pgfscope}%
\pgfpathrectangle{\pgfqpoint{0.519743in}{0.414757in}}{\pgfqpoint{0.983773in}{1.143573in}}%
\pgfusepath{clip}%
\pgfsetrectcap%
\pgfsetroundjoin%
\pgfsetlinewidth{1.505625pt}%
\definecolor{currentstroke}{rgb}{0.003922,0.450980,0.698039}%
\pgfsetstrokecolor{currentstroke}%
\pgfsetstrokeopacity{0.200000}%
\pgfsetdash{}{0pt}%
\pgfpathmoveto{\pgfqpoint{0.640135in}{0.467824in}}%
\pgfpathlineto{\pgfqpoint{0.698611in}{0.470397in}}%
\pgfpathlineto{\pgfqpoint{0.743328in}{0.473816in}}%
\pgfpathlineto{\pgfqpoint{0.788045in}{0.478517in}}%
\pgfpathlineto{\pgfqpoint{0.832762in}{0.484554in}}%
\pgfpathlineto{\pgfqpoint{0.877479in}{0.491952in}}%
\pgfpathlineto{\pgfqpoint{0.922196in}{0.500712in}}%
\pgfpathlineto{\pgfqpoint{0.966913in}{0.510812in}}%
\pgfpathlineto{\pgfqpoint{1.011630in}{0.522212in}}%
\pgfpathlineto{\pgfqpoint{1.056347in}{0.534851in}}%
\pgfpathlineto{\pgfqpoint{1.101064in}{0.548653in}}%
\pgfpathlineto{\pgfqpoint{1.145781in}{0.563523in}}%
\pgfpathlineto{\pgfqpoint{1.190498in}{0.579352in}}%
\pgfpathlineto{\pgfqpoint{1.235215in}{0.596013in}}%
\pgfpathlineto{\pgfqpoint{1.279932in}{0.613357in}}%
\pgfpathlineto{\pgfqpoint{1.324649in}{0.631193in}}%
\pgfpathlineto{\pgfqpoint{1.369365in}{0.649238in}}%
\pgfpathlineto{\pgfqpoint{1.414082in}{0.666844in}}%
\pgfpathlineto{\pgfqpoint{1.458799in}{0.621915in}}%
\pgfusepath{stroke}%
\end{pgfscope}%
\begin{pgfscope}%
\pgfpathrectangle{\pgfqpoint{0.519743in}{0.414757in}}{\pgfqpoint{0.983773in}{1.143573in}}%
\pgfusepath{clip}%
\pgfsetrectcap%
\pgfsetroundjoin%
\pgfsetlinewidth{1.505625pt}%
\definecolor{currentstroke}{rgb}{0.003922,0.450980,0.698039}%
\pgfsetstrokecolor{currentstroke}%
\pgfsetstrokeopacity{0.200000}%
\pgfsetdash{}{0pt}%
\pgfpathmoveto{\pgfqpoint{0.647506in}{0.467858in}}%
\pgfpathlineto{\pgfqpoint{0.698611in}{0.470002in}}%
\pgfpathlineto{\pgfqpoint{0.743328in}{0.473118in}}%
\pgfpathlineto{\pgfqpoint{0.788045in}{0.477447in}}%
\pgfpathlineto{\pgfqpoint{0.832762in}{0.483051in}}%
\pgfpathlineto{\pgfqpoint{0.877479in}{0.489969in}}%
\pgfpathlineto{\pgfqpoint{0.922196in}{0.498214in}}%
\pgfpathlineto{\pgfqpoint{0.966913in}{0.507781in}}%
\pgfpathlineto{\pgfqpoint{1.011630in}{0.518640in}}%
\pgfpathlineto{\pgfqpoint{1.056347in}{0.530748in}}%
\pgfpathlineto{\pgfqpoint{1.101064in}{0.544041in}}%
\pgfpathlineto{\pgfqpoint{1.145781in}{0.558440in}}%
\pgfpathlineto{\pgfqpoint{1.190498in}{0.573850in}}%
\pgfpathlineto{\pgfqpoint{1.235215in}{0.590160in}}%
\pgfpathlineto{\pgfqpoint{1.279932in}{0.607239in}}%
\pgfpathlineto{\pgfqpoint{1.324649in}{0.624928in}}%
\pgfpathlineto{\pgfqpoint{1.369365in}{0.642993in}}%
\pgfpathlineto{\pgfqpoint{1.414082in}{0.660943in}}%
\pgfpathlineto{\pgfqpoint{1.458799in}{0.632821in}}%
\pgfusepath{stroke}%
\end{pgfscope}%
\begin{pgfscope}%
\pgfpathrectangle{\pgfqpoint{0.519743in}{0.414757in}}{\pgfqpoint{0.983773in}{1.143573in}}%
\pgfusepath{clip}%
\pgfsetrectcap%
\pgfsetroundjoin%
\pgfsetlinewidth{1.505625pt}%
\definecolor{currentstroke}{rgb}{0.003922,0.450980,0.698039}%
\pgfsetstrokecolor{currentstroke}%
\pgfsetstrokeopacity{0.200000}%
\pgfsetdash{}{0pt}%
\pgfpathmoveto{\pgfqpoint{0.637125in}{0.467699in}}%
\pgfpathlineto{\pgfqpoint{0.698611in}{0.470243in}}%
\pgfpathlineto{\pgfqpoint{0.743328in}{0.473541in}}%
\pgfpathlineto{\pgfqpoint{0.788045in}{0.478089in}}%
\pgfpathlineto{\pgfqpoint{0.832762in}{0.483944in}}%
\pgfpathlineto{\pgfqpoint{0.877479in}{0.491133in}}%
\pgfpathlineto{\pgfqpoint{0.922196in}{0.499662in}}%
\pgfpathlineto{\pgfqpoint{0.966913in}{0.509513in}}%
\pgfpathlineto{\pgfqpoint{1.011630in}{0.520649in}}%
\pgfpathlineto{\pgfqpoint{1.056347in}{0.533014in}}%
\pgfpathlineto{\pgfqpoint{1.101064in}{0.546534in}}%
\pgfpathlineto{\pgfqpoint{1.145781in}{0.561118in}}%
\pgfpathlineto{\pgfqpoint{1.190498in}{0.576660in}}%
\pgfpathlineto{\pgfqpoint{1.235215in}{0.593033in}}%
\pgfpathlineto{\pgfqpoint{1.279932in}{0.610089in}}%
\pgfpathlineto{\pgfqpoint{1.324649in}{0.627635in}}%
\pgfpathlineto{\pgfqpoint{1.369365in}{0.645375in}}%
\pgfpathlineto{\pgfqpoint{1.414082in}{0.662626in}}%
\pgfpathlineto{\pgfqpoint{1.458799in}{0.613831in}}%
\pgfusepath{stroke}%
\end{pgfscope}%
\begin{pgfscope}%
\pgfpathrectangle{\pgfqpoint{0.519743in}{0.414757in}}{\pgfqpoint{0.983773in}{1.143573in}}%
\pgfusepath{clip}%
\pgfsetbuttcap%
\pgfsetroundjoin%
\pgfsetlinewidth{1.505625pt}%
\definecolor{currentstroke}{rgb}{0.501961,0.501961,0.501961}%
\pgfsetstrokecolor{currentstroke}%
\pgfsetdash{{5.550000pt}{2.400000pt}}{0.000000pt}%
\pgfpathmoveto{\pgfqpoint{0.564460in}{0.726641in}}%
\pgfpathlineto{\pgfqpoint{1.458799in}{0.726641in}}%
\pgfusepath{stroke}%
\end{pgfscope}%
\begin{pgfscope}%
\pgfpathrectangle{\pgfqpoint{0.519743in}{0.414757in}}{\pgfqpoint{0.983773in}{1.143573in}}%
\pgfusepath{clip}%
\pgfsetbuttcap%
\pgfsetroundjoin%
\pgfsetlinewidth{1.505625pt}%
\definecolor{currentstroke}{rgb}{0.501961,0.501961,0.501961}%
\pgfsetstrokecolor{currentstroke}%
\pgfsetdash{{5.550000pt}{2.400000pt}}{0.000000pt}%
\pgfpathmoveto{\pgfqpoint{0.788045in}{0.466738in}}%
\pgfpathlineto{\pgfqpoint{0.788045in}{1.506349in}}%
\pgfusepath{stroke}%
\end{pgfscope}%
\begin{pgfscope}%
\pgfsetrectcap%
\pgfsetmiterjoin%
\pgfsetlinewidth{0.803000pt}%
\definecolor{currentstroke}{rgb}{0.000000,0.000000,0.000000}%
\pgfsetstrokecolor{currentstroke}%
\pgfsetdash{}{0pt}%
\pgfpathmoveto{\pgfqpoint{0.519743in}{0.414757in}}%
\pgfpathlineto{\pgfqpoint{0.519743in}{1.558330in}}%
\pgfusepath{stroke}%
\end{pgfscope}%
\begin{pgfscope}%
\pgfsetrectcap%
\pgfsetmiterjoin%
\pgfsetlinewidth{0.803000pt}%
\definecolor{currentstroke}{rgb}{0.000000,0.000000,0.000000}%
\pgfsetstrokecolor{currentstroke}%
\pgfsetdash{}{0pt}%
\pgfpathmoveto{\pgfqpoint{1.503516in}{0.414757in}}%
\pgfpathlineto{\pgfqpoint{1.503516in}{1.558330in}}%
\pgfusepath{stroke}%
\end{pgfscope}%
\begin{pgfscope}%
\pgfsetrectcap%
\pgfsetmiterjoin%
\pgfsetlinewidth{0.803000pt}%
\definecolor{currentstroke}{rgb}{0.000000,0.000000,0.000000}%
\pgfsetstrokecolor{currentstroke}%
\pgfsetdash{}{0pt}%
\pgfpathmoveto{\pgfqpoint{0.519743in}{0.414757in}}%
\pgfpathlineto{\pgfqpoint{1.503516in}{0.414757in}}%
\pgfusepath{stroke}%
\end{pgfscope}%
\begin{pgfscope}%
\pgfsetrectcap%
\pgfsetmiterjoin%
\pgfsetlinewidth{0.803000pt}%
\definecolor{currentstroke}{rgb}{0.000000,0.000000,0.000000}%
\pgfsetstrokecolor{currentstroke}%
\pgfsetdash{}{0pt}%
\pgfpathmoveto{\pgfqpoint{0.519743in}{1.558330in}}%
\pgfpathlineto{\pgfqpoint{1.503516in}{1.558330in}}%
\pgfusepath{stroke}%
\end{pgfscope}%
\begin{pgfscope}%
\pgfsetbuttcap%
\pgfsetmiterjoin%
\definecolor{currentfill}{rgb}{1.000000,1.000000,1.000000}%
\pgfsetfillcolor{currentfill}%
\pgfsetfillopacity{0.800000}%
\pgfsetlinewidth{1.003750pt}%
\definecolor{currentstroke}{rgb}{0.800000,0.800000,0.800000}%
\pgfsetstrokecolor{currentstroke}%
\pgfsetstrokeopacity{0.800000}%
\pgfsetdash{}{0pt}%
\pgfpathmoveto{\pgfqpoint{0.542878in}{1.008611in}}%
\pgfpathlineto{\pgfqpoint{1.537122in}{1.008611in}}%
\pgfpathquadraticcurveto{\pgfqpoint{1.556567in}{1.008611in}}{\pgfqpoint{1.556567in}{1.028056in}}%
\pgfpathlineto{\pgfqpoint{1.556567in}{1.446433in}}%
\pgfpathquadraticcurveto{\pgfqpoint{1.556567in}{1.465878in}}{\pgfqpoint{1.537122in}{1.465878in}}%
\pgfpathlineto{\pgfqpoint{0.542878in}{1.465878in}}%
\pgfpathquadraticcurveto{\pgfqpoint{0.523433in}{1.465878in}}{\pgfqpoint{0.523433in}{1.446433in}}%
\pgfpathlineto{\pgfqpoint{0.523433in}{1.028056in}}%
\pgfpathquadraticcurveto{\pgfqpoint{0.523433in}{1.008611in}}{\pgfqpoint{0.542878in}{1.008611in}}%
\pgfpathlineto{\pgfqpoint{0.542878in}{1.008611in}}%
\pgfpathclose%
\pgfusepath{stroke,fill}%
\end{pgfscope}%
\begin{pgfscope}%
\pgfsetbuttcap%
\pgfsetmiterjoin%
\definecolor{currentfill}{rgb}{0.007843,0.619608,0.447059}%
\pgfsetfillcolor{currentfill}%
\pgfsetfillopacity{0.850000}%
\pgfsetlinewidth{0.501875pt}%
\definecolor{currentstroke}{rgb}{0.000000,0.000000,0.000000}%
\pgfsetstrokecolor{currentstroke}%
\pgfsetstrokeopacity{0.850000}%
\pgfsetdash{}{0pt}%
\pgfpathmoveto{\pgfqpoint{0.562322in}{1.353123in}}%
\pgfpathlineto{\pgfqpoint{0.756766in}{1.353123in}}%
\pgfpathlineto{\pgfqpoint{0.756766in}{1.421178in}}%
\pgfpathlineto{\pgfqpoint{0.562322in}{1.421178in}}%
\pgfpathlineto{\pgfqpoint{0.562322in}{1.353123in}}%
\pgfpathclose%
\pgfusepath{stroke,fill}%
\end{pgfscope}%
\begin{pgfscope}%
\definecolor{textcolor}{rgb}{0.000000,0.000000,0.000000}%
\pgfsetstrokecolor{textcolor}%
\pgfsetfillcolor{textcolor}%
\pgftext[x=0.834544in,y=1.353123in,left,base]{\color{textcolor}{\rmfamily\fontsize{7.000000}{8.400000}\selectfont\catcode`\^=\active\def^{\ifmmode\sp\else\^{}\fi}\catcode`\%=\active\def
\end{pgfscope}%
\begin{pgfscope}%
\pgfsetbuttcap%
\pgfsetmiterjoin%
\definecolor{currentfill}{rgb}{0.870588,0.560784,0.011765}%
\pgfsetfillcolor{currentfill}%
\pgfsetfillopacity{0.850000}%
\pgfsetlinewidth{0.501875pt}%
\definecolor{currentstroke}{rgb}{0.000000,0.000000,0.000000}%
\pgfsetstrokecolor{currentstroke}%
\pgfsetstrokeopacity{0.850000}%
\pgfsetdash{}{0pt}%
\pgfpathmoveto{\pgfqpoint{0.562322in}{1.210423in}}%
\pgfpathlineto{\pgfqpoint{0.756766in}{1.210423in}}%
\pgfpathlineto{\pgfqpoint{0.756766in}{1.278478in}}%
\pgfpathlineto{\pgfqpoint{0.562322in}{1.278478in}}%
\pgfpathlineto{\pgfqpoint{0.562322in}{1.210423in}}%
\pgfpathclose%
\pgfusepath{stroke,fill}%
\end{pgfscope}%
\begin{pgfscope}%
\definecolor{textcolor}{rgb}{0.000000,0.000000,0.000000}%
\pgfsetstrokecolor{textcolor}%
\pgfsetfillcolor{textcolor}%
\pgftext[x=0.834544in,y=1.210423in,left,base]{\color{textcolor}{\rmfamily\fontsize{7.000000}{8.400000}\selectfont\catcode`\^=\active\def^{\ifmmode\sp\else\^{}\fi}\catcode`\%=\active\def
\end{pgfscope}%
\begin{pgfscope}%
\pgfsetbuttcap%
\pgfsetroundjoin%
\pgfsetlinewidth{0.803000pt}%
\definecolor{currentstroke}{rgb}{0.000000,0.000000,0.000000}%
\pgfsetstrokecolor{currentstroke}%
\pgfsetstrokeopacity{0.600000}%
\pgfsetdash{{2.960000pt}{1.280000pt}}{0.000000pt}%
\pgfpathmoveto{\pgfqpoint{0.562322in}{1.101751in}}%
\pgfpathlineto{\pgfqpoint{0.659544in}{1.101751in}}%
\pgfpathlineto{\pgfqpoint{0.756766in}{1.101751in}}%
\pgfusepath{stroke}%
\end{pgfscope}%
\begin{pgfscope}%
\definecolor{textcolor}{rgb}{0.000000,0.000000,0.000000}%
\pgfsetstrokecolor{textcolor}%
\pgfsetfillcolor{textcolor}%
\pgftext[x=0.834544in,y=1.067723in,left,base]{\color{textcolor}{\rmfamily\fontsize{7.000000}{8.400000}\selectfont\catcode`\^=\active\def^{\ifmmode\sp\else\^{}\fi}\catcode`\%=\active\def
\end{pgfscope}%
\end{pgfpicture}%
\makeatother%
\endgroup%

%% file: figures/calibration_maps/SIMPLE_QA_VERIFIED_gemma-3-12b-it.pgf
\begingroup%
\makeatletter%
\begin{pgfpicture}%
\pgfpathrectangle{\pgfpointorigin}{\pgfqpoint{1.600000in}{1.600000in}}%
\pgfusepath{use as bounding box, clip}%
\begin{pgfscope}%
\pgfsetbuttcap%
\pgfsetmiterjoin%
\definecolor{currentfill}{rgb}{1.000000,1.000000,1.000000}%
\pgfsetfillcolor{currentfill}%
\pgfsetlinewidth{0.000000pt}%
\definecolor{currentstroke}{rgb}{1.000000,1.000000,1.000000}%
\pgfsetstrokecolor{currentstroke}%
\pgfsetdash{}{0pt}%
\pgfpathmoveto{\pgfqpoint{0.000000in}{0.000000in}}%
\pgfpathlineto{\pgfqpoint{1.600000in}{0.000000in}}%
\pgfpathlineto{\pgfqpoint{1.600000in}{1.600000in}}%
\pgfpathlineto{\pgfqpoint{0.000000in}{1.600000in}}%
\pgfpathlineto{\pgfqpoint{0.000000in}{0.000000in}}%
\pgfpathclose%
\pgfusepath{fill}%
\end{pgfscope}%
\begin{pgfscope}%
\pgfsetbuttcap%
\pgfsetmiterjoin%
\definecolor{currentfill}{rgb}{1.000000,1.000000,1.000000}%
\pgfsetfillcolor{currentfill}%
\pgfsetlinewidth{0.000000pt}%
\definecolor{currentstroke}{rgb}{0.000000,0.000000,0.000000}%
\pgfsetstrokecolor{currentstroke}%
\pgfsetstrokeopacity{0.000000}%
\pgfsetdash{}{0pt}%
\pgfpathmoveto{\pgfqpoint{0.519743in}{0.414757in}}%
\pgfpathlineto{\pgfqpoint{1.503516in}{0.414757in}}%
\pgfpathlineto{\pgfqpoint{1.503516in}{1.558330in}}%
\pgfpathlineto{\pgfqpoint{0.519743in}{1.558330in}}%
\pgfpathlineto{\pgfqpoint{0.519743in}{0.414757in}}%
\pgfpathclose%
\pgfusepath{fill}%
\end{pgfscope}%
\begin{pgfscope}%
\pgfpathrectangle{\pgfqpoint{0.519743in}{0.414757in}}{\pgfqpoint{0.983773in}{1.143573in}}%
\pgfusepath{clip}%
\pgfsetbuttcap%
\pgfsetroundjoin%
\definecolor{currentfill}{rgb}{0.870588,0.560784,0.011765}%
\pgfsetfillcolor{currentfill}%
\pgfsetfillopacity{0.200000}%
\pgfsetlinewidth{1.003750pt}%
\definecolor{currentstroke}{rgb}{0.870588,0.560784,0.011765}%
\pgfsetstrokecolor{currentstroke}%
\pgfsetstrokeopacity{0.200000}%
\pgfsetdash{}{0pt}%
\pgfsys@defobject{currentmarker}{\pgfqpoint{0.564460in}{0.726641in}}{\pgfqpoint{0.788045in}{1.506349in}}{%
\pgfpathmoveto{\pgfqpoint{0.788045in}{0.726641in}}%
\pgfpathlineto{\pgfqpoint{0.564460in}{0.726641in}}%
\pgfpathlineto{\pgfqpoint{0.564460in}{1.506349in}}%
\pgfpathlineto{\pgfqpoint{0.788045in}{1.506349in}}%
\pgfpathlineto{\pgfqpoint{0.788045in}{1.506349in}}%
\pgfpathlineto{\pgfqpoint{0.788045in}{0.726641in}}%
\pgfpathlineto{\pgfqpoint{0.788045in}{0.726641in}}%
\pgfpathclose%
\pgfusepath{stroke,fill}%
}%
\begin{pgfscope}%
\pgfsys@transformshift{0.000000in}{0.000000in}%
\pgfsys@useobject{currentmarker}{}%
\end{pgfscope}%
\end{pgfscope}%
\begin{pgfscope}%
\pgfpathrectangle{\pgfqpoint{0.519743in}{0.414757in}}{\pgfqpoint{0.983773in}{1.143573in}}%
\pgfusepath{clip}%
\pgfsetbuttcap%
\pgfsetroundjoin%
\definecolor{currentfill}{rgb}{0.007843,0.619608,0.447059}%
\pgfsetfillcolor{currentfill}%
\pgfsetfillopacity{0.200000}%
\pgfsetlinewidth{1.003750pt}%
\definecolor{currentstroke}{rgb}{0.007843,0.619608,0.447059}%
\pgfsetstrokecolor{currentstroke}%
\pgfsetstrokeopacity{0.200000}%
\pgfsetdash{}{0pt}%
\pgfsys@defobject{currentmarker}{\pgfqpoint{0.788045in}{0.726641in}}{\pgfqpoint{1.458799in}{1.506349in}}{%
\pgfpathmoveto{\pgfqpoint{1.458799in}{0.726641in}}%
\pgfpathlineto{\pgfqpoint{0.788045in}{0.726641in}}%
\pgfpathlineto{\pgfqpoint{0.788045in}{1.506349in}}%
\pgfpathlineto{\pgfqpoint{1.458799in}{1.506349in}}%
\pgfpathlineto{\pgfqpoint{1.458799in}{1.506349in}}%
\pgfpathlineto{\pgfqpoint{1.458799in}{0.726641in}}%
\pgfpathlineto{\pgfqpoint{1.458799in}{0.726641in}}%
\pgfpathclose%
\pgfusepath{stroke,fill}%
}%
\begin{pgfscope}%
\pgfsys@transformshift{0.000000in}{0.000000in}%
\pgfsys@useobject{currentmarker}{}%
\end{pgfscope}%
\end{pgfscope}%
\begin{pgfscope}%
\pgfpathrectangle{\pgfqpoint{0.519743in}{0.414757in}}{\pgfqpoint{0.983773in}{1.143573in}}%
\pgfusepath{clip}%
\pgfsetbuttcap%
\pgfsetroundjoin%
\definecolor{currentfill}{rgb}{0.870588,0.560784,0.011765}%
\pgfsetfillcolor{currentfill}%
\pgfsetfillopacity{0.200000}%
\pgfsetlinewidth{1.003750pt}%
\definecolor{currentstroke}{rgb}{0.870588,0.560784,0.011765}%
\pgfsetstrokecolor{currentstroke}%
\pgfsetstrokeopacity{0.200000}%
\pgfsetdash{}{0pt}%
\pgfsys@defobject{currentmarker}{\pgfqpoint{0.788045in}{0.466738in}}{\pgfqpoint{1.458799in}{0.726641in}}{%
\pgfpathmoveto{\pgfqpoint{1.458799in}{0.466738in}}%
\pgfpathlineto{\pgfqpoint{0.788045in}{0.466738in}}%
\pgfpathlineto{\pgfqpoint{0.788045in}{0.726641in}}%
\pgfpathlineto{\pgfqpoint{1.458799in}{0.726641in}}%
\pgfpathlineto{\pgfqpoint{1.458799in}{0.726641in}}%
\pgfpathlineto{\pgfqpoint{1.458799in}{0.466738in}}%
\pgfpathlineto{\pgfqpoint{1.458799in}{0.466738in}}%
\pgfpathclose%
\pgfusepath{stroke,fill}%
}%
\begin{pgfscope}%
\pgfsys@transformshift{0.000000in}{0.000000in}%
\pgfsys@useobject{currentmarker}{}%
\end{pgfscope}%
\end{pgfscope}%
\begin{pgfscope}%
\pgfpathrectangle{\pgfqpoint{0.519743in}{0.414757in}}{\pgfqpoint{0.983773in}{1.143573in}}%
\pgfusepath{clip}%
\pgfsetbuttcap%
\pgfsetroundjoin%
\definecolor{currentfill}{rgb}{0.007843,0.619608,0.447059}%
\pgfsetfillcolor{currentfill}%
\pgfsetfillopacity{0.200000}%
\pgfsetlinewidth{1.003750pt}%
\definecolor{currentstroke}{rgb}{0.007843,0.619608,0.447059}%
\pgfsetstrokecolor{currentstroke}%
\pgfsetstrokeopacity{0.200000}%
\pgfsetdash{}{0pt}%
\pgfsys@defobject{currentmarker}{\pgfqpoint{0.564460in}{0.466738in}}{\pgfqpoint{0.788045in}{0.726641in}}{%
\pgfpathmoveto{\pgfqpoint{0.788045in}{0.466738in}}%
\pgfpathlineto{\pgfqpoint{0.564460in}{0.466738in}}%
\pgfpathlineto{\pgfqpoint{0.564460in}{0.726641in}}%
\pgfpathlineto{\pgfqpoint{0.788045in}{0.726641in}}%
\pgfpathlineto{\pgfqpoint{0.788045in}{0.726641in}}%
\pgfpathlineto{\pgfqpoint{0.788045in}{0.466738in}}%
\pgfpathlineto{\pgfqpoint{0.788045in}{0.466738in}}%
\pgfpathclose%
\pgfusepath{stroke,fill}%
}%
\begin{pgfscope}%
\pgfsys@transformshift{0.000000in}{0.000000in}%
\pgfsys@useobject{currentmarker}{}%
\end{pgfscope}%
\end{pgfscope}%
\begin{pgfscope}%
\pgfpathrectangle{\pgfqpoint{0.519743in}{0.414757in}}{\pgfqpoint{0.983773in}{1.143573in}}%
\pgfusepath{clip}%
\pgfsetrectcap%
\pgfsetroundjoin%
\pgfsetlinewidth{0.803000pt}%
\definecolor{currentstroke}{rgb}{0.690196,0.690196,0.690196}%
\pgfsetstrokecolor{currentstroke}%
\pgfsetdash{}{0pt}%
\pgfpathmoveto{\pgfqpoint{0.564460in}{0.414757in}}%
\pgfpathlineto{\pgfqpoint{0.564460in}{1.558330in}}%
\pgfusepath{stroke}%
\end{pgfscope}%
\begin{pgfscope}%
\pgfsetbuttcap%
\pgfsetroundjoin%
\definecolor{currentfill}{rgb}{0.000000,0.000000,0.000000}%
\pgfsetfillcolor{currentfill}%
\pgfsetlinewidth{0.803000pt}%
\definecolor{currentstroke}{rgb}{0.000000,0.000000,0.000000}%
\pgfsetstrokecolor{currentstroke}%
\pgfsetdash{}{0pt}%
\pgfsys@defobject{currentmarker}{\pgfqpoint{0.000000in}{-0.048611in}}{\pgfqpoint{0.000000in}{0.000000in}}{%
\pgfpathmoveto{\pgfqpoint{0.000000in}{0.000000in}}%
\pgfpathlineto{\pgfqpoint{0.000000in}{-0.048611in}}%
\pgfusepath{stroke,fill}%
}%
\begin{pgfscope}%
\pgfsys@transformshift{0.564460in}{0.414757in}%
\pgfsys@useobject{currentmarker}{}%
\end{pgfscope}%
\end{pgfscope}%
\begin{pgfscope}%
\definecolor{textcolor}{rgb}{0.000000,0.000000,0.000000}%
\pgfsetstrokecolor{textcolor}%
\pgfsetfillcolor{textcolor}%
\pgftext[x=0.564460in,y=0.317535in,,top]{\color{textcolor}{\rmfamily\fontsize{7.000000}{8.400000}\selectfont\catcode`\^=\active\def^{\ifmmode\sp\else\^{}\fi}\catcode`\%=\active\def
\end{pgfscope}%
\begin{pgfscope}%
\pgfpathrectangle{\pgfqpoint{0.519743in}{0.414757in}}{\pgfqpoint{0.983773in}{1.143573in}}%
\pgfusepath{clip}%
\pgfsetrectcap%
\pgfsetroundjoin%
\pgfsetlinewidth{0.803000pt}%
\definecolor{currentstroke}{rgb}{0.690196,0.690196,0.690196}%
\pgfsetstrokecolor{currentstroke}%
\pgfsetdash{}{0pt}%
\pgfpathmoveto{\pgfqpoint{0.788045in}{0.414757in}}%
\pgfpathlineto{\pgfqpoint{0.788045in}{1.558330in}}%
\pgfusepath{stroke}%
\end{pgfscope}%
\begin{pgfscope}%
\pgfsetbuttcap%
\pgfsetroundjoin%
\definecolor{currentfill}{rgb}{0.000000,0.000000,0.000000}%
\pgfsetfillcolor{currentfill}%
\pgfsetlinewidth{0.803000pt}%
\definecolor{currentstroke}{rgb}{0.000000,0.000000,0.000000}%
\pgfsetstrokecolor{currentstroke}%
\pgfsetdash{}{0pt}%
\pgfsys@defobject{currentmarker}{\pgfqpoint{0.000000in}{-0.048611in}}{\pgfqpoint{0.000000in}{0.000000in}}{%
\pgfpathmoveto{\pgfqpoint{0.000000in}{0.000000in}}%
\pgfpathlineto{\pgfqpoint{0.000000in}{-0.048611in}}%
\pgfusepath{stroke,fill}%
}%
\begin{pgfscope}%
\pgfsys@transformshift{0.788045in}{0.414757in}%
\pgfsys@useobject{currentmarker}{}%
\end{pgfscope}%
\end{pgfscope}%
\begin{pgfscope}%
\definecolor{textcolor}{rgb}{0.000000,0.000000,0.000000}%
\pgfsetstrokecolor{textcolor}%
\pgfsetfillcolor{textcolor}%
\pgftext[x=0.788045in,y=0.317535in,,top]{\color{textcolor}{\rmfamily\fontsize{7.000000}{8.400000}\selectfont\catcode`\^=\active\def^{\ifmmode\sp\else\^{}\fi}\catcode`\%=\active\def
\end{pgfscope}%
\begin{pgfscope}%
\pgfpathrectangle{\pgfqpoint{0.519743in}{0.414757in}}{\pgfqpoint{0.983773in}{1.143573in}}%
\pgfusepath{clip}%
\pgfsetrectcap%
\pgfsetroundjoin%
\pgfsetlinewidth{0.803000pt}%
\definecolor{currentstroke}{rgb}{0.690196,0.690196,0.690196}%
\pgfsetstrokecolor{currentstroke}%
\pgfsetdash{}{0pt}%
\pgfpathmoveto{\pgfqpoint{1.011630in}{0.414757in}}%
\pgfpathlineto{\pgfqpoint{1.011630in}{1.558330in}}%
\pgfusepath{stroke}%
\end{pgfscope}%
\begin{pgfscope}%
\pgfsetbuttcap%
\pgfsetroundjoin%
\definecolor{currentfill}{rgb}{0.000000,0.000000,0.000000}%
\pgfsetfillcolor{currentfill}%
\pgfsetlinewidth{0.803000pt}%
\definecolor{currentstroke}{rgb}{0.000000,0.000000,0.000000}%
\pgfsetstrokecolor{currentstroke}%
\pgfsetdash{}{0pt}%
\pgfsys@defobject{currentmarker}{\pgfqpoint{0.000000in}{-0.048611in}}{\pgfqpoint{0.000000in}{0.000000in}}{%
\pgfpathmoveto{\pgfqpoint{0.000000in}{0.000000in}}%
\pgfpathlineto{\pgfqpoint{0.000000in}{-0.048611in}}%
\pgfusepath{stroke,fill}%
}%
\begin{pgfscope}%
\pgfsys@transformshift{1.011630in}{0.414757in}%
\pgfsys@useobject{currentmarker}{}%
\end{pgfscope}%
\end{pgfscope}%
\begin{pgfscope}%
\definecolor{textcolor}{rgb}{0.000000,0.000000,0.000000}%
\pgfsetstrokecolor{textcolor}%
\pgfsetfillcolor{textcolor}%
\pgftext[x=1.011630in,y=0.317535in,,top]{\color{textcolor}{\rmfamily\fontsize{7.000000}{8.400000}\selectfont\catcode`\^=\active\def^{\ifmmode\sp\else\^{}\fi}\catcode`\%=\active\def
\end{pgfscope}%
\begin{pgfscope}%
\pgfpathrectangle{\pgfqpoint{0.519743in}{0.414757in}}{\pgfqpoint{0.983773in}{1.143573in}}%
\pgfusepath{clip}%
\pgfsetrectcap%
\pgfsetroundjoin%
\pgfsetlinewidth{0.803000pt}%
\definecolor{currentstroke}{rgb}{0.690196,0.690196,0.690196}%
\pgfsetstrokecolor{currentstroke}%
\pgfsetdash{}{0pt}%
\pgfpathmoveto{\pgfqpoint{1.235215in}{0.414757in}}%
\pgfpathlineto{\pgfqpoint{1.235215in}{1.558330in}}%
\pgfusepath{stroke}%
\end{pgfscope}%
\begin{pgfscope}%
\pgfsetbuttcap%
\pgfsetroundjoin%
\definecolor{currentfill}{rgb}{0.000000,0.000000,0.000000}%
\pgfsetfillcolor{currentfill}%
\pgfsetlinewidth{0.803000pt}%
\definecolor{currentstroke}{rgb}{0.000000,0.000000,0.000000}%
\pgfsetstrokecolor{currentstroke}%
\pgfsetdash{}{0pt}%
\pgfsys@defobject{currentmarker}{\pgfqpoint{0.000000in}{-0.048611in}}{\pgfqpoint{0.000000in}{0.000000in}}{%
\pgfpathmoveto{\pgfqpoint{0.000000in}{0.000000in}}%
\pgfpathlineto{\pgfqpoint{0.000000in}{-0.048611in}}%
\pgfusepath{stroke,fill}%
}%
\begin{pgfscope}%
\pgfsys@transformshift{1.235215in}{0.414757in}%
\pgfsys@useobject{currentmarker}{}%
\end{pgfscope}%
\end{pgfscope}%
\begin{pgfscope}%
\definecolor{textcolor}{rgb}{0.000000,0.000000,0.000000}%
\pgfsetstrokecolor{textcolor}%
\pgfsetfillcolor{textcolor}%
\pgftext[x=1.235215in,y=0.317535in,,top]{\color{textcolor}{\rmfamily\fontsize{7.000000}{8.400000}\selectfont\catcode`\^=\active\def^{\ifmmode\sp\else\^{}\fi}\catcode`\%=\active\def
\end{pgfscope}%
\begin{pgfscope}%
\pgfpathrectangle{\pgfqpoint{0.519743in}{0.414757in}}{\pgfqpoint{0.983773in}{1.143573in}}%
\pgfusepath{clip}%
\pgfsetrectcap%
\pgfsetroundjoin%
\pgfsetlinewidth{0.803000pt}%
\definecolor{currentstroke}{rgb}{0.690196,0.690196,0.690196}%
\pgfsetstrokecolor{currentstroke}%
\pgfsetdash{}{0pt}%
\pgfpathmoveto{\pgfqpoint{1.458799in}{0.414757in}}%
\pgfpathlineto{\pgfqpoint{1.458799in}{1.558330in}}%
\pgfusepath{stroke}%
\end{pgfscope}%
\begin{pgfscope}%
\pgfsetbuttcap%
\pgfsetroundjoin%
\definecolor{currentfill}{rgb}{0.000000,0.000000,0.000000}%
\pgfsetfillcolor{currentfill}%
\pgfsetlinewidth{0.803000pt}%
\definecolor{currentstroke}{rgb}{0.000000,0.000000,0.000000}%
\pgfsetstrokecolor{currentstroke}%
\pgfsetdash{}{0pt}%
\pgfsys@defobject{currentmarker}{\pgfqpoint{0.000000in}{-0.048611in}}{\pgfqpoint{0.000000in}{0.000000in}}{%
\pgfpathmoveto{\pgfqpoint{0.000000in}{0.000000in}}%
\pgfpathlineto{\pgfqpoint{0.000000in}{-0.048611in}}%
\pgfusepath{stroke,fill}%
}%
\begin{pgfscope}%
\pgfsys@transformshift{1.458799in}{0.414757in}%
\pgfsys@useobject{currentmarker}{}%
\end{pgfscope}%
\end{pgfscope}%
\begin{pgfscope}%
\definecolor{textcolor}{rgb}{0.000000,0.000000,0.000000}%
\pgfsetstrokecolor{textcolor}%
\pgfsetfillcolor{textcolor}%
\pgftext[x=1.458799in,y=0.317535in,,top]{\color{textcolor}{\rmfamily\fontsize{7.000000}{8.400000}\selectfont\catcode`\^=\active\def^{\ifmmode\sp\else\^{}\fi}\catcode`\%=\active\def
\end{pgfscope}%
\begin{pgfscope}%
\definecolor{textcolor}{rgb}{0.000000,0.000000,0.000000}%
\pgfsetstrokecolor{textcolor}%
\pgfsetfillcolor{textcolor}%
\pgftext[x=1.011630in,y=0.167891in,,top]{\color{textcolor}{\rmfamily\fontsize{9.000000}{10.800000}\selectfont\catcode`\^=\active\def^{\ifmmode\sp\else\^{}\fi}\catcode`\%=\active\def
\end{pgfscope}%
\begin{pgfscope}%
\pgfpathrectangle{\pgfqpoint{0.519743in}{0.414757in}}{\pgfqpoint{0.983773in}{1.143573in}}%
\pgfusepath{clip}%
\pgfsetrectcap%
\pgfsetroundjoin%
\pgfsetlinewidth{0.803000pt}%
\definecolor{currentstroke}{rgb}{0.690196,0.690196,0.690196}%
\pgfsetstrokecolor{currentstroke}%
\pgfsetdash{}{0pt}%
\pgfpathmoveto{\pgfqpoint{0.519743in}{0.466738in}}%
\pgfpathlineto{\pgfqpoint{1.503516in}{0.466738in}}%
\pgfusepath{stroke}%
\end{pgfscope}%
\begin{pgfscope}%
\pgfsetbuttcap%
\pgfsetroundjoin%
\definecolor{currentfill}{rgb}{0.000000,0.000000,0.000000}%
\pgfsetfillcolor{currentfill}%
\pgfsetlinewidth{0.803000pt}%
\definecolor{currentstroke}{rgb}{0.000000,0.000000,0.000000}%
\pgfsetstrokecolor{currentstroke}%
\pgfsetdash{}{0pt}%
\pgfsys@defobject{currentmarker}{\pgfqpoint{-0.048611in}{0.000000in}}{\pgfqpoint{-0.000000in}{0.000000in}}{%
\pgfpathmoveto{\pgfqpoint{-0.000000in}{0.000000in}}%
\pgfpathlineto{\pgfqpoint{-0.048611in}{0.000000in}}%
\pgfusepath{stroke,fill}%
}%
\begin{pgfscope}%
\pgfsys@transformshift{0.519743in}{0.466738in}%
\pgfsys@useobject{currentmarker}{}%
\end{pgfscope}%
\end{pgfscope}%
\begin{pgfscope}%
\definecolor{textcolor}{rgb}{0.000000,0.000000,0.000000}%
\pgfsetstrokecolor{textcolor}%
\pgfsetfillcolor{textcolor}%
\pgftext[x=0.223446in, y=0.429805in, left, base]{\color{textcolor}{\rmfamily\fontsize{7.000000}{8.400000}\selectfont\catcode`\^=\active\def^{\ifmmode\sp\else\^{}\fi}\catcode`\%=\active\def
\end{pgfscope}%
\begin{pgfscope}%
\pgfpathrectangle{\pgfqpoint{0.519743in}{0.414757in}}{\pgfqpoint{0.983773in}{1.143573in}}%
\pgfusepath{clip}%
\pgfsetrectcap%
\pgfsetroundjoin%
\pgfsetlinewidth{0.803000pt}%
\definecolor{currentstroke}{rgb}{0.690196,0.690196,0.690196}%
\pgfsetstrokecolor{currentstroke}%
\pgfsetdash{}{0pt}%
\pgfpathmoveto{\pgfqpoint{0.519743in}{0.726641in}}%
\pgfpathlineto{\pgfqpoint{1.503516in}{0.726641in}}%
\pgfusepath{stroke}%
\end{pgfscope}%
\begin{pgfscope}%
\pgfsetbuttcap%
\pgfsetroundjoin%
\definecolor{currentfill}{rgb}{0.000000,0.000000,0.000000}%
\pgfsetfillcolor{currentfill}%
\pgfsetlinewidth{0.803000pt}%
\definecolor{currentstroke}{rgb}{0.000000,0.000000,0.000000}%
\pgfsetstrokecolor{currentstroke}%
\pgfsetdash{}{0pt}%
\pgfsys@defobject{currentmarker}{\pgfqpoint{-0.048611in}{0.000000in}}{\pgfqpoint{-0.000000in}{0.000000in}}{%
\pgfpathmoveto{\pgfqpoint{-0.000000in}{0.000000in}}%
\pgfpathlineto{\pgfqpoint{-0.048611in}{0.000000in}}%
\pgfusepath{stroke,fill}%
}%
\begin{pgfscope}%
\pgfsys@transformshift{0.519743in}{0.726641in}%
\pgfsys@useobject{currentmarker}{}%
\end{pgfscope}%
\end{pgfscope}%
\begin{pgfscope}%
\definecolor{textcolor}{rgb}{0.000000,0.000000,0.000000}%
\pgfsetstrokecolor{textcolor}%
\pgfsetfillcolor{textcolor}%
\pgftext[x=0.223446in, y=0.689708in, left, base]{\color{textcolor}{\rmfamily\fontsize{7.000000}{8.400000}\selectfont\catcode`\^=\active\def^{\ifmmode\sp\else\^{}\fi}\catcode`\%=\active\def
\end{pgfscope}%
\begin{pgfscope}%
\pgfpathrectangle{\pgfqpoint{0.519743in}{0.414757in}}{\pgfqpoint{0.983773in}{1.143573in}}%
\pgfusepath{clip}%
\pgfsetrectcap%
\pgfsetroundjoin%
\pgfsetlinewidth{0.803000pt}%
\definecolor{currentstroke}{rgb}{0.690196,0.690196,0.690196}%
\pgfsetstrokecolor{currentstroke}%
\pgfsetdash{}{0pt}%
\pgfpathmoveto{\pgfqpoint{0.519743in}{0.986544in}}%
\pgfpathlineto{\pgfqpoint{1.503516in}{0.986544in}}%
\pgfusepath{stroke}%
\end{pgfscope}%
\begin{pgfscope}%
\pgfsetbuttcap%
\pgfsetroundjoin%
\definecolor{currentfill}{rgb}{0.000000,0.000000,0.000000}%
\pgfsetfillcolor{currentfill}%
\pgfsetlinewidth{0.803000pt}%
\definecolor{currentstroke}{rgb}{0.000000,0.000000,0.000000}%
\pgfsetstrokecolor{currentstroke}%
\pgfsetdash{}{0pt}%
\pgfsys@defobject{currentmarker}{\pgfqpoint{-0.048611in}{0.000000in}}{\pgfqpoint{-0.000000in}{0.000000in}}{%
\pgfpathmoveto{\pgfqpoint{-0.000000in}{0.000000in}}%
\pgfpathlineto{\pgfqpoint{-0.048611in}{0.000000in}}%
\pgfusepath{stroke,fill}%
}%
\begin{pgfscope}%
\pgfsys@transformshift{0.519743in}{0.986544in}%
\pgfsys@useobject{currentmarker}{}%
\end{pgfscope}%
\end{pgfscope}%
\begin{pgfscope}%
\definecolor{textcolor}{rgb}{0.000000,0.000000,0.000000}%
\pgfsetstrokecolor{textcolor}%
\pgfsetfillcolor{textcolor}%
\pgftext[x=0.223446in, y=0.949611in, left, base]{\color{textcolor}{\rmfamily\fontsize{7.000000}{8.400000}\selectfont\catcode`\^=\active\def^{\ifmmode\sp\else\^{}\fi}\catcode`\%=\active\def
\end{pgfscope}%
\begin{pgfscope}%
\pgfpathrectangle{\pgfqpoint{0.519743in}{0.414757in}}{\pgfqpoint{0.983773in}{1.143573in}}%
\pgfusepath{clip}%
\pgfsetrectcap%
\pgfsetroundjoin%
\pgfsetlinewidth{0.803000pt}%
\definecolor{currentstroke}{rgb}{0.690196,0.690196,0.690196}%
\pgfsetstrokecolor{currentstroke}%
\pgfsetdash{}{0pt}%
\pgfpathmoveto{\pgfqpoint{0.519743in}{1.246447in}}%
\pgfpathlineto{\pgfqpoint{1.503516in}{1.246447in}}%
\pgfusepath{stroke}%
\end{pgfscope}%
\begin{pgfscope}%
\pgfsetbuttcap%
\pgfsetroundjoin%
\definecolor{currentfill}{rgb}{0.000000,0.000000,0.000000}%
\pgfsetfillcolor{currentfill}%
\pgfsetlinewidth{0.803000pt}%
\definecolor{currentstroke}{rgb}{0.000000,0.000000,0.000000}%
\pgfsetstrokecolor{currentstroke}%
\pgfsetdash{}{0pt}%
\pgfsys@defobject{currentmarker}{\pgfqpoint{-0.048611in}{0.000000in}}{\pgfqpoint{-0.000000in}{0.000000in}}{%
\pgfpathmoveto{\pgfqpoint{-0.000000in}{0.000000in}}%
\pgfpathlineto{\pgfqpoint{-0.048611in}{0.000000in}}%
\pgfusepath{stroke,fill}%
}%
\begin{pgfscope}%
\pgfsys@transformshift{0.519743in}{1.246447in}%
\pgfsys@useobject{currentmarker}{}%
\end{pgfscope}%
\end{pgfscope}%
\begin{pgfscope}%
\definecolor{textcolor}{rgb}{0.000000,0.000000,0.000000}%
\pgfsetstrokecolor{textcolor}%
\pgfsetfillcolor{textcolor}%
\pgftext[x=0.223446in, y=1.209513in, left, base]{\color{textcolor}{\rmfamily\fontsize{7.000000}{8.400000}\selectfont\catcode`\^=\active\def^{\ifmmode\sp\else\^{}\fi}\catcode`\%=\active\def
\end{pgfscope}%
\begin{pgfscope}%
\pgfpathrectangle{\pgfqpoint{0.519743in}{0.414757in}}{\pgfqpoint{0.983773in}{1.143573in}}%
\pgfusepath{clip}%
\pgfsetrectcap%
\pgfsetroundjoin%
\pgfsetlinewidth{0.803000pt}%
\definecolor{currentstroke}{rgb}{0.690196,0.690196,0.690196}%
\pgfsetstrokecolor{currentstroke}%
\pgfsetdash{}{0pt}%
\pgfpathmoveto{\pgfqpoint{0.519743in}{1.506349in}}%
\pgfpathlineto{\pgfqpoint{1.503516in}{1.506349in}}%
\pgfusepath{stroke}%
\end{pgfscope}%
\begin{pgfscope}%
\pgfsetbuttcap%
\pgfsetroundjoin%
\definecolor{currentfill}{rgb}{0.000000,0.000000,0.000000}%
\pgfsetfillcolor{currentfill}%
\pgfsetlinewidth{0.803000pt}%
\definecolor{currentstroke}{rgb}{0.000000,0.000000,0.000000}%
\pgfsetstrokecolor{currentstroke}%
\pgfsetdash{}{0pt}%
\pgfsys@defobject{currentmarker}{\pgfqpoint{-0.048611in}{0.000000in}}{\pgfqpoint{-0.000000in}{0.000000in}}{%
\pgfpathmoveto{\pgfqpoint{-0.000000in}{0.000000in}}%
\pgfpathlineto{\pgfqpoint{-0.048611in}{0.000000in}}%
\pgfusepath{stroke,fill}%
}%
\begin{pgfscope}%
\pgfsys@transformshift{0.519743in}{1.506349in}%
\pgfsys@useobject{currentmarker}{}%
\end{pgfscope}%
\end{pgfscope}%
\begin{pgfscope}%
\definecolor{textcolor}{rgb}{0.000000,0.000000,0.000000}%
\pgfsetstrokecolor{textcolor}%
\pgfsetfillcolor{textcolor}%
\pgftext[x=0.223446in, y=1.469416in, left, base]{\color{textcolor}{\rmfamily\fontsize{7.000000}{8.400000}\selectfont\catcode`\^=\active\def^{\ifmmode\sp\else\^{}\fi}\catcode`\%=\active\def
\end{pgfscope}%
\begin{pgfscope}%
\definecolor{textcolor}{rgb}{0.000000,0.000000,0.000000}%
\pgfsetstrokecolor{textcolor}%
\pgfsetfillcolor{textcolor}%
\pgftext[x=0.167891in,y=0.986544in,,bottom,rotate=90.000000]{\color{textcolor}{\rmfamily\fontsize{9.000000}{10.800000}\selectfont\catcode`\^=\active\def^{\ifmmode\sp\else\^{}\fi}\catcode`\%=\active\def
\end{pgfscope}%
\begin{pgfscope}%
\pgfpathrectangle{\pgfqpoint{0.519743in}{0.414757in}}{\pgfqpoint{0.983773in}{1.143573in}}%
\pgfusepath{clip}%
\pgfsetrectcap%
\pgfsetroundjoin%
\pgfsetlinewidth{1.505625pt}%
\definecolor{currentstroke}{rgb}{0.003922,0.450980,0.698039}%
\pgfsetstrokecolor{currentstroke}%
\pgfsetstrokeopacity{0.200000}%
\pgfsetdash{}{0pt}%
\pgfpathmoveto{\pgfqpoint{0.638112in}{0.468006in}}%
\pgfpathlineto{\pgfqpoint{0.698611in}{0.471055in}}%
\pgfpathlineto{\pgfqpoint{0.743328in}{0.474897in}}%
\pgfpathlineto{\pgfqpoint{0.788045in}{0.480074in}}%
\pgfpathlineto{\pgfqpoint{0.832762in}{0.486615in}}%
\pgfpathlineto{\pgfqpoint{0.877479in}{0.494521in}}%
\pgfpathlineto{\pgfqpoint{0.922196in}{0.503774in}}%
\pgfpathlineto{\pgfqpoint{0.966913in}{0.514337in}}%
\pgfpathlineto{\pgfqpoint{1.011630in}{0.526156in}}%
\pgfpathlineto{\pgfqpoint{1.056347in}{0.539162in}}%
\pgfpathlineto{\pgfqpoint{1.101064in}{0.553270in}}%
\pgfpathlineto{\pgfqpoint{1.145781in}{0.568387in}}%
\pgfpathlineto{\pgfqpoint{1.190498in}{0.584408in}}%
\pgfpathlineto{\pgfqpoint{1.235215in}{0.601216in}}%
\pgfpathlineto{\pgfqpoint{1.279932in}{0.618683in}}%
\pgfpathlineto{\pgfqpoint{1.324649in}{0.636656in}}%
\pgfpathlineto{\pgfqpoint{1.369365in}{0.654933in}}%
\pgfpathlineto{\pgfqpoint{1.414082in}{0.673101in}}%
\pgfpathlineto{\pgfqpoint{1.458799in}{0.653346in}}%
\pgfusepath{stroke}%
\end{pgfscope}%
\begin{pgfscope}%
\pgfpathrectangle{\pgfqpoint{0.519743in}{0.414757in}}{\pgfqpoint{0.983773in}{1.143573in}}%
\pgfusepath{clip}%
\pgfsetrectcap%
\pgfsetroundjoin%
\pgfsetlinewidth{1.505625pt}%
\definecolor{currentstroke}{rgb}{0.003922,0.450980,0.698039}%
\pgfsetstrokecolor{currentstroke}%
\pgfsetstrokeopacity{0.200000}%
\pgfsetdash{}{0pt}%
\pgfpathmoveto{\pgfqpoint{0.631536in}{0.467656in}}%
\pgfpathlineto{\pgfqpoint{0.698611in}{0.470550in}}%
\pgfpathlineto{\pgfqpoint{0.743328in}{0.474039in}}%
\pgfpathlineto{\pgfqpoint{0.788045in}{0.478799in}}%
\pgfpathlineto{\pgfqpoint{0.832762in}{0.484872in}}%
\pgfpathlineto{\pgfqpoint{0.877479in}{0.492279in}}%
\pgfpathlineto{\pgfqpoint{0.922196in}{0.501016in}}%
\pgfpathlineto{\pgfqpoint{0.966913in}{0.511063in}}%
\pgfpathlineto{\pgfqpoint{1.011630in}{0.522378in}}%
\pgfpathlineto{\pgfqpoint{1.056347in}{0.534908in}}%
\pgfpathlineto{\pgfqpoint{1.101064in}{0.548582in}}%
\pgfpathlineto{\pgfqpoint{1.145781in}{0.563318in}}%
\pgfpathlineto{\pgfqpoint{1.190498in}{0.579024in}}%
\pgfpathlineto{\pgfqpoint{1.235215in}{0.595595in}}%
\pgfpathlineto{\pgfqpoint{1.279932in}{0.612916in}}%
\pgfpathlineto{\pgfqpoint{1.324649in}{0.630860in}}%
\pgfpathlineto{\pgfqpoint{1.369365in}{0.649262in}}%
\pgfpathlineto{\pgfqpoint{1.414082in}{0.667839in}}%
\pgfpathlineto{\pgfqpoint{1.458799in}{0.663804in}}%
\pgfusepath{stroke}%
\end{pgfscope}%
\begin{pgfscope}%
\pgfpathrectangle{\pgfqpoint{0.519743in}{0.414757in}}{\pgfqpoint{0.983773in}{1.143573in}}%
\pgfusepath{clip}%
\pgfsetrectcap%
\pgfsetroundjoin%
\pgfsetlinewidth{1.505625pt}%
\definecolor{currentstroke}{rgb}{0.003922,0.450980,0.698039}%
\pgfsetstrokecolor{currentstroke}%
\pgfsetstrokeopacity{0.200000}%
\pgfsetdash{}{0pt}%
\pgfpathmoveto{\pgfqpoint{0.645510in}{0.467713in}}%
\pgfpathlineto{\pgfqpoint{0.698611in}{0.469755in}}%
\pgfpathlineto{\pgfqpoint{0.743328in}{0.472712in}}%
\pgfpathlineto{\pgfqpoint{0.788045in}{0.476865in}}%
\pgfpathlineto{\pgfqpoint{0.832762in}{0.482292in}}%
\pgfpathlineto{\pgfqpoint{0.877479in}{0.489041in}}%
\pgfpathlineto{\pgfqpoint{0.922196in}{0.497138in}}%
\pgfpathlineto{\pgfqpoint{0.966913in}{0.506585in}}%
\pgfpathlineto{\pgfqpoint{1.011630in}{0.517363in}}%
\pgfpathlineto{\pgfqpoint{1.056347in}{0.529432in}}%
\pgfpathlineto{\pgfqpoint{1.101064in}{0.542734in}}%
\pgfpathlineto{\pgfqpoint{1.145781in}{0.557190in}}%
\pgfpathlineto{\pgfqpoint{1.190498in}{0.572706in}}%
\pgfpathlineto{\pgfqpoint{1.235215in}{0.589169in}}%
\pgfpathlineto{\pgfqpoint{1.279932in}{0.606443in}}%
\pgfpathlineto{\pgfqpoint{1.324649in}{0.624357in}}%
\pgfpathlineto{\pgfqpoint{1.369365in}{0.642662in}}%
\pgfpathlineto{\pgfqpoint{1.414082in}{0.660819in}}%
\pgfpathlineto{\pgfqpoint{1.458799in}{0.628539in}}%
\pgfusepath{stroke}%
\end{pgfscope}%
\begin{pgfscope}%
\pgfpathrectangle{\pgfqpoint{0.519743in}{0.414757in}}{\pgfqpoint{0.983773in}{1.143573in}}%
\pgfusepath{clip}%
\pgfsetrectcap%
\pgfsetroundjoin%
\pgfsetlinewidth{1.505625pt}%
\definecolor{currentstroke}{rgb}{0.003922,0.450980,0.698039}%
\pgfsetstrokecolor{currentstroke}%
\pgfsetstrokeopacity{0.200000}%
\pgfsetdash{}{0pt}%
\pgfpathmoveto{\pgfqpoint{0.638989in}{0.468127in}}%
\pgfpathlineto{\pgfqpoint{0.698611in}{0.471312in}}%
\pgfpathlineto{\pgfqpoint{0.743328in}{0.475317in}}%
\pgfpathlineto{\pgfqpoint{0.788045in}{0.480679in}}%
\pgfpathlineto{\pgfqpoint{0.832762in}{0.487416in}}%
\pgfpathlineto{\pgfqpoint{0.877479in}{0.495522in}}%
\pgfpathlineto{\pgfqpoint{0.922196in}{0.504972in}}%
\pgfpathlineto{\pgfqpoint{0.966913in}{0.515721in}}%
\pgfpathlineto{\pgfqpoint{1.011630in}{0.527710in}}%
\pgfpathlineto{\pgfqpoint{1.056347in}{0.540865in}}%
\pgfpathlineto{\pgfqpoint{1.101064in}{0.555100in}}%
\pgfpathlineto{\pgfqpoint{1.145781in}{0.570317in}}%
\pgfpathlineto{\pgfqpoint{1.190498in}{0.586411in}}%
\pgfpathlineto{\pgfqpoint{1.235215in}{0.603265in}}%
\pgfpathlineto{\pgfqpoint{1.279932in}{0.620751in}}%
\pgfpathlineto{\pgfqpoint{1.324649in}{0.638721in}}%
\pgfpathlineto{\pgfqpoint{1.369365in}{0.656978in}}%
\pgfpathlineto{\pgfqpoint{1.414082in}{0.675127in}}%
\pgfpathlineto{\pgfqpoint{1.458799in}{0.657119in}}%
\pgfusepath{stroke}%
\end{pgfscope}%
\begin{pgfscope}%
\pgfpathrectangle{\pgfqpoint{0.519743in}{0.414757in}}{\pgfqpoint{0.983773in}{1.143573in}}%
\pgfusepath{clip}%
\pgfsetrectcap%
\pgfsetroundjoin%
\pgfsetlinewidth{1.505625pt}%
\definecolor{currentstroke}{rgb}{0.003922,0.450980,0.698039}%
\pgfsetstrokecolor{currentstroke}%
\pgfsetstrokeopacity{0.200000}%
\pgfsetdash{}{0pt}%
\pgfpathmoveto{\pgfqpoint{0.639920in}{0.467865in}}%
\pgfpathlineto{\pgfqpoint{0.698611in}{0.470510in}}%
\pgfpathlineto{\pgfqpoint{0.743328in}{0.473984in}}%
\pgfpathlineto{\pgfqpoint{0.788045in}{0.478733in}}%
\pgfpathlineto{\pgfqpoint{0.832762in}{0.484802in}}%
\pgfpathlineto{\pgfqpoint{0.877479in}{0.492213in}}%
\pgfpathlineto{\pgfqpoint{0.922196in}{0.500962in}}%
\pgfpathlineto{\pgfqpoint{0.966913in}{0.511027in}}%
\pgfpathlineto{\pgfqpoint{1.011630in}{0.522366in}}%
\pgfpathlineto{\pgfqpoint{1.056347in}{0.534922in}}%
\pgfpathlineto{\pgfqpoint{1.101064in}{0.548619in}}%
\pgfpathlineto{\pgfqpoint{1.145781in}{0.563371in}}%
\pgfpathlineto{\pgfqpoint{1.190498in}{0.579076in}}%
\pgfpathlineto{\pgfqpoint{1.235215in}{0.595619in}}%
\pgfpathlineto{\pgfqpoint{1.279932in}{0.612869in}}%
\pgfpathlineto{\pgfqpoint{1.324649in}{0.630667in}}%
\pgfpathlineto{\pgfqpoint{1.369365in}{0.648790in}}%
\pgfpathlineto{\pgfqpoint{1.414082in}{0.666769in}}%
\pgfpathlineto{\pgfqpoint{1.458799in}{0.640837in}}%
\pgfusepath{stroke}%
\end{pgfscope}%
\begin{pgfscope}%
\pgfpathrectangle{\pgfqpoint{0.519743in}{0.414757in}}{\pgfqpoint{0.983773in}{1.143573in}}%
\pgfusepath{clip}%
\pgfsetbuttcap%
\pgfsetroundjoin%
\pgfsetlinewidth{1.505625pt}%
\definecolor{currentstroke}{rgb}{0.501961,0.501961,0.501961}%
\pgfsetstrokecolor{currentstroke}%
\pgfsetdash{{5.550000pt}{2.400000pt}}{0.000000pt}%
\pgfpathmoveto{\pgfqpoint{0.564460in}{0.726641in}}%
\pgfpathlineto{\pgfqpoint{1.458799in}{0.726641in}}%
\pgfusepath{stroke}%
\end{pgfscope}%
\begin{pgfscope}%
\pgfpathrectangle{\pgfqpoint{0.519743in}{0.414757in}}{\pgfqpoint{0.983773in}{1.143573in}}%
\pgfusepath{clip}%
\pgfsetbuttcap%
\pgfsetroundjoin%
\pgfsetlinewidth{1.505625pt}%
\definecolor{currentstroke}{rgb}{0.501961,0.501961,0.501961}%
\pgfsetstrokecolor{currentstroke}%
\pgfsetdash{{5.550000pt}{2.400000pt}}{0.000000pt}%
\pgfpathmoveto{\pgfqpoint{0.788045in}{0.466738in}}%
\pgfpathlineto{\pgfqpoint{0.788045in}{1.506349in}}%
\pgfusepath{stroke}%
\end{pgfscope}%
\begin{pgfscope}%
\pgfsetrectcap%
\pgfsetmiterjoin%
\pgfsetlinewidth{0.803000pt}%
\definecolor{currentstroke}{rgb}{0.000000,0.000000,0.000000}%
\pgfsetstrokecolor{currentstroke}%
\pgfsetdash{}{0pt}%
\pgfpathmoveto{\pgfqpoint{0.519743in}{0.414757in}}%
\pgfpathlineto{\pgfqpoint{0.519743in}{1.558330in}}%
\pgfusepath{stroke}%
\end{pgfscope}%
\begin{pgfscope}%
\pgfsetrectcap%
\pgfsetmiterjoin%
\pgfsetlinewidth{0.803000pt}%
\definecolor{currentstroke}{rgb}{0.000000,0.000000,0.000000}%
\pgfsetstrokecolor{currentstroke}%
\pgfsetdash{}{0pt}%
\pgfpathmoveto{\pgfqpoint{1.503516in}{0.414757in}}%
\pgfpathlineto{\pgfqpoint{1.503516in}{1.558330in}}%
\pgfusepath{stroke}%
\end{pgfscope}%
\begin{pgfscope}%
\pgfsetrectcap%
\pgfsetmiterjoin%
\pgfsetlinewidth{0.803000pt}%
\definecolor{currentstroke}{rgb}{0.000000,0.000000,0.000000}%
\pgfsetstrokecolor{currentstroke}%
\pgfsetdash{}{0pt}%
\pgfpathmoveto{\pgfqpoint{0.519743in}{0.414757in}}%
\pgfpathlineto{\pgfqpoint{1.503516in}{0.414757in}}%
\pgfusepath{stroke}%
\end{pgfscope}%
\begin{pgfscope}%
\pgfsetrectcap%
\pgfsetmiterjoin%
\pgfsetlinewidth{0.803000pt}%
\definecolor{currentstroke}{rgb}{0.000000,0.000000,0.000000}%
\pgfsetstrokecolor{currentstroke}%
\pgfsetdash{}{0pt}%
\pgfpathmoveto{\pgfqpoint{0.519743in}{1.558330in}}%
\pgfpathlineto{\pgfqpoint{1.503516in}{1.558330in}}%
\pgfusepath{stroke}%
\end{pgfscope}%
\begin{pgfscope}%
\pgfsetbuttcap%
\pgfsetmiterjoin%
\definecolor{currentfill}{rgb}{1.000000,1.000000,1.000000}%
\pgfsetfillcolor{currentfill}%
\pgfsetfillopacity{0.800000}%
\pgfsetlinewidth{1.003750pt}%
\definecolor{currentstroke}{rgb}{0.800000,0.800000,0.800000}%
\pgfsetstrokecolor{currentstroke}%
\pgfsetstrokeopacity{0.800000}%
\pgfsetdash{}{0pt}%
\pgfpathmoveto{\pgfqpoint{0.542878in}{1.008611in}}%
\pgfpathlineto{\pgfqpoint{1.537122in}{1.008611in}}%
\pgfpathquadraticcurveto{\pgfqpoint{1.556567in}{1.008611in}}{\pgfqpoint{1.556567in}{1.028056in}}%
\pgfpathlineto{\pgfqpoint{1.556567in}{1.446433in}}%
\pgfpathquadraticcurveto{\pgfqpoint{1.556567in}{1.465878in}}{\pgfqpoint{1.537122in}{1.465878in}}%
\pgfpathlineto{\pgfqpoint{0.542878in}{1.465878in}}%
\pgfpathquadraticcurveto{\pgfqpoint{0.523433in}{1.465878in}}{\pgfqpoint{0.523433in}{1.446433in}}%
\pgfpathlineto{\pgfqpoint{0.523433in}{1.028056in}}%
\pgfpathquadraticcurveto{\pgfqpoint{0.523433in}{1.008611in}}{\pgfqpoint{0.542878in}{1.008611in}}%
\pgfpathlineto{\pgfqpoint{0.542878in}{1.008611in}}%
\pgfpathclose%
\pgfusepath{stroke,fill}%
\end{pgfscope}%
\begin{pgfscope}%
\pgfsetbuttcap%
\pgfsetmiterjoin%
\definecolor{currentfill}{rgb}{0.007843,0.619608,0.447059}%
\pgfsetfillcolor{currentfill}%
\pgfsetfillopacity{0.850000}%
\pgfsetlinewidth{0.501875pt}%
\definecolor{currentstroke}{rgb}{0.000000,0.000000,0.000000}%
\pgfsetstrokecolor{currentstroke}%
\pgfsetstrokeopacity{0.850000}%
\pgfsetdash{}{0pt}%
\pgfpathmoveto{\pgfqpoint{0.562322in}{1.353123in}}%
\pgfpathlineto{\pgfqpoint{0.756766in}{1.353123in}}%
\pgfpathlineto{\pgfqpoint{0.756766in}{1.421178in}}%
\pgfpathlineto{\pgfqpoint{0.562322in}{1.421178in}}%
\pgfpathlineto{\pgfqpoint{0.562322in}{1.353123in}}%
\pgfpathclose%
\pgfusepath{stroke,fill}%
\end{pgfscope}%
\begin{pgfscope}%
\definecolor{textcolor}{rgb}{0.000000,0.000000,0.000000}%
\pgfsetstrokecolor{textcolor}%
\pgfsetfillcolor{textcolor}%
\pgftext[x=0.834544in,y=1.353123in,left,base]{\color{textcolor}{\rmfamily\fontsize{7.000000}{8.400000}\selectfont\catcode`\^=\active\def^{\ifmmode\sp\else\^{}\fi}\catcode`\%=\active\def
\end{pgfscope}%
\begin{pgfscope}%
\pgfsetbuttcap%
\pgfsetmiterjoin%
\definecolor{currentfill}{rgb}{0.870588,0.560784,0.011765}%
\pgfsetfillcolor{currentfill}%
\pgfsetfillopacity{0.850000}%
\pgfsetlinewidth{0.501875pt}%
\definecolor{currentstroke}{rgb}{0.000000,0.000000,0.000000}%
\pgfsetstrokecolor{currentstroke}%
\pgfsetstrokeopacity{0.850000}%
\pgfsetdash{}{0pt}%
\pgfpathmoveto{\pgfqpoint{0.562322in}{1.210423in}}%
\pgfpathlineto{\pgfqpoint{0.756766in}{1.210423in}}%
\pgfpathlineto{\pgfqpoint{0.756766in}{1.278478in}}%
\pgfpathlineto{\pgfqpoint{0.562322in}{1.278478in}}%
\pgfpathlineto{\pgfqpoint{0.562322in}{1.210423in}}%
\pgfpathclose%
\pgfusepath{stroke,fill}%
\end{pgfscope}%
\begin{pgfscope}%
\definecolor{textcolor}{rgb}{0.000000,0.000000,0.000000}%
\pgfsetstrokecolor{textcolor}%
\pgfsetfillcolor{textcolor}%
\pgftext[x=0.834544in,y=1.210423in,left,base]{\color{textcolor}{\rmfamily\fontsize{7.000000}{8.400000}\selectfont\catcode`\^=\active\def^{\ifmmode\sp\else\^{}\fi}\catcode`\%=\active\def
\end{pgfscope}%
\begin{pgfscope}%
\pgfsetbuttcap%
\pgfsetroundjoin%
\pgfsetlinewidth{0.803000pt}%
\definecolor{currentstroke}{rgb}{0.000000,0.000000,0.000000}%
\pgfsetstrokecolor{currentstroke}%
\pgfsetstrokeopacity{0.600000}%
\pgfsetdash{{2.960000pt}{1.280000pt}}{0.000000pt}%
\pgfpathmoveto{\pgfqpoint{0.562322in}{1.101751in}}%
\pgfpathlineto{\pgfqpoint{0.659544in}{1.101751in}}%
\pgfpathlineto{\pgfqpoint{0.756766in}{1.101751in}}%
\pgfusepath{stroke}%
\end{pgfscope}%
\begin{pgfscope}%
\definecolor{textcolor}{rgb}{0.000000,0.000000,0.000000}%
\pgfsetstrokecolor{textcolor}%
\pgfsetfillcolor{textcolor}%
\pgftext[x=0.834544in,y=1.067723in,left,base]{\color{textcolor}{\rmfamily\fontsize{7.000000}{8.400000}\selectfont\catcode`\^=\active\def^{\ifmmode\sp\else\^{}\fi}\catcode`\%=\active\def
\end{pgfscope}%
\end{pgfpicture}%
\makeatother%
\endgroup%

%% file: figures/calibration_maps/SIMPLE_QA_VERIFIED_Qwen3-4B-Instruct-2507.pgf
\begingroup%
\makeatletter%
\begin{pgfpicture}%
\pgfpathrectangle{\pgfpointorigin}{\pgfqpoint{1.600000in}{1.600000in}}%
\pgfusepath{use as bounding box, clip}%
\begin{pgfscope}%
\pgfsetbuttcap%
\pgfsetmiterjoin%
\definecolor{currentfill}{rgb}{1.000000,1.000000,1.000000}%
\pgfsetfillcolor{currentfill}%
\pgfsetlinewidth{0.000000pt}%
\definecolor{currentstroke}{rgb}{1.000000,1.000000,1.000000}%
\pgfsetstrokecolor{currentstroke}%
\pgfsetdash{}{0pt}%
\pgfpathmoveto{\pgfqpoint{0.000000in}{0.000000in}}%
\pgfpathlineto{\pgfqpoint{1.600000in}{0.000000in}}%
\pgfpathlineto{\pgfqpoint{1.600000in}{1.600000in}}%
\pgfpathlineto{\pgfqpoint{0.000000in}{1.600000in}}%
\pgfpathlineto{\pgfqpoint{0.000000in}{0.000000in}}%
\pgfpathclose%
\pgfusepath{fill}%
\end{pgfscope}%
\begin{pgfscope}%
\pgfsetbuttcap%
\pgfsetmiterjoin%
\definecolor{currentfill}{rgb}{1.000000,1.000000,1.000000}%
\pgfsetfillcolor{currentfill}%
\pgfsetlinewidth{0.000000pt}%
\definecolor{currentstroke}{rgb}{0.000000,0.000000,0.000000}%
\pgfsetstrokecolor{currentstroke}%
\pgfsetstrokeopacity{0.000000}%
\pgfsetdash{}{0pt}%
\pgfpathmoveto{\pgfqpoint{0.519743in}{0.414757in}}%
\pgfpathlineto{\pgfqpoint{1.503516in}{0.414757in}}%
\pgfpathlineto{\pgfqpoint{1.503516in}{1.558330in}}%
\pgfpathlineto{\pgfqpoint{0.519743in}{1.558330in}}%
\pgfpathlineto{\pgfqpoint{0.519743in}{0.414757in}}%
\pgfpathclose%
\pgfusepath{fill}%
\end{pgfscope}%
\begin{pgfscope}%
\pgfpathrectangle{\pgfqpoint{0.519743in}{0.414757in}}{\pgfqpoint{0.983773in}{1.143573in}}%
\pgfusepath{clip}%
\pgfsetbuttcap%
\pgfsetroundjoin%
\definecolor{currentfill}{rgb}{0.870588,0.560784,0.011765}%
\pgfsetfillcolor{currentfill}%
\pgfsetfillopacity{0.200000}%
\pgfsetlinewidth{1.003750pt}%
\definecolor{currentstroke}{rgb}{0.870588,0.560784,0.011765}%
\pgfsetstrokecolor{currentstroke}%
\pgfsetstrokeopacity{0.200000}%
\pgfsetdash{}{0pt}%
\pgfsys@defobject{currentmarker}{\pgfqpoint{0.564460in}{0.726641in}}{\pgfqpoint{0.788045in}{1.506349in}}{%
\pgfpathmoveto{\pgfqpoint{0.788045in}{0.726641in}}%
\pgfpathlineto{\pgfqpoint{0.564460in}{0.726641in}}%
\pgfpathlineto{\pgfqpoint{0.564460in}{1.506349in}}%
\pgfpathlineto{\pgfqpoint{0.788045in}{1.506349in}}%
\pgfpathlineto{\pgfqpoint{0.788045in}{1.506349in}}%
\pgfpathlineto{\pgfqpoint{0.788045in}{0.726641in}}%
\pgfpathlineto{\pgfqpoint{0.788045in}{0.726641in}}%
\pgfpathclose%
\pgfusepath{stroke,fill}%
}%
\begin{pgfscope}%
\pgfsys@transformshift{0.000000in}{0.000000in}%
\pgfsys@useobject{currentmarker}{}%
\end{pgfscope}%
\end{pgfscope}%
\begin{pgfscope}%
\pgfpathrectangle{\pgfqpoint{0.519743in}{0.414757in}}{\pgfqpoint{0.983773in}{1.143573in}}%
\pgfusepath{clip}%
\pgfsetbuttcap%
\pgfsetroundjoin%
\definecolor{currentfill}{rgb}{0.007843,0.619608,0.447059}%
\pgfsetfillcolor{currentfill}%
\pgfsetfillopacity{0.200000}%
\pgfsetlinewidth{1.003750pt}%
\definecolor{currentstroke}{rgb}{0.007843,0.619608,0.447059}%
\pgfsetstrokecolor{currentstroke}%
\pgfsetstrokeopacity{0.200000}%
\pgfsetdash{}{0pt}%
\pgfsys@defobject{currentmarker}{\pgfqpoint{0.788045in}{0.726641in}}{\pgfqpoint{1.458799in}{1.506349in}}{%
\pgfpathmoveto{\pgfqpoint{1.458799in}{0.726641in}}%
\pgfpathlineto{\pgfqpoint{0.788045in}{0.726641in}}%
\pgfpathlineto{\pgfqpoint{0.788045in}{1.506349in}}%
\pgfpathlineto{\pgfqpoint{1.458799in}{1.506349in}}%
\pgfpathlineto{\pgfqpoint{1.458799in}{1.506349in}}%
\pgfpathlineto{\pgfqpoint{1.458799in}{0.726641in}}%
\pgfpathlineto{\pgfqpoint{1.458799in}{0.726641in}}%
\pgfpathclose%
\pgfusepath{stroke,fill}%
}%
\begin{pgfscope}%
\pgfsys@transformshift{0.000000in}{0.000000in}%
\pgfsys@useobject{currentmarker}{}%
\end{pgfscope}%
\end{pgfscope}%
\begin{pgfscope}%
\pgfpathrectangle{\pgfqpoint{0.519743in}{0.414757in}}{\pgfqpoint{0.983773in}{1.143573in}}%
\pgfusepath{clip}%
\pgfsetbuttcap%
\pgfsetroundjoin%
\definecolor{currentfill}{rgb}{0.870588,0.560784,0.011765}%
\pgfsetfillcolor{currentfill}%
\pgfsetfillopacity{0.200000}%
\pgfsetlinewidth{1.003750pt}%
\definecolor{currentstroke}{rgb}{0.870588,0.560784,0.011765}%
\pgfsetstrokecolor{currentstroke}%
\pgfsetstrokeopacity{0.200000}%
\pgfsetdash{}{0pt}%
\pgfsys@defobject{currentmarker}{\pgfqpoint{0.788045in}{0.466738in}}{\pgfqpoint{1.458799in}{0.726641in}}{%
\pgfpathmoveto{\pgfqpoint{1.458799in}{0.466738in}}%
\pgfpathlineto{\pgfqpoint{0.788045in}{0.466738in}}%
\pgfpathlineto{\pgfqpoint{0.788045in}{0.726641in}}%
\pgfpathlineto{\pgfqpoint{1.458799in}{0.726641in}}%
\pgfpathlineto{\pgfqpoint{1.458799in}{0.726641in}}%
\pgfpathlineto{\pgfqpoint{1.458799in}{0.466738in}}%
\pgfpathlineto{\pgfqpoint{1.458799in}{0.466738in}}%
\pgfpathclose%
\pgfusepath{stroke,fill}%
}%
\begin{pgfscope}%
\pgfsys@transformshift{0.000000in}{0.000000in}%
\pgfsys@useobject{currentmarker}{}%
\end{pgfscope}%
\end{pgfscope}%
\begin{pgfscope}%
\pgfpathrectangle{\pgfqpoint{0.519743in}{0.414757in}}{\pgfqpoint{0.983773in}{1.143573in}}%
\pgfusepath{clip}%
\pgfsetbuttcap%
\pgfsetroundjoin%
\definecolor{currentfill}{rgb}{0.007843,0.619608,0.447059}%
\pgfsetfillcolor{currentfill}%
\pgfsetfillopacity{0.200000}%
\pgfsetlinewidth{1.003750pt}%
\definecolor{currentstroke}{rgb}{0.007843,0.619608,0.447059}%
\pgfsetstrokecolor{currentstroke}%
\pgfsetstrokeopacity{0.200000}%
\pgfsetdash{}{0pt}%
\pgfsys@defobject{currentmarker}{\pgfqpoint{0.564460in}{0.466738in}}{\pgfqpoint{0.788045in}{0.726641in}}{%
\pgfpathmoveto{\pgfqpoint{0.788045in}{0.466738in}}%
\pgfpathlineto{\pgfqpoint{0.564460in}{0.466738in}}%
\pgfpathlineto{\pgfqpoint{0.564460in}{0.726641in}}%
\pgfpathlineto{\pgfqpoint{0.788045in}{0.726641in}}%
\pgfpathlineto{\pgfqpoint{0.788045in}{0.726641in}}%
\pgfpathlineto{\pgfqpoint{0.788045in}{0.466738in}}%
\pgfpathlineto{\pgfqpoint{0.788045in}{0.466738in}}%
\pgfpathclose%
\pgfusepath{stroke,fill}%
}%
\begin{pgfscope}%
\pgfsys@transformshift{0.000000in}{0.000000in}%
\pgfsys@useobject{currentmarker}{}%
\end{pgfscope}%
\end{pgfscope}%
\begin{pgfscope}%
\pgfpathrectangle{\pgfqpoint{0.519743in}{0.414757in}}{\pgfqpoint{0.983773in}{1.143573in}}%
\pgfusepath{clip}%
\pgfsetrectcap%
\pgfsetroundjoin%
\pgfsetlinewidth{0.803000pt}%
\definecolor{currentstroke}{rgb}{0.690196,0.690196,0.690196}%
\pgfsetstrokecolor{currentstroke}%
\pgfsetdash{}{0pt}%
\pgfpathmoveto{\pgfqpoint{0.564460in}{0.414757in}}%
\pgfpathlineto{\pgfqpoint{0.564460in}{1.558330in}}%
\pgfusepath{stroke}%
\end{pgfscope}%
\begin{pgfscope}%
\pgfsetbuttcap%
\pgfsetroundjoin%
\definecolor{currentfill}{rgb}{0.000000,0.000000,0.000000}%
\pgfsetfillcolor{currentfill}%
\pgfsetlinewidth{0.803000pt}%
\definecolor{currentstroke}{rgb}{0.000000,0.000000,0.000000}%
\pgfsetstrokecolor{currentstroke}%
\pgfsetdash{}{0pt}%
\pgfsys@defobject{currentmarker}{\pgfqpoint{0.000000in}{-0.048611in}}{\pgfqpoint{0.000000in}{0.000000in}}{%
\pgfpathmoveto{\pgfqpoint{0.000000in}{0.000000in}}%
\pgfpathlineto{\pgfqpoint{0.000000in}{-0.048611in}}%
\pgfusepath{stroke,fill}%
}%
\begin{pgfscope}%
\pgfsys@transformshift{0.564460in}{0.414757in}%
\pgfsys@useobject{currentmarker}{}%
\end{pgfscope}%
\end{pgfscope}%
\begin{pgfscope}%
\definecolor{textcolor}{rgb}{0.000000,0.000000,0.000000}%
\pgfsetstrokecolor{textcolor}%
\pgfsetfillcolor{textcolor}%
\pgftext[x=0.564460in,y=0.317535in,,top]{\color{textcolor}{\rmfamily\fontsize{7.000000}{8.400000}\selectfont\catcode`\^=\active\def^{\ifmmode\sp\else\^{}\fi}\catcode`\%=\active\def
\end{pgfscope}%
\begin{pgfscope}%
\pgfpathrectangle{\pgfqpoint{0.519743in}{0.414757in}}{\pgfqpoint{0.983773in}{1.143573in}}%
\pgfusepath{clip}%
\pgfsetrectcap%
\pgfsetroundjoin%
\pgfsetlinewidth{0.803000pt}%
\definecolor{currentstroke}{rgb}{0.690196,0.690196,0.690196}%
\pgfsetstrokecolor{currentstroke}%
\pgfsetdash{}{0pt}%
\pgfpathmoveto{\pgfqpoint{0.788045in}{0.414757in}}%
\pgfpathlineto{\pgfqpoint{0.788045in}{1.558330in}}%
\pgfusepath{stroke}%
\end{pgfscope}%
\begin{pgfscope}%
\pgfsetbuttcap%
\pgfsetroundjoin%
\definecolor{currentfill}{rgb}{0.000000,0.000000,0.000000}%
\pgfsetfillcolor{currentfill}%
\pgfsetlinewidth{0.803000pt}%
\definecolor{currentstroke}{rgb}{0.000000,0.000000,0.000000}%
\pgfsetstrokecolor{currentstroke}%
\pgfsetdash{}{0pt}%
\pgfsys@defobject{currentmarker}{\pgfqpoint{0.000000in}{-0.048611in}}{\pgfqpoint{0.000000in}{0.000000in}}{%
\pgfpathmoveto{\pgfqpoint{0.000000in}{0.000000in}}%
\pgfpathlineto{\pgfqpoint{0.000000in}{-0.048611in}}%
\pgfusepath{stroke,fill}%
}%
\begin{pgfscope}%
\pgfsys@transformshift{0.788045in}{0.414757in}%
\pgfsys@useobject{currentmarker}{}%
\end{pgfscope}%
\end{pgfscope}%
\begin{pgfscope}%
\definecolor{textcolor}{rgb}{0.000000,0.000000,0.000000}%
\pgfsetstrokecolor{textcolor}%
\pgfsetfillcolor{textcolor}%
\pgftext[x=0.788045in,y=0.317535in,,top]{\color{textcolor}{\rmfamily\fontsize{7.000000}{8.400000}\selectfont\catcode`\^=\active\def^{\ifmmode\sp\else\^{}\fi}\catcode`\%=\active\def
\end{pgfscope}%
\begin{pgfscope}%
\pgfpathrectangle{\pgfqpoint{0.519743in}{0.414757in}}{\pgfqpoint{0.983773in}{1.143573in}}%
\pgfusepath{clip}%
\pgfsetrectcap%
\pgfsetroundjoin%
\pgfsetlinewidth{0.803000pt}%
\definecolor{currentstroke}{rgb}{0.690196,0.690196,0.690196}%
\pgfsetstrokecolor{currentstroke}%
\pgfsetdash{}{0pt}%
\pgfpathmoveto{\pgfqpoint{1.011630in}{0.414757in}}%
\pgfpathlineto{\pgfqpoint{1.011630in}{1.558330in}}%
\pgfusepath{stroke}%
\end{pgfscope}%
\begin{pgfscope}%
\pgfsetbuttcap%
\pgfsetroundjoin%
\definecolor{currentfill}{rgb}{0.000000,0.000000,0.000000}%
\pgfsetfillcolor{currentfill}%
\pgfsetlinewidth{0.803000pt}%
\definecolor{currentstroke}{rgb}{0.000000,0.000000,0.000000}%
\pgfsetstrokecolor{currentstroke}%
\pgfsetdash{}{0pt}%
\pgfsys@defobject{currentmarker}{\pgfqpoint{0.000000in}{-0.048611in}}{\pgfqpoint{0.000000in}{0.000000in}}{%
\pgfpathmoveto{\pgfqpoint{0.000000in}{0.000000in}}%
\pgfpathlineto{\pgfqpoint{0.000000in}{-0.048611in}}%
\pgfusepath{stroke,fill}%
}%
\begin{pgfscope}%
\pgfsys@transformshift{1.011630in}{0.414757in}%
\pgfsys@useobject{currentmarker}{}%
\end{pgfscope}%
\end{pgfscope}%
\begin{pgfscope}%
\definecolor{textcolor}{rgb}{0.000000,0.000000,0.000000}%
\pgfsetstrokecolor{textcolor}%
\pgfsetfillcolor{textcolor}%
\pgftext[x=1.011630in,y=0.317535in,,top]{\color{textcolor}{\rmfamily\fontsize{7.000000}{8.400000}\selectfont\catcode`\^=\active\def^{\ifmmode\sp\else\^{}\fi}\catcode`\%=\active\def
\end{pgfscope}%
\begin{pgfscope}%
\pgfpathrectangle{\pgfqpoint{0.519743in}{0.414757in}}{\pgfqpoint{0.983773in}{1.143573in}}%
\pgfusepath{clip}%
\pgfsetrectcap%
\pgfsetroundjoin%
\pgfsetlinewidth{0.803000pt}%
\definecolor{currentstroke}{rgb}{0.690196,0.690196,0.690196}%
\pgfsetstrokecolor{currentstroke}%
\pgfsetdash{}{0pt}%
\pgfpathmoveto{\pgfqpoint{1.235215in}{0.414757in}}%
\pgfpathlineto{\pgfqpoint{1.235215in}{1.558330in}}%
\pgfusepath{stroke}%
\end{pgfscope}%
\begin{pgfscope}%
\pgfsetbuttcap%
\pgfsetroundjoin%
\definecolor{currentfill}{rgb}{0.000000,0.000000,0.000000}%
\pgfsetfillcolor{currentfill}%
\pgfsetlinewidth{0.803000pt}%
\definecolor{currentstroke}{rgb}{0.000000,0.000000,0.000000}%
\pgfsetstrokecolor{currentstroke}%
\pgfsetdash{}{0pt}%
\pgfsys@defobject{currentmarker}{\pgfqpoint{0.000000in}{-0.048611in}}{\pgfqpoint{0.000000in}{0.000000in}}{%
\pgfpathmoveto{\pgfqpoint{0.000000in}{0.000000in}}%
\pgfpathlineto{\pgfqpoint{0.000000in}{-0.048611in}}%
\pgfusepath{stroke,fill}%
}%
\begin{pgfscope}%
\pgfsys@transformshift{1.235215in}{0.414757in}%
\pgfsys@useobject{currentmarker}{}%
\end{pgfscope}%
\end{pgfscope}%
\begin{pgfscope}%
\definecolor{textcolor}{rgb}{0.000000,0.000000,0.000000}%
\pgfsetstrokecolor{textcolor}%
\pgfsetfillcolor{textcolor}%
\pgftext[x=1.235215in,y=0.317535in,,top]{\color{textcolor}{\rmfamily\fontsize{7.000000}{8.400000}\selectfont\catcode`\^=\active\def^{\ifmmode\sp\else\^{}\fi}\catcode`\%=\active\def
\end{pgfscope}%
\begin{pgfscope}%
\pgfpathrectangle{\pgfqpoint{0.519743in}{0.414757in}}{\pgfqpoint{0.983773in}{1.143573in}}%
\pgfusepath{clip}%
\pgfsetrectcap%
\pgfsetroundjoin%
\pgfsetlinewidth{0.803000pt}%
\definecolor{currentstroke}{rgb}{0.690196,0.690196,0.690196}%
\pgfsetstrokecolor{currentstroke}%
\pgfsetdash{}{0pt}%
\pgfpathmoveto{\pgfqpoint{1.458799in}{0.414757in}}%
\pgfpathlineto{\pgfqpoint{1.458799in}{1.558330in}}%
\pgfusepath{stroke}%
\end{pgfscope}%
\begin{pgfscope}%
\pgfsetbuttcap%
\pgfsetroundjoin%
\definecolor{currentfill}{rgb}{0.000000,0.000000,0.000000}%
\pgfsetfillcolor{currentfill}%
\pgfsetlinewidth{0.803000pt}%
\definecolor{currentstroke}{rgb}{0.000000,0.000000,0.000000}%
\pgfsetstrokecolor{currentstroke}%
\pgfsetdash{}{0pt}%
\pgfsys@defobject{currentmarker}{\pgfqpoint{0.000000in}{-0.048611in}}{\pgfqpoint{0.000000in}{0.000000in}}{%
\pgfpathmoveto{\pgfqpoint{0.000000in}{0.000000in}}%
\pgfpathlineto{\pgfqpoint{0.000000in}{-0.048611in}}%
\pgfusepath{stroke,fill}%
}%
\begin{pgfscope}%
\pgfsys@transformshift{1.458799in}{0.414757in}%
\pgfsys@useobject{currentmarker}{}%
\end{pgfscope}%
\end{pgfscope}%
\begin{pgfscope}%
\definecolor{textcolor}{rgb}{0.000000,0.000000,0.000000}%
\pgfsetstrokecolor{textcolor}%
\pgfsetfillcolor{textcolor}%
\pgftext[x=1.458799in,y=0.317535in,,top]{\color{textcolor}{\rmfamily\fontsize{7.000000}{8.400000}\selectfont\catcode`\^=\active\def^{\ifmmode\sp\else\^{}\fi}\catcode`\%=\active\def
\end{pgfscope}%
\begin{pgfscope}%
\definecolor{textcolor}{rgb}{0.000000,0.000000,0.000000}%
\pgfsetstrokecolor{textcolor}%
\pgfsetfillcolor{textcolor}%
\pgftext[x=1.011630in,y=0.167891in,,top]{\color{textcolor}{\rmfamily\fontsize{9.000000}{10.800000}\selectfont\catcode`\^=\active\def^{\ifmmode\sp\else\^{}\fi}\catcode`\%=\active\def
\end{pgfscope}%
\begin{pgfscope}%
\pgfpathrectangle{\pgfqpoint{0.519743in}{0.414757in}}{\pgfqpoint{0.983773in}{1.143573in}}%
\pgfusepath{clip}%
\pgfsetrectcap%
\pgfsetroundjoin%
\pgfsetlinewidth{0.803000pt}%
\definecolor{currentstroke}{rgb}{0.690196,0.690196,0.690196}%
\pgfsetstrokecolor{currentstroke}%
\pgfsetdash{}{0pt}%
\pgfpathmoveto{\pgfqpoint{0.519743in}{0.466738in}}%
\pgfpathlineto{\pgfqpoint{1.503516in}{0.466738in}}%
\pgfusepath{stroke}%
\end{pgfscope}%
\begin{pgfscope}%
\pgfsetbuttcap%
\pgfsetroundjoin%
\definecolor{currentfill}{rgb}{0.000000,0.000000,0.000000}%
\pgfsetfillcolor{currentfill}%
\pgfsetlinewidth{0.803000pt}%
\definecolor{currentstroke}{rgb}{0.000000,0.000000,0.000000}%
\pgfsetstrokecolor{currentstroke}%
\pgfsetdash{}{0pt}%
\pgfsys@defobject{currentmarker}{\pgfqpoint{-0.048611in}{0.000000in}}{\pgfqpoint{-0.000000in}{0.000000in}}{%
\pgfpathmoveto{\pgfqpoint{-0.000000in}{0.000000in}}%
\pgfpathlineto{\pgfqpoint{-0.048611in}{0.000000in}}%
\pgfusepath{stroke,fill}%
}%
\begin{pgfscope}%
\pgfsys@transformshift{0.519743in}{0.466738in}%
\pgfsys@useobject{currentmarker}{}%
\end{pgfscope}%
\end{pgfscope}%
\begin{pgfscope}%
\definecolor{textcolor}{rgb}{0.000000,0.000000,0.000000}%
\pgfsetstrokecolor{textcolor}%
\pgfsetfillcolor{textcolor}%
\pgftext[x=0.223446in, y=0.429805in, left, base]{\color{textcolor}{\rmfamily\fontsize{7.000000}{8.400000}\selectfont\catcode`\^=\active\def^{\ifmmode\sp\else\^{}\fi}\catcode`\%=\active\def
\end{pgfscope}%
\begin{pgfscope}%
\pgfpathrectangle{\pgfqpoint{0.519743in}{0.414757in}}{\pgfqpoint{0.983773in}{1.143573in}}%
\pgfusepath{clip}%
\pgfsetrectcap%
\pgfsetroundjoin%
\pgfsetlinewidth{0.803000pt}%
\definecolor{currentstroke}{rgb}{0.690196,0.690196,0.690196}%
\pgfsetstrokecolor{currentstroke}%
\pgfsetdash{}{0pt}%
\pgfpathmoveto{\pgfqpoint{0.519743in}{0.726641in}}%
\pgfpathlineto{\pgfqpoint{1.503516in}{0.726641in}}%
\pgfusepath{stroke}%
\end{pgfscope}%
\begin{pgfscope}%
\pgfsetbuttcap%
\pgfsetroundjoin%
\definecolor{currentfill}{rgb}{0.000000,0.000000,0.000000}%
\pgfsetfillcolor{currentfill}%
\pgfsetlinewidth{0.803000pt}%
\definecolor{currentstroke}{rgb}{0.000000,0.000000,0.000000}%
\pgfsetstrokecolor{currentstroke}%
\pgfsetdash{}{0pt}%
\pgfsys@defobject{currentmarker}{\pgfqpoint{-0.048611in}{0.000000in}}{\pgfqpoint{-0.000000in}{0.000000in}}{%
\pgfpathmoveto{\pgfqpoint{-0.000000in}{0.000000in}}%
\pgfpathlineto{\pgfqpoint{-0.048611in}{0.000000in}}%
\pgfusepath{stroke,fill}%
}%
\begin{pgfscope}%
\pgfsys@transformshift{0.519743in}{0.726641in}%
\pgfsys@useobject{currentmarker}{}%
\end{pgfscope}%
\end{pgfscope}%
\begin{pgfscope}%
\definecolor{textcolor}{rgb}{0.000000,0.000000,0.000000}%
\pgfsetstrokecolor{textcolor}%
\pgfsetfillcolor{textcolor}%
\pgftext[x=0.223446in, y=0.689708in, left, base]{\color{textcolor}{\rmfamily\fontsize{7.000000}{8.400000}\selectfont\catcode`\^=\active\def^{\ifmmode\sp\else\^{}\fi}\catcode`\%=\active\def
\end{pgfscope}%
\begin{pgfscope}%
\pgfpathrectangle{\pgfqpoint{0.519743in}{0.414757in}}{\pgfqpoint{0.983773in}{1.143573in}}%
\pgfusepath{clip}%
\pgfsetrectcap%
\pgfsetroundjoin%
\pgfsetlinewidth{0.803000pt}%
\definecolor{currentstroke}{rgb}{0.690196,0.690196,0.690196}%
\pgfsetstrokecolor{currentstroke}%
\pgfsetdash{}{0pt}%
\pgfpathmoveto{\pgfqpoint{0.519743in}{0.986544in}}%
\pgfpathlineto{\pgfqpoint{1.503516in}{0.986544in}}%
\pgfusepath{stroke}%
\end{pgfscope}%
\begin{pgfscope}%
\pgfsetbuttcap%
\pgfsetroundjoin%
\definecolor{currentfill}{rgb}{0.000000,0.000000,0.000000}%
\pgfsetfillcolor{currentfill}%
\pgfsetlinewidth{0.803000pt}%
\definecolor{currentstroke}{rgb}{0.000000,0.000000,0.000000}%
\pgfsetstrokecolor{currentstroke}%
\pgfsetdash{}{0pt}%
\pgfsys@defobject{currentmarker}{\pgfqpoint{-0.048611in}{0.000000in}}{\pgfqpoint{-0.000000in}{0.000000in}}{%
\pgfpathmoveto{\pgfqpoint{-0.000000in}{0.000000in}}%
\pgfpathlineto{\pgfqpoint{-0.048611in}{0.000000in}}%
\pgfusepath{stroke,fill}%
}%
\begin{pgfscope}%
\pgfsys@transformshift{0.519743in}{0.986544in}%
\pgfsys@useobject{currentmarker}{}%
\end{pgfscope}%
\end{pgfscope}%
\begin{pgfscope}%
\definecolor{textcolor}{rgb}{0.000000,0.000000,0.000000}%
\pgfsetstrokecolor{textcolor}%
\pgfsetfillcolor{textcolor}%
\pgftext[x=0.223446in, y=0.949611in, left, base]{\color{textcolor}{\rmfamily\fontsize{7.000000}{8.400000}\selectfont\catcode`\^=\active\def^{\ifmmode\sp\else\^{}\fi}\catcode`\%=\active\def
\end{pgfscope}%
\begin{pgfscope}%
\pgfpathrectangle{\pgfqpoint{0.519743in}{0.414757in}}{\pgfqpoint{0.983773in}{1.143573in}}%
\pgfusepath{clip}%
\pgfsetrectcap%
\pgfsetroundjoin%
\pgfsetlinewidth{0.803000pt}%
\definecolor{currentstroke}{rgb}{0.690196,0.690196,0.690196}%
\pgfsetstrokecolor{currentstroke}%
\pgfsetdash{}{0pt}%
\pgfpathmoveto{\pgfqpoint{0.519743in}{1.246447in}}%
\pgfpathlineto{\pgfqpoint{1.503516in}{1.246447in}}%
\pgfusepath{stroke}%
\end{pgfscope}%
\begin{pgfscope}%
\pgfsetbuttcap%
\pgfsetroundjoin%
\definecolor{currentfill}{rgb}{0.000000,0.000000,0.000000}%
\pgfsetfillcolor{currentfill}%
\pgfsetlinewidth{0.803000pt}%
\definecolor{currentstroke}{rgb}{0.000000,0.000000,0.000000}%
\pgfsetstrokecolor{currentstroke}%
\pgfsetdash{}{0pt}%
\pgfsys@defobject{currentmarker}{\pgfqpoint{-0.048611in}{0.000000in}}{\pgfqpoint{-0.000000in}{0.000000in}}{%
\pgfpathmoveto{\pgfqpoint{-0.000000in}{0.000000in}}%
\pgfpathlineto{\pgfqpoint{-0.048611in}{0.000000in}}%
\pgfusepath{stroke,fill}%
}%
\begin{pgfscope}%
\pgfsys@transformshift{0.519743in}{1.246447in}%
\pgfsys@useobject{currentmarker}{}%
\end{pgfscope}%
\end{pgfscope}%
\begin{pgfscope}%
\definecolor{textcolor}{rgb}{0.000000,0.000000,0.000000}%
\pgfsetstrokecolor{textcolor}%
\pgfsetfillcolor{textcolor}%
\pgftext[x=0.223446in, y=1.209513in, left, base]{\color{textcolor}{\rmfamily\fontsize{7.000000}{8.400000}\selectfont\catcode`\^=\active\def^{\ifmmode\sp\else\^{}\fi}\catcode`\%=\active\def
\end{pgfscope}%
\begin{pgfscope}%
\pgfpathrectangle{\pgfqpoint{0.519743in}{0.414757in}}{\pgfqpoint{0.983773in}{1.143573in}}%
\pgfusepath{clip}%
\pgfsetrectcap%
\pgfsetroundjoin%
\pgfsetlinewidth{0.803000pt}%
\definecolor{currentstroke}{rgb}{0.690196,0.690196,0.690196}%
\pgfsetstrokecolor{currentstroke}%
\pgfsetdash{}{0pt}%
\pgfpathmoveto{\pgfqpoint{0.519743in}{1.506349in}}%
\pgfpathlineto{\pgfqpoint{1.503516in}{1.506349in}}%
\pgfusepath{stroke}%
\end{pgfscope}%
\begin{pgfscope}%
\pgfsetbuttcap%
\pgfsetroundjoin%
\definecolor{currentfill}{rgb}{0.000000,0.000000,0.000000}%
\pgfsetfillcolor{currentfill}%
\pgfsetlinewidth{0.803000pt}%
\definecolor{currentstroke}{rgb}{0.000000,0.000000,0.000000}%
\pgfsetstrokecolor{currentstroke}%
\pgfsetdash{}{0pt}%
\pgfsys@defobject{currentmarker}{\pgfqpoint{-0.048611in}{0.000000in}}{\pgfqpoint{-0.000000in}{0.000000in}}{%
\pgfpathmoveto{\pgfqpoint{-0.000000in}{0.000000in}}%
\pgfpathlineto{\pgfqpoint{-0.048611in}{0.000000in}}%
\pgfusepath{stroke,fill}%
}%
\begin{pgfscope}%
\pgfsys@transformshift{0.519743in}{1.506349in}%
\pgfsys@useobject{currentmarker}{}%
\end{pgfscope}%
\end{pgfscope}%
\begin{pgfscope}%
\definecolor{textcolor}{rgb}{0.000000,0.000000,0.000000}%
\pgfsetstrokecolor{textcolor}%
\pgfsetfillcolor{textcolor}%
\pgftext[x=0.223446in, y=1.469416in, left, base]{\color{textcolor}{\rmfamily\fontsize{7.000000}{8.400000}\selectfont\catcode`\^=\active\def^{\ifmmode\sp\else\^{}\fi}\catcode`\%=\active\def
\end{pgfscope}%
\begin{pgfscope}%
\definecolor{textcolor}{rgb}{0.000000,0.000000,0.000000}%
\pgfsetstrokecolor{textcolor}%
\pgfsetfillcolor{textcolor}%
\pgftext[x=0.167891in,y=0.986544in,,bottom,rotate=90.000000]{\color{textcolor}{\rmfamily\fontsize{9.000000}{10.800000}\selectfont\catcode`\^=\active\def^{\ifmmode\sp\else\^{}\fi}\catcode`\%=\active\def
\end{pgfscope}%
\begin{pgfscope}%
\pgfpathrectangle{\pgfqpoint{0.519743in}{0.414757in}}{\pgfqpoint{0.983773in}{1.143573in}}%
\pgfusepath{clip}%
\pgfsetrectcap%
\pgfsetroundjoin%
\pgfsetlinewidth{1.505625pt}%
\definecolor{currentstroke}{rgb}{0.003922,0.450980,0.698039}%
\pgfsetstrokecolor{currentstroke}%
\pgfsetstrokeopacity{0.200000}%
\pgfsetdash{}{0pt}%
\pgfpathmoveto{\pgfqpoint{0.640479in}{0.469800in}}%
\pgfpathlineto{\pgfqpoint{0.698611in}{0.475183in}}%
\pgfpathlineto{\pgfqpoint{0.743328in}{0.481240in}}%
\pgfpathlineto{\pgfqpoint{0.788045in}{0.488731in}}%
\pgfpathlineto{\pgfqpoint{0.832762in}{0.497553in}}%
\pgfpathlineto{\pgfqpoint{0.877479in}{0.507602in}}%
\pgfpathlineto{\pgfqpoint{0.922196in}{0.518773in}}%
\pgfpathlineto{\pgfqpoint{0.966913in}{0.530959in}}%
\pgfpathlineto{\pgfqpoint{1.011630in}{0.544052in}}%
\pgfpathlineto{\pgfqpoint{1.056347in}{0.557943in}}%
\pgfpathlineto{\pgfqpoint{1.101064in}{0.572521in}}%
\pgfpathlineto{\pgfqpoint{1.145781in}{0.587677in}}%
\pgfpathlineto{\pgfqpoint{1.190498in}{0.603299in}}%
\pgfpathlineto{\pgfqpoint{1.235215in}{0.619272in}}%
\pgfpathlineto{\pgfqpoint{1.279932in}{0.635468in}}%
\pgfpathlineto{\pgfqpoint{1.324649in}{0.651731in}}%
\pgfpathlineto{\pgfqpoint{1.369365in}{0.667823in}}%
\pgfpathlineto{\pgfqpoint{1.414082in}{0.683166in}}%
\pgfpathlineto{\pgfqpoint{1.458799in}{0.637580in}}%
\pgfusepath{stroke}%
\end{pgfscope}%
\begin{pgfscope}%
\pgfpathrectangle{\pgfqpoint{0.519743in}{0.414757in}}{\pgfqpoint{0.983773in}{1.143573in}}%
\pgfusepath{clip}%
\pgfsetrectcap%
\pgfsetroundjoin%
\pgfsetlinewidth{1.505625pt}%
\definecolor{currentstroke}{rgb}{0.003922,0.450980,0.698039}%
\pgfsetstrokecolor{currentstroke}%
\pgfsetstrokeopacity{0.200000}%
\pgfsetdash{}{0pt}%
\pgfpathmoveto{\pgfqpoint{0.644312in}{0.469132in}}%
\pgfpathlineto{\pgfqpoint{0.698611in}{0.473285in}}%
\pgfpathlineto{\pgfqpoint{0.743328in}{0.478489in}}%
\pgfpathlineto{\pgfqpoint{0.788045in}{0.485180in}}%
\pgfpathlineto{\pgfqpoint{0.832762in}{0.493311in}}%
\pgfpathlineto{\pgfqpoint{0.877479in}{0.502817in}}%
\pgfpathlineto{\pgfqpoint{0.922196in}{0.513622in}}%
\pgfpathlineto{\pgfqpoint{0.966913in}{0.525634in}}%
\pgfpathlineto{\pgfqpoint{1.011630in}{0.538754in}}%
\pgfpathlineto{\pgfqpoint{1.056347in}{0.552872in}}%
\pgfpathlineto{\pgfqpoint{1.101064in}{0.567874in}}%
\pgfpathlineto{\pgfqpoint{1.145781in}{0.583637in}}%
\pgfpathlineto{\pgfqpoint{1.190498in}{0.600031in}}%
\pgfpathlineto{\pgfqpoint{1.235215in}{0.616917in}}%
\pgfpathlineto{\pgfqpoint{1.324649in}{0.651497in}}%
\pgfpathlineto{\pgfqpoint{1.369365in}{0.668685in}}%
\pgfpathlineto{\pgfqpoint{1.414082in}{0.684954in}}%
\pgfpathlineto{\pgfqpoint{1.458799in}{0.624913in}}%
\pgfusepath{stroke}%
\end{pgfscope}%
\begin{pgfscope}%
\pgfpathrectangle{\pgfqpoint{0.519743in}{0.414757in}}{\pgfqpoint{0.983773in}{1.143573in}}%
\pgfusepath{clip}%
\pgfsetrectcap%
\pgfsetroundjoin%
\pgfsetlinewidth{1.505625pt}%
\definecolor{currentstroke}{rgb}{0.003922,0.450980,0.698039}%
\pgfsetstrokecolor{currentstroke}%
\pgfsetstrokeopacity{0.200000}%
\pgfsetdash{}{0pt}%
\pgfpathmoveto{\pgfqpoint{0.638989in}{0.469853in}}%
\pgfpathlineto{\pgfqpoint{0.698611in}{0.475575in}}%
\pgfpathlineto{\pgfqpoint{0.743328in}{0.481881in}}%
\pgfpathlineto{\pgfqpoint{0.788045in}{0.489658in}}%
\pgfpathlineto{\pgfqpoint{0.832762in}{0.498794in}}%
\pgfpathlineto{\pgfqpoint{0.877479in}{0.509176in}}%
\pgfpathlineto{\pgfqpoint{0.922196in}{0.520688in}}%
\pgfpathlineto{\pgfqpoint{0.966913in}{0.533213in}}%
\pgfpathlineto{\pgfqpoint{1.011630in}{0.546636in}}%
\pgfpathlineto{\pgfqpoint{1.056347in}{0.560835in}}%
\pgfpathlineto{\pgfqpoint{1.101064in}{0.575694in}}%
\pgfpathlineto{\pgfqpoint{1.145781in}{0.591090in}}%
\pgfpathlineto{\pgfqpoint{1.190498in}{0.606901in}}%
\pgfpathlineto{\pgfqpoint{1.235215in}{0.622996in}}%
\pgfpathlineto{\pgfqpoint{1.279932in}{0.639225in}}%
\pgfpathlineto{\pgfqpoint{1.324649in}{0.655396in}}%
\pgfpathlineto{\pgfqpoint{1.369365in}{0.671192in}}%
\pgfpathlineto{\pgfqpoint{1.414082in}{0.685795in}}%
\pgfpathlineto{\pgfqpoint{1.458799in}{0.615220in}}%
\pgfusepath{stroke}%
\end{pgfscope}%
\begin{pgfscope}%
\pgfpathrectangle{\pgfqpoint{0.519743in}{0.414757in}}{\pgfqpoint{0.983773in}{1.143573in}}%
\pgfusepath{clip}%
\pgfsetrectcap%
\pgfsetroundjoin%
\pgfsetlinewidth{1.505625pt}%
\definecolor{currentstroke}{rgb}{0.003922,0.450980,0.698039}%
\pgfsetstrokecolor{currentstroke}%
\pgfsetstrokeopacity{0.200000}%
\pgfsetdash{}{0pt}%
\pgfpathmoveto{\pgfqpoint{0.646441in}{0.470172in}}%
\pgfpathlineto{\pgfqpoint{0.698611in}{0.475203in}}%
\pgfpathlineto{\pgfqpoint{0.743328in}{0.481328in}}%
\pgfpathlineto{\pgfqpoint{0.788045in}{0.488925in}}%
\pgfpathlineto{\pgfqpoint{0.832762in}{0.497891in}}%
\pgfpathlineto{\pgfqpoint{0.877479in}{0.508124in}}%
\pgfpathlineto{\pgfqpoint{0.922196in}{0.519515in}}%
\pgfpathlineto{\pgfqpoint{0.966913in}{0.531955in}}%
\pgfpathlineto{\pgfqpoint{1.011630in}{0.545331in}}%
\pgfpathlineto{\pgfqpoint{1.056347in}{0.559530in}}%
\pgfpathlineto{\pgfqpoint{1.101064in}{0.574439in}}%
\pgfpathlineto{\pgfqpoint{1.145781in}{0.589942in}}%
\pgfpathlineto{\pgfqpoint{1.190498in}{0.605924in}}%
\pgfpathlineto{\pgfqpoint{1.235215in}{0.622263in}}%
\pgfpathlineto{\pgfqpoint{1.279932in}{0.638828in}}%
\pgfpathlineto{\pgfqpoint{1.324649in}{0.655456in}}%
\pgfpathlineto{\pgfqpoint{1.369365in}{0.671903in}}%
\pgfpathlineto{\pgfqpoint{1.414082in}{0.687575in}}%
\pgfpathlineto{\pgfqpoint{1.458799in}{0.640740in}}%
\pgfusepath{stroke}%
\end{pgfscope}%
\begin{pgfscope}%
\pgfpathrectangle{\pgfqpoint{0.519743in}{0.414757in}}{\pgfqpoint{0.983773in}{1.143573in}}%
\pgfusepath{clip}%
\pgfsetrectcap%
\pgfsetroundjoin%
\pgfsetlinewidth{1.505625pt}%
\definecolor{currentstroke}{rgb}{0.003922,0.450980,0.698039}%
\pgfsetstrokecolor{currentstroke}%
\pgfsetstrokeopacity{0.200000}%
\pgfsetdash{}{0pt}%
\pgfpathmoveto{\pgfqpoint{0.642715in}{0.470007in}}%
\pgfpathlineto{\pgfqpoint{0.698611in}{0.475351in}}%
\pgfpathlineto{\pgfqpoint{0.743328in}{0.481531in}}%
\pgfpathlineto{\pgfqpoint{0.788045in}{0.489169in}}%
\pgfpathlineto{\pgfqpoint{0.832762in}{0.498161in}}%
\pgfpathlineto{\pgfqpoint{0.877479in}{0.508399in}}%
\pgfpathlineto{\pgfqpoint{0.922196in}{0.519772in}}%
\pgfpathlineto{\pgfqpoint{0.966913in}{0.532169in}}%
\pgfpathlineto{\pgfqpoint{1.011630in}{0.545476in}}%
\pgfpathlineto{\pgfqpoint{1.056347in}{0.559579in}}%
\pgfpathlineto{\pgfqpoint{1.101064in}{0.574363in}}%
\pgfpathlineto{\pgfqpoint{1.145781in}{0.589714in}}%
\pgfpathlineto{\pgfqpoint{1.190498in}{0.605512in}}%
\pgfpathlineto{\pgfqpoint{1.235215in}{0.621635in}}%
\pgfpathlineto{\pgfqpoint{1.279932in}{0.637944in}}%
\pgfpathlineto{\pgfqpoint{1.324649in}{0.654268in}}%
\pgfpathlineto{\pgfqpoint{1.369365in}{0.670331in}}%
\pgfpathlineto{\pgfqpoint{1.414082in}{0.685453in}}%
\pgfpathlineto{\pgfqpoint{1.458799in}{0.628694in}}%
\pgfusepath{stroke}%
\end{pgfscope}%
\begin{pgfscope}%
\pgfpathrectangle{\pgfqpoint{0.519743in}{0.414757in}}{\pgfqpoint{0.983773in}{1.143573in}}%
\pgfusepath{clip}%
\pgfsetbuttcap%
\pgfsetroundjoin%
\pgfsetlinewidth{1.505625pt}%
\definecolor{currentstroke}{rgb}{0.501961,0.501961,0.501961}%
\pgfsetstrokecolor{currentstroke}%
\pgfsetdash{{5.550000pt}{2.400000pt}}{0.000000pt}%
\pgfpathmoveto{\pgfqpoint{0.564460in}{0.726641in}}%
\pgfpathlineto{\pgfqpoint{1.458799in}{0.726641in}}%
\pgfusepath{stroke}%
\end{pgfscope}%
\begin{pgfscope}%
\pgfpathrectangle{\pgfqpoint{0.519743in}{0.414757in}}{\pgfqpoint{0.983773in}{1.143573in}}%
\pgfusepath{clip}%
\pgfsetbuttcap%
\pgfsetroundjoin%
\pgfsetlinewidth{1.505625pt}%
\definecolor{currentstroke}{rgb}{0.501961,0.501961,0.501961}%
\pgfsetstrokecolor{currentstroke}%
\pgfsetdash{{5.550000pt}{2.400000pt}}{0.000000pt}%
\pgfpathmoveto{\pgfqpoint{0.788045in}{0.466738in}}%
\pgfpathlineto{\pgfqpoint{0.788045in}{1.506349in}}%
\pgfusepath{stroke}%
\end{pgfscope}%
\begin{pgfscope}%
\pgfsetrectcap%
\pgfsetmiterjoin%
\pgfsetlinewidth{0.803000pt}%
\definecolor{currentstroke}{rgb}{0.000000,0.000000,0.000000}%
\pgfsetstrokecolor{currentstroke}%
\pgfsetdash{}{0pt}%
\pgfpathmoveto{\pgfqpoint{0.519743in}{0.414757in}}%
\pgfpathlineto{\pgfqpoint{0.519743in}{1.558330in}}%
\pgfusepath{stroke}%
\end{pgfscope}%
\begin{pgfscope}%
\pgfsetrectcap%
\pgfsetmiterjoin%
\pgfsetlinewidth{0.803000pt}%
\definecolor{currentstroke}{rgb}{0.000000,0.000000,0.000000}%
\pgfsetstrokecolor{currentstroke}%
\pgfsetdash{}{0pt}%
\pgfpathmoveto{\pgfqpoint{1.503516in}{0.414757in}}%
\pgfpathlineto{\pgfqpoint{1.503516in}{1.558330in}}%
\pgfusepath{stroke}%
\end{pgfscope}%
\begin{pgfscope}%
\pgfsetrectcap%
\pgfsetmiterjoin%
\pgfsetlinewidth{0.803000pt}%
\definecolor{currentstroke}{rgb}{0.000000,0.000000,0.000000}%
\pgfsetstrokecolor{currentstroke}%
\pgfsetdash{}{0pt}%
\pgfpathmoveto{\pgfqpoint{0.519743in}{0.414757in}}%
\pgfpathlineto{\pgfqpoint{1.503516in}{0.414757in}}%
\pgfusepath{stroke}%
\end{pgfscope}%
\begin{pgfscope}%
\pgfsetrectcap%
\pgfsetmiterjoin%
\pgfsetlinewidth{0.803000pt}%
\definecolor{currentstroke}{rgb}{0.000000,0.000000,0.000000}%
\pgfsetstrokecolor{currentstroke}%
\pgfsetdash{}{0pt}%
\pgfpathmoveto{\pgfqpoint{0.519743in}{1.558330in}}%
\pgfpathlineto{\pgfqpoint{1.503516in}{1.558330in}}%
\pgfusepath{stroke}%
\end{pgfscope}%
\begin{pgfscope}%
\pgfsetbuttcap%
\pgfsetmiterjoin%
\definecolor{currentfill}{rgb}{1.000000,1.000000,1.000000}%
\pgfsetfillcolor{currentfill}%
\pgfsetfillopacity{0.800000}%
\pgfsetlinewidth{1.003750pt}%
\definecolor{currentstroke}{rgb}{0.800000,0.800000,0.800000}%
\pgfsetstrokecolor{currentstroke}%
\pgfsetstrokeopacity{0.800000}%
\pgfsetdash{}{0pt}%
\pgfpathmoveto{\pgfqpoint{0.542878in}{1.008611in}}%
\pgfpathlineto{\pgfqpoint{1.537122in}{1.008611in}}%
\pgfpathquadraticcurveto{\pgfqpoint{1.556567in}{1.008611in}}{\pgfqpoint{1.556567in}{1.028056in}}%
\pgfpathlineto{\pgfqpoint{1.556567in}{1.446433in}}%
\pgfpathquadraticcurveto{\pgfqpoint{1.556567in}{1.465878in}}{\pgfqpoint{1.537122in}{1.465878in}}%
\pgfpathlineto{\pgfqpoint{0.542878in}{1.465878in}}%
\pgfpathquadraticcurveto{\pgfqpoint{0.523433in}{1.465878in}}{\pgfqpoint{0.523433in}{1.446433in}}%
\pgfpathlineto{\pgfqpoint{0.523433in}{1.028056in}}%
\pgfpathquadraticcurveto{\pgfqpoint{0.523433in}{1.008611in}}{\pgfqpoint{0.542878in}{1.008611in}}%
\pgfpathlineto{\pgfqpoint{0.542878in}{1.008611in}}%
\pgfpathclose%
\pgfusepath{stroke,fill}%
\end{pgfscope}%
\begin{pgfscope}%
\pgfsetbuttcap%
\pgfsetmiterjoin%
\definecolor{currentfill}{rgb}{0.007843,0.619608,0.447059}%
\pgfsetfillcolor{currentfill}%
\pgfsetfillopacity{0.850000}%
\pgfsetlinewidth{0.501875pt}%
\definecolor{currentstroke}{rgb}{0.000000,0.000000,0.000000}%
\pgfsetstrokecolor{currentstroke}%
\pgfsetstrokeopacity{0.850000}%
\pgfsetdash{}{0pt}%
\pgfpathmoveto{\pgfqpoint{0.562322in}{1.353123in}}%
\pgfpathlineto{\pgfqpoint{0.756766in}{1.353123in}}%
\pgfpathlineto{\pgfqpoint{0.756766in}{1.421178in}}%
\pgfpathlineto{\pgfqpoint{0.562322in}{1.421178in}}%
\pgfpathlineto{\pgfqpoint{0.562322in}{1.353123in}}%
\pgfpathclose%
\pgfusepath{stroke,fill}%
\end{pgfscope}%
\begin{pgfscope}%
\definecolor{textcolor}{rgb}{0.000000,0.000000,0.000000}%
\pgfsetstrokecolor{textcolor}%
\pgfsetfillcolor{textcolor}%
\pgftext[x=0.834544in,y=1.353123in,left,base]{\color{textcolor}{\rmfamily\fontsize{7.000000}{8.400000}\selectfont\catcode`\^=\active\def^{\ifmmode\sp\else\^{}\fi}\catcode`\%=\active\def
\end{pgfscope}%
\begin{pgfscope}%
\pgfsetbuttcap%
\pgfsetmiterjoin%
\definecolor{currentfill}{rgb}{0.870588,0.560784,0.011765}%
\pgfsetfillcolor{currentfill}%
\pgfsetfillopacity{0.850000}%
\pgfsetlinewidth{0.501875pt}%
\definecolor{currentstroke}{rgb}{0.000000,0.000000,0.000000}%
\pgfsetstrokecolor{currentstroke}%
\pgfsetstrokeopacity{0.850000}%
\pgfsetdash{}{0pt}%
\pgfpathmoveto{\pgfqpoint{0.562322in}{1.210423in}}%
\pgfpathlineto{\pgfqpoint{0.756766in}{1.210423in}}%
\pgfpathlineto{\pgfqpoint{0.756766in}{1.278478in}}%
\pgfpathlineto{\pgfqpoint{0.562322in}{1.278478in}}%
\pgfpathlineto{\pgfqpoint{0.562322in}{1.210423in}}%
\pgfpathclose%
\pgfusepath{stroke,fill}%
\end{pgfscope}%
\begin{pgfscope}%
\definecolor{textcolor}{rgb}{0.000000,0.000000,0.000000}%
\pgfsetstrokecolor{textcolor}%
\pgfsetfillcolor{textcolor}%
\pgftext[x=0.834544in,y=1.210423in,left,base]{\color{textcolor}{\rmfamily\fontsize{7.000000}{8.400000}\selectfont\catcode`\^=\active\def^{\ifmmode\sp\else\^{}\fi}\catcode`\%=\active\def
\end{pgfscope}%
\begin{pgfscope}%
\pgfsetbuttcap%
\pgfsetroundjoin%
\pgfsetlinewidth{0.803000pt}%
\definecolor{currentstroke}{rgb}{0.000000,0.000000,0.000000}%
\pgfsetstrokecolor{currentstroke}%
\pgfsetstrokeopacity{0.600000}%
\pgfsetdash{{2.960000pt}{1.280000pt}}{0.000000pt}%
\pgfpathmoveto{\pgfqpoint{0.562322in}{1.101751in}}%
\pgfpathlineto{\pgfqpoint{0.659544in}{1.101751in}}%
\pgfpathlineto{\pgfqpoint{0.756766in}{1.101751in}}%
\pgfusepath{stroke}%
\end{pgfscope}%
\begin{pgfscope}%
\definecolor{textcolor}{rgb}{0.000000,0.000000,0.000000}%
\pgfsetstrokecolor{textcolor}%
\pgfsetfillcolor{textcolor}%
\pgftext[x=0.834544in,y=1.067723in,left,base]{\color{textcolor}{\rmfamily\fontsize{7.000000}{8.400000}\selectfont\catcode`\^=\active\def^{\ifmmode\sp\else\^{}\fi}\catcode`\%=\active\def
\end{pgfscope}%
\end{pgfpicture}%
\makeatother%
\endgroup%

%% file: figures/calibration_maps/SIMPLE_QA_VERIFIED_Qwen3-30B-A3B-Instruct-2507.pgf
\begingroup%
\makeatletter%
\begin{pgfpicture}%
\pgfpathrectangle{\pgfpointorigin}{\pgfqpoint{1.600000in}{1.600000in}}%
\pgfusepath{use as bounding box, clip}%
\begin{pgfscope}%
\pgfsetbuttcap%
\pgfsetmiterjoin%
\definecolor{currentfill}{rgb}{1.000000,1.000000,1.000000}%
\pgfsetfillcolor{currentfill}%
\pgfsetlinewidth{0.000000pt}%
\definecolor{currentstroke}{rgb}{1.000000,1.000000,1.000000}%
\pgfsetstrokecolor{currentstroke}%
\pgfsetdash{}{0pt}%
\pgfpathmoveto{\pgfqpoint{0.000000in}{0.000000in}}%
\pgfpathlineto{\pgfqpoint{1.600000in}{0.000000in}}%
\pgfpathlineto{\pgfqpoint{1.600000in}{1.600000in}}%
\pgfpathlineto{\pgfqpoint{0.000000in}{1.600000in}}%
\pgfpathlineto{\pgfqpoint{0.000000in}{0.000000in}}%
\pgfpathclose%
\pgfusepath{fill}%
\end{pgfscope}%
\begin{pgfscope}%
\pgfsetbuttcap%
\pgfsetmiterjoin%
\definecolor{currentfill}{rgb}{1.000000,1.000000,1.000000}%
\pgfsetfillcolor{currentfill}%
\pgfsetlinewidth{0.000000pt}%
\definecolor{currentstroke}{rgb}{0.000000,0.000000,0.000000}%
\pgfsetstrokecolor{currentstroke}%
\pgfsetstrokeopacity{0.000000}%
\pgfsetdash{}{0pt}%
\pgfpathmoveto{\pgfqpoint{0.519743in}{0.414757in}}%
\pgfpathlineto{\pgfqpoint{1.503516in}{0.414757in}}%
\pgfpathlineto{\pgfqpoint{1.503516in}{1.558330in}}%
\pgfpathlineto{\pgfqpoint{0.519743in}{1.558330in}}%
\pgfpathlineto{\pgfqpoint{0.519743in}{0.414757in}}%
\pgfpathclose%
\pgfusepath{fill}%
\end{pgfscope}%
\begin{pgfscope}%
\pgfpathrectangle{\pgfqpoint{0.519743in}{0.414757in}}{\pgfqpoint{0.983773in}{1.143573in}}%
\pgfusepath{clip}%
\pgfsetbuttcap%
\pgfsetroundjoin%
\definecolor{currentfill}{rgb}{0.870588,0.560784,0.011765}%
\pgfsetfillcolor{currentfill}%
\pgfsetfillopacity{0.200000}%
\pgfsetlinewidth{1.003750pt}%
\definecolor{currentstroke}{rgb}{0.870588,0.560784,0.011765}%
\pgfsetstrokecolor{currentstroke}%
\pgfsetstrokeopacity{0.200000}%
\pgfsetdash{}{0pt}%
\pgfsys@defobject{currentmarker}{\pgfqpoint{0.564460in}{0.726641in}}{\pgfqpoint{0.788045in}{1.506349in}}{%
\pgfpathmoveto{\pgfqpoint{0.788045in}{0.726641in}}%
\pgfpathlineto{\pgfqpoint{0.564460in}{0.726641in}}%
\pgfpathlineto{\pgfqpoint{0.564460in}{1.506349in}}%
\pgfpathlineto{\pgfqpoint{0.788045in}{1.506349in}}%
\pgfpathlineto{\pgfqpoint{0.788045in}{1.506349in}}%
\pgfpathlineto{\pgfqpoint{0.788045in}{0.726641in}}%
\pgfpathlineto{\pgfqpoint{0.788045in}{0.726641in}}%
\pgfpathclose%
\pgfusepath{stroke,fill}%
}%
\begin{pgfscope}%
\pgfsys@transformshift{0.000000in}{0.000000in}%
\pgfsys@useobject{currentmarker}{}%
\end{pgfscope}%
\end{pgfscope}%
\begin{pgfscope}%
\pgfpathrectangle{\pgfqpoint{0.519743in}{0.414757in}}{\pgfqpoint{0.983773in}{1.143573in}}%
\pgfusepath{clip}%
\pgfsetbuttcap%
\pgfsetroundjoin%
\definecolor{currentfill}{rgb}{0.007843,0.619608,0.447059}%
\pgfsetfillcolor{currentfill}%
\pgfsetfillopacity{0.200000}%
\pgfsetlinewidth{1.003750pt}%
\definecolor{currentstroke}{rgb}{0.007843,0.619608,0.447059}%
\pgfsetstrokecolor{currentstroke}%
\pgfsetstrokeopacity{0.200000}%
\pgfsetdash{}{0pt}%
\pgfsys@defobject{currentmarker}{\pgfqpoint{0.788045in}{0.726641in}}{\pgfqpoint{1.458799in}{1.506349in}}{%
\pgfpathmoveto{\pgfqpoint{1.458799in}{0.726641in}}%
\pgfpathlineto{\pgfqpoint{0.788045in}{0.726641in}}%
\pgfpathlineto{\pgfqpoint{0.788045in}{1.506349in}}%
\pgfpathlineto{\pgfqpoint{1.458799in}{1.506349in}}%
\pgfpathlineto{\pgfqpoint{1.458799in}{1.506349in}}%
\pgfpathlineto{\pgfqpoint{1.458799in}{0.726641in}}%
\pgfpathlineto{\pgfqpoint{1.458799in}{0.726641in}}%
\pgfpathclose%
\pgfusepath{stroke,fill}%
}%
\begin{pgfscope}%
\pgfsys@transformshift{0.000000in}{0.000000in}%
\pgfsys@useobject{currentmarker}{}%
\end{pgfscope}%
\end{pgfscope}%
\begin{pgfscope}%
\pgfpathrectangle{\pgfqpoint{0.519743in}{0.414757in}}{\pgfqpoint{0.983773in}{1.143573in}}%
\pgfusepath{clip}%
\pgfsetbuttcap%
\pgfsetroundjoin%
\definecolor{currentfill}{rgb}{0.870588,0.560784,0.011765}%
\pgfsetfillcolor{currentfill}%
\pgfsetfillopacity{0.200000}%
\pgfsetlinewidth{1.003750pt}%
\definecolor{currentstroke}{rgb}{0.870588,0.560784,0.011765}%
\pgfsetstrokecolor{currentstroke}%
\pgfsetstrokeopacity{0.200000}%
\pgfsetdash{}{0pt}%
\pgfsys@defobject{currentmarker}{\pgfqpoint{0.788045in}{0.466738in}}{\pgfqpoint{1.458799in}{0.726641in}}{%
\pgfpathmoveto{\pgfqpoint{1.458799in}{0.466738in}}%
\pgfpathlineto{\pgfqpoint{0.788045in}{0.466738in}}%
\pgfpathlineto{\pgfqpoint{0.788045in}{0.726641in}}%
\pgfpathlineto{\pgfqpoint{1.458799in}{0.726641in}}%
\pgfpathlineto{\pgfqpoint{1.458799in}{0.726641in}}%
\pgfpathlineto{\pgfqpoint{1.458799in}{0.466738in}}%
\pgfpathlineto{\pgfqpoint{1.458799in}{0.466738in}}%
\pgfpathclose%
\pgfusepath{stroke,fill}%
}%
\begin{pgfscope}%
\pgfsys@transformshift{0.000000in}{0.000000in}%
\pgfsys@useobject{currentmarker}{}%
\end{pgfscope}%
\end{pgfscope}%
\begin{pgfscope}%
\pgfpathrectangle{\pgfqpoint{0.519743in}{0.414757in}}{\pgfqpoint{0.983773in}{1.143573in}}%
\pgfusepath{clip}%
\pgfsetbuttcap%
\pgfsetroundjoin%
\definecolor{currentfill}{rgb}{0.007843,0.619608,0.447059}%
\pgfsetfillcolor{currentfill}%
\pgfsetfillopacity{0.200000}%
\pgfsetlinewidth{1.003750pt}%
\definecolor{currentstroke}{rgb}{0.007843,0.619608,0.447059}%
\pgfsetstrokecolor{currentstroke}%
\pgfsetstrokeopacity{0.200000}%
\pgfsetdash{}{0pt}%
\pgfsys@defobject{currentmarker}{\pgfqpoint{0.564460in}{0.466738in}}{\pgfqpoint{0.788045in}{0.726641in}}{%
\pgfpathmoveto{\pgfqpoint{0.788045in}{0.466738in}}%
\pgfpathlineto{\pgfqpoint{0.564460in}{0.466738in}}%
\pgfpathlineto{\pgfqpoint{0.564460in}{0.726641in}}%
\pgfpathlineto{\pgfqpoint{0.788045in}{0.726641in}}%
\pgfpathlineto{\pgfqpoint{0.788045in}{0.726641in}}%
\pgfpathlineto{\pgfqpoint{0.788045in}{0.466738in}}%
\pgfpathlineto{\pgfqpoint{0.788045in}{0.466738in}}%
\pgfpathclose%
\pgfusepath{stroke,fill}%
}%
\begin{pgfscope}%
\pgfsys@transformshift{0.000000in}{0.000000in}%
\pgfsys@useobject{currentmarker}{}%
\end{pgfscope}%
\end{pgfscope}%
\begin{pgfscope}%
\pgfpathrectangle{\pgfqpoint{0.519743in}{0.414757in}}{\pgfqpoint{0.983773in}{1.143573in}}%
\pgfusepath{clip}%
\pgfsetrectcap%
\pgfsetroundjoin%
\pgfsetlinewidth{0.803000pt}%
\definecolor{currentstroke}{rgb}{0.690196,0.690196,0.690196}%
\pgfsetstrokecolor{currentstroke}%
\pgfsetdash{}{0pt}%
\pgfpathmoveto{\pgfqpoint{0.564460in}{0.414757in}}%
\pgfpathlineto{\pgfqpoint{0.564460in}{1.558330in}}%
\pgfusepath{stroke}%
\end{pgfscope}%
\begin{pgfscope}%
\pgfsetbuttcap%
\pgfsetroundjoin%
\definecolor{currentfill}{rgb}{0.000000,0.000000,0.000000}%
\pgfsetfillcolor{currentfill}%
\pgfsetlinewidth{0.803000pt}%
\definecolor{currentstroke}{rgb}{0.000000,0.000000,0.000000}%
\pgfsetstrokecolor{currentstroke}%
\pgfsetdash{}{0pt}%
\pgfsys@defobject{currentmarker}{\pgfqpoint{0.000000in}{-0.048611in}}{\pgfqpoint{0.000000in}{0.000000in}}{%
\pgfpathmoveto{\pgfqpoint{0.000000in}{0.000000in}}%
\pgfpathlineto{\pgfqpoint{0.000000in}{-0.048611in}}%
\pgfusepath{stroke,fill}%
}%
\begin{pgfscope}%
\pgfsys@transformshift{0.564460in}{0.414757in}%
\pgfsys@useobject{currentmarker}{}%
\end{pgfscope}%
\end{pgfscope}%
\begin{pgfscope}%
\definecolor{textcolor}{rgb}{0.000000,0.000000,0.000000}%
\pgfsetstrokecolor{textcolor}%
\pgfsetfillcolor{textcolor}%
\pgftext[x=0.564460in,y=0.317535in,,top]{\color{textcolor}{\rmfamily\fontsize{7.000000}{8.400000}\selectfont\catcode`\^=\active\def^{\ifmmode\sp\else\^{}\fi}\catcode`\%=\active\def
\end{pgfscope}%
\begin{pgfscope}%
\pgfpathrectangle{\pgfqpoint{0.519743in}{0.414757in}}{\pgfqpoint{0.983773in}{1.143573in}}%
\pgfusepath{clip}%
\pgfsetrectcap%
\pgfsetroundjoin%
\pgfsetlinewidth{0.803000pt}%
\definecolor{currentstroke}{rgb}{0.690196,0.690196,0.690196}%
\pgfsetstrokecolor{currentstroke}%
\pgfsetdash{}{0pt}%
\pgfpathmoveto{\pgfqpoint{0.788045in}{0.414757in}}%
\pgfpathlineto{\pgfqpoint{0.788045in}{1.558330in}}%
\pgfusepath{stroke}%
\end{pgfscope}%
\begin{pgfscope}%
\pgfsetbuttcap%
\pgfsetroundjoin%
\definecolor{currentfill}{rgb}{0.000000,0.000000,0.000000}%
\pgfsetfillcolor{currentfill}%
\pgfsetlinewidth{0.803000pt}%
\definecolor{currentstroke}{rgb}{0.000000,0.000000,0.000000}%
\pgfsetstrokecolor{currentstroke}%
\pgfsetdash{}{0pt}%
\pgfsys@defobject{currentmarker}{\pgfqpoint{0.000000in}{-0.048611in}}{\pgfqpoint{0.000000in}{0.000000in}}{%
\pgfpathmoveto{\pgfqpoint{0.000000in}{0.000000in}}%
\pgfpathlineto{\pgfqpoint{0.000000in}{-0.048611in}}%
\pgfusepath{stroke,fill}%
}%
\begin{pgfscope}%
\pgfsys@transformshift{0.788045in}{0.414757in}%
\pgfsys@useobject{currentmarker}{}%
\end{pgfscope}%
\end{pgfscope}%
\begin{pgfscope}%
\definecolor{textcolor}{rgb}{0.000000,0.000000,0.000000}%
\pgfsetstrokecolor{textcolor}%
\pgfsetfillcolor{textcolor}%
\pgftext[x=0.788045in,y=0.317535in,,top]{\color{textcolor}{\rmfamily\fontsize{7.000000}{8.400000}\selectfont\catcode`\^=\active\def^{\ifmmode\sp\else\^{}\fi}\catcode`\%=\active\def
\end{pgfscope}%
\begin{pgfscope}%
\pgfpathrectangle{\pgfqpoint{0.519743in}{0.414757in}}{\pgfqpoint{0.983773in}{1.143573in}}%
\pgfusepath{clip}%
\pgfsetrectcap%
\pgfsetroundjoin%
\pgfsetlinewidth{0.803000pt}%
\definecolor{currentstroke}{rgb}{0.690196,0.690196,0.690196}%
\pgfsetstrokecolor{currentstroke}%
\pgfsetdash{}{0pt}%
\pgfpathmoveto{\pgfqpoint{1.011630in}{0.414757in}}%
\pgfpathlineto{\pgfqpoint{1.011630in}{1.558330in}}%
\pgfusepath{stroke}%
\end{pgfscope}%
\begin{pgfscope}%
\pgfsetbuttcap%
\pgfsetroundjoin%
\definecolor{currentfill}{rgb}{0.000000,0.000000,0.000000}%
\pgfsetfillcolor{currentfill}%
\pgfsetlinewidth{0.803000pt}%
\definecolor{currentstroke}{rgb}{0.000000,0.000000,0.000000}%
\pgfsetstrokecolor{currentstroke}%
\pgfsetdash{}{0pt}%
\pgfsys@defobject{currentmarker}{\pgfqpoint{0.000000in}{-0.048611in}}{\pgfqpoint{0.000000in}{0.000000in}}{%
\pgfpathmoveto{\pgfqpoint{0.000000in}{0.000000in}}%
\pgfpathlineto{\pgfqpoint{0.000000in}{-0.048611in}}%
\pgfusepath{stroke,fill}%
}%
\begin{pgfscope}%
\pgfsys@transformshift{1.011630in}{0.414757in}%
\pgfsys@useobject{currentmarker}{}%
\end{pgfscope}%
\end{pgfscope}%
\begin{pgfscope}%
\definecolor{textcolor}{rgb}{0.000000,0.000000,0.000000}%
\pgfsetstrokecolor{textcolor}%
\pgfsetfillcolor{textcolor}%
\pgftext[x=1.011630in,y=0.317535in,,top]{\color{textcolor}{\rmfamily\fontsize{7.000000}{8.400000}\selectfont\catcode`\^=\active\def^{\ifmmode\sp\else\^{}\fi}\catcode`\%=\active\def
\end{pgfscope}%
\begin{pgfscope}%
\pgfpathrectangle{\pgfqpoint{0.519743in}{0.414757in}}{\pgfqpoint{0.983773in}{1.143573in}}%
\pgfusepath{clip}%
\pgfsetrectcap%
\pgfsetroundjoin%
\pgfsetlinewidth{0.803000pt}%
\definecolor{currentstroke}{rgb}{0.690196,0.690196,0.690196}%
\pgfsetstrokecolor{currentstroke}%
\pgfsetdash{}{0pt}%
\pgfpathmoveto{\pgfqpoint{1.235215in}{0.414757in}}%
\pgfpathlineto{\pgfqpoint{1.235215in}{1.558330in}}%
\pgfusepath{stroke}%
\end{pgfscope}%
\begin{pgfscope}%
\pgfsetbuttcap%
\pgfsetroundjoin%
\definecolor{currentfill}{rgb}{0.000000,0.000000,0.000000}%
\pgfsetfillcolor{currentfill}%
\pgfsetlinewidth{0.803000pt}%
\definecolor{currentstroke}{rgb}{0.000000,0.000000,0.000000}%
\pgfsetstrokecolor{currentstroke}%
\pgfsetdash{}{0pt}%
\pgfsys@defobject{currentmarker}{\pgfqpoint{0.000000in}{-0.048611in}}{\pgfqpoint{0.000000in}{0.000000in}}{%
\pgfpathmoveto{\pgfqpoint{0.000000in}{0.000000in}}%
\pgfpathlineto{\pgfqpoint{0.000000in}{-0.048611in}}%
\pgfusepath{stroke,fill}%
}%
\begin{pgfscope}%
\pgfsys@transformshift{1.235215in}{0.414757in}%
\pgfsys@useobject{currentmarker}{}%
\end{pgfscope}%
\end{pgfscope}%
\begin{pgfscope}%
\definecolor{textcolor}{rgb}{0.000000,0.000000,0.000000}%
\pgfsetstrokecolor{textcolor}%
\pgfsetfillcolor{textcolor}%
\pgftext[x=1.235215in,y=0.317535in,,top]{\color{textcolor}{\rmfamily\fontsize{7.000000}{8.400000}\selectfont\catcode`\^=\active\def^{\ifmmode\sp\else\^{}\fi}\catcode`\%=\active\def
\end{pgfscope}%
\begin{pgfscope}%
\pgfpathrectangle{\pgfqpoint{0.519743in}{0.414757in}}{\pgfqpoint{0.983773in}{1.143573in}}%
\pgfusepath{clip}%
\pgfsetrectcap%
\pgfsetroundjoin%
\pgfsetlinewidth{0.803000pt}%
\definecolor{currentstroke}{rgb}{0.690196,0.690196,0.690196}%
\pgfsetstrokecolor{currentstroke}%
\pgfsetdash{}{0pt}%
\pgfpathmoveto{\pgfqpoint{1.458799in}{0.414757in}}%
\pgfpathlineto{\pgfqpoint{1.458799in}{1.558330in}}%
\pgfusepath{stroke}%
\end{pgfscope}%
\begin{pgfscope}%
\pgfsetbuttcap%
\pgfsetroundjoin%
\definecolor{currentfill}{rgb}{0.000000,0.000000,0.000000}%
\pgfsetfillcolor{currentfill}%
\pgfsetlinewidth{0.803000pt}%
\definecolor{currentstroke}{rgb}{0.000000,0.000000,0.000000}%
\pgfsetstrokecolor{currentstroke}%
\pgfsetdash{}{0pt}%
\pgfsys@defobject{currentmarker}{\pgfqpoint{0.000000in}{-0.048611in}}{\pgfqpoint{0.000000in}{0.000000in}}{%
\pgfpathmoveto{\pgfqpoint{0.000000in}{0.000000in}}%
\pgfpathlineto{\pgfqpoint{0.000000in}{-0.048611in}}%
\pgfusepath{stroke,fill}%
}%
\begin{pgfscope}%
\pgfsys@transformshift{1.458799in}{0.414757in}%
\pgfsys@useobject{currentmarker}{}%
\end{pgfscope}%
\end{pgfscope}%
\begin{pgfscope}%
\definecolor{textcolor}{rgb}{0.000000,0.000000,0.000000}%
\pgfsetstrokecolor{textcolor}%
\pgfsetfillcolor{textcolor}%
\pgftext[x=1.458799in,y=0.317535in,,top]{\color{textcolor}{\rmfamily\fontsize{7.000000}{8.400000}\selectfont\catcode`\^=\active\def^{\ifmmode\sp\else\^{}\fi}\catcode`\%=\active\def
\end{pgfscope}%
\begin{pgfscope}%
\definecolor{textcolor}{rgb}{0.000000,0.000000,0.000000}%
\pgfsetstrokecolor{textcolor}%
\pgfsetfillcolor{textcolor}%
\pgftext[x=1.011630in,y=0.167891in,,top]{\color{textcolor}{\rmfamily\fontsize{9.000000}{10.800000}\selectfont\catcode`\^=\active\def^{\ifmmode\sp\else\^{}\fi}\catcode`\%=\active\def
\end{pgfscope}%
\begin{pgfscope}%
\pgfpathrectangle{\pgfqpoint{0.519743in}{0.414757in}}{\pgfqpoint{0.983773in}{1.143573in}}%
\pgfusepath{clip}%
\pgfsetrectcap%
\pgfsetroundjoin%
\pgfsetlinewidth{0.803000pt}%
\definecolor{currentstroke}{rgb}{0.690196,0.690196,0.690196}%
\pgfsetstrokecolor{currentstroke}%
\pgfsetdash{}{0pt}%
\pgfpathmoveto{\pgfqpoint{0.519743in}{0.466738in}}%
\pgfpathlineto{\pgfqpoint{1.503516in}{0.466738in}}%
\pgfusepath{stroke}%
\end{pgfscope}%
\begin{pgfscope}%
\pgfsetbuttcap%
\pgfsetroundjoin%
\definecolor{currentfill}{rgb}{0.000000,0.000000,0.000000}%
\pgfsetfillcolor{currentfill}%
\pgfsetlinewidth{0.803000pt}%
\definecolor{currentstroke}{rgb}{0.000000,0.000000,0.000000}%
\pgfsetstrokecolor{currentstroke}%
\pgfsetdash{}{0pt}%
\pgfsys@defobject{currentmarker}{\pgfqpoint{-0.048611in}{0.000000in}}{\pgfqpoint{-0.000000in}{0.000000in}}{%
\pgfpathmoveto{\pgfqpoint{-0.000000in}{0.000000in}}%
\pgfpathlineto{\pgfqpoint{-0.048611in}{0.000000in}}%
\pgfusepath{stroke,fill}%
}%
\begin{pgfscope}%
\pgfsys@transformshift{0.519743in}{0.466738in}%
\pgfsys@useobject{currentmarker}{}%
\end{pgfscope}%
\end{pgfscope}%
\begin{pgfscope}%
\definecolor{textcolor}{rgb}{0.000000,0.000000,0.000000}%
\pgfsetstrokecolor{textcolor}%
\pgfsetfillcolor{textcolor}%
\pgftext[x=0.223446in, y=0.429805in, left, base]{\color{textcolor}{\rmfamily\fontsize{7.000000}{8.400000}\selectfont\catcode`\^=\active\def^{\ifmmode\sp\else\^{}\fi}\catcode`\%=\active\def
\end{pgfscope}%
\begin{pgfscope}%
\pgfpathrectangle{\pgfqpoint{0.519743in}{0.414757in}}{\pgfqpoint{0.983773in}{1.143573in}}%
\pgfusepath{clip}%
\pgfsetrectcap%
\pgfsetroundjoin%
\pgfsetlinewidth{0.803000pt}%
\definecolor{currentstroke}{rgb}{0.690196,0.690196,0.690196}%
\pgfsetstrokecolor{currentstroke}%
\pgfsetdash{}{0pt}%
\pgfpathmoveto{\pgfqpoint{0.519743in}{0.726641in}}%
\pgfpathlineto{\pgfqpoint{1.503516in}{0.726641in}}%
\pgfusepath{stroke}%
\end{pgfscope}%
\begin{pgfscope}%
\pgfsetbuttcap%
\pgfsetroundjoin%
\definecolor{currentfill}{rgb}{0.000000,0.000000,0.000000}%
\pgfsetfillcolor{currentfill}%
\pgfsetlinewidth{0.803000pt}%
\definecolor{currentstroke}{rgb}{0.000000,0.000000,0.000000}%
\pgfsetstrokecolor{currentstroke}%
\pgfsetdash{}{0pt}%
\pgfsys@defobject{currentmarker}{\pgfqpoint{-0.048611in}{0.000000in}}{\pgfqpoint{-0.000000in}{0.000000in}}{%
\pgfpathmoveto{\pgfqpoint{-0.000000in}{0.000000in}}%
\pgfpathlineto{\pgfqpoint{-0.048611in}{0.000000in}}%
\pgfusepath{stroke,fill}%
}%
\begin{pgfscope}%
\pgfsys@transformshift{0.519743in}{0.726641in}%
\pgfsys@useobject{currentmarker}{}%
\end{pgfscope}%
\end{pgfscope}%
\begin{pgfscope}%
\definecolor{textcolor}{rgb}{0.000000,0.000000,0.000000}%
\pgfsetstrokecolor{textcolor}%
\pgfsetfillcolor{textcolor}%
\pgftext[x=0.223446in, y=0.689708in, left, base]{\color{textcolor}{\rmfamily\fontsize{7.000000}{8.400000}\selectfont\catcode`\^=\active\def^{\ifmmode\sp\else\^{}\fi}\catcode`\%=\active\def
\end{pgfscope}%
\begin{pgfscope}%
\pgfpathrectangle{\pgfqpoint{0.519743in}{0.414757in}}{\pgfqpoint{0.983773in}{1.143573in}}%
\pgfusepath{clip}%
\pgfsetrectcap%
\pgfsetroundjoin%
\pgfsetlinewidth{0.803000pt}%
\definecolor{currentstroke}{rgb}{0.690196,0.690196,0.690196}%
\pgfsetstrokecolor{currentstroke}%
\pgfsetdash{}{0pt}%
\pgfpathmoveto{\pgfqpoint{0.519743in}{0.986544in}}%
\pgfpathlineto{\pgfqpoint{1.503516in}{0.986544in}}%
\pgfusepath{stroke}%
\end{pgfscope}%
\begin{pgfscope}%
\pgfsetbuttcap%
\pgfsetroundjoin%
\definecolor{currentfill}{rgb}{0.000000,0.000000,0.000000}%
\pgfsetfillcolor{currentfill}%
\pgfsetlinewidth{0.803000pt}%
\definecolor{currentstroke}{rgb}{0.000000,0.000000,0.000000}%
\pgfsetstrokecolor{currentstroke}%
\pgfsetdash{}{0pt}%
\pgfsys@defobject{currentmarker}{\pgfqpoint{-0.048611in}{0.000000in}}{\pgfqpoint{-0.000000in}{0.000000in}}{%
\pgfpathmoveto{\pgfqpoint{-0.000000in}{0.000000in}}%
\pgfpathlineto{\pgfqpoint{-0.048611in}{0.000000in}}%
\pgfusepath{stroke,fill}%
}%
\begin{pgfscope}%
\pgfsys@transformshift{0.519743in}{0.986544in}%
\pgfsys@useobject{currentmarker}{}%
\end{pgfscope}%
\end{pgfscope}%
\begin{pgfscope}%
\definecolor{textcolor}{rgb}{0.000000,0.000000,0.000000}%
\pgfsetstrokecolor{textcolor}%
\pgfsetfillcolor{textcolor}%
\pgftext[x=0.223446in, y=0.949611in, left, base]{\color{textcolor}{\rmfamily\fontsize{7.000000}{8.400000}\selectfont\catcode`\^=\active\def^{\ifmmode\sp\else\^{}\fi}\catcode`\%=\active\def
\end{pgfscope}%
\begin{pgfscope}%
\pgfpathrectangle{\pgfqpoint{0.519743in}{0.414757in}}{\pgfqpoint{0.983773in}{1.143573in}}%
\pgfusepath{clip}%
\pgfsetrectcap%
\pgfsetroundjoin%
\pgfsetlinewidth{0.803000pt}%
\definecolor{currentstroke}{rgb}{0.690196,0.690196,0.690196}%
\pgfsetstrokecolor{currentstroke}%
\pgfsetdash{}{0pt}%
\pgfpathmoveto{\pgfqpoint{0.519743in}{1.246447in}}%
\pgfpathlineto{\pgfqpoint{1.503516in}{1.246447in}}%
\pgfusepath{stroke}%
\end{pgfscope}%
\begin{pgfscope}%
\pgfsetbuttcap%
\pgfsetroundjoin%
\definecolor{currentfill}{rgb}{0.000000,0.000000,0.000000}%
\pgfsetfillcolor{currentfill}%
\pgfsetlinewidth{0.803000pt}%
\definecolor{currentstroke}{rgb}{0.000000,0.000000,0.000000}%
\pgfsetstrokecolor{currentstroke}%
\pgfsetdash{}{0pt}%
\pgfsys@defobject{currentmarker}{\pgfqpoint{-0.048611in}{0.000000in}}{\pgfqpoint{-0.000000in}{0.000000in}}{%
\pgfpathmoveto{\pgfqpoint{-0.000000in}{0.000000in}}%
\pgfpathlineto{\pgfqpoint{-0.048611in}{0.000000in}}%
\pgfusepath{stroke,fill}%
}%
\begin{pgfscope}%
\pgfsys@transformshift{0.519743in}{1.246447in}%
\pgfsys@useobject{currentmarker}{}%
\end{pgfscope}%
\end{pgfscope}%
\begin{pgfscope}%
\definecolor{textcolor}{rgb}{0.000000,0.000000,0.000000}%
\pgfsetstrokecolor{textcolor}%
\pgfsetfillcolor{textcolor}%
\pgftext[x=0.223446in, y=1.209513in, left, base]{\color{textcolor}{\rmfamily\fontsize{7.000000}{8.400000}\selectfont\catcode`\^=\active\def^{\ifmmode\sp\else\^{}\fi}\catcode`\%=\active\def
\end{pgfscope}%
\begin{pgfscope}%
\pgfpathrectangle{\pgfqpoint{0.519743in}{0.414757in}}{\pgfqpoint{0.983773in}{1.143573in}}%
\pgfusepath{clip}%
\pgfsetrectcap%
\pgfsetroundjoin%
\pgfsetlinewidth{0.803000pt}%
\definecolor{currentstroke}{rgb}{0.690196,0.690196,0.690196}%
\pgfsetstrokecolor{currentstroke}%
\pgfsetdash{}{0pt}%
\pgfpathmoveto{\pgfqpoint{0.519743in}{1.506349in}}%
\pgfpathlineto{\pgfqpoint{1.503516in}{1.506349in}}%
\pgfusepath{stroke}%
\end{pgfscope}%
\begin{pgfscope}%
\pgfsetbuttcap%
\pgfsetroundjoin%
\definecolor{currentfill}{rgb}{0.000000,0.000000,0.000000}%
\pgfsetfillcolor{currentfill}%
\pgfsetlinewidth{0.803000pt}%
\definecolor{currentstroke}{rgb}{0.000000,0.000000,0.000000}%
\pgfsetstrokecolor{currentstroke}%
\pgfsetdash{}{0pt}%
\pgfsys@defobject{currentmarker}{\pgfqpoint{-0.048611in}{0.000000in}}{\pgfqpoint{-0.000000in}{0.000000in}}{%
\pgfpathmoveto{\pgfqpoint{-0.000000in}{0.000000in}}%
\pgfpathlineto{\pgfqpoint{-0.048611in}{0.000000in}}%
\pgfusepath{stroke,fill}%
}%
\begin{pgfscope}%
\pgfsys@transformshift{0.519743in}{1.506349in}%
\pgfsys@useobject{currentmarker}{}%
\end{pgfscope}%
\end{pgfscope}%
\begin{pgfscope}%
\definecolor{textcolor}{rgb}{0.000000,0.000000,0.000000}%
\pgfsetstrokecolor{textcolor}%
\pgfsetfillcolor{textcolor}%
\pgftext[x=0.223446in, y=1.469416in, left, base]{\color{textcolor}{\rmfamily\fontsize{7.000000}{8.400000}\selectfont\catcode`\^=\active\def^{\ifmmode\sp\else\^{}\fi}\catcode`\%=\active\def
\end{pgfscope}%
\begin{pgfscope}%
\definecolor{textcolor}{rgb}{0.000000,0.000000,0.000000}%
\pgfsetstrokecolor{textcolor}%
\pgfsetfillcolor{textcolor}%
\pgftext[x=0.167891in,y=0.986544in,,bottom,rotate=90.000000]{\color{textcolor}{\rmfamily\fontsize{9.000000}{10.800000}\selectfont\catcode`\^=\active\def^{\ifmmode\sp\else\^{}\fi}\catcode`\%=\active\def
\end{pgfscope}%
\begin{pgfscope}%
\pgfpathrectangle{\pgfqpoint{0.519743in}{0.414757in}}{\pgfqpoint{0.983773in}{1.143573in}}%
\pgfusepath{clip}%
\pgfsetrectcap%
\pgfsetroundjoin%
\pgfsetlinewidth{1.505625pt}%
\definecolor{currentstroke}{rgb}{0.003922,0.450980,0.698039}%
\pgfsetstrokecolor{currentstroke}%
\pgfsetstrokeopacity{0.200000}%
\pgfsetdash{}{0pt}%
\pgfpathmoveto{\pgfqpoint{0.634730in}{0.471030in}}%
\pgfpathlineto{\pgfqpoint{0.698611in}{0.479352in}}%
\pgfpathlineto{\pgfqpoint{0.743328in}{0.487637in}}%
\pgfpathlineto{\pgfqpoint{0.788045in}{0.497552in}}%
\pgfpathlineto{\pgfqpoint{0.832762in}{0.508922in}}%
\pgfpathlineto{\pgfqpoint{0.877479in}{0.521584in}}%
\pgfpathlineto{\pgfqpoint{0.922196in}{0.535388in}}%
\pgfpathlineto{\pgfqpoint{0.966913in}{0.550189in}}%
\pgfpathlineto{\pgfqpoint{1.011630in}{0.565855in}}%
\pgfpathlineto{\pgfqpoint{1.056347in}{0.582260in}}%
\pgfpathlineto{\pgfqpoint{1.101064in}{0.599286in}}%
\pgfpathlineto{\pgfqpoint{1.145781in}{0.616828in}}%
\pgfpathlineto{\pgfqpoint{1.190498in}{0.634791in}}%
\pgfpathlineto{\pgfqpoint{1.235215in}{0.653092in}}%
\pgfpathlineto{\pgfqpoint{1.279932in}{0.671669in}}%
\pgfpathlineto{\pgfqpoint{1.324649in}{0.690490in}}%
\pgfpathlineto{\pgfqpoint{1.369365in}{0.709588in}}%
\pgfpathlineto{\pgfqpoint{1.414082in}{0.729237in}}%
\pgfpathlineto{\pgfqpoint{1.458799in}{0.797063in}}%
\pgfusepath{stroke}%
\end{pgfscope}%
\begin{pgfscope}%
\pgfpathrectangle{\pgfqpoint{0.519743in}{0.414757in}}{\pgfqpoint{0.983773in}{1.143573in}}%
\pgfusepath{clip}%
\pgfsetrectcap%
\pgfsetroundjoin%
\pgfsetlinewidth{1.505625pt}%
\definecolor{currentstroke}{rgb}{0.003922,0.450980,0.698039}%
\pgfsetstrokecolor{currentstroke}%
\pgfsetstrokeopacity{0.200000}%
\pgfsetdash{}{0pt}%
\pgfpathmoveto{\pgfqpoint{0.642715in}{0.471662in}}%
\pgfpathlineto{\pgfqpoint{0.698611in}{0.479075in}}%
\pgfpathlineto{\pgfqpoint{0.743328in}{0.487270in}}%
\pgfpathlineto{\pgfqpoint{0.788045in}{0.497117in}}%
\pgfpathlineto{\pgfqpoint{0.832762in}{0.508443in}}%
\pgfpathlineto{\pgfqpoint{0.877479in}{0.521090in}}%
\pgfpathlineto{\pgfqpoint{0.922196in}{0.534907in}}%
\pgfpathlineto{\pgfqpoint{0.966913in}{0.549750in}}%
\pgfpathlineto{\pgfqpoint{1.011630in}{0.565485in}}%
\pgfpathlineto{\pgfqpoint{1.056347in}{0.581984in}}%
\pgfpathlineto{\pgfqpoint{1.101064in}{0.599129in}}%
\pgfpathlineto{\pgfqpoint{1.145781in}{0.616810in}}%
\pgfpathlineto{\pgfqpoint{1.190498in}{0.634930in}}%
\pgfpathlineto{\pgfqpoint{1.235215in}{0.653404in}}%
\pgfpathlineto{\pgfqpoint{1.279932in}{0.672167in}}%
\pgfpathlineto{\pgfqpoint{1.324649in}{0.691184in}}%
\pgfpathlineto{\pgfqpoint{1.369365in}{0.710487in}}%
\pgfpathlineto{\pgfqpoint{1.414082in}{0.730348in}}%
\pgfpathlineto{\pgfqpoint{1.458799in}{0.798626in}}%
\pgfusepath{stroke}%
\end{pgfscope}%
\begin{pgfscope}%
\pgfpathrectangle{\pgfqpoint{0.519743in}{0.414757in}}{\pgfqpoint{0.983773in}{1.143573in}}%
\pgfusepath{clip}%
\pgfsetrectcap%
\pgfsetroundjoin%
\pgfsetlinewidth{1.505625pt}%
\definecolor{currentstroke}{rgb}{0.003922,0.450980,0.698039}%
\pgfsetstrokecolor{currentstroke}%
\pgfsetstrokeopacity{0.200000}%
\pgfsetdash{}{0pt}%
\pgfpathmoveto{\pgfqpoint{0.638989in}{0.471204in}}%
\pgfpathlineto{\pgfqpoint{0.698611in}{0.478844in}}%
\pgfpathlineto{\pgfqpoint{0.743328in}{0.486934in}}%
\pgfpathlineto{\pgfqpoint{0.788045in}{0.496678in}}%
\pgfpathlineto{\pgfqpoint{0.832762in}{0.507909in}}%
\pgfpathlineto{\pgfqpoint{0.877479in}{0.520472in}}%
\pgfpathlineto{\pgfqpoint{0.922196in}{0.534219in}}%
\pgfpathlineto{\pgfqpoint{0.966913in}{0.549011in}}%
\pgfpathlineto{\pgfqpoint{1.011630in}{0.564714in}}%
\pgfpathlineto{\pgfqpoint{1.056347in}{0.581204in}}%
\pgfpathlineto{\pgfqpoint{1.101064in}{0.598362in}}%
\pgfpathlineto{\pgfqpoint{1.145781in}{0.616083in}}%
\pgfpathlineto{\pgfqpoint{1.190498in}{0.634271in}}%
\pgfpathlineto{\pgfqpoint{1.235215in}{0.652846in}}%
\pgfpathlineto{\pgfqpoint{1.279932in}{0.671750in}}%
\pgfpathlineto{\pgfqpoint{1.324649in}{0.690963in}}%
\pgfpathlineto{\pgfqpoint{1.369365in}{0.710548in}}%
\pgfpathlineto{\pgfqpoint{1.414082in}{0.730881in}}%
\pgfpathlineto{\pgfqpoint{1.458799in}{0.815394in}}%
\pgfusepath{stroke}%
\end{pgfscope}%
\begin{pgfscope}%
\pgfpathrectangle{\pgfqpoint{0.519743in}{0.414757in}}{\pgfqpoint{0.983773in}{1.143573in}}%
\pgfusepath{clip}%
\pgfsetrectcap%
\pgfsetroundjoin%
\pgfsetlinewidth{1.505625pt}%
\definecolor{currentstroke}{rgb}{0.003922,0.450980,0.698039}%
\pgfsetstrokecolor{currentstroke}%
\pgfsetstrokeopacity{0.200000}%
\pgfsetdash{}{0pt}%
\pgfpathmoveto{\pgfqpoint{0.653894in}{0.474311in}}%
\pgfpathlineto{\pgfqpoint{0.698611in}{0.481705in}}%
\pgfpathlineto{\pgfqpoint{0.743328in}{0.490916in}}%
\pgfpathlineto{\pgfqpoint{0.788045in}{0.501689in}}%
\pgfpathlineto{\pgfqpoint{0.832762in}{0.513812in}}%
\pgfpathlineto{\pgfqpoint{0.877479in}{0.527101in}}%
\pgfpathlineto{\pgfqpoint{0.922196in}{0.541393in}}%
\pgfpathlineto{\pgfqpoint{0.966913in}{0.556538in}}%
\pgfpathlineto{\pgfqpoint{1.011630in}{0.572401in}}%
\pgfpathlineto{\pgfqpoint{1.056347in}{0.588858in}}%
\pgfpathlineto{\pgfqpoint{1.101064in}{0.605799in}}%
\pgfpathlineto{\pgfqpoint{1.145781in}{0.623123in}}%
\pgfpathlineto{\pgfqpoint{1.190498in}{0.640744in}}%
\pgfpathlineto{\pgfqpoint{1.235215in}{0.658588in}}%
\pgfpathlineto{\pgfqpoint{1.279932in}{0.676598in}}%
\pgfpathlineto{\pgfqpoint{1.324649in}{0.694745in}}%
\pgfpathlineto{\pgfqpoint{1.369365in}{0.713053in}}%
\pgfpathlineto{\pgfqpoint{1.414082in}{0.731743in}}%
\pgfpathlineto{\pgfqpoint{1.458799in}{0.789122in}}%
\pgfusepath{stroke}%
\end{pgfscope}%
\begin{pgfscope}%
\pgfpathrectangle{\pgfqpoint{0.519743in}{0.414757in}}{\pgfqpoint{0.983773in}{1.143573in}}%
\pgfusepath{clip}%
\pgfsetrectcap%
\pgfsetroundjoin%
\pgfsetlinewidth{1.505625pt}%
\definecolor{currentstroke}{rgb}{0.003922,0.450980,0.698039}%
\pgfsetstrokecolor{currentstroke}%
\pgfsetstrokeopacity{0.200000}%
\pgfsetdash{}{0pt}%
\pgfpathmoveto{\pgfqpoint{0.653894in}{0.471759in}}%
\pgfpathlineto{\pgfqpoint{0.698611in}{0.477412in}}%
\pgfpathlineto{\pgfqpoint{0.743328in}{0.484905in}}%
\pgfpathlineto{\pgfqpoint{0.788045in}{0.494093in}}%
\pgfpathlineto{\pgfqpoint{0.832762in}{0.504838in}}%
\pgfpathlineto{\pgfqpoint{0.877479in}{0.517003in}}%
\pgfpathlineto{\pgfqpoint{0.922196in}{0.530453in}}%
\pgfpathlineto{\pgfqpoint{0.966913in}{0.545052in}}%
\pgfpathlineto{\pgfqpoint{1.011630in}{0.560670in}}%
\pgfpathlineto{\pgfqpoint{1.056347in}{0.577179in}}%
\pgfpathlineto{\pgfqpoint{1.101064in}{0.594456in}}%
\pgfpathlineto{\pgfqpoint{1.145781in}{0.612388in}}%
\pgfpathlineto{\pgfqpoint{1.190498in}{0.630869in}}%
\pgfpathlineto{\pgfqpoint{1.235215in}{0.649809in}}%
\pgfpathlineto{\pgfqpoint{1.279932in}{0.669134in}}%
\pgfpathlineto{\pgfqpoint{1.324649in}{0.688806in}}%
\pgfpathlineto{\pgfqpoint{1.369365in}{0.708862in}}%
\pgfpathlineto{\pgfqpoint{1.414082in}{0.729610in}}%
\pgfpathlineto{\pgfqpoint{1.458799in}{0.805908in}}%
\pgfusepath{stroke}%
\end{pgfscope}%
\begin{pgfscope}%
\pgfpathrectangle{\pgfqpoint{0.519743in}{0.414757in}}{\pgfqpoint{0.983773in}{1.143573in}}%
\pgfusepath{clip}%
\pgfsetbuttcap%
\pgfsetroundjoin%
\pgfsetlinewidth{1.505625pt}%
\definecolor{currentstroke}{rgb}{0.501961,0.501961,0.501961}%
\pgfsetstrokecolor{currentstroke}%
\pgfsetdash{{5.550000pt}{2.400000pt}}{0.000000pt}%
\pgfpathmoveto{\pgfqpoint{0.564460in}{0.726641in}}%
\pgfpathlineto{\pgfqpoint{1.458799in}{0.726641in}}%
\pgfusepath{stroke}%
\end{pgfscope}%
\begin{pgfscope}%
\pgfpathrectangle{\pgfqpoint{0.519743in}{0.414757in}}{\pgfqpoint{0.983773in}{1.143573in}}%
\pgfusepath{clip}%
\pgfsetbuttcap%
\pgfsetroundjoin%
\pgfsetlinewidth{1.505625pt}%
\definecolor{currentstroke}{rgb}{0.501961,0.501961,0.501961}%
\pgfsetstrokecolor{currentstroke}%
\pgfsetdash{{5.550000pt}{2.400000pt}}{0.000000pt}%
\pgfpathmoveto{\pgfqpoint{0.788045in}{0.466738in}}%
\pgfpathlineto{\pgfqpoint{0.788045in}{1.506349in}}%
\pgfusepath{stroke}%
\end{pgfscope}%
\begin{pgfscope}%
\pgfsetrectcap%
\pgfsetmiterjoin%
\pgfsetlinewidth{0.803000pt}%
\definecolor{currentstroke}{rgb}{0.000000,0.000000,0.000000}%
\pgfsetstrokecolor{currentstroke}%
\pgfsetdash{}{0pt}%
\pgfpathmoveto{\pgfqpoint{0.519743in}{0.414757in}}%
\pgfpathlineto{\pgfqpoint{0.519743in}{1.558330in}}%
\pgfusepath{stroke}%
\end{pgfscope}%
\begin{pgfscope}%
\pgfsetrectcap%
\pgfsetmiterjoin%
\pgfsetlinewidth{0.803000pt}%
\definecolor{currentstroke}{rgb}{0.000000,0.000000,0.000000}%
\pgfsetstrokecolor{currentstroke}%
\pgfsetdash{}{0pt}%
\pgfpathmoveto{\pgfqpoint{1.503516in}{0.414757in}}%
\pgfpathlineto{\pgfqpoint{1.503516in}{1.558330in}}%
\pgfusepath{stroke}%
\end{pgfscope}%
\begin{pgfscope}%
\pgfsetrectcap%
\pgfsetmiterjoin%
\pgfsetlinewidth{0.803000pt}%
\definecolor{currentstroke}{rgb}{0.000000,0.000000,0.000000}%
\pgfsetstrokecolor{currentstroke}%
\pgfsetdash{}{0pt}%
\pgfpathmoveto{\pgfqpoint{0.519743in}{0.414757in}}%
\pgfpathlineto{\pgfqpoint{1.503516in}{0.414757in}}%
\pgfusepath{stroke}%
\end{pgfscope}%
\begin{pgfscope}%
\pgfsetrectcap%
\pgfsetmiterjoin%
\pgfsetlinewidth{0.803000pt}%
\definecolor{currentstroke}{rgb}{0.000000,0.000000,0.000000}%
\pgfsetstrokecolor{currentstroke}%
\pgfsetdash{}{0pt}%
\pgfpathmoveto{\pgfqpoint{0.519743in}{1.558330in}}%
\pgfpathlineto{\pgfqpoint{1.503516in}{1.558330in}}%
\pgfusepath{stroke}%
\end{pgfscope}%
\begin{pgfscope}%
\pgfsetbuttcap%
\pgfsetmiterjoin%
\definecolor{currentfill}{rgb}{1.000000,1.000000,1.000000}%
\pgfsetfillcolor{currentfill}%
\pgfsetfillopacity{0.800000}%
\pgfsetlinewidth{1.003750pt}%
\definecolor{currentstroke}{rgb}{0.800000,0.800000,0.800000}%
\pgfsetstrokecolor{currentstroke}%
\pgfsetstrokeopacity{0.800000}%
\pgfsetdash{}{0pt}%
\pgfpathmoveto{\pgfqpoint{0.542878in}{1.008611in}}%
\pgfpathlineto{\pgfqpoint{1.537122in}{1.008611in}}%
\pgfpathquadraticcurveto{\pgfqpoint{1.556567in}{1.008611in}}{\pgfqpoint{1.556567in}{1.028056in}}%
\pgfpathlineto{\pgfqpoint{1.556567in}{1.446433in}}%
\pgfpathquadraticcurveto{\pgfqpoint{1.556567in}{1.465878in}}{\pgfqpoint{1.537122in}{1.465878in}}%
\pgfpathlineto{\pgfqpoint{0.542878in}{1.465878in}}%
\pgfpathquadraticcurveto{\pgfqpoint{0.523433in}{1.465878in}}{\pgfqpoint{0.523433in}{1.446433in}}%
\pgfpathlineto{\pgfqpoint{0.523433in}{1.028056in}}%
\pgfpathquadraticcurveto{\pgfqpoint{0.523433in}{1.008611in}}{\pgfqpoint{0.542878in}{1.008611in}}%
\pgfpathlineto{\pgfqpoint{0.542878in}{1.008611in}}%
\pgfpathclose%
\pgfusepath{stroke,fill}%
\end{pgfscope}%
\begin{pgfscope}%
\pgfsetbuttcap%
\pgfsetmiterjoin%
\definecolor{currentfill}{rgb}{0.007843,0.619608,0.447059}%
\pgfsetfillcolor{currentfill}%
\pgfsetfillopacity{0.850000}%
\pgfsetlinewidth{0.501875pt}%
\definecolor{currentstroke}{rgb}{0.000000,0.000000,0.000000}%
\pgfsetstrokecolor{currentstroke}%
\pgfsetstrokeopacity{0.850000}%
\pgfsetdash{}{0pt}%
\pgfpathmoveto{\pgfqpoint{0.562322in}{1.353123in}}%
\pgfpathlineto{\pgfqpoint{0.756766in}{1.353123in}}%
\pgfpathlineto{\pgfqpoint{0.756766in}{1.421178in}}%
\pgfpathlineto{\pgfqpoint{0.562322in}{1.421178in}}%
\pgfpathlineto{\pgfqpoint{0.562322in}{1.353123in}}%
\pgfpathclose%
\pgfusepath{stroke,fill}%
\end{pgfscope}%
\begin{pgfscope}%
\definecolor{textcolor}{rgb}{0.000000,0.000000,0.000000}%
\pgfsetstrokecolor{textcolor}%
\pgfsetfillcolor{textcolor}%
\pgftext[x=0.834544in,y=1.353123in,left,base]{\color{textcolor}{\rmfamily\fontsize{7.000000}{8.400000}\selectfont\catcode`\^=\active\def^{\ifmmode\sp\else\^{}\fi}\catcode`\%=\active\def
\end{pgfscope}%
\begin{pgfscope}%
\pgfsetbuttcap%
\pgfsetmiterjoin%
\definecolor{currentfill}{rgb}{0.870588,0.560784,0.011765}%
\pgfsetfillcolor{currentfill}%
\pgfsetfillopacity{0.850000}%
\pgfsetlinewidth{0.501875pt}%
\definecolor{currentstroke}{rgb}{0.000000,0.000000,0.000000}%
\pgfsetstrokecolor{currentstroke}%
\pgfsetstrokeopacity{0.850000}%
\pgfsetdash{}{0pt}%
\pgfpathmoveto{\pgfqpoint{0.562322in}{1.210423in}}%
\pgfpathlineto{\pgfqpoint{0.756766in}{1.210423in}}%
\pgfpathlineto{\pgfqpoint{0.756766in}{1.278478in}}%
\pgfpathlineto{\pgfqpoint{0.562322in}{1.278478in}}%
\pgfpathlineto{\pgfqpoint{0.562322in}{1.210423in}}%
\pgfpathclose%
\pgfusepath{stroke,fill}%
\end{pgfscope}%
\begin{pgfscope}%
\definecolor{textcolor}{rgb}{0.000000,0.000000,0.000000}%
\pgfsetstrokecolor{textcolor}%
\pgfsetfillcolor{textcolor}%
\pgftext[x=0.834544in,y=1.210423in,left,base]{\color{textcolor}{\rmfamily\fontsize{7.000000}{8.400000}\selectfont\catcode`\^=\active\def^{\ifmmode\sp\else\^{}\fi}\catcode`\%=\active\def
\end{pgfscope}%
\begin{pgfscope}%
\pgfsetbuttcap%
\pgfsetroundjoin%
\pgfsetlinewidth{0.803000pt}%
\definecolor{currentstroke}{rgb}{0.000000,0.000000,0.000000}%
\pgfsetstrokecolor{currentstroke}%
\pgfsetstrokeopacity{0.600000}%
\pgfsetdash{{2.960000pt}{1.280000pt}}{0.000000pt}%
\pgfpathmoveto{\pgfqpoint{0.562322in}{1.101751in}}%
\pgfpathlineto{\pgfqpoint{0.659544in}{1.101751in}}%
\pgfpathlineto{\pgfqpoint{0.756766in}{1.101751in}}%
\pgfusepath{stroke}%
\end{pgfscope}%
\begin{pgfscope}%
\definecolor{textcolor}{rgb}{0.000000,0.000000,0.000000}%
\pgfsetstrokecolor{textcolor}%
\pgfsetfillcolor{textcolor}%
\pgftext[x=0.834544in,y=1.067723in,left,base]{\color{textcolor}{\rmfamily\fontsize{7.000000}{8.400000}\selectfont\catcode`\^=\active\def^{\ifmmode\sp\else\^{}\fi}\catcode`\%=\active\def
\end{pgfscope}%
\end{pgfpicture}%
\makeatother%
\endgroup%

%% file: figures/calibration_maps/TRIVIAQA_gemma-3-4b-it.pgf
\begingroup%
\makeatletter%
\begin{pgfpicture}%
\pgfpathrectangle{\pgfpointorigin}{\pgfqpoint{1.600000in}{1.600000in}}%
\pgfusepath{use as bounding box, clip}%
\begin{pgfscope}%
\pgfsetbuttcap%
\pgfsetmiterjoin%
\definecolor{currentfill}{rgb}{1.000000,1.000000,1.000000}%
\pgfsetfillcolor{currentfill}%
\pgfsetlinewidth{0.000000pt}%
\definecolor{currentstroke}{rgb}{1.000000,1.000000,1.000000}%
\pgfsetstrokecolor{currentstroke}%
\pgfsetdash{}{0pt}%
\pgfpathmoveto{\pgfqpoint{0.000000in}{0.000000in}}%
\pgfpathlineto{\pgfqpoint{1.600000in}{0.000000in}}%
\pgfpathlineto{\pgfqpoint{1.600000in}{1.600000in}}%
\pgfpathlineto{\pgfqpoint{0.000000in}{1.600000in}}%
\pgfpathlineto{\pgfqpoint{0.000000in}{0.000000in}}%
\pgfpathclose%
\pgfusepath{fill}%
\end{pgfscope}%
\begin{pgfscope}%
\pgfsetbuttcap%
\pgfsetmiterjoin%
\definecolor{currentfill}{rgb}{1.000000,1.000000,1.000000}%
\pgfsetfillcolor{currentfill}%
\pgfsetlinewidth{0.000000pt}%
\definecolor{currentstroke}{rgb}{0.000000,0.000000,0.000000}%
\pgfsetstrokecolor{currentstroke}%
\pgfsetstrokeopacity{0.000000}%
\pgfsetdash{}{0pt}%
\pgfpathmoveto{\pgfqpoint{0.519743in}{0.414757in}}%
\pgfpathlineto{\pgfqpoint{1.503516in}{0.414757in}}%
\pgfpathlineto{\pgfqpoint{1.503516in}{1.558330in}}%
\pgfpathlineto{\pgfqpoint{0.519743in}{1.558330in}}%
\pgfpathlineto{\pgfqpoint{0.519743in}{0.414757in}}%
\pgfpathclose%
\pgfusepath{fill}%
\end{pgfscope}%
\begin{pgfscope}%
\pgfpathrectangle{\pgfqpoint{0.519743in}{0.414757in}}{\pgfqpoint{0.983773in}{1.143573in}}%
\pgfusepath{clip}%
\pgfsetbuttcap%
\pgfsetroundjoin%
\definecolor{currentfill}{rgb}{0.870588,0.560784,0.011765}%
\pgfsetfillcolor{currentfill}%
\pgfsetfillopacity{0.200000}%
\pgfsetlinewidth{1.003750pt}%
\definecolor{currentstroke}{rgb}{0.870588,0.560784,0.011765}%
\pgfsetstrokecolor{currentstroke}%
\pgfsetstrokeopacity{0.200000}%
\pgfsetdash{}{0pt}%
\pgfsys@defobject{currentmarker}{\pgfqpoint{0.564460in}{0.726641in}}{\pgfqpoint{0.788045in}{1.506349in}}{%
\pgfpathmoveto{\pgfqpoint{0.788045in}{0.726641in}}%
\pgfpathlineto{\pgfqpoint{0.564460in}{0.726641in}}%
\pgfpathlineto{\pgfqpoint{0.564460in}{1.506349in}}%
\pgfpathlineto{\pgfqpoint{0.788045in}{1.506349in}}%
\pgfpathlineto{\pgfqpoint{0.788045in}{1.506349in}}%
\pgfpathlineto{\pgfqpoint{0.788045in}{0.726641in}}%
\pgfpathlineto{\pgfqpoint{0.788045in}{0.726641in}}%
\pgfpathclose%
\pgfusepath{stroke,fill}%
}%
\begin{pgfscope}%
\pgfsys@transformshift{0.000000in}{0.000000in}%
\pgfsys@useobject{currentmarker}{}%
\end{pgfscope}%
\end{pgfscope}%
\begin{pgfscope}%
\pgfpathrectangle{\pgfqpoint{0.519743in}{0.414757in}}{\pgfqpoint{0.983773in}{1.143573in}}%
\pgfusepath{clip}%
\pgfsetbuttcap%
\pgfsetroundjoin%
\definecolor{currentfill}{rgb}{0.007843,0.619608,0.447059}%
\pgfsetfillcolor{currentfill}%
\pgfsetfillopacity{0.200000}%
\pgfsetlinewidth{1.003750pt}%
\definecolor{currentstroke}{rgb}{0.007843,0.619608,0.447059}%
\pgfsetstrokecolor{currentstroke}%
\pgfsetstrokeopacity{0.200000}%
\pgfsetdash{}{0pt}%
\pgfsys@defobject{currentmarker}{\pgfqpoint{0.788045in}{0.726641in}}{\pgfqpoint{1.458799in}{1.506349in}}{%
\pgfpathmoveto{\pgfqpoint{1.458799in}{0.726641in}}%
\pgfpathlineto{\pgfqpoint{0.788045in}{0.726641in}}%
\pgfpathlineto{\pgfqpoint{0.788045in}{1.506349in}}%
\pgfpathlineto{\pgfqpoint{1.458799in}{1.506349in}}%
\pgfpathlineto{\pgfqpoint{1.458799in}{1.506349in}}%
\pgfpathlineto{\pgfqpoint{1.458799in}{0.726641in}}%
\pgfpathlineto{\pgfqpoint{1.458799in}{0.726641in}}%
\pgfpathclose%
\pgfusepath{stroke,fill}%
}%
\begin{pgfscope}%
\pgfsys@transformshift{0.000000in}{0.000000in}%
\pgfsys@useobject{currentmarker}{}%
\end{pgfscope}%
\end{pgfscope}%
\begin{pgfscope}%
\pgfpathrectangle{\pgfqpoint{0.519743in}{0.414757in}}{\pgfqpoint{0.983773in}{1.143573in}}%
\pgfusepath{clip}%
\pgfsetbuttcap%
\pgfsetroundjoin%
\definecolor{currentfill}{rgb}{0.870588,0.560784,0.011765}%
\pgfsetfillcolor{currentfill}%
\pgfsetfillopacity{0.200000}%
\pgfsetlinewidth{1.003750pt}%
\definecolor{currentstroke}{rgb}{0.870588,0.560784,0.011765}%
\pgfsetstrokecolor{currentstroke}%
\pgfsetstrokeopacity{0.200000}%
\pgfsetdash{}{0pt}%
\pgfsys@defobject{currentmarker}{\pgfqpoint{0.788045in}{0.466738in}}{\pgfqpoint{1.458799in}{0.726641in}}{%
\pgfpathmoveto{\pgfqpoint{1.458799in}{0.466738in}}%
\pgfpathlineto{\pgfqpoint{0.788045in}{0.466738in}}%
\pgfpathlineto{\pgfqpoint{0.788045in}{0.726641in}}%
\pgfpathlineto{\pgfqpoint{1.458799in}{0.726641in}}%
\pgfpathlineto{\pgfqpoint{1.458799in}{0.726641in}}%
\pgfpathlineto{\pgfqpoint{1.458799in}{0.466738in}}%
\pgfpathlineto{\pgfqpoint{1.458799in}{0.466738in}}%
\pgfpathclose%
\pgfusepath{stroke,fill}%
}%
\begin{pgfscope}%
\pgfsys@transformshift{0.000000in}{0.000000in}%
\pgfsys@useobject{currentmarker}{}%
\end{pgfscope}%
\end{pgfscope}%
\begin{pgfscope}%
\pgfpathrectangle{\pgfqpoint{0.519743in}{0.414757in}}{\pgfqpoint{0.983773in}{1.143573in}}%
\pgfusepath{clip}%
\pgfsetbuttcap%
\pgfsetroundjoin%
\definecolor{currentfill}{rgb}{0.007843,0.619608,0.447059}%
\pgfsetfillcolor{currentfill}%
\pgfsetfillopacity{0.200000}%
\pgfsetlinewidth{1.003750pt}%
\definecolor{currentstroke}{rgb}{0.007843,0.619608,0.447059}%
\pgfsetstrokecolor{currentstroke}%
\pgfsetstrokeopacity{0.200000}%
\pgfsetdash{}{0pt}%
\pgfsys@defobject{currentmarker}{\pgfqpoint{0.564460in}{0.466738in}}{\pgfqpoint{0.788045in}{0.726641in}}{%
\pgfpathmoveto{\pgfqpoint{0.788045in}{0.466738in}}%
\pgfpathlineto{\pgfqpoint{0.564460in}{0.466738in}}%
\pgfpathlineto{\pgfqpoint{0.564460in}{0.726641in}}%
\pgfpathlineto{\pgfqpoint{0.788045in}{0.726641in}}%
\pgfpathlineto{\pgfqpoint{0.788045in}{0.726641in}}%
\pgfpathlineto{\pgfqpoint{0.788045in}{0.466738in}}%
\pgfpathlineto{\pgfqpoint{0.788045in}{0.466738in}}%
\pgfpathclose%
\pgfusepath{stroke,fill}%
}%
\begin{pgfscope}%
\pgfsys@transformshift{0.000000in}{0.000000in}%
\pgfsys@useobject{currentmarker}{}%
\end{pgfscope}%
\end{pgfscope}%
\begin{pgfscope}%
\pgfpathrectangle{\pgfqpoint{0.519743in}{0.414757in}}{\pgfqpoint{0.983773in}{1.143573in}}%
\pgfusepath{clip}%
\pgfsetrectcap%
\pgfsetroundjoin%
\pgfsetlinewidth{0.803000pt}%
\definecolor{currentstroke}{rgb}{0.690196,0.690196,0.690196}%
\pgfsetstrokecolor{currentstroke}%
\pgfsetdash{}{0pt}%
\pgfpathmoveto{\pgfqpoint{0.564460in}{0.414757in}}%
\pgfpathlineto{\pgfqpoint{0.564460in}{1.558330in}}%
\pgfusepath{stroke}%
\end{pgfscope}%
\begin{pgfscope}%
\pgfsetbuttcap%
\pgfsetroundjoin%
\definecolor{currentfill}{rgb}{0.000000,0.000000,0.000000}%
\pgfsetfillcolor{currentfill}%
\pgfsetlinewidth{0.803000pt}%
\definecolor{currentstroke}{rgb}{0.000000,0.000000,0.000000}%
\pgfsetstrokecolor{currentstroke}%
\pgfsetdash{}{0pt}%
\pgfsys@defobject{currentmarker}{\pgfqpoint{0.000000in}{-0.048611in}}{\pgfqpoint{0.000000in}{0.000000in}}{%
\pgfpathmoveto{\pgfqpoint{0.000000in}{0.000000in}}%
\pgfpathlineto{\pgfqpoint{0.000000in}{-0.048611in}}%
\pgfusepath{stroke,fill}%
}%
\begin{pgfscope}%
\pgfsys@transformshift{0.564460in}{0.414757in}%
\pgfsys@useobject{currentmarker}{}%
\end{pgfscope}%
\end{pgfscope}%
\begin{pgfscope}%
\definecolor{textcolor}{rgb}{0.000000,0.000000,0.000000}%
\pgfsetstrokecolor{textcolor}%
\pgfsetfillcolor{textcolor}%
\pgftext[x=0.564460in,y=0.317535in,,top]{\color{textcolor}{\rmfamily\fontsize{7.000000}{8.400000}\selectfont\catcode`\^=\active\def^{\ifmmode\sp\else\^{}\fi}\catcode`\%=\active\def
\end{pgfscope}%
\begin{pgfscope}%
\pgfpathrectangle{\pgfqpoint{0.519743in}{0.414757in}}{\pgfqpoint{0.983773in}{1.143573in}}%
\pgfusepath{clip}%
\pgfsetrectcap%
\pgfsetroundjoin%
\pgfsetlinewidth{0.803000pt}%
\definecolor{currentstroke}{rgb}{0.690196,0.690196,0.690196}%
\pgfsetstrokecolor{currentstroke}%
\pgfsetdash{}{0pt}%
\pgfpathmoveto{\pgfqpoint{0.788045in}{0.414757in}}%
\pgfpathlineto{\pgfqpoint{0.788045in}{1.558330in}}%
\pgfusepath{stroke}%
\end{pgfscope}%
\begin{pgfscope}%
\pgfsetbuttcap%
\pgfsetroundjoin%
\definecolor{currentfill}{rgb}{0.000000,0.000000,0.000000}%
\pgfsetfillcolor{currentfill}%
\pgfsetlinewidth{0.803000pt}%
\definecolor{currentstroke}{rgb}{0.000000,0.000000,0.000000}%
\pgfsetstrokecolor{currentstroke}%
\pgfsetdash{}{0pt}%
\pgfsys@defobject{currentmarker}{\pgfqpoint{0.000000in}{-0.048611in}}{\pgfqpoint{0.000000in}{0.000000in}}{%
\pgfpathmoveto{\pgfqpoint{0.000000in}{0.000000in}}%
\pgfpathlineto{\pgfqpoint{0.000000in}{-0.048611in}}%
\pgfusepath{stroke,fill}%
}%
\begin{pgfscope}%
\pgfsys@transformshift{0.788045in}{0.414757in}%
\pgfsys@useobject{currentmarker}{}%
\end{pgfscope}%
\end{pgfscope}%
\begin{pgfscope}%
\definecolor{textcolor}{rgb}{0.000000,0.000000,0.000000}%
\pgfsetstrokecolor{textcolor}%
\pgfsetfillcolor{textcolor}%
\pgftext[x=0.788045in,y=0.317535in,,top]{\color{textcolor}{\rmfamily\fontsize{7.000000}{8.400000}\selectfont\catcode`\^=\active\def^{\ifmmode\sp\else\^{}\fi}\catcode`\%=\active\def
\end{pgfscope}%
\begin{pgfscope}%
\pgfpathrectangle{\pgfqpoint{0.519743in}{0.414757in}}{\pgfqpoint{0.983773in}{1.143573in}}%
\pgfusepath{clip}%
\pgfsetrectcap%
\pgfsetroundjoin%
\pgfsetlinewidth{0.803000pt}%
\definecolor{currentstroke}{rgb}{0.690196,0.690196,0.690196}%
\pgfsetstrokecolor{currentstroke}%
\pgfsetdash{}{0pt}%
\pgfpathmoveto{\pgfqpoint{1.011630in}{0.414757in}}%
\pgfpathlineto{\pgfqpoint{1.011630in}{1.558330in}}%
\pgfusepath{stroke}%
\end{pgfscope}%
\begin{pgfscope}%
\pgfsetbuttcap%
\pgfsetroundjoin%
\definecolor{currentfill}{rgb}{0.000000,0.000000,0.000000}%
\pgfsetfillcolor{currentfill}%
\pgfsetlinewidth{0.803000pt}%
\definecolor{currentstroke}{rgb}{0.000000,0.000000,0.000000}%
\pgfsetstrokecolor{currentstroke}%
\pgfsetdash{}{0pt}%
\pgfsys@defobject{currentmarker}{\pgfqpoint{0.000000in}{-0.048611in}}{\pgfqpoint{0.000000in}{0.000000in}}{%
\pgfpathmoveto{\pgfqpoint{0.000000in}{0.000000in}}%
\pgfpathlineto{\pgfqpoint{0.000000in}{-0.048611in}}%
\pgfusepath{stroke,fill}%
}%
\begin{pgfscope}%
\pgfsys@transformshift{1.011630in}{0.414757in}%
\pgfsys@useobject{currentmarker}{}%
\end{pgfscope}%
\end{pgfscope}%
\begin{pgfscope}%
\definecolor{textcolor}{rgb}{0.000000,0.000000,0.000000}%
\pgfsetstrokecolor{textcolor}%
\pgfsetfillcolor{textcolor}%
\pgftext[x=1.011630in,y=0.317535in,,top]{\color{textcolor}{\rmfamily\fontsize{7.000000}{8.400000}\selectfont\catcode`\^=\active\def^{\ifmmode\sp\else\^{}\fi}\catcode`\%=\active\def
\end{pgfscope}%
\begin{pgfscope}%
\pgfpathrectangle{\pgfqpoint{0.519743in}{0.414757in}}{\pgfqpoint{0.983773in}{1.143573in}}%
\pgfusepath{clip}%
\pgfsetrectcap%
\pgfsetroundjoin%
\pgfsetlinewidth{0.803000pt}%
\definecolor{currentstroke}{rgb}{0.690196,0.690196,0.690196}%
\pgfsetstrokecolor{currentstroke}%
\pgfsetdash{}{0pt}%
\pgfpathmoveto{\pgfqpoint{1.235215in}{0.414757in}}%
\pgfpathlineto{\pgfqpoint{1.235215in}{1.558330in}}%
\pgfusepath{stroke}%
\end{pgfscope}%
\begin{pgfscope}%
\pgfsetbuttcap%
\pgfsetroundjoin%
\definecolor{currentfill}{rgb}{0.000000,0.000000,0.000000}%
\pgfsetfillcolor{currentfill}%
\pgfsetlinewidth{0.803000pt}%
\definecolor{currentstroke}{rgb}{0.000000,0.000000,0.000000}%
\pgfsetstrokecolor{currentstroke}%
\pgfsetdash{}{0pt}%
\pgfsys@defobject{currentmarker}{\pgfqpoint{0.000000in}{-0.048611in}}{\pgfqpoint{0.000000in}{0.000000in}}{%
\pgfpathmoveto{\pgfqpoint{0.000000in}{0.000000in}}%
\pgfpathlineto{\pgfqpoint{0.000000in}{-0.048611in}}%
\pgfusepath{stroke,fill}%
}%
\begin{pgfscope}%
\pgfsys@transformshift{1.235215in}{0.414757in}%
\pgfsys@useobject{currentmarker}{}%
\end{pgfscope}%
\end{pgfscope}%
\begin{pgfscope}%
\definecolor{textcolor}{rgb}{0.000000,0.000000,0.000000}%
\pgfsetstrokecolor{textcolor}%
\pgfsetfillcolor{textcolor}%
\pgftext[x=1.235215in,y=0.317535in,,top]{\color{textcolor}{\rmfamily\fontsize{7.000000}{8.400000}\selectfont\catcode`\^=\active\def^{\ifmmode\sp\else\^{}\fi}\catcode`\%=\active\def
\end{pgfscope}%
\begin{pgfscope}%
\pgfpathrectangle{\pgfqpoint{0.519743in}{0.414757in}}{\pgfqpoint{0.983773in}{1.143573in}}%
\pgfusepath{clip}%
\pgfsetrectcap%
\pgfsetroundjoin%
\pgfsetlinewidth{0.803000pt}%
\definecolor{currentstroke}{rgb}{0.690196,0.690196,0.690196}%
\pgfsetstrokecolor{currentstroke}%
\pgfsetdash{}{0pt}%
\pgfpathmoveto{\pgfqpoint{1.458799in}{0.414757in}}%
\pgfpathlineto{\pgfqpoint{1.458799in}{1.558330in}}%
\pgfusepath{stroke}%
\end{pgfscope}%
\begin{pgfscope}%
\pgfsetbuttcap%
\pgfsetroundjoin%
\definecolor{currentfill}{rgb}{0.000000,0.000000,0.000000}%
\pgfsetfillcolor{currentfill}%
\pgfsetlinewidth{0.803000pt}%
\definecolor{currentstroke}{rgb}{0.000000,0.000000,0.000000}%
\pgfsetstrokecolor{currentstroke}%
\pgfsetdash{}{0pt}%
\pgfsys@defobject{currentmarker}{\pgfqpoint{0.000000in}{-0.048611in}}{\pgfqpoint{0.000000in}{0.000000in}}{%
\pgfpathmoveto{\pgfqpoint{0.000000in}{0.000000in}}%
\pgfpathlineto{\pgfqpoint{0.000000in}{-0.048611in}}%
\pgfusepath{stroke,fill}%
}%
\begin{pgfscope}%
\pgfsys@transformshift{1.458799in}{0.414757in}%
\pgfsys@useobject{currentmarker}{}%
\end{pgfscope}%
\end{pgfscope}%
\begin{pgfscope}%
\definecolor{textcolor}{rgb}{0.000000,0.000000,0.000000}%
\pgfsetstrokecolor{textcolor}%
\pgfsetfillcolor{textcolor}%
\pgftext[x=1.458799in,y=0.317535in,,top]{\color{textcolor}{\rmfamily\fontsize{7.000000}{8.400000}\selectfont\catcode`\^=\active\def^{\ifmmode\sp\else\^{}\fi}\catcode`\%=\active\def
\end{pgfscope}%
\begin{pgfscope}%
\definecolor{textcolor}{rgb}{0.000000,0.000000,0.000000}%
\pgfsetstrokecolor{textcolor}%
\pgfsetfillcolor{textcolor}%
\pgftext[x=1.011630in,y=0.167891in,,top]{\color{textcolor}{\rmfamily\fontsize{9.000000}{10.800000}\selectfont\catcode`\^=\active\def^{\ifmmode\sp\else\^{}\fi}\catcode`\%=\active\def
\end{pgfscope}%
\begin{pgfscope}%
\pgfpathrectangle{\pgfqpoint{0.519743in}{0.414757in}}{\pgfqpoint{0.983773in}{1.143573in}}%
\pgfusepath{clip}%
\pgfsetrectcap%
\pgfsetroundjoin%
\pgfsetlinewidth{0.803000pt}%
\definecolor{currentstroke}{rgb}{0.690196,0.690196,0.690196}%
\pgfsetstrokecolor{currentstroke}%
\pgfsetdash{}{0pt}%
\pgfpathmoveto{\pgfqpoint{0.519743in}{0.466738in}}%
\pgfpathlineto{\pgfqpoint{1.503516in}{0.466738in}}%
\pgfusepath{stroke}%
\end{pgfscope}%
\begin{pgfscope}%
\pgfsetbuttcap%
\pgfsetroundjoin%
\definecolor{currentfill}{rgb}{0.000000,0.000000,0.000000}%
\pgfsetfillcolor{currentfill}%
\pgfsetlinewidth{0.803000pt}%
\definecolor{currentstroke}{rgb}{0.000000,0.000000,0.000000}%
\pgfsetstrokecolor{currentstroke}%
\pgfsetdash{}{0pt}%
\pgfsys@defobject{currentmarker}{\pgfqpoint{-0.048611in}{0.000000in}}{\pgfqpoint{-0.000000in}{0.000000in}}{%
\pgfpathmoveto{\pgfqpoint{-0.000000in}{0.000000in}}%
\pgfpathlineto{\pgfqpoint{-0.048611in}{0.000000in}}%
\pgfusepath{stroke,fill}%
}%
\begin{pgfscope}%
\pgfsys@transformshift{0.519743in}{0.466738in}%
\pgfsys@useobject{currentmarker}{}%
\end{pgfscope}%
\end{pgfscope}%
\begin{pgfscope}%
\definecolor{textcolor}{rgb}{0.000000,0.000000,0.000000}%
\pgfsetstrokecolor{textcolor}%
\pgfsetfillcolor{textcolor}%
\pgftext[x=0.223446in, y=0.429805in, left, base]{\color{textcolor}{\rmfamily\fontsize{7.000000}{8.400000}\selectfont\catcode`\^=\active\def^{\ifmmode\sp\else\^{}\fi}\catcode`\%=\active\def
\end{pgfscope}%
\begin{pgfscope}%
\pgfpathrectangle{\pgfqpoint{0.519743in}{0.414757in}}{\pgfqpoint{0.983773in}{1.143573in}}%
\pgfusepath{clip}%
\pgfsetrectcap%
\pgfsetroundjoin%
\pgfsetlinewidth{0.803000pt}%
\definecolor{currentstroke}{rgb}{0.690196,0.690196,0.690196}%
\pgfsetstrokecolor{currentstroke}%
\pgfsetdash{}{0pt}%
\pgfpathmoveto{\pgfqpoint{0.519743in}{0.726641in}}%
\pgfpathlineto{\pgfqpoint{1.503516in}{0.726641in}}%
\pgfusepath{stroke}%
\end{pgfscope}%
\begin{pgfscope}%
\pgfsetbuttcap%
\pgfsetroundjoin%
\definecolor{currentfill}{rgb}{0.000000,0.000000,0.000000}%
\pgfsetfillcolor{currentfill}%
\pgfsetlinewidth{0.803000pt}%
\definecolor{currentstroke}{rgb}{0.000000,0.000000,0.000000}%
\pgfsetstrokecolor{currentstroke}%
\pgfsetdash{}{0pt}%
\pgfsys@defobject{currentmarker}{\pgfqpoint{-0.048611in}{0.000000in}}{\pgfqpoint{-0.000000in}{0.000000in}}{%
\pgfpathmoveto{\pgfqpoint{-0.000000in}{0.000000in}}%
\pgfpathlineto{\pgfqpoint{-0.048611in}{0.000000in}}%
\pgfusepath{stroke,fill}%
}%
\begin{pgfscope}%
\pgfsys@transformshift{0.519743in}{0.726641in}%
\pgfsys@useobject{currentmarker}{}%
\end{pgfscope}%
\end{pgfscope}%
\begin{pgfscope}%
\definecolor{textcolor}{rgb}{0.000000,0.000000,0.000000}%
\pgfsetstrokecolor{textcolor}%
\pgfsetfillcolor{textcolor}%
\pgftext[x=0.223446in, y=0.689708in, left, base]{\color{textcolor}{\rmfamily\fontsize{7.000000}{8.400000}\selectfont\catcode`\^=\active\def^{\ifmmode\sp\else\^{}\fi}\catcode`\%=\active\def
\end{pgfscope}%
\begin{pgfscope}%
\pgfpathrectangle{\pgfqpoint{0.519743in}{0.414757in}}{\pgfqpoint{0.983773in}{1.143573in}}%
\pgfusepath{clip}%
\pgfsetrectcap%
\pgfsetroundjoin%
\pgfsetlinewidth{0.803000pt}%
\definecolor{currentstroke}{rgb}{0.690196,0.690196,0.690196}%
\pgfsetstrokecolor{currentstroke}%
\pgfsetdash{}{0pt}%
\pgfpathmoveto{\pgfqpoint{0.519743in}{0.986544in}}%
\pgfpathlineto{\pgfqpoint{1.503516in}{0.986544in}}%
\pgfusepath{stroke}%
\end{pgfscope}%
\begin{pgfscope}%
\pgfsetbuttcap%
\pgfsetroundjoin%
\definecolor{currentfill}{rgb}{0.000000,0.000000,0.000000}%
\pgfsetfillcolor{currentfill}%
\pgfsetlinewidth{0.803000pt}%
\definecolor{currentstroke}{rgb}{0.000000,0.000000,0.000000}%
\pgfsetstrokecolor{currentstroke}%
\pgfsetdash{}{0pt}%
\pgfsys@defobject{currentmarker}{\pgfqpoint{-0.048611in}{0.000000in}}{\pgfqpoint{-0.000000in}{0.000000in}}{%
\pgfpathmoveto{\pgfqpoint{-0.000000in}{0.000000in}}%
\pgfpathlineto{\pgfqpoint{-0.048611in}{0.000000in}}%
\pgfusepath{stroke,fill}%
}%
\begin{pgfscope}%
\pgfsys@transformshift{0.519743in}{0.986544in}%
\pgfsys@useobject{currentmarker}{}%
\end{pgfscope}%
\end{pgfscope}%
\begin{pgfscope}%
\definecolor{textcolor}{rgb}{0.000000,0.000000,0.000000}%
\pgfsetstrokecolor{textcolor}%
\pgfsetfillcolor{textcolor}%
\pgftext[x=0.223446in, y=0.949611in, left, base]{\color{textcolor}{\rmfamily\fontsize{7.000000}{8.400000}\selectfont\catcode`\^=\active\def^{\ifmmode\sp\else\^{}\fi}\catcode`\%=\active\def
\end{pgfscope}%
\begin{pgfscope}%
\pgfpathrectangle{\pgfqpoint{0.519743in}{0.414757in}}{\pgfqpoint{0.983773in}{1.143573in}}%
\pgfusepath{clip}%
\pgfsetrectcap%
\pgfsetroundjoin%
\pgfsetlinewidth{0.803000pt}%
\definecolor{currentstroke}{rgb}{0.690196,0.690196,0.690196}%
\pgfsetstrokecolor{currentstroke}%
\pgfsetdash{}{0pt}%
\pgfpathmoveto{\pgfqpoint{0.519743in}{1.246447in}}%
\pgfpathlineto{\pgfqpoint{1.503516in}{1.246447in}}%
\pgfusepath{stroke}%
\end{pgfscope}%
\begin{pgfscope}%
\pgfsetbuttcap%
\pgfsetroundjoin%
\definecolor{currentfill}{rgb}{0.000000,0.000000,0.000000}%
\pgfsetfillcolor{currentfill}%
\pgfsetlinewidth{0.803000pt}%
\definecolor{currentstroke}{rgb}{0.000000,0.000000,0.000000}%
\pgfsetstrokecolor{currentstroke}%
\pgfsetdash{}{0pt}%
\pgfsys@defobject{currentmarker}{\pgfqpoint{-0.048611in}{0.000000in}}{\pgfqpoint{-0.000000in}{0.000000in}}{%
\pgfpathmoveto{\pgfqpoint{-0.000000in}{0.000000in}}%
\pgfpathlineto{\pgfqpoint{-0.048611in}{0.000000in}}%
\pgfusepath{stroke,fill}%
}%
\begin{pgfscope}%
\pgfsys@transformshift{0.519743in}{1.246447in}%
\pgfsys@useobject{currentmarker}{}%
\end{pgfscope}%
\end{pgfscope}%
\begin{pgfscope}%
\definecolor{textcolor}{rgb}{0.000000,0.000000,0.000000}%
\pgfsetstrokecolor{textcolor}%
\pgfsetfillcolor{textcolor}%
\pgftext[x=0.223446in, y=1.209513in, left, base]{\color{textcolor}{\rmfamily\fontsize{7.000000}{8.400000}\selectfont\catcode`\^=\active\def^{\ifmmode\sp\else\^{}\fi}\catcode`\%=\active\def
\end{pgfscope}%
\begin{pgfscope}%
\pgfpathrectangle{\pgfqpoint{0.519743in}{0.414757in}}{\pgfqpoint{0.983773in}{1.143573in}}%
\pgfusepath{clip}%
\pgfsetrectcap%
\pgfsetroundjoin%
\pgfsetlinewidth{0.803000pt}%
\definecolor{currentstroke}{rgb}{0.690196,0.690196,0.690196}%
\pgfsetstrokecolor{currentstroke}%
\pgfsetdash{}{0pt}%
\pgfpathmoveto{\pgfqpoint{0.519743in}{1.506349in}}%
\pgfpathlineto{\pgfqpoint{1.503516in}{1.506349in}}%
\pgfusepath{stroke}%
\end{pgfscope}%
\begin{pgfscope}%
\pgfsetbuttcap%
\pgfsetroundjoin%
\definecolor{currentfill}{rgb}{0.000000,0.000000,0.000000}%
\pgfsetfillcolor{currentfill}%
\pgfsetlinewidth{0.803000pt}%
\definecolor{currentstroke}{rgb}{0.000000,0.000000,0.000000}%
\pgfsetstrokecolor{currentstroke}%
\pgfsetdash{}{0pt}%
\pgfsys@defobject{currentmarker}{\pgfqpoint{-0.048611in}{0.000000in}}{\pgfqpoint{-0.000000in}{0.000000in}}{%
\pgfpathmoveto{\pgfqpoint{-0.000000in}{0.000000in}}%
\pgfpathlineto{\pgfqpoint{-0.048611in}{0.000000in}}%
\pgfusepath{stroke,fill}%
}%
\begin{pgfscope}%
\pgfsys@transformshift{0.519743in}{1.506349in}%
\pgfsys@useobject{currentmarker}{}%
\end{pgfscope}%
\end{pgfscope}%
\begin{pgfscope}%
\definecolor{textcolor}{rgb}{0.000000,0.000000,0.000000}%
\pgfsetstrokecolor{textcolor}%
\pgfsetfillcolor{textcolor}%
\pgftext[x=0.223446in, y=1.469416in, left, base]{\color{textcolor}{\rmfamily\fontsize{7.000000}{8.400000}\selectfont\catcode`\^=\active\def^{\ifmmode\sp\else\^{}\fi}\catcode`\%=\active\def
\end{pgfscope}%
\begin{pgfscope}%
\definecolor{textcolor}{rgb}{0.000000,0.000000,0.000000}%
\pgfsetstrokecolor{textcolor}%
\pgfsetfillcolor{textcolor}%
\pgftext[x=0.167891in,y=0.986544in,,bottom,rotate=90.000000]{\color{textcolor}{\rmfamily\fontsize{9.000000}{10.800000}\selectfont\catcode`\^=\active\def^{\ifmmode\sp\else\^{}\fi}\catcode`\%=\active\def
\end{pgfscope}%
\begin{pgfscope}%
\pgfpathrectangle{\pgfqpoint{0.519743in}{0.414757in}}{\pgfqpoint{0.983773in}{1.143573in}}%
\pgfusepath{clip}%
\pgfsetrectcap%
\pgfsetroundjoin%
\pgfsetlinewidth{1.505625pt}%
\definecolor{currentstroke}{rgb}{0.003922,0.450980,0.698039}%
\pgfsetstrokecolor{currentstroke}%
\pgfsetstrokeopacity{0.200000}%
\pgfsetdash{}{0pt}%
\pgfpathmoveto{\pgfqpoint{0.685835in}{0.498766in}}%
\pgfpathlineto{\pgfqpoint{0.743328in}{0.518899in}}%
\pgfpathlineto{\pgfqpoint{0.788045in}{0.535677in}}%
\pgfpathlineto{\pgfqpoint{0.832762in}{0.553040in}}%
\pgfpathlineto{\pgfqpoint{0.877479in}{0.570801in}}%
\pgfpathlineto{\pgfqpoint{0.922196in}{0.588824in}}%
\pgfpathlineto{\pgfqpoint{0.966913in}{0.607008in}}%
\pgfpathlineto{\pgfqpoint{1.011630in}{0.625280in}}%
\pgfpathlineto{\pgfqpoint{1.056347in}{0.643587in}}%
\pgfpathlineto{\pgfqpoint{1.101064in}{0.661902in}}%
\pgfpathlineto{\pgfqpoint{1.145781in}{0.680219in}}%
\pgfpathlineto{\pgfqpoint{1.190498in}{0.698558in}}%
\pgfpathlineto{\pgfqpoint{1.235215in}{0.716980in}}%
\pgfpathlineto{\pgfqpoint{1.279932in}{0.735614in}}%
\pgfpathlineto{\pgfqpoint{1.324649in}{0.754722in}}%
\pgfpathlineto{\pgfqpoint{1.369365in}{0.774920in}}%
\pgfpathlineto{\pgfqpoint{1.414082in}{0.798200in}}%
\pgfpathlineto{\pgfqpoint{1.458799in}{1.119756in}}%
\pgfusepath{stroke}%
\end{pgfscope}%
\begin{pgfscope}%
\pgfpathrectangle{\pgfqpoint{0.519743in}{0.414757in}}{\pgfqpoint{0.983773in}{1.143573in}}%
\pgfusepath{clip}%
\pgfsetrectcap%
\pgfsetroundjoin%
\pgfsetlinewidth{1.505625pt}%
\definecolor{currentstroke}{rgb}{0.003922,0.450980,0.698039}%
\pgfsetstrokecolor{currentstroke}%
\pgfsetstrokeopacity{0.200000}%
\pgfsetdash{}{0pt}%
\pgfpathmoveto{\pgfqpoint{0.653894in}{0.484813in}}%
\pgfpathlineto{\pgfqpoint{0.743328in}{0.512896in}}%
\pgfpathlineto{\pgfqpoint{0.788045in}{0.528797in}}%
\pgfpathlineto{\pgfqpoint{0.832762in}{0.545522in}}%
\pgfpathlineto{\pgfqpoint{0.877479in}{0.562864in}}%
\pgfpathlineto{\pgfqpoint{0.922196in}{0.580663in}}%
\pgfpathlineto{\pgfqpoint{0.966913in}{0.598796in}}%
\pgfpathlineto{\pgfqpoint{1.011630in}{0.617169in}}%
\pgfpathlineto{\pgfqpoint{1.056347in}{0.635712in}}%
\pgfpathlineto{\pgfqpoint{1.101064in}{0.654381in}}%
\pgfpathlineto{\pgfqpoint{1.145781in}{0.673155in}}%
\pgfpathlineto{\pgfqpoint{1.190498in}{0.692041in}}%
\pgfpathlineto{\pgfqpoint{1.235215in}{0.711093in}}%
\pgfpathlineto{\pgfqpoint{1.279932in}{0.730429in}}%
\pgfpathlineto{\pgfqpoint{1.324649in}{0.750315in}}%
\pgfpathlineto{\pgfqpoint{1.369365in}{0.771381in}}%
\pgfpathlineto{\pgfqpoint{1.414082in}{0.795688in}}%
\pgfpathlineto{\pgfqpoint{1.458799in}{1.129528in}}%
\pgfusepath{stroke}%
\end{pgfscope}%
\begin{pgfscope}%
\pgfpathrectangle{\pgfqpoint{0.519743in}{0.414757in}}{\pgfqpoint{0.983773in}{1.143573in}}%
\pgfusepath{clip}%
\pgfsetrectcap%
\pgfsetroundjoin%
\pgfsetlinewidth{1.505625pt}%
\definecolor{currentstroke}{rgb}{0.003922,0.450980,0.698039}%
\pgfsetstrokecolor{currentstroke}%
\pgfsetstrokeopacity{0.200000}%
\pgfsetdash{}{0pt}%
\pgfpathmoveto{\pgfqpoint{0.676253in}{0.491375in}}%
\pgfpathlineto{\pgfqpoint{0.743328in}{0.512630in}}%
\pgfpathlineto{\pgfqpoint{0.788045in}{0.528366in}}%
\pgfpathlineto{\pgfqpoint{0.832762in}{0.544904in}}%
\pgfpathlineto{\pgfqpoint{0.877479in}{0.562041in}}%
\pgfpathlineto{\pgfqpoint{0.922196in}{0.579624in}}%
\pgfpathlineto{\pgfqpoint{0.966913in}{0.597533in}}%
\pgfpathlineto{\pgfqpoint{1.011630in}{0.615677in}}%
\pgfpathlineto{\pgfqpoint{1.056347in}{0.633990in}}%
\pgfpathlineto{\pgfqpoint{1.101064in}{0.652428in}}%
\pgfpathlineto{\pgfqpoint{1.145781in}{0.670973in}}%
\pgfpathlineto{\pgfqpoint{1.190498in}{0.689634in}}%
\pgfpathlineto{\pgfqpoint{1.235215in}{0.708465in}}%
\pgfpathlineto{\pgfqpoint{1.279932in}{0.727585in}}%
\pgfpathlineto{\pgfqpoint{1.324649in}{0.747259in}}%
\pgfpathlineto{\pgfqpoint{1.369365in}{0.768115in}}%
\pgfpathlineto{\pgfqpoint{1.414082in}{0.792206in}}%
\pgfpathlineto{\pgfqpoint{1.458799in}{1.125813in}}%
\pgfusepath{stroke}%
\end{pgfscope}%
\begin{pgfscope}%
\pgfpathrectangle{\pgfqpoint{0.519743in}{0.414757in}}{\pgfqpoint{0.983773in}{1.143573in}}%
\pgfusepath{clip}%
\pgfsetrectcap%
\pgfsetroundjoin%
\pgfsetlinewidth{1.505625pt}%
\definecolor{currentstroke}{rgb}{0.003922,0.450980,0.698039}%
\pgfsetstrokecolor{currentstroke}%
\pgfsetstrokeopacity{0.200000}%
\pgfsetdash{}{0pt}%
\pgfpathmoveto{\pgfqpoint{0.689668in}{0.505059in}}%
\pgfpathlineto{\pgfqpoint{0.743328in}{0.524507in}}%
\pgfpathlineto{\pgfqpoint{0.788045in}{0.541212in}}%
\pgfpathlineto{\pgfqpoint{0.832762in}{0.558129in}}%
\pgfpathlineto{\pgfqpoint{0.877479in}{0.575144in}}%
\pgfpathlineto{\pgfqpoint{0.922196in}{0.592176in}}%
\pgfpathlineto{\pgfqpoint{0.966913in}{0.609173in}}%
\pgfpathlineto{\pgfqpoint{1.011630in}{0.626102in}}%
\pgfpathlineto{\pgfqpoint{1.056347in}{0.642948in}}%
\pgfpathlineto{\pgfqpoint{1.101064in}{0.659712in}}%
\pgfpathlineto{\pgfqpoint{1.145781in}{0.676413in}}%
\pgfpathlineto{\pgfqpoint{1.190498in}{0.693095in}}%
\pgfpathlineto{\pgfqpoint{1.235215in}{0.709838in}}%
\pgfpathlineto{\pgfqpoint{1.279932in}{0.726787in}}%
\pgfpathlineto{\pgfqpoint{1.324649in}{0.744224in}}%
\pgfpathlineto{\pgfqpoint{1.369365in}{0.762784in}}%
\pgfpathlineto{\pgfqpoint{1.414082in}{0.784501in}}%
\pgfpathlineto{\pgfqpoint{1.458799in}{1.114834in}}%
\pgfusepath{stroke}%
\end{pgfscope}%
\begin{pgfscope}%
\pgfpathrectangle{\pgfqpoint{0.519743in}{0.414757in}}{\pgfqpoint{0.983773in}{1.143573in}}%
\pgfusepath{clip}%
\pgfsetrectcap%
\pgfsetroundjoin%
\pgfsetlinewidth{1.505625pt}%
\definecolor{currentstroke}{rgb}{0.003922,0.450980,0.698039}%
\pgfsetstrokecolor{currentstroke}%
\pgfsetstrokeopacity{0.200000}%
\pgfsetdash{}{0pt}%
\pgfpathmoveto{\pgfqpoint{0.671781in}{0.494310in}}%
\pgfpathlineto{\pgfqpoint{0.743328in}{0.519081in}}%
\pgfpathlineto{\pgfqpoint{0.788045in}{0.536324in}}%
\pgfpathlineto{\pgfqpoint{0.832762in}{0.554250in}}%
\pgfpathlineto{\pgfqpoint{0.877479in}{0.572651in}}%
\pgfpathlineto{\pgfqpoint{0.922196in}{0.591371in}}%
\pgfpathlineto{\pgfqpoint{0.966913in}{0.610293in}}%
\pgfpathlineto{\pgfqpoint{1.011630in}{0.629332in}}%
\pgfpathlineto{\pgfqpoint{1.056347in}{0.648425in}}%
\pgfpathlineto{\pgfqpoint{1.101064in}{0.667536in}}%
\pgfpathlineto{\pgfqpoint{1.145781in}{0.686651in}}%
\pgfpathlineto{\pgfqpoint{1.190498in}{0.705788in}}%
\pgfpathlineto{\pgfqpoint{1.235215in}{0.725005in}}%
\pgfpathlineto{\pgfqpoint{1.279932in}{0.744429in}}%
\pgfpathlineto{\pgfqpoint{1.324649in}{0.764329in}}%
\pgfpathlineto{\pgfqpoint{1.369365in}{0.785336in}}%
\pgfpathlineto{\pgfqpoint{1.414082in}{0.809500in}}%
\pgfpathlineto{\pgfqpoint{1.458799in}{1.136314in}}%
\pgfusepath{stroke}%
\end{pgfscope}%
\begin{pgfscope}%
\pgfpathrectangle{\pgfqpoint{0.519743in}{0.414757in}}{\pgfqpoint{0.983773in}{1.143573in}}%
\pgfusepath{clip}%
\pgfsetbuttcap%
\pgfsetroundjoin%
\pgfsetlinewidth{1.505625pt}%
\definecolor{currentstroke}{rgb}{0.501961,0.501961,0.501961}%
\pgfsetstrokecolor{currentstroke}%
\pgfsetdash{{5.550000pt}{2.400000pt}}{0.000000pt}%
\pgfpathmoveto{\pgfqpoint{0.564460in}{0.726641in}}%
\pgfpathlineto{\pgfqpoint{1.458799in}{0.726641in}}%
\pgfusepath{stroke}%
\end{pgfscope}%
\begin{pgfscope}%
\pgfpathrectangle{\pgfqpoint{0.519743in}{0.414757in}}{\pgfqpoint{0.983773in}{1.143573in}}%
\pgfusepath{clip}%
\pgfsetbuttcap%
\pgfsetroundjoin%
\pgfsetlinewidth{1.505625pt}%
\definecolor{currentstroke}{rgb}{0.501961,0.501961,0.501961}%
\pgfsetstrokecolor{currentstroke}%
\pgfsetdash{{5.550000pt}{2.400000pt}}{0.000000pt}%
\pgfpathmoveto{\pgfqpoint{0.788045in}{0.466738in}}%
\pgfpathlineto{\pgfqpoint{0.788045in}{1.506349in}}%
\pgfusepath{stroke}%
\end{pgfscope}%
\begin{pgfscope}%
\pgfsetrectcap%
\pgfsetmiterjoin%
\pgfsetlinewidth{0.803000pt}%
\definecolor{currentstroke}{rgb}{0.000000,0.000000,0.000000}%
\pgfsetstrokecolor{currentstroke}%
\pgfsetdash{}{0pt}%
\pgfpathmoveto{\pgfqpoint{0.519743in}{0.414757in}}%
\pgfpathlineto{\pgfqpoint{0.519743in}{1.558330in}}%
\pgfusepath{stroke}%
\end{pgfscope}%
\begin{pgfscope}%
\pgfsetrectcap%
\pgfsetmiterjoin%
\pgfsetlinewidth{0.803000pt}%
\definecolor{currentstroke}{rgb}{0.000000,0.000000,0.000000}%
\pgfsetstrokecolor{currentstroke}%
\pgfsetdash{}{0pt}%
\pgfpathmoveto{\pgfqpoint{1.503516in}{0.414757in}}%
\pgfpathlineto{\pgfqpoint{1.503516in}{1.558330in}}%
\pgfusepath{stroke}%
\end{pgfscope}%
\begin{pgfscope}%
\pgfsetrectcap%
\pgfsetmiterjoin%
\pgfsetlinewidth{0.803000pt}%
\definecolor{currentstroke}{rgb}{0.000000,0.000000,0.000000}%
\pgfsetstrokecolor{currentstroke}%
\pgfsetdash{}{0pt}%
\pgfpathmoveto{\pgfqpoint{0.519743in}{0.414757in}}%
\pgfpathlineto{\pgfqpoint{1.503516in}{0.414757in}}%
\pgfusepath{stroke}%
\end{pgfscope}%
\begin{pgfscope}%
\pgfsetrectcap%
\pgfsetmiterjoin%
\pgfsetlinewidth{0.803000pt}%
\definecolor{currentstroke}{rgb}{0.000000,0.000000,0.000000}%
\pgfsetstrokecolor{currentstroke}%
\pgfsetdash{}{0pt}%
\pgfpathmoveto{\pgfqpoint{0.519743in}{1.558330in}}%
\pgfpathlineto{\pgfqpoint{1.503516in}{1.558330in}}%
\pgfusepath{stroke}%
\end{pgfscope}%
\begin{pgfscope}%
\pgfsetbuttcap%
\pgfsetmiterjoin%
\definecolor{currentfill}{rgb}{1.000000,1.000000,1.000000}%
\pgfsetfillcolor{currentfill}%
\pgfsetfillopacity{0.800000}%
\pgfsetlinewidth{1.003750pt}%
\definecolor{currentstroke}{rgb}{0.800000,0.800000,0.800000}%
\pgfsetstrokecolor{currentstroke}%
\pgfsetstrokeopacity{0.800000}%
\pgfsetdash{}{0pt}%
\pgfpathmoveto{\pgfqpoint{0.542878in}{1.008611in}}%
\pgfpathlineto{\pgfqpoint{1.537122in}{1.008611in}}%
\pgfpathquadraticcurveto{\pgfqpoint{1.556567in}{1.008611in}}{\pgfqpoint{1.556567in}{1.028056in}}%
\pgfpathlineto{\pgfqpoint{1.556567in}{1.446433in}}%
\pgfpathquadraticcurveto{\pgfqpoint{1.556567in}{1.465878in}}{\pgfqpoint{1.537122in}{1.465878in}}%
\pgfpathlineto{\pgfqpoint{0.542878in}{1.465878in}}%
\pgfpathquadraticcurveto{\pgfqpoint{0.523433in}{1.465878in}}{\pgfqpoint{0.523433in}{1.446433in}}%
\pgfpathlineto{\pgfqpoint{0.523433in}{1.028056in}}%
\pgfpathquadraticcurveto{\pgfqpoint{0.523433in}{1.008611in}}{\pgfqpoint{0.542878in}{1.008611in}}%
\pgfpathlineto{\pgfqpoint{0.542878in}{1.008611in}}%
\pgfpathclose%
\pgfusepath{stroke,fill}%
\end{pgfscope}%
\begin{pgfscope}%
\pgfsetbuttcap%
\pgfsetmiterjoin%
\definecolor{currentfill}{rgb}{0.007843,0.619608,0.447059}%
\pgfsetfillcolor{currentfill}%
\pgfsetfillopacity{0.850000}%
\pgfsetlinewidth{0.501875pt}%
\definecolor{currentstroke}{rgb}{0.000000,0.000000,0.000000}%
\pgfsetstrokecolor{currentstroke}%
\pgfsetstrokeopacity{0.850000}%
\pgfsetdash{}{0pt}%
\pgfpathmoveto{\pgfqpoint{0.562322in}{1.353123in}}%
\pgfpathlineto{\pgfqpoint{0.756766in}{1.353123in}}%
\pgfpathlineto{\pgfqpoint{0.756766in}{1.421178in}}%
\pgfpathlineto{\pgfqpoint{0.562322in}{1.421178in}}%
\pgfpathlineto{\pgfqpoint{0.562322in}{1.353123in}}%
\pgfpathclose%
\pgfusepath{stroke,fill}%
\end{pgfscope}%
\begin{pgfscope}%
\definecolor{textcolor}{rgb}{0.000000,0.000000,0.000000}%
\pgfsetstrokecolor{textcolor}%
\pgfsetfillcolor{textcolor}%
\pgftext[x=0.834544in,y=1.353123in,left,base]{\color{textcolor}{\rmfamily\fontsize{7.000000}{8.400000}\selectfont\catcode`\^=\active\def^{\ifmmode\sp\else\^{}\fi}\catcode`\%=\active\def
\end{pgfscope}%
\begin{pgfscope}%
\pgfsetbuttcap%
\pgfsetmiterjoin%
\definecolor{currentfill}{rgb}{0.870588,0.560784,0.011765}%
\pgfsetfillcolor{currentfill}%
\pgfsetfillopacity{0.850000}%
\pgfsetlinewidth{0.501875pt}%
\definecolor{currentstroke}{rgb}{0.000000,0.000000,0.000000}%
\pgfsetstrokecolor{currentstroke}%
\pgfsetstrokeopacity{0.850000}%
\pgfsetdash{}{0pt}%
\pgfpathmoveto{\pgfqpoint{0.562322in}{1.210423in}}%
\pgfpathlineto{\pgfqpoint{0.756766in}{1.210423in}}%
\pgfpathlineto{\pgfqpoint{0.756766in}{1.278478in}}%
\pgfpathlineto{\pgfqpoint{0.562322in}{1.278478in}}%
\pgfpathlineto{\pgfqpoint{0.562322in}{1.210423in}}%
\pgfpathclose%
\pgfusepath{stroke,fill}%
\end{pgfscope}%
\begin{pgfscope}%
\definecolor{textcolor}{rgb}{0.000000,0.000000,0.000000}%
\pgfsetstrokecolor{textcolor}%
\pgfsetfillcolor{textcolor}%
\pgftext[x=0.834544in,y=1.210423in,left,base]{\color{textcolor}{\rmfamily\fontsize{7.000000}{8.400000}\selectfont\catcode`\^=\active\def^{\ifmmode\sp\else\^{}\fi}\catcode`\%=\active\def
\end{pgfscope}%
\begin{pgfscope}%
\pgfsetbuttcap%
\pgfsetroundjoin%
\pgfsetlinewidth{0.803000pt}%
\definecolor{currentstroke}{rgb}{0.000000,0.000000,0.000000}%
\pgfsetstrokecolor{currentstroke}%
\pgfsetstrokeopacity{0.600000}%
\pgfsetdash{{2.960000pt}{1.280000pt}}{0.000000pt}%
\pgfpathmoveto{\pgfqpoint{0.562322in}{1.101751in}}%
\pgfpathlineto{\pgfqpoint{0.659544in}{1.101751in}}%
\pgfpathlineto{\pgfqpoint{0.756766in}{1.101751in}}%
\pgfusepath{stroke}%
\end{pgfscope}%
\begin{pgfscope}%
\definecolor{textcolor}{rgb}{0.000000,0.000000,0.000000}%
\pgfsetstrokecolor{textcolor}%
\pgfsetfillcolor{textcolor}%
\pgftext[x=0.834544in,y=1.067723in,left,base]{\color{textcolor}{\rmfamily\fontsize{7.000000}{8.400000}\selectfont\catcode`\^=\active\def^{\ifmmode\sp\else\^{}\fi}\catcode`\%=\active\def
\end{pgfscope}%
\end{pgfpicture}%
\makeatother%
\endgroup%

%% file: figures/calibration_maps/TRIVIAQA_gemma-3-12b-it.pgf
\begingroup%
\makeatletter%
\begin{pgfpicture}%
\pgfpathrectangle{\pgfpointorigin}{\pgfqpoint{1.600000in}{1.600000in}}%
\pgfusepath{use as bounding box, clip}%
\begin{pgfscope}%
\pgfsetbuttcap%
\pgfsetmiterjoin%
\definecolor{currentfill}{rgb}{1.000000,1.000000,1.000000}%
\pgfsetfillcolor{currentfill}%
\pgfsetlinewidth{0.000000pt}%
\definecolor{currentstroke}{rgb}{1.000000,1.000000,1.000000}%
\pgfsetstrokecolor{currentstroke}%
\pgfsetdash{}{0pt}%
\pgfpathmoveto{\pgfqpoint{0.000000in}{0.000000in}}%
\pgfpathlineto{\pgfqpoint{1.600000in}{0.000000in}}%
\pgfpathlineto{\pgfqpoint{1.600000in}{1.600000in}}%
\pgfpathlineto{\pgfqpoint{0.000000in}{1.600000in}}%
\pgfpathlineto{\pgfqpoint{0.000000in}{0.000000in}}%
\pgfpathclose%
\pgfusepath{fill}%
\end{pgfscope}%
\begin{pgfscope}%
\pgfsetbuttcap%
\pgfsetmiterjoin%
\definecolor{currentfill}{rgb}{1.000000,1.000000,1.000000}%
\pgfsetfillcolor{currentfill}%
\pgfsetlinewidth{0.000000pt}%
\definecolor{currentstroke}{rgb}{0.000000,0.000000,0.000000}%
\pgfsetstrokecolor{currentstroke}%
\pgfsetstrokeopacity{0.000000}%
\pgfsetdash{}{0pt}%
\pgfpathmoveto{\pgfqpoint{0.519743in}{0.414757in}}%
\pgfpathlineto{\pgfqpoint{1.503516in}{0.414757in}}%
\pgfpathlineto{\pgfqpoint{1.503516in}{1.558330in}}%
\pgfpathlineto{\pgfqpoint{0.519743in}{1.558330in}}%
\pgfpathlineto{\pgfqpoint{0.519743in}{0.414757in}}%
\pgfpathclose%
\pgfusepath{fill}%
\end{pgfscope}%
\begin{pgfscope}%
\pgfpathrectangle{\pgfqpoint{0.519743in}{0.414757in}}{\pgfqpoint{0.983773in}{1.143573in}}%
\pgfusepath{clip}%
\pgfsetbuttcap%
\pgfsetroundjoin%
\definecolor{currentfill}{rgb}{0.870588,0.560784,0.011765}%
\pgfsetfillcolor{currentfill}%
\pgfsetfillopacity{0.200000}%
\pgfsetlinewidth{1.003750pt}%
\definecolor{currentstroke}{rgb}{0.870588,0.560784,0.011765}%
\pgfsetstrokecolor{currentstroke}%
\pgfsetstrokeopacity{0.200000}%
\pgfsetdash{}{0pt}%
\pgfsys@defobject{currentmarker}{\pgfqpoint{0.564460in}{0.726641in}}{\pgfqpoint{0.788045in}{1.506349in}}{%
\pgfpathmoveto{\pgfqpoint{0.788045in}{0.726641in}}%
\pgfpathlineto{\pgfqpoint{0.564460in}{0.726641in}}%
\pgfpathlineto{\pgfqpoint{0.564460in}{1.506349in}}%
\pgfpathlineto{\pgfqpoint{0.788045in}{1.506349in}}%
\pgfpathlineto{\pgfqpoint{0.788045in}{1.506349in}}%
\pgfpathlineto{\pgfqpoint{0.788045in}{0.726641in}}%
\pgfpathlineto{\pgfqpoint{0.788045in}{0.726641in}}%
\pgfpathclose%
\pgfusepath{stroke,fill}%
}%
\begin{pgfscope}%
\pgfsys@transformshift{0.000000in}{0.000000in}%
\pgfsys@useobject{currentmarker}{}%
\end{pgfscope}%
\end{pgfscope}%
\begin{pgfscope}%
\pgfpathrectangle{\pgfqpoint{0.519743in}{0.414757in}}{\pgfqpoint{0.983773in}{1.143573in}}%
\pgfusepath{clip}%
\pgfsetbuttcap%
\pgfsetroundjoin%
\definecolor{currentfill}{rgb}{0.007843,0.619608,0.447059}%
\pgfsetfillcolor{currentfill}%
\pgfsetfillopacity{0.200000}%
\pgfsetlinewidth{1.003750pt}%
\definecolor{currentstroke}{rgb}{0.007843,0.619608,0.447059}%
\pgfsetstrokecolor{currentstroke}%
\pgfsetstrokeopacity{0.200000}%
\pgfsetdash{}{0pt}%
\pgfsys@defobject{currentmarker}{\pgfqpoint{0.788045in}{0.726641in}}{\pgfqpoint{1.458799in}{1.506349in}}{%
\pgfpathmoveto{\pgfqpoint{1.458799in}{0.726641in}}%
\pgfpathlineto{\pgfqpoint{0.788045in}{0.726641in}}%
\pgfpathlineto{\pgfqpoint{0.788045in}{1.506349in}}%
\pgfpathlineto{\pgfqpoint{1.458799in}{1.506349in}}%
\pgfpathlineto{\pgfqpoint{1.458799in}{1.506349in}}%
\pgfpathlineto{\pgfqpoint{1.458799in}{0.726641in}}%
\pgfpathlineto{\pgfqpoint{1.458799in}{0.726641in}}%
\pgfpathclose%
\pgfusepath{stroke,fill}%
}%
\begin{pgfscope}%
\pgfsys@transformshift{0.000000in}{0.000000in}%
\pgfsys@useobject{currentmarker}{}%
\end{pgfscope}%
\end{pgfscope}%
\begin{pgfscope}%
\pgfpathrectangle{\pgfqpoint{0.519743in}{0.414757in}}{\pgfqpoint{0.983773in}{1.143573in}}%
\pgfusepath{clip}%
\pgfsetbuttcap%
\pgfsetroundjoin%
\definecolor{currentfill}{rgb}{0.870588,0.560784,0.011765}%
\pgfsetfillcolor{currentfill}%
\pgfsetfillopacity{0.200000}%
\pgfsetlinewidth{1.003750pt}%
\definecolor{currentstroke}{rgb}{0.870588,0.560784,0.011765}%
\pgfsetstrokecolor{currentstroke}%
\pgfsetstrokeopacity{0.200000}%
\pgfsetdash{}{0pt}%
\pgfsys@defobject{currentmarker}{\pgfqpoint{0.788045in}{0.466738in}}{\pgfqpoint{1.458799in}{0.726641in}}{%
\pgfpathmoveto{\pgfqpoint{1.458799in}{0.466738in}}%
\pgfpathlineto{\pgfqpoint{0.788045in}{0.466738in}}%
\pgfpathlineto{\pgfqpoint{0.788045in}{0.726641in}}%
\pgfpathlineto{\pgfqpoint{1.458799in}{0.726641in}}%
\pgfpathlineto{\pgfqpoint{1.458799in}{0.726641in}}%
\pgfpathlineto{\pgfqpoint{1.458799in}{0.466738in}}%
\pgfpathlineto{\pgfqpoint{1.458799in}{0.466738in}}%
\pgfpathclose%
\pgfusepath{stroke,fill}%
}%
\begin{pgfscope}%
\pgfsys@transformshift{0.000000in}{0.000000in}%
\pgfsys@useobject{currentmarker}{}%
\end{pgfscope}%
\end{pgfscope}%
\begin{pgfscope}%
\pgfpathrectangle{\pgfqpoint{0.519743in}{0.414757in}}{\pgfqpoint{0.983773in}{1.143573in}}%
\pgfusepath{clip}%
\pgfsetbuttcap%
\pgfsetroundjoin%
\definecolor{currentfill}{rgb}{0.007843,0.619608,0.447059}%
\pgfsetfillcolor{currentfill}%
\pgfsetfillopacity{0.200000}%
\pgfsetlinewidth{1.003750pt}%
\definecolor{currentstroke}{rgb}{0.007843,0.619608,0.447059}%
\pgfsetstrokecolor{currentstroke}%
\pgfsetstrokeopacity{0.200000}%
\pgfsetdash{}{0pt}%
\pgfsys@defobject{currentmarker}{\pgfqpoint{0.564460in}{0.466738in}}{\pgfqpoint{0.788045in}{0.726641in}}{%
\pgfpathmoveto{\pgfqpoint{0.788045in}{0.466738in}}%
\pgfpathlineto{\pgfqpoint{0.564460in}{0.466738in}}%
\pgfpathlineto{\pgfqpoint{0.564460in}{0.726641in}}%
\pgfpathlineto{\pgfqpoint{0.788045in}{0.726641in}}%
\pgfpathlineto{\pgfqpoint{0.788045in}{0.726641in}}%
\pgfpathlineto{\pgfqpoint{0.788045in}{0.466738in}}%
\pgfpathlineto{\pgfqpoint{0.788045in}{0.466738in}}%
\pgfpathclose%
\pgfusepath{stroke,fill}%
}%
\begin{pgfscope}%
\pgfsys@transformshift{0.000000in}{0.000000in}%
\pgfsys@useobject{currentmarker}{}%
\end{pgfscope}%
\end{pgfscope}%
\begin{pgfscope}%
\pgfpathrectangle{\pgfqpoint{0.519743in}{0.414757in}}{\pgfqpoint{0.983773in}{1.143573in}}%
\pgfusepath{clip}%
\pgfsetrectcap%
\pgfsetroundjoin%
\pgfsetlinewidth{0.803000pt}%
\definecolor{currentstroke}{rgb}{0.690196,0.690196,0.690196}%
\pgfsetstrokecolor{currentstroke}%
\pgfsetdash{}{0pt}%
\pgfpathmoveto{\pgfqpoint{0.564460in}{0.414757in}}%
\pgfpathlineto{\pgfqpoint{0.564460in}{1.558330in}}%
\pgfusepath{stroke}%
\end{pgfscope}%
\begin{pgfscope}%
\pgfsetbuttcap%
\pgfsetroundjoin%
\definecolor{currentfill}{rgb}{0.000000,0.000000,0.000000}%
\pgfsetfillcolor{currentfill}%
\pgfsetlinewidth{0.803000pt}%
\definecolor{currentstroke}{rgb}{0.000000,0.000000,0.000000}%
\pgfsetstrokecolor{currentstroke}%
\pgfsetdash{}{0pt}%
\pgfsys@defobject{currentmarker}{\pgfqpoint{0.000000in}{-0.048611in}}{\pgfqpoint{0.000000in}{0.000000in}}{%
\pgfpathmoveto{\pgfqpoint{0.000000in}{0.000000in}}%
\pgfpathlineto{\pgfqpoint{0.000000in}{-0.048611in}}%
\pgfusepath{stroke,fill}%
}%
\begin{pgfscope}%
\pgfsys@transformshift{0.564460in}{0.414757in}%
\pgfsys@useobject{currentmarker}{}%
\end{pgfscope}%
\end{pgfscope}%
\begin{pgfscope}%
\definecolor{textcolor}{rgb}{0.000000,0.000000,0.000000}%
\pgfsetstrokecolor{textcolor}%
\pgfsetfillcolor{textcolor}%
\pgftext[x=0.564460in,y=0.317535in,,top]{\color{textcolor}{\rmfamily\fontsize{7.000000}{8.400000}\selectfont\catcode`\^=\active\def^{\ifmmode\sp\else\^{}\fi}\catcode`\%=\active\def
\end{pgfscope}%
\begin{pgfscope}%
\pgfpathrectangle{\pgfqpoint{0.519743in}{0.414757in}}{\pgfqpoint{0.983773in}{1.143573in}}%
\pgfusepath{clip}%
\pgfsetrectcap%
\pgfsetroundjoin%
\pgfsetlinewidth{0.803000pt}%
\definecolor{currentstroke}{rgb}{0.690196,0.690196,0.690196}%
\pgfsetstrokecolor{currentstroke}%
\pgfsetdash{}{0pt}%
\pgfpathmoveto{\pgfqpoint{0.788045in}{0.414757in}}%
\pgfpathlineto{\pgfqpoint{0.788045in}{1.558330in}}%
\pgfusepath{stroke}%
\end{pgfscope}%
\begin{pgfscope}%
\pgfsetbuttcap%
\pgfsetroundjoin%
\definecolor{currentfill}{rgb}{0.000000,0.000000,0.000000}%
\pgfsetfillcolor{currentfill}%
\pgfsetlinewidth{0.803000pt}%
\definecolor{currentstroke}{rgb}{0.000000,0.000000,0.000000}%
\pgfsetstrokecolor{currentstroke}%
\pgfsetdash{}{0pt}%
\pgfsys@defobject{currentmarker}{\pgfqpoint{0.000000in}{-0.048611in}}{\pgfqpoint{0.000000in}{0.000000in}}{%
\pgfpathmoveto{\pgfqpoint{0.000000in}{0.000000in}}%
\pgfpathlineto{\pgfqpoint{0.000000in}{-0.048611in}}%
\pgfusepath{stroke,fill}%
}%
\begin{pgfscope}%
\pgfsys@transformshift{0.788045in}{0.414757in}%
\pgfsys@useobject{currentmarker}{}%
\end{pgfscope}%
\end{pgfscope}%
\begin{pgfscope}%
\definecolor{textcolor}{rgb}{0.000000,0.000000,0.000000}%
\pgfsetstrokecolor{textcolor}%
\pgfsetfillcolor{textcolor}%
\pgftext[x=0.788045in,y=0.317535in,,top]{\color{textcolor}{\rmfamily\fontsize{7.000000}{8.400000}\selectfont\catcode`\^=\active\def^{\ifmmode\sp\else\^{}\fi}\catcode`\%=\active\def
\end{pgfscope}%
\begin{pgfscope}%
\pgfpathrectangle{\pgfqpoint{0.519743in}{0.414757in}}{\pgfqpoint{0.983773in}{1.143573in}}%
\pgfusepath{clip}%
\pgfsetrectcap%
\pgfsetroundjoin%
\pgfsetlinewidth{0.803000pt}%
\definecolor{currentstroke}{rgb}{0.690196,0.690196,0.690196}%
\pgfsetstrokecolor{currentstroke}%
\pgfsetdash{}{0pt}%
\pgfpathmoveto{\pgfqpoint{1.011630in}{0.414757in}}%
\pgfpathlineto{\pgfqpoint{1.011630in}{1.558330in}}%
\pgfusepath{stroke}%
\end{pgfscope}%
\begin{pgfscope}%
\pgfsetbuttcap%
\pgfsetroundjoin%
\definecolor{currentfill}{rgb}{0.000000,0.000000,0.000000}%
\pgfsetfillcolor{currentfill}%
\pgfsetlinewidth{0.803000pt}%
\definecolor{currentstroke}{rgb}{0.000000,0.000000,0.000000}%
\pgfsetstrokecolor{currentstroke}%
\pgfsetdash{}{0pt}%
\pgfsys@defobject{currentmarker}{\pgfqpoint{0.000000in}{-0.048611in}}{\pgfqpoint{0.000000in}{0.000000in}}{%
\pgfpathmoveto{\pgfqpoint{0.000000in}{0.000000in}}%
\pgfpathlineto{\pgfqpoint{0.000000in}{-0.048611in}}%
\pgfusepath{stroke,fill}%
}%
\begin{pgfscope}%
\pgfsys@transformshift{1.011630in}{0.414757in}%
\pgfsys@useobject{currentmarker}{}%
\end{pgfscope}%
\end{pgfscope}%
\begin{pgfscope}%
\definecolor{textcolor}{rgb}{0.000000,0.000000,0.000000}%
\pgfsetstrokecolor{textcolor}%
\pgfsetfillcolor{textcolor}%
\pgftext[x=1.011630in,y=0.317535in,,top]{\color{textcolor}{\rmfamily\fontsize{7.000000}{8.400000}\selectfont\catcode`\^=\active\def^{\ifmmode\sp\else\^{}\fi}\catcode`\%=\active\def
\end{pgfscope}%
\begin{pgfscope}%
\pgfpathrectangle{\pgfqpoint{0.519743in}{0.414757in}}{\pgfqpoint{0.983773in}{1.143573in}}%
\pgfusepath{clip}%
\pgfsetrectcap%
\pgfsetroundjoin%
\pgfsetlinewidth{0.803000pt}%
\definecolor{currentstroke}{rgb}{0.690196,0.690196,0.690196}%
\pgfsetstrokecolor{currentstroke}%
\pgfsetdash{}{0pt}%
\pgfpathmoveto{\pgfqpoint{1.235215in}{0.414757in}}%
\pgfpathlineto{\pgfqpoint{1.235215in}{1.558330in}}%
\pgfusepath{stroke}%
\end{pgfscope}%
\begin{pgfscope}%
\pgfsetbuttcap%
\pgfsetroundjoin%
\definecolor{currentfill}{rgb}{0.000000,0.000000,0.000000}%
\pgfsetfillcolor{currentfill}%
\pgfsetlinewidth{0.803000pt}%
\definecolor{currentstroke}{rgb}{0.000000,0.000000,0.000000}%
\pgfsetstrokecolor{currentstroke}%
\pgfsetdash{}{0pt}%
\pgfsys@defobject{currentmarker}{\pgfqpoint{0.000000in}{-0.048611in}}{\pgfqpoint{0.000000in}{0.000000in}}{%
\pgfpathmoveto{\pgfqpoint{0.000000in}{0.000000in}}%
\pgfpathlineto{\pgfqpoint{0.000000in}{-0.048611in}}%
\pgfusepath{stroke,fill}%
}%
\begin{pgfscope}%
\pgfsys@transformshift{1.235215in}{0.414757in}%
\pgfsys@useobject{currentmarker}{}%
\end{pgfscope}%
\end{pgfscope}%
\begin{pgfscope}%
\definecolor{textcolor}{rgb}{0.000000,0.000000,0.000000}%
\pgfsetstrokecolor{textcolor}%
\pgfsetfillcolor{textcolor}%
\pgftext[x=1.235215in,y=0.317535in,,top]{\color{textcolor}{\rmfamily\fontsize{7.000000}{8.400000}\selectfont\catcode`\^=\active\def^{\ifmmode\sp\else\^{}\fi}\catcode`\%=\active\def
\end{pgfscope}%
\begin{pgfscope}%
\pgfpathrectangle{\pgfqpoint{0.519743in}{0.414757in}}{\pgfqpoint{0.983773in}{1.143573in}}%
\pgfusepath{clip}%
\pgfsetrectcap%
\pgfsetroundjoin%
\pgfsetlinewidth{0.803000pt}%
\definecolor{currentstroke}{rgb}{0.690196,0.690196,0.690196}%
\pgfsetstrokecolor{currentstroke}%
\pgfsetdash{}{0pt}%
\pgfpathmoveto{\pgfqpoint{1.458799in}{0.414757in}}%
\pgfpathlineto{\pgfqpoint{1.458799in}{1.558330in}}%
\pgfusepath{stroke}%
\end{pgfscope}%
\begin{pgfscope}%
\pgfsetbuttcap%
\pgfsetroundjoin%
\definecolor{currentfill}{rgb}{0.000000,0.000000,0.000000}%
\pgfsetfillcolor{currentfill}%
\pgfsetlinewidth{0.803000pt}%
\definecolor{currentstroke}{rgb}{0.000000,0.000000,0.000000}%
\pgfsetstrokecolor{currentstroke}%
\pgfsetdash{}{0pt}%
\pgfsys@defobject{currentmarker}{\pgfqpoint{0.000000in}{-0.048611in}}{\pgfqpoint{0.000000in}{0.000000in}}{%
\pgfpathmoveto{\pgfqpoint{0.000000in}{0.000000in}}%
\pgfpathlineto{\pgfqpoint{0.000000in}{-0.048611in}}%
\pgfusepath{stroke,fill}%
}%
\begin{pgfscope}%
\pgfsys@transformshift{1.458799in}{0.414757in}%
\pgfsys@useobject{currentmarker}{}%
\end{pgfscope}%
\end{pgfscope}%
\begin{pgfscope}%
\definecolor{textcolor}{rgb}{0.000000,0.000000,0.000000}%
\pgfsetstrokecolor{textcolor}%
\pgfsetfillcolor{textcolor}%
\pgftext[x=1.458799in,y=0.317535in,,top]{\color{textcolor}{\rmfamily\fontsize{7.000000}{8.400000}\selectfont\catcode`\^=\active\def^{\ifmmode\sp\else\^{}\fi}\catcode`\%=\active\def
\end{pgfscope}%
\begin{pgfscope}%
\definecolor{textcolor}{rgb}{0.000000,0.000000,0.000000}%
\pgfsetstrokecolor{textcolor}%
\pgfsetfillcolor{textcolor}%
\pgftext[x=1.011630in,y=0.167891in,,top]{\color{textcolor}{\rmfamily\fontsize{9.000000}{10.800000}\selectfont\catcode`\^=\active\def^{\ifmmode\sp\else\^{}\fi}\catcode`\%=\active\def
\end{pgfscope}%
\begin{pgfscope}%
\pgfpathrectangle{\pgfqpoint{0.519743in}{0.414757in}}{\pgfqpoint{0.983773in}{1.143573in}}%
\pgfusepath{clip}%
\pgfsetrectcap%
\pgfsetroundjoin%
\pgfsetlinewidth{0.803000pt}%
\definecolor{currentstroke}{rgb}{0.690196,0.690196,0.690196}%
\pgfsetstrokecolor{currentstroke}%
\pgfsetdash{}{0pt}%
\pgfpathmoveto{\pgfqpoint{0.519743in}{0.466738in}}%
\pgfpathlineto{\pgfqpoint{1.503516in}{0.466738in}}%
\pgfusepath{stroke}%
\end{pgfscope}%
\begin{pgfscope}%
\pgfsetbuttcap%
\pgfsetroundjoin%
\definecolor{currentfill}{rgb}{0.000000,0.000000,0.000000}%
\pgfsetfillcolor{currentfill}%
\pgfsetlinewidth{0.803000pt}%
\definecolor{currentstroke}{rgb}{0.000000,0.000000,0.000000}%
\pgfsetstrokecolor{currentstroke}%
\pgfsetdash{}{0pt}%
\pgfsys@defobject{currentmarker}{\pgfqpoint{-0.048611in}{0.000000in}}{\pgfqpoint{-0.000000in}{0.000000in}}{%
\pgfpathmoveto{\pgfqpoint{-0.000000in}{0.000000in}}%
\pgfpathlineto{\pgfqpoint{-0.048611in}{0.000000in}}%
\pgfusepath{stroke,fill}%
}%
\begin{pgfscope}%
\pgfsys@transformshift{0.519743in}{0.466738in}%
\pgfsys@useobject{currentmarker}{}%
\end{pgfscope}%
\end{pgfscope}%
\begin{pgfscope}%
\definecolor{textcolor}{rgb}{0.000000,0.000000,0.000000}%
\pgfsetstrokecolor{textcolor}%
\pgfsetfillcolor{textcolor}%
\pgftext[x=0.223446in, y=0.429805in, left, base]{\color{textcolor}{\rmfamily\fontsize{7.000000}{8.400000}\selectfont\catcode`\^=\active\def^{\ifmmode\sp\else\^{}\fi}\catcode`\%=\active\def
\end{pgfscope}%
\begin{pgfscope}%
\pgfpathrectangle{\pgfqpoint{0.519743in}{0.414757in}}{\pgfqpoint{0.983773in}{1.143573in}}%
\pgfusepath{clip}%
\pgfsetrectcap%
\pgfsetroundjoin%
\pgfsetlinewidth{0.803000pt}%
\definecolor{currentstroke}{rgb}{0.690196,0.690196,0.690196}%
\pgfsetstrokecolor{currentstroke}%
\pgfsetdash{}{0pt}%
\pgfpathmoveto{\pgfqpoint{0.519743in}{0.726641in}}%
\pgfpathlineto{\pgfqpoint{1.503516in}{0.726641in}}%
\pgfusepath{stroke}%
\end{pgfscope}%
\begin{pgfscope}%
\pgfsetbuttcap%
\pgfsetroundjoin%
\definecolor{currentfill}{rgb}{0.000000,0.000000,0.000000}%
\pgfsetfillcolor{currentfill}%
\pgfsetlinewidth{0.803000pt}%
\definecolor{currentstroke}{rgb}{0.000000,0.000000,0.000000}%
\pgfsetstrokecolor{currentstroke}%
\pgfsetdash{}{0pt}%
\pgfsys@defobject{currentmarker}{\pgfqpoint{-0.048611in}{0.000000in}}{\pgfqpoint{-0.000000in}{0.000000in}}{%
\pgfpathmoveto{\pgfqpoint{-0.000000in}{0.000000in}}%
\pgfpathlineto{\pgfqpoint{-0.048611in}{0.000000in}}%
\pgfusepath{stroke,fill}%
}%
\begin{pgfscope}%
\pgfsys@transformshift{0.519743in}{0.726641in}%
\pgfsys@useobject{currentmarker}{}%
\end{pgfscope}%
\end{pgfscope}%
\begin{pgfscope}%
\definecolor{textcolor}{rgb}{0.000000,0.000000,0.000000}%
\pgfsetstrokecolor{textcolor}%
\pgfsetfillcolor{textcolor}%
\pgftext[x=0.223446in, y=0.689708in, left, base]{\color{textcolor}{\rmfamily\fontsize{7.000000}{8.400000}\selectfont\catcode`\^=\active\def^{\ifmmode\sp\else\^{}\fi}\catcode`\%=\active\def
\end{pgfscope}%
\begin{pgfscope}%
\pgfpathrectangle{\pgfqpoint{0.519743in}{0.414757in}}{\pgfqpoint{0.983773in}{1.143573in}}%
\pgfusepath{clip}%
\pgfsetrectcap%
\pgfsetroundjoin%
\pgfsetlinewidth{0.803000pt}%
\definecolor{currentstroke}{rgb}{0.690196,0.690196,0.690196}%
\pgfsetstrokecolor{currentstroke}%
\pgfsetdash{}{0pt}%
\pgfpathmoveto{\pgfqpoint{0.519743in}{0.986544in}}%
\pgfpathlineto{\pgfqpoint{1.503516in}{0.986544in}}%
\pgfusepath{stroke}%
\end{pgfscope}%
\begin{pgfscope}%
\pgfsetbuttcap%
\pgfsetroundjoin%
\definecolor{currentfill}{rgb}{0.000000,0.000000,0.000000}%
\pgfsetfillcolor{currentfill}%
\pgfsetlinewidth{0.803000pt}%
\definecolor{currentstroke}{rgb}{0.000000,0.000000,0.000000}%
\pgfsetstrokecolor{currentstroke}%
\pgfsetdash{}{0pt}%
\pgfsys@defobject{currentmarker}{\pgfqpoint{-0.048611in}{0.000000in}}{\pgfqpoint{-0.000000in}{0.000000in}}{%
\pgfpathmoveto{\pgfqpoint{-0.000000in}{0.000000in}}%
\pgfpathlineto{\pgfqpoint{-0.048611in}{0.000000in}}%
\pgfusepath{stroke,fill}%
}%
\begin{pgfscope}%
\pgfsys@transformshift{0.519743in}{0.986544in}%
\pgfsys@useobject{currentmarker}{}%
\end{pgfscope}%
\end{pgfscope}%
\begin{pgfscope}%
\definecolor{textcolor}{rgb}{0.000000,0.000000,0.000000}%
\pgfsetstrokecolor{textcolor}%
\pgfsetfillcolor{textcolor}%
\pgftext[x=0.223446in, y=0.949611in, left, base]{\color{textcolor}{\rmfamily\fontsize{7.000000}{8.400000}\selectfont\catcode`\^=\active\def^{\ifmmode\sp\else\^{}\fi}\catcode`\%=\active\def
\end{pgfscope}%
\begin{pgfscope}%
\pgfpathrectangle{\pgfqpoint{0.519743in}{0.414757in}}{\pgfqpoint{0.983773in}{1.143573in}}%
\pgfusepath{clip}%
\pgfsetrectcap%
\pgfsetroundjoin%
\pgfsetlinewidth{0.803000pt}%
\definecolor{currentstroke}{rgb}{0.690196,0.690196,0.690196}%
\pgfsetstrokecolor{currentstroke}%
\pgfsetdash{}{0pt}%
\pgfpathmoveto{\pgfqpoint{0.519743in}{1.246447in}}%
\pgfpathlineto{\pgfqpoint{1.503516in}{1.246447in}}%
\pgfusepath{stroke}%
\end{pgfscope}%
\begin{pgfscope}%
\pgfsetbuttcap%
\pgfsetroundjoin%
\definecolor{currentfill}{rgb}{0.000000,0.000000,0.000000}%
\pgfsetfillcolor{currentfill}%
\pgfsetlinewidth{0.803000pt}%
\definecolor{currentstroke}{rgb}{0.000000,0.000000,0.000000}%
\pgfsetstrokecolor{currentstroke}%
\pgfsetdash{}{0pt}%
\pgfsys@defobject{currentmarker}{\pgfqpoint{-0.048611in}{0.000000in}}{\pgfqpoint{-0.000000in}{0.000000in}}{%
\pgfpathmoveto{\pgfqpoint{-0.000000in}{0.000000in}}%
\pgfpathlineto{\pgfqpoint{-0.048611in}{0.000000in}}%
\pgfusepath{stroke,fill}%
}%
\begin{pgfscope}%
\pgfsys@transformshift{0.519743in}{1.246447in}%
\pgfsys@useobject{currentmarker}{}%
\end{pgfscope}%
\end{pgfscope}%
\begin{pgfscope}%
\definecolor{textcolor}{rgb}{0.000000,0.000000,0.000000}%
\pgfsetstrokecolor{textcolor}%
\pgfsetfillcolor{textcolor}%
\pgftext[x=0.223446in, y=1.209513in, left, base]{\color{textcolor}{\rmfamily\fontsize{7.000000}{8.400000}\selectfont\catcode`\^=\active\def^{\ifmmode\sp\else\^{}\fi}\catcode`\%=\active\def
\end{pgfscope}%
\begin{pgfscope}%
\pgfpathrectangle{\pgfqpoint{0.519743in}{0.414757in}}{\pgfqpoint{0.983773in}{1.143573in}}%
\pgfusepath{clip}%
\pgfsetrectcap%
\pgfsetroundjoin%
\pgfsetlinewidth{0.803000pt}%
\definecolor{currentstroke}{rgb}{0.690196,0.690196,0.690196}%
\pgfsetstrokecolor{currentstroke}%
\pgfsetdash{}{0pt}%
\pgfpathmoveto{\pgfqpoint{0.519743in}{1.506349in}}%
\pgfpathlineto{\pgfqpoint{1.503516in}{1.506349in}}%
\pgfusepath{stroke}%
\end{pgfscope}%
\begin{pgfscope}%
\pgfsetbuttcap%
\pgfsetroundjoin%
\definecolor{currentfill}{rgb}{0.000000,0.000000,0.000000}%
\pgfsetfillcolor{currentfill}%
\pgfsetlinewidth{0.803000pt}%
\definecolor{currentstroke}{rgb}{0.000000,0.000000,0.000000}%
\pgfsetstrokecolor{currentstroke}%
\pgfsetdash{}{0pt}%
\pgfsys@defobject{currentmarker}{\pgfqpoint{-0.048611in}{0.000000in}}{\pgfqpoint{-0.000000in}{0.000000in}}{%
\pgfpathmoveto{\pgfqpoint{-0.000000in}{0.000000in}}%
\pgfpathlineto{\pgfqpoint{-0.048611in}{0.000000in}}%
\pgfusepath{stroke,fill}%
}%
\begin{pgfscope}%
\pgfsys@transformshift{0.519743in}{1.506349in}%
\pgfsys@useobject{currentmarker}{}%
\end{pgfscope}%
\end{pgfscope}%
\begin{pgfscope}%
\definecolor{textcolor}{rgb}{0.000000,0.000000,0.000000}%
\pgfsetstrokecolor{textcolor}%
\pgfsetfillcolor{textcolor}%
\pgftext[x=0.223446in, y=1.469416in, left, base]{\color{textcolor}{\rmfamily\fontsize{7.000000}{8.400000}\selectfont\catcode`\^=\active\def^{\ifmmode\sp\else\^{}\fi}\catcode`\%=\active\def
\end{pgfscope}%
\begin{pgfscope}%
\definecolor{textcolor}{rgb}{0.000000,0.000000,0.000000}%
\pgfsetstrokecolor{textcolor}%
\pgfsetfillcolor{textcolor}%
\pgftext[x=0.167891in,y=0.986544in,,bottom,rotate=90.000000]{\color{textcolor}{\rmfamily\fontsize{9.000000}{10.800000}\selectfont\catcode`\^=\active\def^{\ifmmode\sp\else\^{}\fi}\catcode`\%=\active\def
\end{pgfscope}%
\begin{pgfscope}%
\pgfpathrectangle{\pgfqpoint{0.519743in}{0.414757in}}{\pgfqpoint{0.983773in}{1.143573in}}%
\pgfusepath{clip}%
\pgfsetrectcap%
\pgfsetroundjoin%
\pgfsetlinewidth{1.505625pt}%
\definecolor{currentstroke}{rgb}{0.003922,0.450980,0.698039}%
\pgfsetstrokecolor{currentstroke}%
\pgfsetstrokeopacity{0.200000}%
\pgfsetdash{}{0pt}%
\pgfpathmoveto{\pgfqpoint{0.773139in}{0.502062in}}%
\pgfpathlineto{\pgfqpoint{0.877479in}{0.536356in}}%
\pgfpathlineto{\pgfqpoint{0.922196in}{0.553563in}}%
\pgfpathlineto{\pgfqpoint{0.966913in}{0.571974in}}%
\pgfpathlineto{\pgfqpoint{1.011630in}{0.591441in}}%
\pgfpathlineto{\pgfqpoint{1.056347in}{0.611835in}}%
\pgfpathlineto{\pgfqpoint{1.101064in}{0.633048in}}%
\pgfpathlineto{\pgfqpoint{1.145781in}{0.655006in}}%
\pgfpathlineto{\pgfqpoint{1.190498in}{0.677671in}}%
\pgfpathlineto{\pgfqpoint{1.235215in}{0.701067in}}%
\pgfpathlineto{\pgfqpoint{1.279932in}{0.725325in}}%
\pgfpathlineto{\pgfqpoint{1.324649in}{0.750786in}}%
\pgfpathlineto{\pgfqpoint{1.369365in}{0.778343in}}%
\pgfpathlineto{\pgfqpoint{1.414082in}{0.811034in}}%
\pgfpathlineto{\pgfqpoint{1.458799in}{1.268384in}}%
\pgfusepath{stroke}%
\end{pgfscope}%
\begin{pgfscope}%
\pgfpathrectangle{\pgfqpoint{0.519743in}{0.414757in}}{\pgfqpoint{0.983773in}{1.143573in}}%
\pgfusepath{clip}%
\pgfsetrectcap%
\pgfsetroundjoin%
\pgfsetlinewidth{1.505625pt}%
\definecolor{currentstroke}{rgb}{0.003922,0.450980,0.698039}%
\pgfsetstrokecolor{currentstroke}%
\pgfsetstrokeopacity{0.200000}%
\pgfsetdash{}{0pt}%
\pgfpathmoveto{\pgfqpoint{0.698611in}{0.492914in}}%
\pgfpathlineto{\pgfqpoint{0.832762in}{0.536772in}}%
\pgfpathlineto{\pgfqpoint{0.877479in}{0.554155in}}%
\pgfpathlineto{\pgfqpoint{0.922196in}{0.572401in}}%
\pgfpathlineto{\pgfqpoint{0.966913in}{0.591356in}}%
\pgfpathlineto{\pgfqpoint{1.011630in}{0.610896in}}%
\pgfpathlineto{\pgfqpoint{1.056347in}{0.630925in}}%
\pgfpathlineto{\pgfqpoint{1.101064in}{0.651378in}}%
\pgfpathlineto{\pgfqpoint{1.145781in}{0.672219in}}%
\pgfpathlineto{\pgfqpoint{1.190498in}{0.693457in}}%
\pgfpathlineto{\pgfqpoint{1.235215in}{0.715161in}}%
\pgfpathlineto{\pgfqpoint{1.279932in}{0.737504in}}%
\pgfpathlineto{\pgfqpoint{1.324649in}{0.760873in}}%
\pgfpathlineto{\pgfqpoint{1.369365in}{0.786205in}}%
\pgfpathlineto{\pgfqpoint{1.414082in}{0.816601in}}%
\pgfpathlineto{\pgfqpoint{1.458799in}{1.273350in}}%
\pgfusepath{stroke}%
\end{pgfscope}%
\begin{pgfscope}%
\pgfpathrectangle{\pgfqpoint{0.519743in}{0.414757in}}{\pgfqpoint{0.983773in}{1.143573in}}%
\pgfusepath{clip}%
\pgfsetrectcap%
\pgfsetroundjoin%
\pgfsetlinewidth{1.505625pt}%
\definecolor{currentstroke}{rgb}{0.003922,0.450980,0.698039}%
\pgfsetstrokecolor{currentstroke}%
\pgfsetstrokeopacity{0.200000}%
\pgfsetdash{}{0pt}%
\pgfpathmoveto{\pgfqpoint{0.788045in}{0.522422in}}%
\pgfpathlineto{\pgfqpoint{0.832762in}{0.538665in}}%
\pgfpathlineto{\pgfqpoint{0.877479in}{0.555807in}}%
\pgfpathlineto{\pgfqpoint{0.922196in}{0.573677in}}%
\pgfpathlineto{\pgfqpoint{0.966913in}{0.592135in}}%
\pgfpathlineto{\pgfqpoint{1.011630in}{0.611076in}}%
\pgfpathlineto{\pgfqpoint{1.056347in}{0.630418in}}%
\pgfpathlineto{\pgfqpoint{1.101064in}{0.650109in}}%
\pgfpathlineto{\pgfqpoint{1.145781in}{0.670128in}}%
\pgfpathlineto{\pgfqpoint{1.190498in}{0.690491in}}%
\pgfpathlineto{\pgfqpoint{1.235215in}{0.711277in}}%
\pgfpathlineto{\pgfqpoint{1.279932in}{0.732663in}}%
\pgfpathlineto{\pgfqpoint{1.324649in}{0.755032in}}%
\pgfpathlineto{\pgfqpoint{1.369365in}{0.779306in}}%
\pgfpathlineto{\pgfqpoint{1.414082in}{0.808504in}}%
\pgfpathlineto{\pgfqpoint{1.458799in}{1.260589in}}%
\pgfusepath{stroke}%
\end{pgfscope}%
\begin{pgfscope}%
\pgfpathrectangle{\pgfqpoint{0.519743in}{0.414757in}}{\pgfqpoint{0.983773in}{1.143573in}}%
\pgfusepath{clip}%
\pgfsetrectcap%
\pgfsetroundjoin%
\pgfsetlinewidth{1.505625pt}%
\definecolor{currentstroke}{rgb}{0.003922,0.450980,0.698039}%
\pgfsetstrokecolor{currentstroke}%
\pgfsetstrokeopacity{0.200000}%
\pgfsetdash{}{0pt}%
\pgfpathmoveto{\pgfqpoint{0.788045in}{0.504554in}}%
\pgfpathlineto{\pgfqpoint{0.832762in}{0.518309in}}%
\pgfpathlineto{\pgfqpoint{0.877479in}{0.533593in}}%
\pgfpathlineto{\pgfqpoint{0.922196in}{0.550235in}}%
\pgfpathlineto{\pgfqpoint{0.966913in}{0.568077in}}%
\pgfpathlineto{\pgfqpoint{1.011630in}{0.586980in}}%
\pgfpathlineto{\pgfqpoint{1.056347in}{0.606822in}}%
\pgfpathlineto{\pgfqpoint{1.101064in}{0.627503in}}%
\pgfpathlineto{\pgfqpoint{1.145781in}{0.648951in}}%
\pgfpathlineto{\pgfqpoint{1.190498in}{0.671133in}}%
\pgfpathlineto{\pgfqpoint{1.235215in}{0.694076in}}%
\pgfpathlineto{\pgfqpoint{1.279932in}{0.717909in}}%
\pgfpathlineto{\pgfqpoint{1.324649in}{0.742974in}}%
\pgfpathlineto{\pgfqpoint{1.369365in}{0.770156in}}%
\pgfpathlineto{\pgfqpoint{1.414082in}{0.802469in}}%
\pgfpathlineto{\pgfqpoint{1.458799in}{1.261971in}}%
\pgfusepath{stroke}%
\end{pgfscope}%
\begin{pgfscope}%
\pgfpathrectangle{\pgfqpoint{0.519743in}{0.414757in}}{\pgfqpoint{0.983773in}{1.143573in}}%
\pgfusepath{clip}%
\pgfsetrectcap%
\pgfsetroundjoin%
\pgfsetlinewidth{1.505625pt}%
\definecolor{currentstroke}{rgb}{0.003922,0.450980,0.698039}%
\pgfsetstrokecolor{currentstroke}%
\pgfsetstrokeopacity{0.200000}%
\pgfsetdash{}{0pt}%
\pgfpathmoveto{\pgfqpoint{0.716498in}{0.491462in}}%
\pgfpathlineto{\pgfqpoint{0.832762in}{0.527403in}}%
\pgfpathlineto{\pgfqpoint{0.877479in}{0.544269in}}%
\pgfpathlineto{\pgfqpoint{0.922196in}{0.562359in}}%
\pgfpathlineto{\pgfqpoint{0.966913in}{0.581497in}}%
\pgfpathlineto{\pgfqpoint{1.011630in}{0.601530in}}%
\pgfpathlineto{\pgfqpoint{1.056347in}{0.622329in}}%
\pgfpathlineto{\pgfqpoint{1.101064in}{0.643792in}}%
\pgfpathlineto{\pgfqpoint{1.145781in}{0.665851in}}%
\pgfpathlineto{\pgfqpoint{1.190498in}{0.688478in}}%
\pgfpathlineto{\pgfqpoint{1.235215in}{0.711706in}}%
\pgfpathlineto{\pgfqpoint{1.279932in}{0.735675in}}%
\pgfpathlineto{\pgfqpoint{1.324649in}{0.760734in}}%
\pgfpathlineto{\pgfqpoint{1.369365in}{0.787778in}}%
\pgfpathlineto{\pgfqpoint{1.414082in}{0.819819in}}%
\pgfpathlineto{\pgfqpoint{1.458799in}{1.268594in}}%
\pgfusepath{stroke}%
\end{pgfscope}%
\begin{pgfscope}%
\pgfpathrectangle{\pgfqpoint{0.519743in}{0.414757in}}{\pgfqpoint{0.983773in}{1.143573in}}%
\pgfusepath{clip}%
\pgfsetbuttcap%
\pgfsetroundjoin%
\pgfsetlinewidth{1.505625pt}%
\definecolor{currentstroke}{rgb}{0.501961,0.501961,0.501961}%
\pgfsetstrokecolor{currentstroke}%
\pgfsetdash{{5.550000pt}{2.400000pt}}{0.000000pt}%
\pgfpathmoveto{\pgfqpoint{0.564460in}{0.726641in}}%
\pgfpathlineto{\pgfqpoint{1.458799in}{0.726641in}}%
\pgfusepath{stroke}%
\end{pgfscope}%
\begin{pgfscope}%
\pgfpathrectangle{\pgfqpoint{0.519743in}{0.414757in}}{\pgfqpoint{0.983773in}{1.143573in}}%
\pgfusepath{clip}%
\pgfsetbuttcap%
\pgfsetroundjoin%
\pgfsetlinewidth{1.505625pt}%
\definecolor{currentstroke}{rgb}{0.501961,0.501961,0.501961}%
\pgfsetstrokecolor{currentstroke}%
\pgfsetdash{{5.550000pt}{2.400000pt}}{0.000000pt}%
\pgfpathmoveto{\pgfqpoint{0.788045in}{0.466738in}}%
\pgfpathlineto{\pgfqpoint{0.788045in}{1.506349in}}%
\pgfusepath{stroke}%
\end{pgfscope}%
\begin{pgfscope}%
\pgfsetrectcap%
\pgfsetmiterjoin%
\pgfsetlinewidth{0.803000pt}%
\definecolor{currentstroke}{rgb}{0.000000,0.000000,0.000000}%
\pgfsetstrokecolor{currentstroke}%
\pgfsetdash{}{0pt}%
\pgfpathmoveto{\pgfqpoint{0.519743in}{0.414757in}}%
\pgfpathlineto{\pgfqpoint{0.519743in}{1.558330in}}%
\pgfusepath{stroke}%
\end{pgfscope}%
\begin{pgfscope}%
\pgfsetrectcap%
\pgfsetmiterjoin%
\pgfsetlinewidth{0.803000pt}%
\definecolor{currentstroke}{rgb}{0.000000,0.000000,0.000000}%
\pgfsetstrokecolor{currentstroke}%
\pgfsetdash{}{0pt}%
\pgfpathmoveto{\pgfqpoint{1.503516in}{0.414757in}}%
\pgfpathlineto{\pgfqpoint{1.503516in}{1.558330in}}%
\pgfusepath{stroke}%
\end{pgfscope}%
\begin{pgfscope}%
\pgfsetrectcap%
\pgfsetmiterjoin%
\pgfsetlinewidth{0.803000pt}%
\definecolor{currentstroke}{rgb}{0.000000,0.000000,0.000000}%
\pgfsetstrokecolor{currentstroke}%
\pgfsetdash{}{0pt}%
\pgfpathmoveto{\pgfqpoint{0.519743in}{0.414757in}}%
\pgfpathlineto{\pgfqpoint{1.503516in}{0.414757in}}%
\pgfusepath{stroke}%
\end{pgfscope}%
\begin{pgfscope}%
\pgfsetrectcap%
\pgfsetmiterjoin%
\pgfsetlinewidth{0.803000pt}%
\definecolor{currentstroke}{rgb}{0.000000,0.000000,0.000000}%
\pgfsetstrokecolor{currentstroke}%
\pgfsetdash{}{0pt}%
\pgfpathmoveto{\pgfqpoint{0.519743in}{1.558330in}}%
\pgfpathlineto{\pgfqpoint{1.503516in}{1.558330in}}%
\pgfusepath{stroke}%
\end{pgfscope}%
\begin{pgfscope}%
\pgfsetbuttcap%
\pgfsetmiterjoin%
\definecolor{currentfill}{rgb}{1.000000,1.000000,1.000000}%
\pgfsetfillcolor{currentfill}%
\pgfsetfillopacity{0.800000}%
\pgfsetlinewidth{1.003750pt}%
\definecolor{currentstroke}{rgb}{0.800000,0.800000,0.800000}%
\pgfsetstrokecolor{currentstroke}%
\pgfsetstrokeopacity{0.800000}%
\pgfsetdash{}{0pt}%
\pgfpathmoveto{\pgfqpoint{0.542878in}{1.008611in}}%
\pgfpathlineto{\pgfqpoint{1.537122in}{1.008611in}}%
\pgfpathquadraticcurveto{\pgfqpoint{1.556567in}{1.008611in}}{\pgfqpoint{1.556567in}{1.028056in}}%
\pgfpathlineto{\pgfqpoint{1.556567in}{1.446433in}}%
\pgfpathquadraticcurveto{\pgfqpoint{1.556567in}{1.465878in}}{\pgfqpoint{1.537122in}{1.465878in}}%
\pgfpathlineto{\pgfqpoint{0.542878in}{1.465878in}}%
\pgfpathquadraticcurveto{\pgfqpoint{0.523433in}{1.465878in}}{\pgfqpoint{0.523433in}{1.446433in}}%
\pgfpathlineto{\pgfqpoint{0.523433in}{1.028056in}}%
\pgfpathquadraticcurveto{\pgfqpoint{0.523433in}{1.008611in}}{\pgfqpoint{0.542878in}{1.008611in}}%
\pgfpathlineto{\pgfqpoint{0.542878in}{1.008611in}}%
\pgfpathclose%
\pgfusepath{stroke,fill}%
\end{pgfscope}%
\begin{pgfscope}%
\pgfsetbuttcap%
\pgfsetmiterjoin%
\definecolor{currentfill}{rgb}{0.007843,0.619608,0.447059}%
\pgfsetfillcolor{currentfill}%
\pgfsetfillopacity{0.850000}%
\pgfsetlinewidth{0.501875pt}%
\definecolor{currentstroke}{rgb}{0.000000,0.000000,0.000000}%
\pgfsetstrokecolor{currentstroke}%
\pgfsetstrokeopacity{0.850000}%
\pgfsetdash{}{0pt}%
\pgfpathmoveto{\pgfqpoint{0.562322in}{1.353123in}}%
\pgfpathlineto{\pgfqpoint{0.756766in}{1.353123in}}%
\pgfpathlineto{\pgfqpoint{0.756766in}{1.421178in}}%
\pgfpathlineto{\pgfqpoint{0.562322in}{1.421178in}}%
\pgfpathlineto{\pgfqpoint{0.562322in}{1.353123in}}%
\pgfpathclose%
\pgfusepath{stroke,fill}%
\end{pgfscope}%
\begin{pgfscope}%
\definecolor{textcolor}{rgb}{0.000000,0.000000,0.000000}%
\pgfsetstrokecolor{textcolor}%
\pgfsetfillcolor{textcolor}%
\pgftext[x=0.834544in,y=1.353123in,left,base]{\color{textcolor}{\rmfamily\fontsize{7.000000}{8.400000}\selectfont\catcode`\^=\active\def^{\ifmmode\sp\else\^{}\fi}\catcode`\%=\active\def
\end{pgfscope}%
\begin{pgfscope}%
\pgfsetbuttcap%
\pgfsetmiterjoin%
\definecolor{currentfill}{rgb}{0.870588,0.560784,0.011765}%
\pgfsetfillcolor{currentfill}%
\pgfsetfillopacity{0.850000}%
\pgfsetlinewidth{0.501875pt}%
\definecolor{currentstroke}{rgb}{0.000000,0.000000,0.000000}%
\pgfsetstrokecolor{currentstroke}%
\pgfsetstrokeopacity{0.850000}%
\pgfsetdash{}{0pt}%
\pgfpathmoveto{\pgfqpoint{0.562322in}{1.210423in}}%
\pgfpathlineto{\pgfqpoint{0.756766in}{1.210423in}}%
\pgfpathlineto{\pgfqpoint{0.756766in}{1.278478in}}%
\pgfpathlineto{\pgfqpoint{0.562322in}{1.278478in}}%
\pgfpathlineto{\pgfqpoint{0.562322in}{1.210423in}}%
\pgfpathclose%
\pgfusepath{stroke,fill}%
\end{pgfscope}%
\begin{pgfscope}%
\definecolor{textcolor}{rgb}{0.000000,0.000000,0.000000}%
\pgfsetstrokecolor{textcolor}%
\pgfsetfillcolor{textcolor}%
\pgftext[x=0.834544in,y=1.210423in,left,base]{\color{textcolor}{\rmfamily\fontsize{7.000000}{8.400000}\selectfont\catcode`\^=\active\def^{\ifmmode\sp\else\^{}\fi}\catcode`\%=\active\def
\end{pgfscope}%
\begin{pgfscope}%
\pgfsetbuttcap%
\pgfsetroundjoin%
\pgfsetlinewidth{0.803000pt}%
\definecolor{currentstroke}{rgb}{0.000000,0.000000,0.000000}%
\pgfsetstrokecolor{currentstroke}%
\pgfsetstrokeopacity{0.600000}%
\pgfsetdash{{2.960000pt}{1.280000pt}}{0.000000pt}%
\pgfpathmoveto{\pgfqpoint{0.562322in}{1.101751in}}%
\pgfpathlineto{\pgfqpoint{0.659544in}{1.101751in}}%
\pgfpathlineto{\pgfqpoint{0.756766in}{1.101751in}}%
\pgfusepath{stroke}%
\end{pgfscope}%
\begin{pgfscope}%
\definecolor{textcolor}{rgb}{0.000000,0.000000,0.000000}%
\pgfsetstrokecolor{textcolor}%
\pgfsetfillcolor{textcolor}%
\pgftext[x=0.834544in,y=1.067723in,left,base]{\color{textcolor}{\rmfamily\fontsize{7.000000}{8.400000}\selectfont\catcode`\^=\active\def^{\ifmmode\sp\else\^{}\fi}\catcode`\%=\active\def
\end{pgfscope}%
\end{pgfpicture}%
\makeatother%
\endgroup%

%% file: figures/calibration_maps/TRIVIAQA_Qwen3-4B-Instruct-2507.pgf
\begingroup%
\makeatletter%
\begin{pgfpicture}%
\pgfpathrectangle{\pgfpointorigin}{\pgfqpoint{1.600000in}{1.600000in}}%
\pgfusepath{use as bounding box, clip}%
\begin{pgfscope}%
\pgfsetbuttcap%
\pgfsetmiterjoin%
\definecolor{currentfill}{rgb}{1.000000,1.000000,1.000000}%
\pgfsetfillcolor{currentfill}%
\pgfsetlinewidth{0.000000pt}%
\definecolor{currentstroke}{rgb}{1.000000,1.000000,1.000000}%
\pgfsetstrokecolor{currentstroke}%
\pgfsetdash{}{0pt}%
\pgfpathmoveto{\pgfqpoint{0.000000in}{0.000000in}}%
\pgfpathlineto{\pgfqpoint{1.600000in}{0.000000in}}%
\pgfpathlineto{\pgfqpoint{1.600000in}{1.600000in}}%
\pgfpathlineto{\pgfqpoint{0.000000in}{1.600000in}}%
\pgfpathlineto{\pgfqpoint{0.000000in}{0.000000in}}%
\pgfpathclose%
\pgfusepath{fill}%
\end{pgfscope}%
\begin{pgfscope}%
\pgfsetbuttcap%
\pgfsetmiterjoin%
\definecolor{currentfill}{rgb}{1.000000,1.000000,1.000000}%
\pgfsetfillcolor{currentfill}%
\pgfsetlinewidth{0.000000pt}%
\definecolor{currentstroke}{rgb}{0.000000,0.000000,0.000000}%
\pgfsetstrokecolor{currentstroke}%
\pgfsetstrokeopacity{0.000000}%
\pgfsetdash{}{0pt}%
\pgfpathmoveto{\pgfqpoint{0.519743in}{0.414757in}}%
\pgfpathlineto{\pgfqpoint{1.503516in}{0.414757in}}%
\pgfpathlineto{\pgfqpoint{1.503516in}{1.558330in}}%
\pgfpathlineto{\pgfqpoint{0.519743in}{1.558330in}}%
\pgfpathlineto{\pgfqpoint{0.519743in}{0.414757in}}%
\pgfpathclose%
\pgfusepath{fill}%
\end{pgfscope}%
\begin{pgfscope}%
\pgfpathrectangle{\pgfqpoint{0.519743in}{0.414757in}}{\pgfqpoint{0.983773in}{1.143573in}}%
\pgfusepath{clip}%
\pgfsetbuttcap%
\pgfsetroundjoin%
\definecolor{currentfill}{rgb}{0.870588,0.560784,0.011765}%
\pgfsetfillcolor{currentfill}%
\pgfsetfillopacity{0.200000}%
\pgfsetlinewidth{1.003750pt}%
\definecolor{currentstroke}{rgb}{0.870588,0.560784,0.011765}%
\pgfsetstrokecolor{currentstroke}%
\pgfsetstrokeopacity{0.200000}%
\pgfsetdash{}{0pt}%
\pgfsys@defobject{currentmarker}{\pgfqpoint{0.564460in}{0.726641in}}{\pgfqpoint{0.788045in}{1.506349in}}{%
\pgfpathmoveto{\pgfqpoint{0.788045in}{0.726641in}}%
\pgfpathlineto{\pgfqpoint{0.564460in}{0.726641in}}%
\pgfpathlineto{\pgfqpoint{0.564460in}{1.506349in}}%
\pgfpathlineto{\pgfqpoint{0.788045in}{1.506349in}}%
\pgfpathlineto{\pgfqpoint{0.788045in}{1.506349in}}%
\pgfpathlineto{\pgfqpoint{0.788045in}{0.726641in}}%
\pgfpathlineto{\pgfqpoint{0.788045in}{0.726641in}}%
\pgfpathclose%
\pgfusepath{stroke,fill}%
}%
\begin{pgfscope}%
\pgfsys@transformshift{0.000000in}{0.000000in}%
\pgfsys@useobject{currentmarker}{}%
\end{pgfscope}%
\end{pgfscope}%
\begin{pgfscope}%
\pgfpathrectangle{\pgfqpoint{0.519743in}{0.414757in}}{\pgfqpoint{0.983773in}{1.143573in}}%
\pgfusepath{clip}%
\pgfsetbuttcap%
\pgfsetroundjoin%
\definecolor{currentfill}{rgb}{0.007843,0.619608,0.447059}%
\pgfsetfillcolor{currentfill}%
\pgfsetfillopacity{0.200000}%
\pgfsetlinewidth{1.003750pt}%
\definecolor{currentstroke}{rgb}{0.007843,0.619608,0.447059}%
\pgfsetstrokecolor{currentstroke}%
\pgfsetstrokeopacity{0.200000}%
\pgfsetdash{}{0pt}%
\pgfsys@defobject{currentmarker}{\pgfqpoint{0.788045in}{0.726641in}}{\pgfqpoint{1.458799in}{1.506349in}}{%
\pgfpathmoveto{\pgfqpoint{1.458799in}{0.726641in}}%
\pgfpathlineto{\pgfqpoint{0.788045in}{0.726641in}}%
\pgfpathlineto{\pgfqpoint{0.788045in}{1.506349in}}%
\pgfpathlineto{\pgfqpoint{1.458799in}{1.506349in}}%
\pgfpathlineto{\pgfqpoint{1.458799in}{1.506349in}}%
\pgfpathlineto{\pgfqpoint{1.458799in}{0.726641in}}%
\pgfpathlineto{\pgfqpoint{1.458799in}{0.726641in}}%
\pgfpathclose%
\pgfusepath{stroke,fill}%
}%
\begin{pgfscope}%
\pgfsys@transformshift{0.000000in}{0.000000in}%
\pgfsys@useobject{currentmarker}{}%
\end{pgfscope}%
\end{pgfscope}%
\begin{pgfscope}%
\pgfpathrectangle{\pgfqpoint{0.519743in}{0.414757in}}{\pgfqpoint{0.983773in}{1.143573in}}%
\pgfusepath{clip}%
\pgfsetbuttcap%
\pgfsetroundjoin%
\definecolor{currentfill}{rgb}{0.870588,0.560784,0.011765}%
\pgfsetfillcolor{currentfill}%
\pgfsetfillopacity{0.200000}%
\pgfsetlinewidth{1.003750pt}%
\definecolor{currentstroke}{rgb}{0.870588,0.560784,0.011765}%
\pgfsetstrokecolor{currentstroke}%
\pgfsetstrokeopacity{0.200000}%
\pgfsetdash{}{0pt}%
\pgfsys@defobject{currentmarker}{\pgfqpoint{0.788045in}{0.466738in}}{\pgfqpoint{1.458799in}{0.726641in}}{%
\pgfpathmoveto{\pgfqpoint{1.458799in}{0.466738in}}%
\pgfpathlineto{\pgfqpoint{0.788045in}{0.466738in}}%
\pgfpathlineto{\pgfqpoint{0.788045in}{0.726641in}}%
\pgfpathlineto{\pgfqpoint{1.458799in}{0.726641in}}%
\pgfpathlineto{\pgfqpoint{1.458799in}{0.726641in}}%
\pgfpathlineto{\pgfqpoint{1.458799in}{0.466738in}}%
\pgfpathlineto{\pgfqpoint{1.458799in}{0.466738in}}%
\pgfpathclose%
\pgfusepath{stroke,fill}%
}%
\begin{pgfscope}%
\pgfsys@transformshift{0.000000in}{0.000000in}%
\pgfsys@useobject{currentmarker}{}%
\end{pgfscope}%
\end{pgfscope}%
\begin{pgfscope}%
\pgfpathrectangle{\pgfqpoint{0.519743in}{0.414757in}}{\pgfqpoint{0.983773in}{1.143573in}}%
\pgfusepath{clip}%
\pgfsetbuttcap%
\pgfsetroundjoin%
\definecolor{currentfill}{rgb}{0.007843,0.619608,0.447059}%
\pgfsetfillcolor{currentfill}%
\pgfsetfillopacity{0.200000}%
\pgfsetlinewidth{1.003750pt}%
\definecolor{currentstroke}{rgb}{0.007843,0.619608,0.447059}%
\pgfsetstrokecolor{currentstroke}%
\pgfsetstrokeopacity{0.200000}%
\pgfsetdash{}{0pt}%
\pgfsys@defobject{currentmarker}{\pgfqpoint{0.564460in}{0.466738in}}{\pgfqpoint{0.788045in}{0.726641in}}{%
\pgfpathmoveto{\pgfqpoint{0.788045in}{0.466738in}}%
\pgfpathlineto{\pgfqpoint{0.564460in}{0.466738in}}%
\pgfpathlineto{\pgfqpoint{0.564460in}{0.726641in}}%
\pgfpathlineto{\pgfqpoint{0.788045in}{0.726641in}}%
\pgfpathlineto{\pgfqpoint{0.788045in}{0.726641in}}%
\pgfpathlineto{\pgfqpoint{0.788045in}{0.466738in}}%
\pgfpathlineto{\pgfqpoint{0.788045in}{0.466738in}}%
\pgfpathclose%
\pgfusepath{stroke,fill}%
}%
\begin{pgfscope}%
\pgfsys@transformshift{0.000000in}{0.000000in}%
\pgfsys@useobject{currentmarker}{}%
\end{pgfscope}%
\end{pgfscope}%
\begin{pgfscope}%
\pgfpathrectangle{\pgfqpoint{0.519743in}{0.414757in}}{\pgfqpoint{0.983773in}{1.143573in}}%
\pgfusepath{clip}%
\pgfsetrectcap%
\pgfsetroundjoin%
\pgfsetlinewidth{0.803000pt}%
\definecolor{currentstroke}{rgb}{0.690196,0.690196,0.690196}%
\pgfsetstrokecolor{currentstroke}%
\pgfsetdash{}{0pt}%
\pgfpathmoveto{\pgfqpoint{0.564460in}{0.414757in}}%
\pgfpathlineto{\pgfqpoint{0.564460in}{1.558330in}}%
\pgfusepath{stroke}%
\end{pgfscope}%
\begin{pgfscope}%
\pgfsetbuttcap%
\pgfsetroundjoin%
\definecolor{currentfill}{rgb}{0.000000,0.000000,0.000000}%
\pgfsetfillcolor{currentfill}%
\pgfsetlinewidth{0.803000pt}%
\definecolor{currentstroke}{rgb}{0.000000,0.000000,0.000000}%
\pgfsetstrokecolor{currentstroke}%
\pgfsetdash{}{0pt}%
\pgfsys@defobject{currentmarker}{\pgfqpoint{0.000000in}{-0.048611in}}{\pgfqpoint{0.000000in}{0.000000in}}{%
\pgfpathmoveto{\pgfqpoint{0.000000in}{0.000000in}}%
\pgfpathlineto{\pgfqpoint{0.000000in}{-0.048611in}}%
\pgfusepath{stroke,fill}%
}%
\begin{pgfscope}%
\pgfsys@transformshift{0.564460in}{0.414757in}%
\pgfsys@useobject{currentmarker}{}%
\end{pgfscope}%
\end{pgfscope}%
\begin{pgfscope}%
\definecolor{textcolor}{rgb}{0.000000,0.000000,0.000000}%
\pgfsetstrokecolor{textcolor}%
\pgfsetfillcolor{textcolor}%
\pgftext[x=0.564460in,y=0.317535in,,top]{\color{textcolor}{\rmfamily\fontsize{7.000000}{8.400000}\selectfont\catcode`\^=\active\def^{\ifmmode\sp\else\^{}\fi}\catcode`\%=\active\def
\end{pgfscope}%
\begin{pgfscope}%
\pgfpathrectangle{\pgfqpoint{0.519743in}{0.414757in}}{\pgfqpoint{0.983773in}{1.143573in}}%
\pgfusepath{clip}%
\pgfsetrectcap%
\pgfsetroundjoin%
\pgfsetlinewidth{0.803000pt}%
\definecolor{currentstroke}{rgb}{0.690196,0.690196,0.690196}%
\pgfsetstrokecolor{currentstroke}%
\pgfsetdash{}{0pt}%
\pgfpathmoveto{\pgfqpoint{0.788045in}{0.414757in}}%
\pgfpathlineto{\pgfqpoint{0.788045in}{1.558330in}}%
\pgfusepath{stroke}%
\end{pgfscope}%
\begin{pgfscope}%
\pgfsetbuttcap%
\pgfsetroundjoin%
\definecolor{currentfill}{rgb}{0.000000,0.000000,0.000000}%
\pgfsetfillcolor{currentfill}%
\pgfsetlinewidth{0.803000pt}%
\definecolor{currentstroke}{rgb}{0.000000,0.000000,0.000000}%
\pgfsetstrokecolor{currentstroke}%
\pgfsetdash{}{0pt}%
\pgfsys@defobject{currentmarker}{\pgfqpoint{0.000000in}{-0.048611in}}{\pgfqpoint{0.000000in}{0.000000in}}{%
\pgfpathmoveto{\pgfqpoint{0.000000in}{0.000000in}}%
\pgfpathlineto{\pgfqpoint{0.000000in}{-0.048611in}}%
\pgfusepath{stroke,fill}%
}%
\begin{pgfscope}%
\pgfsys@transformshift{0.788045in}{0.414757in}%
\pgfsys@useobject{currentmarker}{}%
\end{pgfscope}%
\end{pgfscope}%
\begin{pgfscope}%
\definecolor{textcolor}{rgb}{0.000000,0.000000,0.000000}%
\pgfsetstrokecolor{textcolor}%
\pgfsetfillcolor{textcolor}%
\pgftext[x=0.788045in,y=0.317535in,,top]{\color{textcolor}{\rmfamily\fontsize{7.000000}{8.400000}\selectfont\catcode`\^=\active\def^{\ifmmode\sp\else\^{}\fi}\catcode`\%=\active\def
\end{pgfscope}%
\begin{pgfscope}%
\pgfpathrectangle{\pgfqpoint{0.519743in}{0.414757in}}{\pgfqpoint{0.983773in}{1.143573in}}%
\pgfusepath{clip}%
\pgfsetrectcap%
\pgfsetroundjoin%
\pgfsetlinewidth{0.803000pt}%
\definecolor{currentstroke}{rgb}{0.690196,0.690196,0.690196}%
\pgfsetstrokecolor{currentstroke}%
\pgfsetdash{}{0pt}%
\pgfpathmoveto{\pgfqpoint{1.011630in}{0.414757in}}%
\pgfpathlineto{\pgfqpoint{1.011630in}{1.558330in}}%
\pgfusepath{stroke}%
\end{pgfscope}%
\begin{pgfscope}%
\pgfsetbuttcap%
\pgfsetroundjoin%
\definecolor{currentfill}{rgb}{0.000000,0.000000,0.000000}%
\pgfsetfillcolor{currentfill}%
\pgfsetlinewidth{0.803000pt}%
\definecolor{currentstroke}{rgb}{0.000000,0.000000,0.000000}%
\pgfsetstrokecolor{currentstroke}%
\pgfsetdash{}{0pt}%
\pgfsys@defobject{currentmarker}{\pgfqpoint{0.000000in}{-0.048611in}}{\pgfqpoint{0.000000in}{0.000000in}}{%
\pgfpathmoveto{\pgfqpoint{0.000000in}{0.000000in}}%
\pgfpathlineto{\pgfqpoint{0.000000in}{-0.048611in}}%
\pgfusepath{stroke,fill}%
}%
\begin{pgfscope}%
\pgfsys@transformshift{1.011630in}{0.414757in}%
\pgfsys@useobject{currentmarker}{}%
\end{pgfscope}%
\end{pgfscope}%
\begin{pgfscope}%
\definecolor{textcolor}{rgb}{0.000000,0.000000,0.000000}%
\pgfsetstrokecolor{textcolor}%
\pgfsetfillcolor{textcolor}%
\pgftext[x=1.011630in,y=0.317535in,,top]{\color{textcolor}{\rmfamily\fontsize{7.000000}{8.400000}\selectfont\catcode`\^=\active\def^{\ifmmode\sp\else\^{}\fi}\catcode`\%=\active\def
\end{pgfscope}%
\begin{pgfscope}%
\pgfpathrectangle{\pgfqpoint{0.519743in}{0.414757in}}{\pgfqpoint{0.983773in}{1.143573in}}%
\pgfusepath{clip}%
\pgfsetrectcap%
\pgfsetroundjoin%
\pgfsetlinewidth{0.803000pt}%
\definecolor{currentstroke}{rgb}{0.690196,0.690196,0.690196}%
\pgfsetstrokecolor{currentstroke}%
\pgfsetdash{}{0pt}%
\pgfpathmoveto{\pgfqpoint{1.235215in}{0.414757in}}%
\pgfpathlineto{\pgfqpoint{1.235215in}{1.558330in}}%
\pgfusepath{stroke}%
\end{pgfscope}%
\begin{pgfscope}%
\pgfsetbuttcap%
\pgfsetroundjoin%
\definecolor{currentfill}{rgb}{0.000000,0.000000,0.000000}%
\pgfsetfillcolor{currentfill}%
\pgfsetlinewidth{0.803000pt}%
\definecolor{currentstroke}{rgb}{0.000000,0.000000,0.000000}%
\pgfsetstrokecolor{currentstroke}%
\pgfsetdash{}{0pt}%
\pgfsys@defobject{currentmarker}{\pgfqpoint{0.000000in}{-0.048611in}}{\pgfqpoint{0.000000in}{0.000000in}}{%
\pgfpathmoveto{\pgfqpoint{0.000000in}{0.000000in}}%
\pgfpathlineto{\pgfqpoint{0.000000in}{-0.048611in}}%
\pgfusepath{stroke,fill}%
}%
\begin{pgfscope}%
\pgfsys@transformshift{1.235215in}{0.414757in}%
\pgfsys@useobject{currentmarker}{}%
\end{pgfscope}%
\end{pgfscope}%
\begin{pgfscope}%
\definecolor{textcolor}{rgb}{0.000000,0.000000,0.000000}%
\pgfsetstrokecolor{textcolor}%
\pgfsetfillcolor{textcolor}%
\pgftext[x=1.235215in,y=0.317535in,,top]{\color{textcolor}{\rmfamily\fontsize{7.000000}{8.400000}\selectfont\catcode`\^=\active\def^{\ifmmode\sp\else\^{}\fi}\catcode`\%=\active\def
\end{pgfscope}%
\begin{pgfscope}%
\pgfpathrectangle{\pgfqpoint{0.519743in}{0.414757in}}{\pgfqpoint{0.983773in}{1.143573in}}%
\pgfusepath{clip}%
\pgfsetrectcap%
\pgfsetroundjoin%
\pgfsetlinewidth{0.803000pt}%
\definecolor{currentstroke}{rgb}{0.690196,0.690196,0.690196}%
\pgfsetstrokecolor{currentstroke}%
\pgfsetdash{}{0pt}%
\pgfpathmoveto{\pgfqpoint{1.458799in}{0.414757in}}%
\pgfpathlineto{\pgfqpoint{1.458799in}{1.558330in}}%
\pgfusepath{stroke}%
\end{pgfscope}%
\begin{pgfscope}%
\pgfsetbuttcap%
\pgfsetroundjoin%
\definecolor{currentfill}{rgb}{0.000000,0.000000,0.000000}%
\pgfsetfillcolor{currentfill}%
\pgfsetlinewidth{0.803000pt}%
\definecolor{currentstroke}{rgb}{0.000000,0.000000,0.000000}%
\pgfsetstrokecolor{currentstroke}%
\pgfsetdash{}{0pt}%
\pgfsys@defobject{currentmarker}{\pgfqpoint{0.000000in}{-0.048611in}}{\pgfqpoint{0.000000in}{0.000000in}}{%
\pgfpathmoveto{\pgfqpoint{0.000000in}{0.000000in}}%
\pgfpathlineto{\pgfqpoint{0.000000in}{-0.048611in}}%
\pgfusepath{stroke,fill}%
}%
\begin{pgfscope}%
\pgfsys@transformshift{1.458799in}{0.414757in}%
\pgfsys@useobject{currentmarker}{}%
\end{pgfscope}%
\end{pgfscope}%
\begin{pgfscope}%
\definecolor{textcolor}{rgb}{0.000000,0.000000,0.000000}%
\pgfsetstrokecolor{textcolor}%
\pgfsetfillcolor{textcolor}%
\pgftext[x=1.458799in,y=0.317535in,,top]{\color{textcolor}{\rmfamily\fontsize{7.000000}{8.400000}\selectfont\catcode`\^=\active\def^{\ifmmode\sp\else\^{}\fi}\catcode`\%=\active\def
\end{pgfscope}%
\begin{pgfscope}%
\definecolor{textcolor}{rgb}{0.000000,0.000000,0.000000}%
\pgfsetstrokecolor{textcolor}%
\pgfsetfillcolor{textcolor}%
\pgftext[x=1.011630in,y=0.167891in,,top]{\color{textcolor}{\rmfamily\fontsize{9.000000}{10.800000}\selectfont\catcode`\^=\active\def^{\ifmmode\sp\else\^{}\fi}\catcode`\%=\active\def
\end{pgfscope}%
\begin{pgfscope}%
\pgfpathrectangle{\pgfqpoint{0.519743in}{0.414757in}}{\pgfqpoint{0.983773in}{1.143573in}}%
\pgfusepath{clip}%
\pgfsetrectcap%
\pgfsetroundjoin%
\pgfsetlinewidth{0.803000pt}%
\definecolor{currentstroke}{rgb}{0.690196,0.690196,0.690196}%
\pgfsetstrokecolor{currentstroke}%
\pgfsetdash{}{0pt}%
\pgfpathmoveto{\pgfqpoint{0.519743in}{0.466738in}}%
\pgfpathlineto{\pgfqpoint{1.503516in}{0.466738in}}%
\pgfusepath{stroke}%
\end{pgfscope}%
\begin{pgfscope}%
\pgfsetbuttcap%
\pgfsetroundjoin%
\definecolor{currentfill}{rgb}{0.000000,0.000000,0.000000}%
\pgfsetfillcolor{currentfill}%
\pgfsetlinewidth{0.803000pt}%
\definecolor{currentstroke}{rgb}{0.000000,0.000000,0.000000}%
\pgfsetstrokecolor{currentstroke}%
\pgfsetdash{}{0pt}%
\pgfsys@defobject{currentmarker}{\pgfqpoint{-0.048611in}{0.000000in}}{\pgfqpoint{-0.000000in}{0.000000in}}{%
\pgfpathmoveto{\pgfqpoint{-0.000000in}{0.000000in}}%
\pgfpathlineto{\pgfqpoint{-0.048611in}{0.000000in}}%
\pgfusepath{stroke,fill}%
}%
\begin{pgfscope}%
\pgfsys@transformshift{0.519743in}{0.466738in}%
\pgfsys@useobject{currentmarker}{}%
\end{pgfscope}%
\end{pgfscope}%
\begin{pgfscope}%
\definecolor{textcolor}{rgb}{0.000000,0.000000,0.000000}%
\pgfsetstrokecolor{textcolor}%
\pgfsetfillcolor{textcolor}%
\pgftext[x=0.223446in, y=0.429805in, left, base]{\color{textcolor}{\rmfamily\fontsize{7.000000}{8.400000}\selectfont\catcode`\^=\active\def^{\ifmmode\sp\else\^{}\fi}\catcode`\%=\active\def
\end{pgfscope}%
\begin{pgfscope}%
\pgfpathrectangle{\pgfqpoint{0.519743in}{0.414757in}}{\pgfqpoint{0.983773in}{1.143573in}}%
\pgfusepath{clip}%
\pgfsetrectcap%
\pgfsetroundjoin%
\pgfsetlinewidth{0.803000pt}%
\definecolor{currentstroke}{rgb}{0.690196,0.690196,0.690196}%
\pgfsetstrokecolor{currentstroke}%
\pgfsetdash{}{0pt}%
\pgfpathmoveto{\pgfqpoint{0.519743in}{0.726641in}}%
\pgfpathlineto{\pgfqpoint{1.503516in}{0.726641in}}%
\pgfusepath{stroke}%
\end{pgfscope}%
\begin{pgfscope}%
\pgfsetbuttcap%
\pgfsetroundjoin%
\definecolor{currentfill}{rgb}{0.000000,0.000000,0.000000}%
\pgfsetfillcolor{currentfill}%
\pgfsetlinewidth{0.803000pt}%
\definecolor{currentstroke}{rgb}{0.000000,0.000000,0.000000}%
\pgfsetstrokecolor{currentstroke}%
\pgfsetdash{}{0pt}%
\pgfsys@defobject{currentmarker}{\pgfqpoint{-0.048611in}{0.000000in}}{\pgfqpoint{-0.000000in}{0.000000in}}{%
\pgfpathmoveto{\pgfqpoint{-0.000000in}{0.000000in}}%
\pgfpathlineto{\pgfqpoint{-0.048611in}{0.000000in}}%
\pgfusepath{stroke,fill}%
}%
\begin{pgfscope}%
\pgfsys@transformshift{0.519743in}{0.726641in}%
\pgfsys@useobject{currentmarker}{}%
\end{pgfscope}%
\end{pgfscope}%
\begin{pgfscope}%
\definecolor{textcolor}{rgb}{0.000000,0.000000,0.000000}%
\pgfsetstrokecolor{textcolor}%
\pgfsetfillcolor{textcolor}%
\pgftext[x=0.223446in, y=0.689708in, left, base]{\color{textcolor}{\rmfamily\fontsize{7.000000}{8.400000}\selectfont\catcode`\^=\active\def^{\ifmmode\sp\else\^{}\fi}\catcode`\%=\active\def
\end{pgfscope}%
\begin{pgfscope}%
\pgfpathrectangle{\pgfqpoint{0.519743in}{0.414757in}}{\pgfqpoint{0.983773in}{1.143573in}}%
\pgfusepath{clip}%
\pgfsetrectcap%
\pgfsetroundjoin%
\pgfsetlinewidth{0.803000pt}%
\definecolor{currentstroke}{rgb}{0.690196,0.690196,0.690196}%
\pgfsetstrokecolor{currentstroke}%
\pgfsetdash{}{0pt}%
\pgfpathmoveto{\pgfqpoint{0.519743in}{0.986544in}}%
\pgfpathlineto{\pgfqpoint{1.503516in}{0.986544in}}%
\pgfusepath{stroke}%
\end{pgfscope}%
\begin{pgfscope}%
\pgfsetbuttcap%
\pgfsetroundjoin%
\definecolor{currentfill}{rgb}{0.000000,0.000000,0.000000}%
\pgfsetfillcolor{currentfill}%
\pgfsetlinewidth{0.803000pt}%
\definecolor{currentstroke}{rgb}{0.000000,0.000000,0.000000}%
\pgfsetstrokecolor{currentstroke}%
\pgfsetdash{}{0pt}%
\pgfsys@defobject{currentmarker}{\pgfqpoint{-0.048611in}{0.000000in}}{\pgfqpoint{-0.000000in}{0.000000in}}{%
\pgfpathmoveto{\pgfqpoint{-0.000000in}{0.000000in}}%
\pgfpathlineto{\pgfqpoint{-0.048611in}{0.000000in}}%
\pgfusepath{stroke,fill}%
}%
\begin{pgfscope}%
\pgfsys@transformshift{0.519743in}{0.986544in}%
\pgfsys@useobject{currentmarker}{}%
\end{pgfscope}%
\end{pgfscope}%
\begin{pgfscope}%
\definecolor{textcolor}{rgb}{0.000000,0.000000,0.000000}%
\pgfsetstrokecolor{textcolor}%
\pgfsetfillcolor{textcolor}%
\pgftext[x=0.223446in, y=0.949611in, left, base]{\color{textcolor}{\rmfamily\fontsize{7.000000}{8.400000}\selectfont\catcode`\^=\active\def^{\ifmmode\sp\else\^{}\fi}\catcode`\%=\active\def
\end{pgfscope}%
\begin{pgfscope}%
\pgfpathrectangle{\pgfqpoint{0.519743in}{0.414757in}}{\pgfqpoint{0.983773in}{1.143573in}}%
\pgfusepath{clip}%
\pgfsetrectcap%
\pgfsetroundjoin%
\pgfsetlinewidth{0.803000pt}%
\definecolor{currentstroke}{rgb}{0.690196,0.690196,0.690196}%
\pgfsetstrokecolor{currentstroke}%
\pgfsetdash{}{0pt}%
\pgfpathmoveto{\pgfqpoint{0.519743in}{1.246447in}}%
\pgfpathlineto{\pgfqpoint{1.503516in}{1.246447in}}%
\pgfusepath{stroke}%
\end{pgfscope}%
\begin{pgfscope}%
\pgfsetbuttcap%
\pgfsetroundjoin%
\definecolor{currentfill}{rgb}{0.000000,0.000000,0.000000}%
\pgfsetfillcolor{currentfill}%
\pgfsetlinewidth{0.803000pt}%
\definecolor{currentstroke}{rgb}{0.000000,0.000000,0.000000}%
\pgfsetstrokecolor{currentstroke}%
\pgfsetdash{}{0pt}%
\pgfsys@defobject{currentmarker}{\pgfqpoint{-0.048611in}{0.000000in}}{\pgfqpoint{-0.000000in}{0.000000in}}{%
\pgfpathmoveto{\pgfqpoint{-0.000000in}{0.000000in}}%
\pgfpathlineto{\pgfqpoint{-0.048611in}{0.000000in}}%
\pgfusepath{stroke,fill}%
}%
\begin{pgfscope}%
\pgfsys@transformshift{0.519743in}{1.246447in}%
\pgfsys@useobject{currentmarker}{}%
\end{pgfscope}%
\end{pgfscope}%
\begin{pgfscope}%
\definecolor{textcolor}{rgb}{0.000000,0.000000,0.000000}%
\pgfsetstrokecolor{textcolor}%
\pgfsetfillcolor{textcolor}%
\pgftext[x=0.223446in, y=1.209513in, left, base]{\color{textcolor}{\rmfamily\fontsize{7.000000}{8.400000}\selectfont\catcode`\^=\active\def^{\ifmmode\sp\else\^{}\fi}\catcode`\%=\active\def
\end{pgfscope}%
\begin{pgfscope}%
\pgfpathrectangle{\pgfqpoint{0.519743in}{0.414757in}}{\pgfqpoint{0.983773in}{1.143573in}}%
\pgfusepath{clip}%
\pgfsetrectcap%
\pgfsetroundjoin%
\pgfsetlinewidth{0.803000pt}%
\definecolor{currentstroke}{rgb}{0.690196,0.690196,0.690196}%
\pgfsetstrokecolor{currentstroke}%
\pgfsetdash{}{0pt}%
\pgfpathmoveto{\pgfqpoint{0.519743in}{1.506349in}}%
\pgfpathlineto{\pgfqpoint{1.503516in}{1.506349in}}%
\pgfusepath{stroke}%
\end{pgfscope}%
\begin{pgfscope}%
\pgfsetbuttcap%
\pgfsetroundjoin%
\definecolor{currentfill}{rgb}{0.000000,0.000000,0.000000}%
\pgfsetfillcolor{currentfill}%
\pgfsetlinewidth{0.803000pt}%
\definecolor{currentstroke}{rgb}{0.000000,0.000000,0.000000}%
\pgfsetstrokecolor{currentstroke}%
\pgfsetdash{}{0pt}%
\pgfsys@defobject{currentmarker}{\pgfqpoint{-0.048611in}{0.000000in}}{\pgfqpoint{-0.000000in}{0.000000in}}{%
\pgfpathmoveto{\pgfqpoint{-0.000000in}{0.000000in}}%
\pgfpathlineto{\pgfqpoint{-0.048611in}{0.000000in}}%
\pgfusepath{stroke,fill}%
}%
\begin{pgfscope}%
\pgfsys@transformshift{0.519743in}{1.506349in}%
\pgfsys@useobject{currentmarker}{}%
\end{pgfscope}%
\end{pgfscope}%
\begin{pgfscope}%
\definecolor{textcolor}{rgb}{0.000000,0.000000,0.000000}%
\pgfsetstrokecolor{textcolor}%
\pgfsetfillcolor{textcolor}%
\pgftext[x=0.223446in, y=1.469416in, left, base]{\color{textcolor}{\rmfamily\fontsize{7.000000}{8.400000}\selectfont\catcode`\^=\active\def^{\ifmmode\sp\else\^{}\fi}\catcode`\%=\active\def
\end{pgfscope}%
\begin{pgfscope}%
\definecolor{textcolor}{rgb}{0.000000,0.000000,0.000000}%
\pgfsetstrokecolor{textcolor}%
\pgfsetfillcolor{textcolor}%
\pgftext[x=0.167891in,y=0.986544in,,bottom,rotate=90.000000]{\color{textcolor}{\rmfamily\fontsize{9.000000}{10.800000}\selectfont\catcode`\^=\active\def^{\ifmmode\sp\else\^{}\fi}\catcode`\%=\active\def
\end{pgfscope}%
\begin{pgfscope}%
\pgfpathrectangle{\pgfqpoint{0.519743in}{0.414757in}}{\pgfqpoint{0.983773in}{1.143573in}}%
\pgfusepath{clip}%
\pgfsetrectcap%
\pgfsetroundjoin%
\pgfsetlinewidth{1.505625pt}%
\definecolor{currentstroke}{rgb}{0.003922,0.450980,0.698039}%
\pgfsetstrokecolor{currentstroke}%
\pgfsetstrokeopacity{0.200000}%
\pgfsetdash{}{0pt}%
\pgfpathmoveto{\pgfqpoint{0.698611in}{0.499941in}}%
\pgfpathlineto{\pgfqpoint{0.743328in}{0.515391in}}%
\pgfpathlineto{\pgfqpoint{0.788045in}{0.531903in}}%
\pgfpathlineto{\pgfqpoint{0.832762in}{0.549183in}}%
\pgfpathlineto{\pgfqpoint{0.877479in}{0.567012in}}%
\pgfpathlineto{\pgfqpoint{0.922196in}{0.585225in}}%
\pgfpathlineto{\pgfqpoint{0.966913in}{0.603692in}}%
\pgfpathlineto{\pgfqpoint{1.011630in}{0.622315in}}%
\pgfpathlineto{\pgfqpoint{1.056347in}{0.641018in}}%
\pgfpathlineto{\pgfqpoint{1.101064in}{0.659750in}}%
\pgfpathlineto{\pgfqpoint{1.145781in}{0.678480in}}%
\pgfpathlineto{\pgfqpoint{1.190498in}{0.697203in}}%
\pgfpathlineto{\pgfqpoint{1.235215in}{0.715948in}}%
\pgfpathlineto{\pgfqpoint{1.279932in}{0.734798in}}%
\pgfpathlineto{\pgfqpoint{1.324649in}{0.753939in}}%
\pgfpathlineto{\pgfqpoint{1.369365in}{0.773829in}}%
\pgfpathlineto{\pgfqpoint{1.414082in}{0.795968in}}%
\pgfpathlineto{\pgfqpoint{1.458799in}{1.045043in}}%
\pgfusepath{stroke}%
\end{pgfscope}%
\begin{pgfscope}%
\pgfpathrectangle{\pgfqpoint{0.519743in}{0.414757in}}{\pgfqpoint{0.983773in}{1.143573in}}%
\pgfusepath{clip}%
\pgfsetrectcap%
\pgfsetroundjoin%
\pgfsetlinewidth{1.505625pt}%
\definecolor{currentstroke}{rgb}{0.003922,0.450980,0.698039}%
\pgfsetstrokecolor{currentstroke}%
\pgfsetstrokeopacity{0.200000}%
\pgfsetdash{}{0pt}%
\pgfpathmoveto{\pgfqpoint{0.687432in}{0.488857in}}%
\pgfpathlineto{\pgfqpoint{0.788045in}{0.518869in}}%
\pgfpathlineto{\pgfqpoint{0.832762in}{0.534246in}}%
\pgfpathlineto{\pgfqpoint{0.877479in}{0.550488in}}%
\pgfpathlineto{\pgfqpoint{0.922196in}{0.567416in}}%
\pgfpathlineto{\pgfqpoint{0.966913in}{0.584884in}}%
\pgfpathlineto{\pgfqpoint{1.011630in}{0.602773in}}%
\pgfpathlineto{\pgfqpoint{1.056347in}{0.620987in}}%
\pgfpathlineto{\pgfqpoint{1.101064in}{0.639452in}}%
\pgfpathlineto{\pgfqpoint{1.145781in}{0.658115in}}%
\pgfpathlineto{\pgfqpoint{1.190498in}{0.676952in}}%
\pgfpathlineto{\pgfqpoint{1.235215in}{0.695971in}}%
\pgfpathlineto{\pgfqpoint{1.279932in}{0.715237in}}%
\pgfpathlineto{\pgfqpoint{1.324649in}{0.734922in}}%
\pgfpathlineto{\pgfqpoint{1.369365in}{0.755473in}}%
\pgfpathlineto{\pgfqpoint{1.414082in}{0.778393in}}%
\pgfpathlineto{\pgfqpoint{1.458799in}{1.033970in}}%
\pgfusepath{stroke}%
\end{pgfscope}%
\begin{pgfscope}%
\pgfpathrectangle{\pgfqpoint{0.519743in}{0.414757in}}{\pgfqpoint{0.983773in}{1.143573in}}%
\pgfusepath{clip}%
\pgfsetrectcap%
\pgfsetroundjoin%
\pgfsetlinewidth{1.505625pt}%
\definecolor{currentstroke}{rgb}{0.003922,0.450980,0.698039}%
\pgfsetstrokecolor{currentstroke}%
\pgfsetstrokeopacity{0.200000}%
\pgfsetdash{}{0pt}%
\pgfpathmoveto{\pgfqpoint{0.698611in}{0.504210in}}%
\pgfpathlineto{\pgfqpoint{0.743328in}{0.520315in}}%
\pgfpathlineto{\pgfqpoint{0.788045in}{0.537156in}}%
\pgfpathlineto{\pgfqpoint{0.832762in}{0.554483in}}%
\pgfpathlineto{\pgfqpoint{0.877479in}{0.572120in}}%
\pgfpathlineto{\pgfqpoint{0.922196in}{0.589937in}}%
\pgfpathlineto{\pgfqpoint{0.966913in}{0.607838in}}%
\pgfpathlineto{\pgfqpoint{1.011630in}{0.625755in}}%
\pgfpathlineto{\pgfqpoint{1.056347in}{0.643637in}}%
\pgfpathlineto{\pgfqpoint{1.101064in}{0.661457in}}%
\pgfpathlineto{\pgfqpoint{1.145781in}{0.679205in}}%
\pgfpathlineto{\pgfqpoint{1.190498in}{0.696894in}}%
\pgfpathlineto{\pgfqpoint{1.235215in}{0.714568in}}%
\pgfpathlineto{\pgfqpoint{1.279932in}{0.732324in}}%
\pgfpathlineto{\pgfqpoint{1.324649in}{0.750363in}}%
\pgfpathlineto{\pgfqpoint{1.369365in}{0.769158in}}%
\pgfpathlineto{\pgfqpoint{1.414082in}{0.790234in}}%
\pgfpathlineto{\pgfqpoint{1.458799in}{1.042467in}}%
\pgfusepath{stroke}%
\end{pgfscope}%
\begin{pgfscope}%
\pgfpathrectangle{\pgfqpoint{0.519743in}{0.414757in}}{\pgfqpoint{0.983773in}{1.143573in}}%
\pgfusepath{clip}%
\pgfsetrectcap%
\pgfsetroundjoin%
\pgfsetlinewidth{1.505625pt}%
\definecolor{currentstroke}{rgb}{0.003922,0.450980,0.698039}%
\pgfsetstrokecolor{currentstroke}%
\pgfsetstrokeopacity{0.200000}%
\pgfsetdash{}{0pt}%
\pgfpathmoveto{\pgfqpoint{0.676253in}{0.488044in}}%
\pgfpathlineto{\pgfqpoint{0.743328in}{0.507971in}}%
\pgfpathlineto{\pgfqpoint{0.788045in}{0.523187in}}%
\pgfpathlineto{\pgfqpoint{0.832762in}{0.539445in}}%
\pgfpathlineto{\pgfqpoint{0.877479in}{0.556512in}}%
\pgfpathlineto{\pgfqpoint{0.922196in}{0.574202in}}%
\pgfpathlineto{\pgfqpoint{0.966913in}{0.592366in}}%
\pgfpathlineto{\pgfqpoint{1.011630in}{0.610879in}}%
\pgfpathlineto{\pgfqpoint{1.056347in}{0.629646in}}%
\pgfpathlineto{\pgfqpoint{1.101064in}{0.648593in}}%
\pgfpathlineto{\pgfqpoint{1.145781in}{0.667667in}}%
\pgfpathlineto{\pgfqpoint{1.190498in}{0.686845in}}%
\pgfpathlineto{\pgfqpoint{1.235215in}{0.706135in}}%
\pgfpathlineto{\pgfqpoint{1.279932in}{0.725601in}}%
\pgfpathlineto{\pgfqpoint{1.324649in}{0.745410in}}%
\pgfpathlineto{\pgfqpoint{1.369365in}{0.765995in}}%
\pgfpathlineto{\pgfqpoint{1.414082in}{0.788801in}}%
\pgfpathlineto{\pgfqpoint{1.458799in}{1.033188in}}%
\pgfusepath{stroke}%
\end{pgfscope}%
\begin{pgfscope}%
\pgfpathrectangle{\pgfqpoint{0.519743in}{0.414757in}}{\pgfqpoint{0.983773in}{1.143573in}}%
\pgfusepath{clip}%
\pgfsetrectcap%
\pgfsetroundjoin%
\pgfsetlinewidth{1.505625pt}%
\definecolor{currentstroke}{rgb}{0.003922,0.450980,0.698039}%
\pgfsetstrokecolor{currentstroke}%
\pgfsetstrokeopacity{0.200000}%
\pgfsetdash{}{0pt}%
\pgfpathmoveto{\pgfqpoint{0.683706in}{0.489226in}}%
\pgfpathlineto{\pgfqpoint{0.743328in}{0.506616in}}%
\pgfpathlineto{\pgfqpoint{0.788045in}{0.521444in}}%
\pgfpathlineto{\pgfqpoint{0.832762in}{0.537330in}}%
\pgfpathlineto{\pgfqpoint{0.877479in}{0.554047in}}%
\pgfpathlineto{\pgfqpoint{0.922196in}{0.571415in}}%
\pgfpathlineto{\pgfqpoint{0.966913in}{0.589288in}}%
\pgfpathlineto{\pgfqpoint{1.011630in}{0.607547in}}%
\pgfpathlineto{\pgfqpoint{1.056347in}{0.626099in}}%
\pgfpathlineto{\pgfqpoint{1.101064in}{0.644871in}}%
\pgfpathlineto{\pgfqpoint{1.145781in}{0.663817in}}%
\pgfpathlineto{\pgfqpoint{1.190498in}{0.682915in}}%
\pgfpathlineto{\pgfqpoint{1.235215in}{0.702182in}}%
\pgfpathlineto{\pgfqpoint{1.279932in}{0.721694in}}%
\pgfpathlineto{\pgfqpoint{1.324649in}{0.741641in}}%
\pgfpathlineto{\pgfqpoint{1.369365in}{0.762507in}}%
\pgfpathlineto{\pgfqpoint{1.414082in}{0.785910in}}%
\pgfpathlineto{\pgfqpoint{1.458799in}{1.057756in}}%
\pgfusepath{stroke}%
\end{pgfscope}%
\begin{pgfscope}%
\pgfpathrectangle{\pgfqpoint{0.519743in}{0.414757in}}{\pgfqpoint{0.983773in}{1.143573in}}%
\pgfusepath{clip}%
\pgfsetbuttcap%
\pgfsetroundjoin%
\pgfsetlinewidth{1.505625pt}%
\definecolor{currentstroke}{rgb}{0.501961,0.501961,0.501961}%
\pgfsetstrokecolor{currentstroke}%
\pgfsetdash{{5.550000pt}{2.400000pt}}{0.000000pt}%
\pgfpathmoveto{\pgfqpoint{0.564460in}{0.726641in}}%
\pgfpathlineto{\pgfqpoint{1.458799in}{0.726641in}}%
\pgfusepath{stroke}%
\end{pgfscope}%
\begin{pgfscope}%
\pgfpathrectangle{\pgfqpoint{0.519743in}{0.414757in}}{\pgfqpoint{0.983773in}{1.143573in}}%
\pgfusepath{clip}%
\pgfsetbuttcap%
\pgfsetroundjoin%
\pgfsetlinewidth{1.505625pt}%
\definecolor{currentstroke}{rgb}{0.501961,0.501961,0.501961}%
\pgfsetstrokecolor{currentstroke}%
\pgfsetdash{{5.550000pt}{2.400000pt}}{0.000000pt}%
\pgfpathmoveto{\pgfqpoint{0.788045in}{0.466738in}}%
\pgfpathlineto{\pgfqpoint{0.788045in}{1.506349in}}%
\pgfusepath{stroke}%
\end{pgfscope}%
\begin{pgfscope}%
\pgfsetrectcap%
\pgfsetmiterjoin%
\pgfsetlinewidth{0.803000pt}%
\definecolor{currentstroke}{rgb}{0.000000,0.000000,0.000000}%
\pgfsetstrokecolor{currentstroke}%
\pgfsetdash{}{0pt}%
\pgfpathmoveto{\pgfqpoint{0.519743in}{0.414757in}}%
\pgfpathlineto{\pgfqpoint{0.519743in}{1.558330in}}%
\pgfusepath{stroke}%
\end{pgfscope}%
\begin{pgfscope}%
\pgfsetrectcap%
\pgfsetmiterjoin%
\pgfsetlinewidth{0.803000pt}%
\definecolor{currentstroke}{rgb}{0.000000,0.000000,0.000000}%
\pgfsetstrokecolor{currentstroke}%
\pgfsetdash{}{0pt}%
\pgfpathmoveto{\pgfqpoint{1.503516in}{0.414757in}}%
\pgfpathlineto{\pgfqpoint{1.503516in}{1.558330in}}%
\pgfusepath{stroke}%
\end{pgfscope}%
\begin{pgfscope}%
\pgfsetrectcap%
\pgfsetmiterjoin%
\pgfsetlinewidth{0.803000pt}%
\definecolor{currentstroke}{rgb}{0.000000,0.000000,0.000000}%
\pgfsetstrokecolor{currentstroke}%
\pgfsetdash{}{0pt}%
\pgfpathmoveto{\pgfqpoint{0.519743in}{0.414757in}}%
\pgfpathlineto{\pgfqpoint{1.503516in}{0.414757in}}%
\pgfusepath{stroke}%
\end{pgfscope}%
\begin{pgfscope}%
\pgfsetrectcap%
\pgfsetmiterjoin%
\pgfsetlinewidth{0.803000pt}%
\definecolor{currentstroke}{rgb}{0.000000,0.000000,0.000000}%
\pgfsetstrokecolor{currentstroke}%
\pgfsetdash{}{0pt}%
\pgfpathmoveto{\pgfqpoint{0.519743in}{1.558330in}}%
\pgfpathlineto{\pgfqpoint{1.503516in}{1.558330in}}%
\pgfusepath{stroke}%
\end{pgfscope}%
\begin{pgfscope}%
\pgfsetbuttcap%
\pgfsetmiterjoin%
\definecolor{currentfill}{rgb}{1.000000,1.000000,1.000000}%
\pgfsetfillcolor{currentfill}%
\pgfsetfillopacity{0.800000}%
\pgfsetlinewidth{1.003750pt}%
\definecolor{currentstroke}{rgb}{0.800000,0.800000,0.800000}%
\pgfsetstrokecolor{currentstroke}%
\pgfsetstrokeopacity{0.800000}%
\pgfsetdash{}{0pt}%
\pgfpathmoveto{\pgfqpoint{0.542878in}{1.008611in}}%
\pgfpathlineto{\pgfqpoint{1.537122in}{1.008611in}}%
\pgfpathquadraticcurveto{\pgfqpoint{1.556567in}{1.008611in}}{\pgfqpoint{1.556567in}{1.028056in}}%
\pgfpathlineto{\pgfqpoint{1.556567in}{1.446433in}}%
\pgfpathquadraticcurveto{\pgfqpoint{1.556567in}{1.465878in}}{\pgfqpoint{1.537122in}{1.465878in}}%
\pgfpathlineto{\pgfqpoint{0.542878in}{1.465878in}}%
\pgfpathquadraticcurveto{\pgfqpoint{0.523433in}{1.465878in}}{\pgfqpoint{0.523433in}{1.446433in}}%
\pgfpathlineto{\pgfqpoint{0.523433in}{1.028056in}}%
\pgfpathquadraticcurveto{\pgfqpoint{0.523433in}{1.008611in}}{\pgfqpoint{0.542878in}{1.008611in}}%
\pgfpathlineto{\pgfqpoint{0.542878in}{1.008611in}}%
\pgfpathclose%
\pgfusepath{stroke,fill}%
\end{pgfscope}%
\begin{pgfscope}%
\pgfsetbuttcap%
\pgfsetmiterjoin%
\definecolor{currentfill}{rgb}{0.007843,0.619608,0.447059}%
\pgfsetfillcolor{currentfill}%
\pgfsetfillopacity{0.850000}%
\pgfsetlinewidth{0.501875pt}%
\definecolor{currentstroke}{rgb}{0.000000,0.000000,0.000000}%
\pgfsetstrokecolor{currentstroke}%
\pgfsetstrokeopacity{0.850000}%
\pgfsetdash{}{0pt}%
\pgfpathmoveto{\pgfqpoint{0.562322in}{1.353123in}}%
\pgfpathlineto{\pgfqpoint{0.756766in}{1.353123in}}%
\pgfpathlineto{\pgfqpoint{0.756766in}{1.421178in}}%
\pgfpathlineto{\pgfqpoint{0.562322in}{1.421178in}}%
\pgfpathlineto{\pgfqpoint{0.562322in}{1.353123in}}%
\pgfpathclose%
\pgfusepath{stroke,fill}%
\end{pgfscope}%
\begin{pgfscope}%
\definecolor{textcolor}{rgb}{0.000000,0.000000,0.000000}%
\pgfsetstrokecolor{textcolor}%
\pgfsetfillcolor{textcolor}%
\pgftext[x=0.834544in,y=1.353123in,left,base]{\color{textcolor}{\rmfamily\fontsize{7.000000}{8.400000}\selectfont\catcode`\^=\active\def^{\ifmmode\sp\else\^{}\fi}\catcode`\%=\active\def
\end{pgfscope}%
\begin{pgfscope}%
\pgfsetbuttcap%
\pgfsetmiterjoin%
\definecolor{currentfill}{rgb}{0.870588,0.560784,0.011765}%
\pgfsetfillcolor{currentfill}%
\pgfsetfillopacity{0.850000}%
\pgfsetlinewidth{0.501875pt}%
\definecolor{currentstroke}{rgb}{0.000000,0.000000,0.000000}%
\pgfsetstrokecolor{currentstroke}%
\pgfsetstrokeopacity{0.850000}%
\pgfsetdash{}{0pt}%
\pgfpathmoveto{\pgfqpoint{0.562322in}{1.210423in}}%
\pgfpathlineto{\pgfqpoint{0.756766in}{1.210423in}}%
\pgfpathlineto{\pgfqpoint{0.756766in}{1.278478in}}%
\pgfpathlineto{\pgfqpoint{0.562322in}{1.278478in}}%
\pgfpathlineto{\pgfqpoint{0.562322in}{1.210423in}}%
\pgfpathclose%
\pgfusepath{stroke,fill}%
\end{pgfscope}%
\begin{pgfscope}%
\definecolor{textcolor}{rgb}{0.000000,0.000000,0.000000}%
\pgfsetstrokecolor{textcolor}%
\pgfsetfillcolor{textcolor}%
\pgftext[x=0.834544in,y=1.210423in,left,base]{\color{textcolor}{\rmfamily\fontsize{7.000000}{8.400000}\selectfont\catcode`\^=\active\def^{\ifmmode\sp\else\^{}\fi}\catcode`\%=\active\def
\end{pgfscope}%
\begin{pgfscope}%
\pgfsetbuttcap%
\pgfsetroundjoin%
\pgfsetlinewidth{0.803000pt}%
\definecolor{currentstroke}{rgb}{0.000000,0.000000,0.000000}%
\pgfsetstrokecolor{currentstroke}%
\pgfsetstrokeopacity{0.600000}%
\pgfsetdash{{2.960000pt}{1.280000pt}}{0.000000pt}%
\pgfpathmoveto{\pgfqpoint{0.562322in}{1.101751in}}%
\pgfpathlineto{\pgfqpoint{0.659544in}{1.101751in}}%
\pgfpathlineto{\pgfqpoint{0.756766in}{1.101751in}}%
\pgfusepath{stroke}%
\end{pgfscope}%
\begin{pgfscope}%
\definecolor{textcolor}{rgb}{0.000000,0.000000,0.000000}%
\pgfsetstrokecolor{textcolor}%
\pgfsetfillcolor{textcolor}%
\pgftext[x=0.834544in,y=1.067723in,left,base]{\color{textcolor}{\rmfamily\fontsize{7.000000}{8.400000}\selectfont\catcode`\^=\active\def^{\ifmmode\sp\else\^{}\fi}\catcode`\%=\active\def
\end{pgfscope}%
\end{pgfpicture}%
\makeatother%
\endgroup%

%% file: tables/calibration_table.tex
\begin{tabular}{ll|c|c|cc|cc|c|c}
\toprule
\multicolumn{2}{c|}{\textbf{Model/Dataset Latent} $\mathcal{L}$} & \multicolumn{1}{c|}{$\{0,...,4\}$ + L1} & \multicolumn{1}{c|}{$\{0,...,5\}$ + L1} & \multicolumn{2}{c|}{$[C]$ + Exact Match} & \multicolumn{2}{c|}{$\{A,\perp\}$ + BAS} & \multicolumn{1}{c|}{$\{0,1\}^{\leq C}$ + Hamming} & \multicolumn{1}{c}{$\{0,...,C\}^2$ + L1} \\
\multicolumn{2}{c|}{} & HelpSteer & STSB & MMLU & When2Call & SimpleQA & TriviaQA & MAQA & HelpSteer \\
\midrule
\multirow{2}{*}{\small{Gemma-3-4B}} & $q$ & {$0.282$}  & {$0.323$}  & {$0.085$}  & {$0.061$}  & {$0.196$}  & {$0.048$}  & {$0.103$}  & {$1.216$}  \\
& $f_{\phi}(q)$ & {$\textcolor{colfirst}{\textbf{0.099}}$}  & {$\textcolor{colfirst}{\textbf{0.050}}$}  & {$\textcolor{colfirst}{\textbf{0.030}}$}  & {$\textcolor{colfirst}{\textbf{0.001}}$}  & {$\textcolor{colfirst}{\textbf{0.000}}$}  & {$\textcolor{colfirst}{\textbf{0.004}}$}  & {$\textcolor{colfirst}{\textbf{0.000}}$}  & {$\textcolor{colfirst}{\textbf{0.424}}$}  \\
\midrule
\multirow{2}{*}{\small{Gemma-3-12B}} & $q$ & {$0.197$}  & {$0.364$}  & {$0.040$}  & {$0.034$}  & {$0.187$}  & {$0.029$}  & {$0.059$}  & {$1.249$}  \\
& $f_{\phi}(q)$ & {$\textcolor{colfirst}{\textbf{0.037}}$}  & {$\textcolor{colfirst}{\textbf{0.023}}$}  & {$\textcolor{colfirst}{\textbf{0.000}}$}  & {$\textcolor{colfirst}{\textbf{0.001}}$}  & {$\textcolor{colfirst}{\textbf{0.000}}$}  & {$\textcolor{colfirst}{\textbf{0.000}}$}  & {$\textcolor{colfirst}{\textbf{0.000}}$}  & {$\textcolor{colfirst}{\textbf{0.547}}$}  \\
\midrule
\multirow{2}{*}{\small{Qwen3-4B}} & $q$ & {$0.139$}  & {$0.692$}  & {$0.043$}  & {$0.018$}  & {$0.182$}  & {$0.048$}  & {$0.137$}  & {$1.375$}  \\
& $f_{\phi}(q)$ & {$\textcolor{colfirst}{\textbf{0.039}}$}  & {$\textcolor{colfirst}{\textbf{0.033}}$}  & {$\textcolor{colfirst}{\textbf{0.006}}$}  & {$\textcolor{colfirst}{\textbf{0.004}}$}  & {$\textcolor{colfirst}{\textbf{0.000}}$}  & {$\textcolor{colfirst}{\textbf{0.004}}$}  & {$\textcolor{colfirst}{\textbf{0.000}}$}  & {$\textcolor{colfirst}{\textbf{0.594}}$}  \\
\midrule
\multirow{2}{*}{\small{Qwen3-30B-A3B}} & $q$ & {$0.117$}  & {$0.589$}  & {$0.083$}  & {$0.035$}  & {$0.087$}  & {$0.009$}  & {$0.052$}  & {$1.449$}  \\
& $f_{\phi}(q)$ & {$\textcolor{colfirst}{\textbf{0.026}}$}  & {$\textcolor{colfirst}{\textbf{0.041}}$}  & {$\textcolor{colfirst}{\textbf{0.006}}$}  & {$\textcolor{colfirst}{\textbf{0.001}}$}  & {$\textcolor{colfirst}{\textbf{0.000}}$}  & {$\textcolor{colfirst}{\textbf{0.007}}$}  & {$\textcolor{colfirst}{\textbf{0.000}}$}  & {$\textcolor{colfirst}{\textbf{0.530}}$}  \\
\bottomrule
\end{tabular}

%% file: tables/calibration_table_kde.tex
\begin{tabular}{ll|c|c|cc|cc|c|c}
\toprule
\multicolumn{2}{c|}{\textbf{Model/Dataset Latent} $\mathcal{L}$} & \multicolumn{1}{c|}{$\{0,...,4\}$ + L1} & \multicolumn{1}{c|}{$\{0,...,5\}$ + L1} & \multicolumn{2}{c|}{$[C]$ + Exact Match} & \multicolumn{2}{c|}{$\{A,\perp\}$ + BAS} & \multicolumn{1}{c|}{$\{0,1\}^{\leq C}$ + Hamming} & \multicolumn{1}{c}{$\{0,...,C\}^2$ + L1} \\
\multicolumn{2}{c|}{} & HelpSteer & STSB & MMLU & When2Call & SimpleQA & TriviaQA & MAQA & HelpSteer \\
\midrule
\multirow{2}{*}{\small{Gemma-3-4B}} & $q$ & {$0.266$}  & {$0.340$}  & {$\textcolor{colfirst}{\textbf{0.124}}$}  & {$0.112$}  & {$0.196$}  & {$0.059$}  & {$0.678$}  & {$\textcolor{colfirst}{\textbf{0.212}}$}  \\
& $f_{\phi}(q)$ & {$\textcolor{colfirst}{\textbf{0.182}}$}  & {$\textcolor{colfirst}{\textbf{0.072}}$}  & {$0.145$}  & {$\textcolor{colfirst}{\textbf{0.001}}$}  & {$\textcolor{colfirst}{\textbf{0.000}}$}  & {$\textcolor{colfirst}{\textbf{0.014}}$}  & {$\textcolor{colfirst}{\textbf{0.000}}$}  & {$0.299$}  \\
\midrule
\multirow{2}{*}{\small{Gemma-3-12B}} & $q$ & {$0.238$}  & {$0.419$}  & {$0.129$}  & {$0.068$}  & {$0.187$}  & {$0.035$}  & {$1.036$}  & {$0.118$}  \\
& $f_{\phi}(q)$ & {$\textcolor{colfirst}{\textbf{0.114}}$}  & {$\textcolor{colfirst}{\textbf{0.048}}$}  & {$\textcolor{colfirst}{\textbf{0.019}}$}  & {$\textcolor{colfirst}{\textbf{0.003}}$}  & {$\textcolor{colfirst}{\textbf{0.000}}$}  & {$\textcolor{colfirst}{\textbf{0.002}}$}  & {$\textcolor{colfirst}{\textbf{0.206}}$}  & {$\textcolor{colfirst}{\textbf{0.115}}$}  \\
\midrule
\multirow{2}{*}{\small{Qwen3-4B}} & $q$ & {$0.253$}  & {$0.719$}  & {$0.094$}  & {$0.047$}  & {$0.182$}  & {$0.052$}  & {$0.806$}  & {$0.143$}  \\
& $f_{\phi}(q)$ & {$\textcolor{colfirst}{\textbf{0.102}}$}  & {$\textcolor{colfirst}{\textbf{0.077}}$}  & {$\textcolor{colfirst}{\textbf{0.037}}$}  & {$\textcolor{colfirst}{\textbf{0.013}}$}  & {$\textcolor{colfirst}{\textbf{0.000}}$}  & {$\textcolor{colfirst}{\textbf{0.004}}$}  & {$\textcolor{colfirst}{\textbf{0.000}}$}  & {$\textcolor{colfirst}{\textbf{0.109}}$}  \\
\midrule
\multirow{2}{*}{\small{Qwen3-30B-A3B}} & $q$ & {$0.223$}  & {$0.617$}  & {$0.148$}  & {$0.055$}  & {$0.101$}  & {$0.024$}  & {$0.836$}  & {$0.155$}  \\
& $f_{\phi}(q)$ & {$\textcolor{colfirst}{\textbf{0.061}}$}  & {$\textcolor{colfirst}{\textbf{0.082}}$}  & {$\textcolor{colfirst}{\textbf{0.031}}$}  & {$\textcolor{colfirst}{\textbf{0.004}}$}  & {$\textcolor{colfirst}{\textbf{0.001}}$}  & {$\textcolor{colfirst}{\textbf{0.009}}$}  & {$\textcolor{colfirst}{\textbf{0.072}}$}  & {$\textcolor{colfirst}{\textbf{0.075}}$}  \\
\bottomrule
\end{tabular}

%% file: appendix/implementation_details.tex
\section{Implementation details}
\label{app:experimental_details}

Here, we present a high-level algorithmic overview of the proposed task-calibrated MBR decoding method. \Cref{alg:fitting_task_calibration} describes how the calibration map $f_\phi$ is fitted on a calibration dataset while \Cref{alg:inference} describes how to perform MBR decoding on the task calibrated latent distribution as shown in \Cref{tab:performance_generations}.

\begin{algorithm}
\caption{Fitting task calibration map $f_\phi$ for a given task $T$}\label{alg:fitting_task_calibration}
\begin{algorithmic}
\Require {Calibration dataset of queries and free-form true answers $\{x^{(i)}, y^{(i)}\}_{i=1}^N$}, task-dependent latent encoder $g_T : \mathcal{V}^\infty \mapsto \mathcal{L}$, $\text{LLM} : \mathcal{V}^\infty \mapsto \mathcal{V}^\infty$
\Ensure {Calibration map $f_\phi : \Delta^{C-1} \mapsto\Delta^{C-1}$}
\For{$i \gets 1$ to $N$}
   \State Generate true latents $\ell^{(i)} \gets g_T(y^{(i)})$
   \For{$j \gets 1$ to $M$}
        \State Sample LLM responses $\hat{y}^{(i, j)} \gets \text{LLM}(x^{(i)})$
        \State Generate sampled latent $\hat{\ell}^{(i,j)} = g_T(y^{(i, j)})$
   \EndFor
   \State Compute empirical distribution $\hat{p}^{(i)}_c \gets \frac{1}{M}\sum_{j=1}^M \mathbf{1}[\hat{\ell}^{(i, j)} = c]$, where $\hat{p}^{(i)} \in \Delta^{C-1}$
\EndFor
\State  $f_\phi \gets \arg\min_{f_\phi}\frac{1}{N} \sum_{i=1}^N\text{Cross-Entropy}(f_\phi(\hat{p}^{(i)}), \text{one-hot}(y^{(i)}))$
\end{algorithmic}
\end{algorithm}

\begin{algorithm}
\caption{MBR decoding with task calibrated LLM $f_\phi$ for a given task $T$}\label{alg:inference}
\begin{algorithmic}
\Require Task-calibration map $f_\phi : \Delta^{C-1} \mapsto\Delta^{C-1}$ from \Cref{alg:fitting_task_calibration}, task-dependent latent encoder $g_T : \mathcal{V}^\infty \mapsto \mathcal{L}$, task-loss $d_T : \mathcal{L} \times \mathcal{L} \mapsto \mathbb{R}$, Query $x \in \mathcal{V}^\infty$, $\text{LLM} : \mathcal{V}^\infty \mapsto \mathcal{V}^\infty$
\Ensure {MBR decoding output $\ell \in \mathcal{L}$}

\For{$j \gets 1$ to $M$}
    \State Sample LLM responses $\hat{y}^{(j)} \gets \text{LLM}(x)$
    \State Generate sampled latents $\hat{\ell}^{(j)} = g_T(y^{(j)})$
\EndFor
\State Compute empirical distribution $\hat{p}_c \gets \frac{1}{M}\sum_{j=1}^M \mathbf{1}[\hat{\ell}^{(j)} = c]$, where $\hat{p} \in \Delta^{C-1}$
\State Compute calibrated distribution $q \gets f_\phi(\hat{{p}})$
\State Output MBR action $\ell \gets \delta^\mathrm{MBR}(q) = \arg\min_{\ell \in \mathcal{L}} \mathbb{E}_{\ell^\prime \sim q} \left[d_T(\ell^\prime, \ell)\right]$
\end{algorithmic}
\end{algorithm}

\begin{algorithm}
\caption{Binning estimator for Task Calibration Error (TCE)}
\label{alg:tce_binning_estimator}
\begin{algorithmic}
\Require {Predicted latent distributions $\{p_i\}_{i=1}^n$ with $p_i \in \Delta^{C-1}$, observed true latents $\{\ell^*_i\}_{i=1}^n$, number of bins per dimension $m$.}
\Ensure {Binned estimate $\widehat{\mathrm{TCE}}$}
\State Partition the simplex $\Delta^{C-1}$ into $m^{C-1}$ equally volume bins
\For{$i \gets 1$ to $n$}
    \State Assign prediction $p_i$ to its bin $B_i$
\EndFor
\vspace{1em}
\State $\widehat{\mathrm{TCE}} \gets 0$
\For{each occupied bin $B$}
    \State Collect indices $I_B \gets \{i : B_i = B\}$
    \State Estimate the true conditional distribution in bin $B$ \quad 
    \(
        \hat q_B \gets \frac{1}{|I_B|}\sum_{i \in I_B} \ell^*_i
    \)
    \State Accumulate the bin contribution \quad
    \(
        \widehat{\mathrm{TCE}}
        \gets
        \widehat{\mathrm{TCE}}
        +
        \sum_{i \in I_B} D_{S_T}(p_i \| \hat q_B)
    \)
\EndFor
\vspace{1em}
\State \(\widehat{\mathrm{TCE}} \gets \frac{1}{n}\widehat{\mathrm{TCE}}\)
\end{algorithmic}
\end{algorithm}

\paragraph{Computational Resources.}
We evaluate all LLMs with default hyperparameters from HuggingFace (top-k, top-p, temperature, etc.) to ensure a realistic practical environment. We perform inference on two types of machines:
\begin{inparaenum}[(i)]
    \item Nvidia H100 GPUs, and
    \item Nvidia H200 GPUs.
\end{inparaenum}
Calibration is done on these machines as well, even though they are computationally lightweight and can easily be done on CPU-only machines as well. While inference (sampling responses from the LLM) takes up to multiple days for the biggest datasets with longest prompts and largest LLMs, calibration runs finish within a few minutes on our machines.

\paragraph{Calibration}
When fitting calibrators (temperature scaling and Dirichlet calibration) using NLL loss, we transform the logarithm of the predicted latent beliefs $\hat{p}(\rx)$. We add $\eps = 10^{-12}$ to these beliefs to prevent taking the logarithm of zero. In particular, temperature scaling with a positive temperature $\tau$ is realized by optimizing over the logarithm of the temperature $t = \log \tau$:
\begin{equation}
    f_t(q) = \mathrm{softmax}(\exp(t)^{-1}\log(q + \epsilon)), \quad\quad t \in \mathbb{R}.
\end{equation}
Similarly, Dirichlet calibration is realized as:
\begin{equation}
    f_{\mW, \vb}(q) = \mathrm{softmax}(\mW \log(q + \epsilon) + \vb), \quad\quad \mW \in \mathbb{R}^{C \times C}, \vb \in \mathbb{R}^C.
\end{equation}

\paragraph{Datasets.}
\begin{table}[]
    \caption{Datasets with tasks $T$, latent structures $\mathcal{L}$, and task-dependent loss functions $d_T$.}
    \centering
    \resizebox{0.99\textwidth}{!}{
        \input{tables/datasets}
    }
    \label{tab:datasets}
\end{table}

We detail the datasets, corresponding tasks and the associated latent structures and loss functions in \Cref{tab:datasets}. Here, we overload elements $\ell$ which do not represent one-hot vectors in this context but instead the true corresponding latent meaning. In practice, each one-hot representation directly corresponds to one latent interpretation, so the one-hot representation of a structure and its ``true meaning'' are isomorphic (denoted by $\cong$).

We also briefly describe each dataset and the associated LLM task.
\begin{inparaenum}[(i)]
    \item For Helpsteer, the LLM is shown a prompt and corresponding response and should rate the response according to one or more criteria. For the ordinal regression task, we choose the criterion ``correctness'' and let the LLM output a natural number for a rating between $0$ and $4$.
    \item For STSB, the LLM is shown two sentences and it should rate their similarity on a scale on $1$ to $5$. We again only consider natural numbers through rounding to the closest integer.
    \item For MMLU, the LLM is shown a question and four potential answers enumerated with (a) to (d). The LLM should select the correct answer among the provided candidates.
    \item For When2Call, we present the LLM with a prompt and a set of available tools. The LLM should then assess whether to call one of the tools, request more information from user or abstain from answering alltogether.
    \item For SimpleQA and \item TriviaQA, we present the LLM with a question and ask for the answer. Whether the LLM action is to answer $A$ or abstain $\perp$ is determined by the relative frequency of the semantic majority answer (see \Cref{app:postprocessing}).
    \item For MAQA, we present the LLM with an ambiguous question that permits multiple answers. The LLM should list all valid answers. We construct a set of all observed answers in the generated responses and the truth and then decompose this decision task into binary decision tasks of whether to include an individual answer $i$ in the final response.
    \item For Vector-valued Ordinal Regression, we again use Helpsteer and show the LLM a prompt and a corresponding response. It now should rate the response in terms of two criteria simultaneously (``correctness'', ``helpfulness'') by outputting two integer scores in $\{0, \dots, 4\}$.
\end{inparaenum}

For all datasets except MAQA, we use HuggingFace as a data source: For STSB, we use \colorbox{lightgray}{mteb/stsbenchmark-sts}, for MMLU, we use \colorbox{lightgray}{cais/mmlu}, for When2Call, we use \colorbox{lightgray}{nvidia/When2Call}, for SimpleQA, we use \colorbox{lightgray}{google/simpleqa-verified}, for TriviaQA, we use \colorbox{lightgray}{mandarjoshi/trivia\_qa}, for Helpsteer, we use \colorbox{lightgray}{nvidia/HelpSteer}. All datasets correspond to their latest available versions as of May 2026. For MAQA we use the official repository \colorbox{lightgray}{%
\href{https://github.com/YangYongJin/MAQA-Official-Repo/blob/master/datasets/MAQA_world_knowledge_nq.json}{MAQA-Data}%
}.

\subsection{Mapping Language Outputs to Latent Structures}\label{app:postprocessing}

Here, we detail how we realize the mappings $g_T$ to free-form text LLM responses into the corresponding latent structures. As our core assumption is that the LLM outputs free-form text that is embedded in the latent structure, i.e. the LLM is able to reasonable address the task we prompt it with, we discard answers that can not properly parsed through any mapping $g_T$. These occurrences are extremely rare as the instruction-tuned LLMs studied in this work respect the instruction prompts sufficiently well. Therefore, $g_T$ can often be implemented just with simple Regular Expressions or lightweight Natural Language Inference models. In general, for all structures $\mathcal{L}$ considered in this work, we empirically find that the mapping $g_T$ does not lose any semantically relevant information and succesfully extracts the correct meaning from a response in almost all cases.

\paragraph{Ordinal Regression}
Models are instructed to follow a specific output format as seen in \Cref{app:prompts}. The LLM's output is post-processed by an auxiliary language model extracting the integer response (again, see \Cref{app:prompts}) and then parsed using a Regex. 

\paragraph{Classification}
Models are instructed to follow a specific output format as seen in \Cref{app:prompts}. The answer is then passed using a Regex. 

\paragraph{Answer-or-Abstain}
Similar to \citet{kuhn2023semanticuncertaintylinguisticinvariances}, we use DeBERTa \cite{he2021deberta} as an entailment model to group the individual free-form text responses into semantic clusters that form the support of the domain over which the most frequent answer is selected. All elements falling into the cluster of the most frequent answer are mapped to \(\texttt{Answer}\) while all others are mapped to \(\perp\). 

\paragraph{Open Set Prediction}
Given an answer string \(s\), we want to extract all \emph{semantically distinct} answers in that string. As an example to the question ``Which ocean borders the USA? `` the string \(s=\)`` Both the Atlantic and Pacific Ocean border the USA. ``gets mapped to \(\{\texttt{Atlantic}, \texttt{Pacific}\}\) by \(g_T\). Concretely \(g_T\) is implemented using \texttt{gpt-5-mini} with prompts in (\Cref{app:prompts}).

\subsection{Task Calibration Error}
\paragraph{Binning Estimator for TCE}
\label{app:binning_estimator}
We assume we have \(m^{C-1}\) equally spaced bins on the simplex. For each prediction \(p_i \in \Delta^{C-1}\), we assign a bin index \(B_i\). For each occupied bin \(B\), let
\[
I_B = \{ i : B_i = B \}.
\]
The empirical conditional label distribution in bin \(B\) is
\[
\hat q_B
=
\frac{1}{|I_B|}
\sum_{i\in I_B} \ell_i,
\]
where \(\ell_i\) is the one-hot encoding of the observed label as explained in \Cref{sec:background}.
The binned TCE estimator is then
\[
\widehat{\mathrm{TCE}}
=
\frac{1}{n}\sum_B
\sum_{i\in I_B}
D(p_i || \hat q_B).
\]
The procedure is also shown in \Cref{alg:tce_binning_estimator}.

%% file: tables/datasets.tex
\begin{tabular}{c|c|c|c|c|c|c|c|c}
        \toprule
         Task & \multicolumn{2}{c|}{Ordinal Regression} & \multicolumn{2}{c|}{Classification} & \multicolumn{2}{c|}{Answer-or-Abstain} & Open-Set Prediction & \makecell{Vector Ordinal \\ Regression} \\
         \midrule
         Dataset & Helpsteer & STSB & MMLU & When2Call & SimpleQA & TriviaQA & MAQA & Helpsteer
         \\
         \midrule
         \makecell{Latent Structure \\ $\mathcal{L} \cong$} & $\{0, \dots, 4\}$ & $\{1, \dots 5\}$ & \makecell{\{(a), (b), \\ (c), (d)\}} & \makecell{\{``call tool'', \\ ``request info'', \\ ``abstain''\}} & $\{A, \perp\}$ & $\{A, \perp\}$ & $\{0, 1\}^{\leq C}$ & $\{0, \dots, 4\}^2$ \\
         \makecell{Task-Dependent \\ Loss $d_T(\ell, \ell^\prime)$} & $\lvert \ell - \ell^\prime \rvert$ & $\lvert \ell - \ell^\prime \rvert$ & $ \mathbf{1}[\ell \neq \ell^\prime]$ & $ \mathbf{1}[\ell \neq \ell^\prime]$ & $d_t^\textrm{BAS}(\ell, \ell^\prime)$ & $d_t^\textrm{BAS}(\ell, \ell^\prime)$ & %
         \makecell{$\mathbf{1}[\ell_i \neq \ell_i^\prime]$ \\ for binary decisions \\ about including \\ answer $i$}
         & $\lVert \ell - \ell^\prime \rVert_1$
         \\
         \bottomrule
    \end{tabular}

%% file: appendix/related_work_extended.tex
\section{Related Work Extended}

\subsection{BAS: A Decision-Theoretic Approach to Evaluating LLM Confidence}
\label{appendix:bas}
The BAS metric of \citet{wu2026basdecisiontheoreticapproachevaluating} can be viewed as a special case of our latent-structure decision framework. Let \(\mathcal L=\{\mathrm{A},\perp\}\), where \(\mathrm{A}\) denotes \texttt{answer} and \(\perp\) denotes \texttt{abstain}. We identify both the latent and action space with these two selective decisions. For a fixed threshold \(t\in[0,1)\), define the task loss
\begin{align}
d_T^\textrm{BAS}(a,b)=
\begin{cases}
0, & b=\perp,\\
-1, & b=\mathrm{A},\ a=\mathrm{A},\\
\frac{t}{1-t}, & b=\mathrm{A},\ a=\perp.
\end{cases}
\label{eq:bas}
\end{align}
Importantly, the loss is asymmetric, and \(a\) corresponds to the true decision while \(b\) denotes the model prediction. For a given \(x\), let \(\hat p(x)_{\mathrm{A}}\) denote the predictive probability that answering is correct. The corresponding Bayes decision is
\[
\delta^{\mathrm{MBR}}_t(\hat p(x))
= \argmin_{b\in\{\mathrm{A},\perp\}} 
\mathbb{E}_{a\sim \hat p(x)}\bigl[d_t^\textrm{BAS}(a,b)\bigr].
\]
This coincides with the BAS answer-or-abstain policy. To see this, we compare the two possible actions:
\[
\mathbb{E}_{a\sim \hat p(x)}[d_t^\textrm{BAS}(a,\perp)] = 0,
\]
and
\[
\mathbb{E}_{a\sim \hat p(x)}[d_t^\textrm{BAS}(a,\mathrm{A})]
= -\hat p(x)_{\mathrm{A}} + \frac{t}{1-t}\hat p(x)_{\perp}.
\]
Thus, choosing \(b=\mathrm{A}\) is optimal iff
\[
-\hat p(x)_{\mathrm{A}} + \frac{t}{1-t}\hat p(x)_{\perp} \le 0.
\]
Using \(\hat p(x)_{\perp} = 1 - \hat p(x)_{\mathrm{A}}\), this is equivalent to
\[
-\hat p(x)_{\mathrm{A}} + \frac{t}{1-t}(1 - \hat p(x)_{\mathrm{A}}) \le 0
\;\Longleftrightarrow\;
\hat p(x)_{\mathrm{A}} \ge t.
\]
Therefore,
\[
\delta^{\mathrm{MBR}}_t(\hat p(x))=
\begin{cases}
\mathrm{A}, & \hat p(x)_{\mathrm{A}}\ge t,\\
\perp, & \text{otherwise},
\end{cases}
\]
which exactly recovers the BAS answer-or-abstain rule. In our experiments, we set \(t=0.25\), and the implementation of the mapping \(g_T\) is detailed in \Cref{app:postprocessing}.

%% file: appendix/uq_and_calibration.tex
\section{On Calibration And Uncertainty Quantification.}
\label{app:uq_and_calibration}

Model calibration enables \emph{estimation of the average loss incurred by a decision rule}.
Importantly, it does not ensure good per-sample decision-making. This is particularly relevant in the context of uncertainty quantification (UQ), a field that the literature often confounds with model calibration \cite{minderer2021revisiting,nixon2019measuring}. UQ methods are typically evaluated by how well they predict model errors or out-of-distribution examples. In the context of LLMs, hallucination detection is a well-established surrogate task for evaluating UQ performance. Here, the uncertainty associated with an individual instance is thresholded to predict if a response is a confabulation. Hence, the reported metrics often depend on the relative rankings of inputs based on the predicted uncertainty. Calibrating a (latent) predictive distribution has no formal implications on either this relative ranking or the per-instance thresholding. In particular, calibration does not claim to improve the discrimination of individual predictions into correct and incorrect answers by using uncertainty. Instead, it ensures that the expected loss under a decision rule computed from the model's beliefs matches the true expected loss. In the context of our work, our calibration map $f^*$ further ensures that MBR is the optimal decision rule with respect to some loss function. This does not imply that, for example, the maximum probability or similar statistics computed from the calibrated latent beliefs (e.g. Bayes risk as proposed for UQ by \citet{tomov2026taskawarenessimprovesllmgenerations}) are more reliable to detect hallucinations or forecast model errors. 
For an in-depth discussion of calibration and its relationship to reliability and uncertainty quantification, with illustrative examples, we refer to the excellent blog post by \citet{xia2026whatandwhat}. 

\begin{figure}[]
    \centering\input{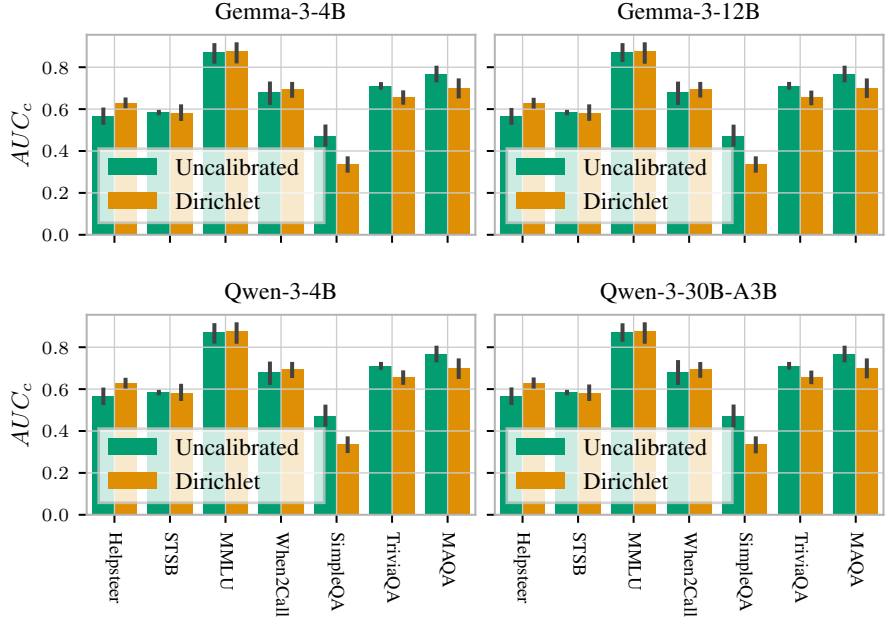} %
      \vspace{-3em}
        \caption{Uncertainty quantification (UQ) $\mathrm{AUC}_c$ scores for Bayes risk of \Cref{eq:mbr_decoding} and the task-specific loss $d_T$ incurred by the MBR decoding action for different tasks and models. There is no clear trend for task calibration: On some datasets it improves UQ, on others, it does not. This aligns with our discussion.}
        \label{fig:uq_auc}
\end{figure}

Nonetheless, we showcase how task calibration affects uncertainty quantification for completeness in \Cref{fig:uq_auc}. Here, we measure the concordance statistic $\mathrm{AUC}_c$ that assesses how well a continuous uncertainty estimator can predict a continuous target variable. The concordance \(\mathrm{AUC}_c\) is an estimate of \(\mathbb{P}(y_i > y_j \mid \ \hat{y}_i > \hat{y}_j)\). We follow \cite{therneau2024concordance} and discard ties on \(y\), i.e. \(y_i = y_j\) and treat ties on the estimator, i.e. \(y_i = y_j\) with a score of \(\frac{1}{2}\). The value of the resulting score can be interpreted analogously to the traditional Area-under-the-Precision-Recall-Curve (AUC-ROC) metrics, with 0.5 corresponding to random chance and 1 to perfect ranking ability. In our case, we evaluate the Bayes risk (\Cref{eq:mbr_decoding} like \citet{tomov2026taskawarenessimprovesllmgenerations} as a proxy for per-instance uncertainty and measure how well it predicts the task dependent loss $d_T$. We observe that on some tasks, there surprisingly is an improved capacity to estimate uncertainty from the task calibrated LLM beliefs (MMLU, When2Call, Helpsteer). On other tasks, however, the quality of the uncertainty estimate deteriorates. Overall, there is no trend that is consistent among all tasks. This aligns well with our discussion: Calibration enables improved decision making in terms of average loss. It does not improve per-instance decisions like predicting the error caused by the prediction from the Bayes risk as a proxy for uncertainty.

%% file: appendix/proofs.tex
\section{Proofs}
\label{app:proofs}
\OptimalStrategyPartition*
\begin{proof}
Define
\[
\rs := \hat p(\rx),
\qquad
\rl := g_T(Y).
\]
Then \(\rs\) is the observable model output, and for any realization of \(\rs\), \(f^*_{g_T}(\rs)\) denotes the true conditional distribution of \(\rl\) given \(\rs\).

Now for any \(\delta'\):
\begin{align}
\mathbb E_{\rx,\ry}\!\left[d_T\!\left(g_T(\ry),\delta'(\hat{p}(\rx))\right) \mid \hat{p}(\rx) \in P\right]
&=
\mathbb E_{\rx,\ry}\!\left[d_T\!\left(\rl,\delta'(\rs)\right) \mid \rs \in P\right]
\\
&=
\mathbb E_{\rs \mid \rs \in P}\!\left[
\mathbb E_{\rx, \ry}\!\left[d_T\!\left(\rl,\delta'(\rs)\right)\mid \rs, \rs \in P\right]
\right]
\label{eq:opt_pf_1}
\\
&=
\mathbb E_{\rs \mid \rs \in P}\!\left[
\mathbb E_{\ell\sim f^*_{g_T}(\rs)}
\!\left[d_T\!\left(\ell,\delta'(\rs)\right)\mid \rs, \rs \in P\right]
\right]
\label{eq:opt_pf_2}
\\
&\ge
\mathbb E_{\rs \mid \rs \in P}\!\left[
\min_{a\in\mathcal L}
\mathbb E_{\ell\sim f^*_{g_T}(\rs)}
\!\left[d_T(\ell,a)\mid \rs, \rs \in P\right]
\right]
\label{eq:opt_pf_3}
\\
&=
\mathbb E_{\rs \mid \rs \in P}\!\left[
\mathbb E_{\ell\sim f^*_{g_T}(\rs)}
\!\left[
d_T\!\left(
\ell,
\delta^\mathrm{MBR}(f^*_{g_T}(\rs))
\right)
\,\middle|\, \rs, \rs \in P
\right]
\right]
\label{eq:opt_pf_4}
\\
&=
\mathbb E_{\rs \mid \rs \in P}\!\left[
\mathbb E_{\rx, \ry}\!\left[d_T\!\left(\rl,\delta^\mathrm{MBR}(f^*_{g_T}(\rs))\right)\mid \rs, \rs \in P\right]
\right]
\label{eq:opt_pf_5}
\\
&=
\mathbb E_{\rx,\ry}\!\left[
d_T\!\left(
\rl,
\delta^\mathrm{MBR}(f^*_{g_T}(\rs))
\right) \mid \rs \in P
\right].
\label{eq:opt_pf_6}
\end{align}

Substituting back \(\rl=g_T(\ry)\) and \(\rs=\hat p(\rx)\) yields
\begin{align}
\mathbb E_{\rx,\ry}\!\left[
d_T\!\left(
g_T(\ry),
\delta^\mathrm{MBR}(f^*_{g_T}(\hat p(\rx)))
\right) \mid \hat{p}(\rx) \in P
\right]
\le
\mathbb E_{\rx,\ry}\!\left[
d_T\!\left(
g_T(\ry),
\delta'(\hat p(\rx))
\right) \mid \hat{p}(\rx) \in P
\right].
\end{align}

In \Cref{eq:opt_pf_1} we apply the tower rule by conditioning on \(\rs=\hat p(\rx)\). In \Cref{eq:opt_pf_2} we use that, conditional on \(S\), the random variable \(\rl=g_T(\ry)\) is distributed according to \(f^*_{g_T}(\rs)\) when taking the expectation over $\rx, \ry$ with respect to the true data generation distribution. In \Cref{eq:opt_pf_3} we use that, for each realization of \(S\), the conditional risk of the action \(\delta'(\rs)\) is at least the minimum achievable conditional risk over all actions \(a\in\mathcal L\). In \Cref{eq:opt_pf_4} we use the definition of the Minimum Bayes Risk decision rule. Finally, \Cref{eq:opt_pf_5} uses again the tower rule. Hence, no other decision strategy depending only on \(\hat p(\rx)\) can attain lower expected task loss for any condition on $\hat{p}(\rx)$.
\end{proof}

\OptimalStrategy*
\begin{proof}
    This follows directly from \Cref{lem:optimal_strategy} by setting $P = \Delta^{C-1}$.
\end{proof}

\begin{restatable}{proposition}{ApproximateTaskMap}
\label{prop:approximate_task_map}
For a given parametric calibration map family $\mathcal{F}$, the map with minimal negative log-likelihood minimizes the expected KL-divergence to $f^*_{g_T}$:
\begin{align}
\begin{split}
    & \arg\min_{f_\phi \in \mathcal{F}} \mathbb{E}_{\rx, \ry}\left[ -g_T(\ry)^T \log f_\phi(\hat{p}(\rx))\right] 
    =  \arg\min_{f_\phi \in \mathcal{F}} \mathbb{E}_\rx \left[ \mathrm{KL}\left[f_{g_T}^*(\hat{p}(\rx) \Vert f_\phi(\hat{p}(\rx))\right]\right].
\end{split}
\end{align}
\end{restatable}

\begin{proof}
    \begin{align}
    \mathbb{E}_{\rx, \ry}\left[ -g_T(\ry^T) \log f_\phi(\hat{p}(\rx)\right] &= \mathbb{E}_{\hat{p}(\rx)} \left[\mathbb{E}_{\rx, \ry}\left[ -g_T(\ry)^T \log f_\phi(\hat{p}(\rx) \right] \mid \hat{p}(\rx)\right] \\
    &= \mathbb{E}_{\hat{p}(\rx)} \left[ 
        - \mathbb{E}_{\rx, \ry}[g_T(\ry) \mid \hat{p}(\rx)]^T \log f_\phi(\hat{p}(\rx))
    \right]
    \\
    &= \mathbb{E}_{\hat{p}(\rx)} \left[ \mathbb{E}_{\rx \mid \hat{p}(\rx)} \left[
        - (f_{g_T}^*(\hat{p}(\rx)))^T \log f_\phi(\hat{p}(\rx))
    \right]\right]
    \\
    \begin{split}
    &= \mathbb{E}_{\hat{p}(\rx)} \left[ \mathbb{E}_{\rx \mid \hat{p}(\rx)} \left[
        - (f_{g_T}^*(\hat{p}(\rx)))^T \log f_\phi(\hat{p}(\rx)) \right]\right]
        \\ & \quad + \mathbb{E}_{\hat{p}(\rx)} \left[\mathbb{E}_{\rx \mid \hat{p}(\rx)} \left[f_{g_T}^*(\hat{p}(\rx))^T \log f_{g_T}^*(\hat{p}(\rx))\right]\right]
        \\ & \quad - \mathbb{E}_{\hat{p}(\rx)} \left[\mathbb{E}_{\rx \mid \hat{p}(\rx)} \left[f_{g_T}^*(\hat{p}(\rx))^T \log f_{g_T}^*(\hat{p}(\rx))
    \right]\right]
    \end{split}
    \\ 
    \begin{split}
    &= \mathbb{E}_{\hat{p}(\rx)}\left[  \mathbb{E}_{\rx\mid \hat{p}(\rx)} \left[
        (f_{g_T}^*(\hat{p}(\rx)))^T \log \frac{f^*_{g_T}(\hat{p}(\rx))}{f_\phi(\hat{p}(\rx))} \right]\right] 
        \\ & \quad - \mathbb{E}_{\hat{p}(\rx)} \left[\mathbb{E}_{\rx \mid \hat{p}(\rx)} \left[f_{g_T}^*(\hat{p}(\rx))^T \log f_{g_T}^*(\hat{p}(\rx))
    \right]\right]
    \end{split}
    \\
    \begin{split}
    &= \mathbb{E}_{\hat{p}(\rx)}\left[  \mathbb{E}_{\rx\mid \hat{p}(\rx)} \left[
        \mathrm{KL}\left[f_{g_T}^*(\hat{p}(\rx)) \Vert f_\phi(\hat{p}(\rx))\right] \right]\right] 
        \\ & \quad + \mathbb{E}_{\hat{p}(\rx)} \left[\mathbb{E}_{\rx \mid \hat{p}(\rx)} \left[\mathbb{H}\left[f_{g_T}^*(\hat{p}(\rx))\right]
    \right]\right]
    \end{split}
    \\ 
    \begin{split}
    &=  \mathbb{E}_{\rx} \left[
        \mathrm{KL}\left[f_{g_T}^*(\hat{p}(\rx)) \Vert f_\phi(\hat{p}(\rx))\right] \right]
        \\ & \quad + \mathbb{E}_{\rx } \left[\mathbb{H}\left[f_{g_T}^*(\hat{p}(\rx))\right]
    \right]
    \end{split}
\end{align}
The second expected entropy term does not depend on $f_\phi(\hat{p}(\rx))$ and hence minimizing the cross entropy also minimizes the expected KL divergence.
\end{proof}

\begin{restatable}{proposition}{DecisionAwareEntropyDivergence}
\label{prop:decision_aware_entropy_divergence}
The generalized entropy \& divergence induced by 
\(S_T(q,\ell)=d_{T}(\ell, \delta^\text{MBR}(q))\) are:
\begin{align}
H_{S_T}(q) = R_{\mathrm{Bayes}}(\delta^\text{MBR}(q), q), \quad
D_{S_T}(q \,\|\, r) = 
R_{\mathrm{Bayes}}(\delta^\text{MBR}(q), r)
- R_{\mathrm{Bayes}}(\delta^\text{MBR}(r), r).
\end{align}
\end{restatable}

\begin{proof}
We first recall the definitions of the generalized entropy and divergence induced by a scoring function \(S_T\):
\begin{align}
H_{S_T}(q)
&=
\mathbb E_{\ell\sim q}[S_T(q,\ell)],
\label{eq:decision_aware_1}
\\
D_{S_T}(q\|r)
&=
\mathbb E_{\ell\sim r}[S_T(q,\ell)]
-
\mathbb E_{\ell\sim r}[S_T(r,\ell)].
\label{eq:decision_aware_2}
\end{align}

Using \(S_T(q,\ell)=d_T(\ell,\delta^\mathrm{MBR}(q))\), we obtain
\begin{align}
H_{S_T}(q)
&=
\mathbb E_{\ell\sim q}[S_T(q,\ell)]
\label{eq:decision_aware_3}
\\
&=
\mathbb E_{\ell\sim q}\!\left[d_T(\ell,\delta^\mathrm{MBR}(q))\right]
\label{eq:decision_aware_4}
\\
&=
R_{\mathrm{Bayes}}(\delta^\mathrm{MBR}(q),q).
\label{eq:decision_aware_5}
\end{align}

Likewise,
\begin{align}
D_{S_T}(q\|r)
&=
\mathbb E_{\ell\sim r}[S_T(q,\ell)]
-
\mathbb E_{\ell\sim r}[S_T(r,\ell)]
\label{eq:decision_aware_6}
\\
&=
\mathbb E_{\ell\sim r}\!\left[d_T(\ell,\delta^\mathrm{MBR}(q))\right]
-
\mathbb E_{\ell\sim r}\!\left[d_T(\ell,\delta^\mathrm{MBR}(q))\right]
\label{eq:decision_aware_7}
\\
&=
R_{\mathrm{Bayes}}(\delta^\mathrm{MBR}(q),r)
-
R_{\mathrm{Bayes}}(\delta^\mathrm{MBR}(r),r).
\label{eq:decision_aware_8}
\end{align}
\end{proof}

\TCE*
\begin{proof}
Follows directly from the known decomposition \cite{kull2015novel}.
\end{proof}

%% file: appendix/prompts.tex
\section{Prompts}
\label{app:prompts}

\begin{lstlisting}[caption={\(g_T\) Prompt for Set-Case},label={lst:set_prompt_gt}]
You extract explicit, semantically distinct atomic answers from a model answer.

Rules:
- Extract only what the answer explicitly states; no inference.
- Split coordinated lists (commas, and/or, slashes) into separate items.
- Merge synonyms/rephrasings; keep one canonical surface form, preserving key head nouns.
- Remove leading articles/hedges; remove trailing punctuation/parentheticals unless part of the name.
- Ignore explanations/justifications; include quantities only if the quantity itself is the item.
- If no explicit item appears (I don't know, etc.), output: I don't know.
- Output strictly: one item per line, no bullets or extra text.

Follow the output format in the examples.

Example 1
Q: Which oceans border the USA?
A: The Atlantic and Pacific Oceans.
Output:
Atlantic Ocean
Pacific Ocean

Example 2
Q: What are the official languages of Switzerland?
A: German, French, Italian, and Romansh are the four official languages.
Output:
German
French
Italian
Romansh

Q: {question}
A: {answer}

Task: List all explicit, distinct answers from the model answer, one per line.
\end{lstlisting}

\begin{lstlisting}[caption={Prompt for When2Call},label={lst:when2call_prompt}]
You are an expert in composing functions. You are given a question and a set of possible functions.
Based on the question, decide what the assistant should do.
If a function can directly answer the question, you may call it.
If the question lacks required parameters, ask a follow-up question.
If none of the functions can be used to answer the question, say so.
If you decide to invoke a function, you MUST put it in the format of [func_name1(params_name1=params_value1, params_name2=params_value2...), func_name2(params)].

You may briefly explain your reasoning first, but your final line must be exactly one of:
FINAL_ACTION: tool_call
FINAL_ACTION: request_for_info
FINAL_ACTION: cannot_answer

Here is a list of functions in JSON format that you can invoke.
{tools}

Question: {question}
\end{lstlisting}

\begin{lstlisting}[caption={Prompt for MMLU},label={lst:mmlu_prompt}]
What is the correct answer to the following multiple choice question? You may very briefly explain your reasoning first but then give your final answer as one of the letters as "Final answer: letter".
Question: <question>
A. <choice 1>
B. <choice 2>
C. <choice 3>
D. <choice 4>
\end{lstlisting}

\begin{lstlisting}[caption={Prompt for HelpSteer},label={lst:helpsteer}]
Assess the following model response for {criterion_name} ({criterion_description})
on a scale from 0 to 4, where higher is better.

You may briefly explain your reasoning first, but your final line must be exactly one
number from 0 to 4 in the format: "Final answer: <number>".

Model response: {question}
\end{lstlisting}

\begin{lstlisting}[caption={Prompt for STSB},label={lst:helpsteer}]
Measure the sentence similarity between two sentences measured on a continous scale of 0 
to 5 with 0 being unrelated and 5 being related.
Now for the following two sentences, what is the similarity score?

You may briefly explain your reasoning first, but your final line must be exactly one 
number from 0 to 5 in the format: "Final answer: <number>".

Sentences: {question}
\end{lstlisting}

\begin{lstlisting}[caption={Prompt for TriviaQA/SimpleQA/MAQA},label={lst:singleqa_prompt}]
{question}. Only provide the answer without explanation.
\end{lstlisting}

\begin{lstlisting}[caption={Prompt for Helpsteer, Two Criteria},label={lst:helpsteer_vector}]
Assess the following model response for {criteria_names_and_descriptions} on a scale
from 0 to 4 each, where higher is better.

You may briefly explain your reasoning first, but your final line must be exactly one
number from 0 to 4 for each criterion in the format: "Final answer: {format_string}".

Model response: {question}

Score: 
\end{lstlisting}

\begin{lstlisting}[caption={Prompt for Extracting Integers from LLM Response}]
You extract the final numeric answer from a model response.

Rules:
- Read the full response and identify the final numeric value the answer commits to.
- Ignore intermediate calculations, option numbers, and numbers mentioned only in reasoning.
- Return exactly one bare number with no explanation or surrounding text.
- If the response does not clearly provide one final numeric answer, return UNKNOWN.

Question: {question}
Answer: {answer}

Task: Return only the final numeric answer stated in the answer.
\end{lstlisting}

%% file: appendix/supp_figures.tex
\section{Additional Figures}
\begin{figure}[h]
    \centering
    \includegraphics[width=\linewidth]{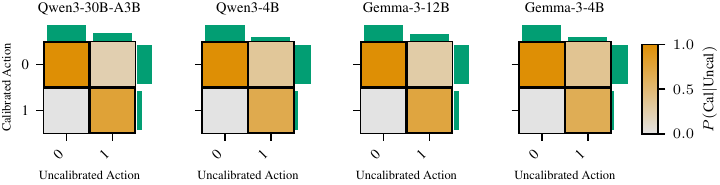}
    \caption{Action movement for MAQA ($\{0,1\}^{\leq C}$).}
    \label{fig:action_movement_maqa}
\end{figure}

\begin{figure}[h]
    \centering
    \begin{subfigure}[t]{\linewidth}
        \centering
        \includegraphics[width=\linewidth]{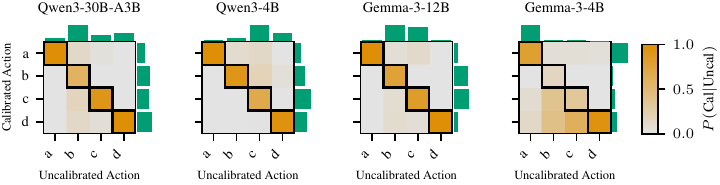}
        \caption{MMLU ($[C]$)}
        \label{fig:action_movement_mmlu}
    \end{subfigure}
    
    \vspace{0.5em}
    
    \begin{subfigure}[t]{\linewidth}
        \centering
        \includegraphics[width=\linewidth]{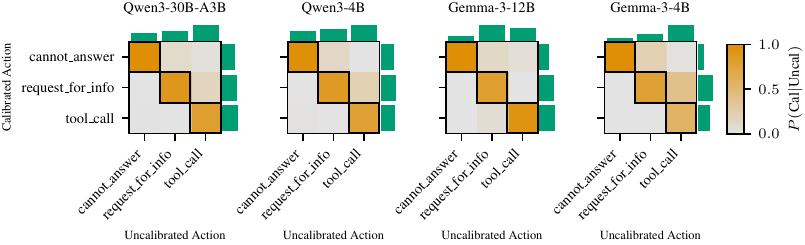}
        \caption{When2Call ($[C]$)}
        \label{fig:action_movement_when2call}
    \end{subfigure}
    
    \caption{Action movement for classification tasks.}
\end{figure}

\begin{figure}[h]
    \centering
    \begin{subfigure}[t]{\linewidth}
        \centering
        \includegraphics[width=\linewidth]{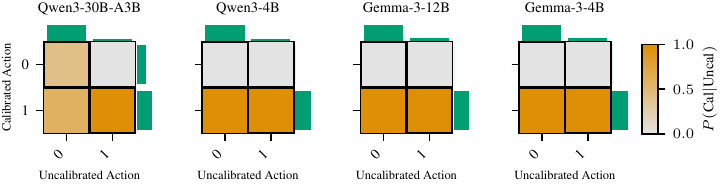}
        \caption{SimpleQA ($\{A,\perp\}$)}
        \label{fig:action_movement_simpleqa}
    \end{subfigure}
    
    \vspace{0.5em}
    
    \begin{subfigure}[t]{\linewidth}
        \centering
        \includegraphics[width=\linewidth]{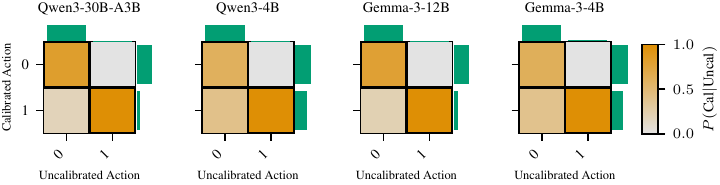}
        \caption{TriviaQA ($\{A,\perp\}$)}
        \label{fig:action_movement_triviaqa}
    \end{subfigure}
    
    \caption{Action movement for abstention tasks.}
\end{figure}

\begin{figure}[h]
    \centering
    \begin{subfigure}[t]{\linewidth}
        \centering
        \includegraphics[width=\linewidth]{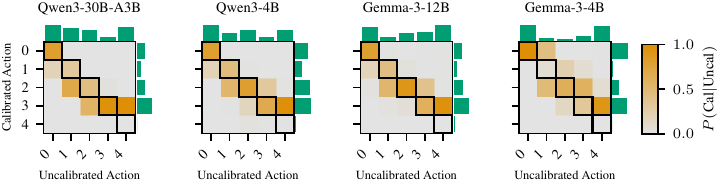}
        \caption{HelpSteer ($\{0,\dots,C\}$)}
        \label{fig:action_movement_helpsteer}
    \end{subfigure}
    
    \vspace{0.5em}
    
    \begin{subfigure}[h]{\linewidth}
        \centering
        \includegraphics[width=\linewidth]{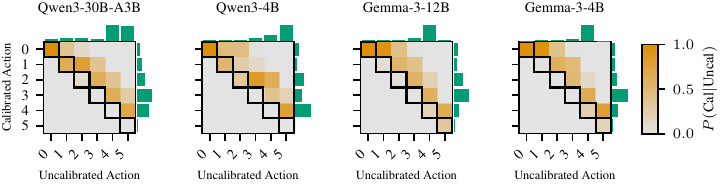}
        \caption{STSB ($\{0,\dots,C\}$)}
        \label{fig:action_movement_stsb}
    \end{subfigure}
    \caption{Action movement for ordinal/interval tasks.}
\end{figure}

\clearpage
\begin{figure}[p]
    \centering
    \vspace*{\fill}
    \includegraphics[width=0.5\linewidth]{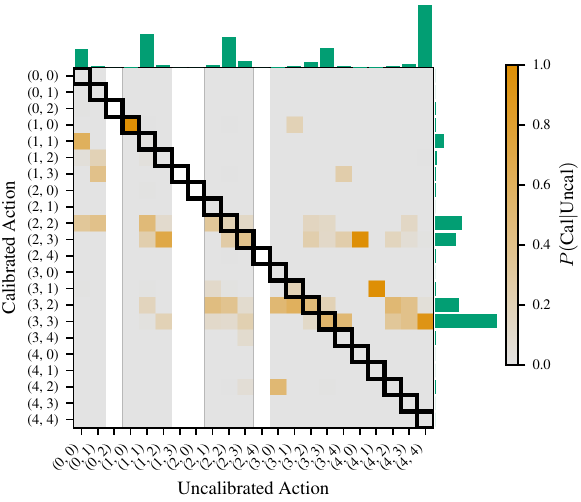}
    \caption{HelpSteer ($\{0,\dots,C\}^2$) Qwen3-4B}
    \label{fig:helpsteer_vector_qwen4}
    \vspace*{\fill}
\end{figure}

\begin{figure}[p]
    \centering
    \vspace*{\fill}
    \includegraphics[width=0.5\linewidth]{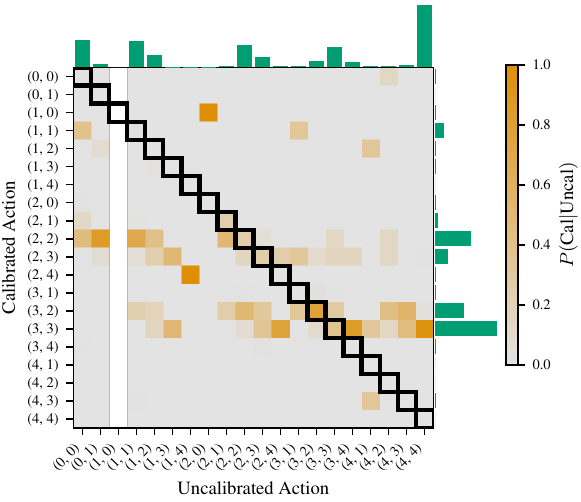}
    \caption{HelpSteer ($\{0,\dots,C\}^2$) Qwen3-30B-A3B}
    \label{fig:helpsteer_vector_qwen30}
    \vspace*{\fill}
\end{figure}

\begin{figure}[p]
    \centering
    \vspace*{\fill}
    \includegraphics[width=0.5\linewidth]{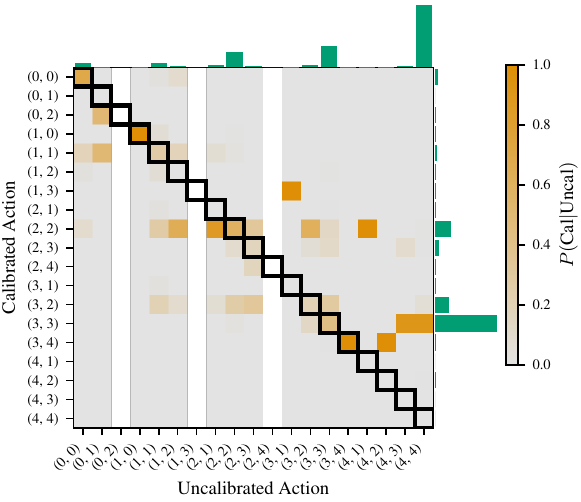}
    \caption{HelpSteer ($\{0,\dots,C\}^2$) Gemma3-4B}
    \label{fig:helpsteer_vector_gemma4}
    \vspace*{\fill}
\end{figure}

\begin{figure}[p]
    \centering
    \vspace*{\fill}
    \includegraphics[width=0.5\linewidth]{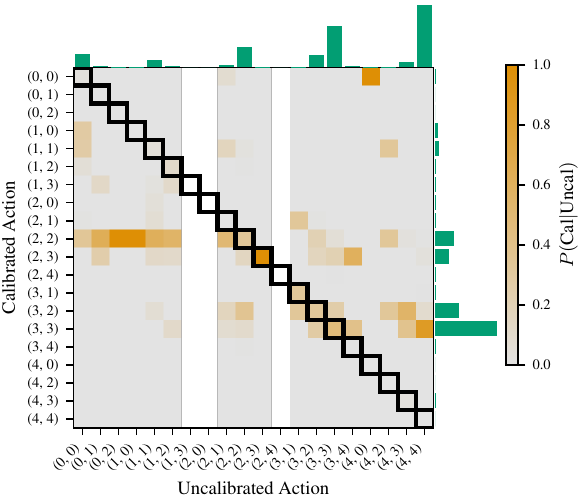}
    \caption{HelpSteer ($\{0,\dots,C\}^2$) Gemma-3-12B}
    \label{fig:helpsteer_vector_gemma12}
    \vspace*{\fill}
\end{figure}